%% file: main.tex
\documentclass{article}

\usepackage{microtype}
\usepackage{graphicx}
\usepackage{subfigure}
\usepackage{algorithm2e}
\usepackage[most]{tcolorbox}

\usepackage{bbm}
\usepackage{titletoc}

\usepackage{mymacros}

\usetikzlibrary{automata, positioning, arrows}
\usetikzlibrary{shapes.geometric}

\usepackage{hyperref}

\usepackage[accepted]{icml2024}

\usepackage{natbib}
\usepackage{enumitem}

\usepackage[capitalize,noabbrev]{cleveref}

\usepackage{natbib}
\usepackage{enumitem}

\usepackage{lscape}

\icmltitlerunning{Offline Inverse RL: New Solution Concepts and Provably Efficient Algorithms}

\begin{document}

\twocolumn[
\icmltitle{Offline Inverse RL:\\
New Solution Concepts and Provably Efficient Algorithms}




\begin{icmlauthorlist}
\icmlauthor{Filippo Lazzati}{yyy}
\icmlauthor{Mirco Mutti}{comp}
\icmlauthor{Alberto Maria Metelli}{yyy}
\end{icmlauthorlist}

\icmlaffiliation{yyy}{Politecnico di Milano, Milan, Italy}
\icmlaffiliation{comp}{Technion, Haifa, Israel}

\icmlcorrespondingauthor{Filippo Lazzati}{filippo.lazzati@polimi.it}

\icmlkeywords{Inverse Reinforcement Learning, Offline, Sample Complexity}

\vskip 0.3in
]



\printAffiliationsAndNotice{}  

\allowdisplaybreaks[4]

\input{main_paper.tex}

\bibliography{refs.bib}
\bibliographystyle{icml2024}

\newpage
\appendix
\onecolumn

\input{appendix.tex}

\end{document}

%% file: main_paper.tex
\thinmuskip=2mu
\medmuskip=2mu
\thickmuskip=2mu

\begin{abstract}
    \emph{Inverse reinforcement learning} (IRL) aims to recover the reward function
    of an \emph{expert} agent from demonstrations of behavior.
    It is well-known
    that the IRL problem is fundamentally ill-posed, i.e., many reward functions
    can explain the demonstrations.
    For this reason, IRL has been recently reframed in terms of
    estimating the \emph{feasible reward set}
    \cite{metelli2021provably}, thus, postponing the selection of a single reward.
    However, so far, the available formulations and algorithmic solutions have been proposed and analyzed mainly for the \emph{online} setting, where the learner can interact with the environment and query the expert at will. This is clearly unrealistic in most practical applications, where the availability of an \emph{offline} dataset is a much more common scenario. In this paper, we introduce a novel notion of feasible reward set capturing the opportunities and limitations of the {offline} setting and we analyze the complexity of its estimation. This requires the introduction of an original learning framework that copes with the intrinsic difficulty of the setting, for which the data coverage is not under control.
   Then, we propose two computationally and statistically
    efficient algorithms, \irlo and \pirlo, for addressing the problem.
    In particular, the latter adopts a specific form of \textit{pessimism} to enforce the novel, desirable property of \textit{inclusion monotonicity} of the delivered feasible set.
    With this work, we aim to provide a panorama of the challenges of the offline IRL problem and how they can be fruitfully addressed.
\end{abstract}

\section{Introduction}\label{sec:intro}
\emph{Inverse reinforcement learning} (IRL), also called inverse optimal control,
consists of recovering a reward function from expert's demonstrations
\cite{russell1998learning}. Specifically, the reward is required to be \textit{compatible} with the expert's behavior, i.e.,
it shall make the expert's policy optimal.
As pointed out in \citet{arora2018survey},
IRL allows mitigating the challenging task of the manual specification of the reward function, thanks to the presence of demonstrations, and provides an effective method for \emph{imitation learning} \cite{osa2018algorithmic}. In opposition to mere \emph{behavioral cloning}, IRL allows focusing on the expert \emph{intent} (instead of \emph{behavior}), and, for this reason, it has the potential to reveal the underlying objectives that drive the expert's choices.
In this sense, IRL enables \emph{interpretability}, improving
the interaction with the expert by explaining and predicting its behavior, and \emph{transferability}, as the reward (more than a policy) can be employed under environment shifts~\citep{adams2022survey}.

One of the main concerns of IRL is that
the problem is inherently \emph{ill-posed} or \emph{ambiguous} \cite{ng2000algorithms}, i.e.,
there exists a variety of reward functions {compatible} with expert's demonstrations.
In the literature, many criteria for the selection of a single reward among the compatible ones were proposed~\citep[e.g.,][]{ng2000algorithms,ratliff2006maximum,ziebart2008maximum,boularias2011relative}.
Nevertheless, the ambiguity issue has limited the theoretical understanding of the IRL problem for a long time.

Recently, IRL has been reframed by \citet{metelli2021provably} into the problem of computing the \textit{set} of all rewards {compatible} with expert's demonstrations, named \textit{feasible reward set} (or just \emph{feasible set}). By postponing the choice of a specific reward within the {feasible set}, this formulation has opened the doors to a new perspective that has enabled a deeper theoretical understanding of the IRL problem. The majority of previous works on the reconstruction of the feasible set have focused mostly on the \textit{online} setting \citep[e.g.,][]{metelli2021provably,lindner2022active,zhao2023inverse,metelli2023towards}, in which the learner is allowed to actively interact with the environment {and} with the expert to collect samples. 

Although these works succeeded in obtaining sample efficient algorithms and represent a fundamental step ahead in the understanding of the challenges of the IRL problem (e.g., providing sample complexity lower bounds), the underlying basic assumption that the learner is allowed to govern the exploration and query the expert wherever is far from being realistic. Indeed, the most common IRL applications are naturally framed in an \emph{offline} scenario, in which the learner is given in advance a dataset of trajectories of the expert (and, possibly, an additional dataset collected with a \emph{behavioral} policy, e.g., \citealt{boularias2011relative}). Typically, no further interaction with the environment and with the expert is allowed~\citep{LikmetaMRTGR21}. The offline setting has been widely studied in (forward) \emph{reinforcement learning} \citep[RL,][]{sutton2018reinforcement},
and a surge of works have analyzed the problem from theoretical
and practical perspectives \citep[e.g.,][]{munos2007performance, levine2020offline,buckman2020importance,yu2020mopo,Jin2021IsPP}. In this context, a powerful technique is represented by \emph{pessimism}, which discourages the learner from assigning credit to options that have not been sufficiently explored in the available dataset, allowing for sample efficiency guarantees~\citep{buckman2020importance}.

The IRL offline setting has been investigated for the problem of recovering the feasible set in the recent preprint~\citep{zhao2023inverse}. The authors consider the same feasible set definition employed for the online case, which enforces the optimality of the expert's policy \emph{in every state}~\citep{metelli2021provably,lindner2022active}. However, in the offline setting, this learning target is unrealistic unless the dataset covers the full space. 
This implies that the produced rewards can be safely used in forward RL when the behavioral policy covers the whole reachable portion of the state-action space \emph{only}. For this reason, \citet{zhao2023inverse} apply a form of \emph{pessimism} which allows delivering rewards that make the expert's policy $\epsilon$-optimal even in the presence of partial covering of the behavioral policy but only when the latter is sufficiently close to the expert's.
These demanding requirements, however, collide with the intuition that, regardless of the sampling policy, if we observe the expert's actions, we can deliver \emph{at least one} reward, making the expert optimal.\footnote{For instance, simply assign $0$ when playing the expert actions and $-1$ otherwise.}


\textbf{Desired Properties}~~In this paper, we seek to develop novel appropriate  \emph{solution concepts}  for the feasible reward set and new effective actionable \emph{algorithms} for recovering them in the offline IRL setting. Specifically, we aim at fulfilling the following three \emph{key properties}:
\begin{enumerate}[noitemsep, leftmargin=*, topsep=-2pt, label=($\roman*$)]
    \item (\emph{Sample Efficiency}) We should output, with high probability, an estimated feasible set using a number of samples polynomial w.r.t. the desired accuracy, error probability, and relevant sizes of the problem.
    \item (\emph{Computational Efficiency}) We should be able to check the \emph{membership} of a candidate reward in the feasible set in polynomial time w.r.t. the relevant sizes of the problem.
    \item (\emph{Inclusion Monotonicity}) We should output one estimated feasible set that \emph{includes} and one that \emph{is included} in the true feasible set with high probability.
\end{enumerate}
While properties ($i$) and ($ii$) are commonly requested, ($iii$) deserves some comments. \emph{Inclusion monotonicity}, intuitively, guarantees that we produce a set that \emph{does not exclude} any reward function that can be feasible and a set that \emph{includes only} reward functions that are surely feasible, given the current samples (Figure~\ref{fig: explanation sanity checker v2 online setting}). This, remarkably, allows delivering (with high probability) reward functions that make the expert's policy optimal (not just $\epsilon$-optimal) regardless of the accuracy with which the feasible set is recovered.

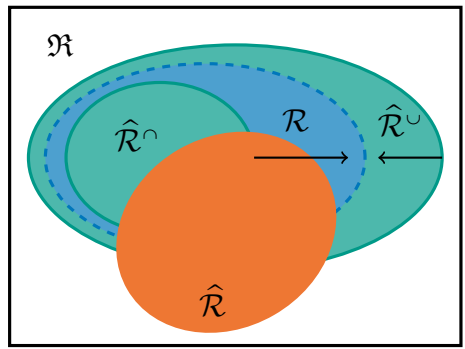
\begin{figure}[t]
\centering
\begin{tikzpicture}
    \filldraw[black,fill=white,very thick] (-3,-2.5) rectangle (3,2);
    \filldraw[vibrantTeal,fill=vibrantTeal!70, very thick] (0,0) ellipse (2.75 and 1.5);
    \filldraw[vibrantBlue,fill=vibrantBlue!70, dashed, very thick] (-0.4,0) ellipse (2.125 and 1.25);
    \filldraw[vibrantTeal,fill=vibrantTeal!70, very thick] (-1,0)  ellipse (1.25 and 1);
     \filldraw[vibrantOrange, fill=vibrantOrange, very thick, rotate=30, fill opacity=0.5] (-0.6,-0.8) ellipse (1.5 and 1.25);
    \node[text width=2cm] at (-1.5,1.5) {\large$\mathfrak{R}$};
    \node at (+2.2,0.5) {\large$\widehat{\R}^\cup$};
    \node at (+0.8,0.5) {\large$\R$};
    \node at (-0.3,-1.9) {\large$\widehat{\R}$};
    \node at (-1.3,0.3) {\large$\widehat{\R}^\cap$};
    \draw[->,  thick] (0.25,0) -- (1.5,0);
    \draw[->,  thick] (2.75,0) -- (1.9,0);
\end{tikzpicture}
\caption{
$\mathfrak{R}$ = set of all rewards,
$\mathcal{R}$ = true feasible set, $\widehat{\mathcal{R}}^{\cap}$ and $\widehat{\mathcal{R}}^{\cup}$ = examples of \emph{inclusion monotonic} estimated feasible set (i.e., $\widehat{\mathcal{R}}^{\cap} \subseteq \mathcal{R} \subseteq \widehat{\mathcal{R}}^{\cup}$), $\widehat{\mathcal{R}}$ = example of \emph{inclusion non-monotonic} estimated feasible set (i.e., $\widehat{\mathcal{R}}\not\subseteq \mathcal{R}$ and $ \mathcal{R}\not\subseteq \widehat{\mathcal{R}}$).
}\label{fig: explanation sanity checker v2 online setting}
\end{figure}




\textbf{Contributions}~~The contributions of this paper are summarized as follows:
\begin{itemize}[noitemsep, leftmargin=*, topsep=-2pt]
    \item We propose a novel definition of \textit{feasible set} that takes into
    account the intrinsic challenges of the \textit{offline} setting (i.e., partial covering).
    Moreover, we introduce appropriate \emph{solution concepts}, which are \emph{learnable} based on the coverage of the given dataset (Section \ref{section: framework}).
    \item We adapt the \emph{probably approximately correct} (PAC) framework
    from \citet{metelli2023towards} to our {offline} setting by proposing novel \emph{semimetrics} which, differently from previous works, allow us to naturally deal with \textit{unbounded} rewards (Section \ref{section: pac framework}).
    \item We present a novel algorithm, named \irlo (\irloExt), for solving {offline} IRL. We show that it satisfies the requirements of ($i$) sample and ($ii$) computational efficiency (Section \ref{section: irlo}). 
    \item After having formally defined the notion of \textit{inclusion monotonicity}, we propose a \textit{pessimism}-based algorithm, named \pirlo (\pirloExt), that achieves ($iii$) inclusion monotonicity preserving sample and computational efficiency, at the price of a larger sample complexity (Section \ref{section: pirlo}).
    \item We discuss a specific application of our algorithm \pirlo for \emph{reward sanity check} (Section~\ref{section: application}).
    \item  We present a negative result for
    \textit{offline} IRL when only data from a deterministic expert are available (Section \ref{section: bitter lesson}).
\end{itemize}
Additional related works are reported in Appendix~\ref{subsection: related works}.
The proofs of all the results are reported in the Appendix~\ref{section: old fs}-\ref{apx:tech}.

\section{Preliminaries}\label{sec: preliminaries}


\textbf{Notation}~~Given a finite set $\X$, we denote by $|\X|$ its cardinality and
 by $\Delta^\X\coloneqq\{q\in [0,1]^{|\X|}\,|\,\sum_{x\in\X}q(x)=1\}$
the  simplex on $\X$. Given two sets $\X$ and $\Y$, we denote the
set of conditional distributions as
$\Delta_\Y^\X\coloneqq\{q:\Y\rightarrow\Delta^\X\}$. Given $N \in \mathbb{N}$, we denote $\dsb{N}\coloneqq\{1,\dots,N\}$.
Given an equivalence relation $\equiv\subseteq\X\times\X$,
and an item $x\in\X$, we denote
by $[x]_{\equiv}$ the equivalence class of $x$.

\textbf{Markov Decision Processes (MDPs) without Reward}~~A finite-horizon \emph{Markov decision process} \citep[MDP,][]{puterman1994markov}
without reward is defined as
$\M\coloneqq\langle\S,\A,\mu_0,p, H\rangle$,
where $\S$ is the finite state space
($S\coloneqq|\S|$), $\A$ is the finite action space
($A\coloneqq|\A|$),
$\mu_0\in\Delta^\S$ is
the initial-state distribution, $p=\{p_h\}_{h\in\dsb{H}}$ where
$p_h\in\Delta^\S_{\SA}$ for every $h\in\dsb{H}$ is the transition model, and $H \in \mathbb{N}$ is the horizon.
A policy is defined as $\pi =\{\pi_h\}_{h \in \dsb{H}}$ where $\pi_h \in \Delta_{\S}^{\A}$ for every $h \in \dsb{H}$.
 $\mathbb{P}_{p,\pi}$ denotes the trajectory distribution 
induced by $\pi$ and $\E_{p,\pi}$ the expectation w.r.t. $\mathbb{P}_{p,\pi}$ (we omit $\mu_0$ in the notation).
The state-action visitation distribution induced by $p$ and $\pi$ is defined as $\rho^{p,\pi}_h(s,a)\coloneqq\mathbb{P}_{p,\pi}(s_h=s,a_h=a)$ and the state visitation distribution as $\rho^{p,\pi}_h(s)\coloneqq\sum_{a\in\A}\rho^{p,\pi}_h(s,a)$, so that $\sum_{s\in\S}\rho^{p,\pi}_h(s) = 1$ for every $h \in \dsb{H}$.

\textbf{Additional Definitions}~~The sets of transition models, policies, and rewards
are denoted as $\P\coloneqq\Delta^\S_{\SAH}$, $\Pi\coloneqq\Delta^\A_{\SH}$,
and $\mathfrak{R}\coloneqq\{r:\SAH\rightarrow\mathbb{R}\}$, respectively.\footnote{
    We remark that we consider \textit{real-valued} rewards without requiring boundedness.
}
For every $h \in \dsb{H}$, we define the set of states and state-action pairs
reachable by $\pi$ at stage $h\in\dsb{H}$
as $\S_h^{p,\pi}\coloneqq\{s\in\S\,|\,\rho^{p,\pi}_h(s)>0\}$
and $\Z_h^{p,\pi}\coloneqq\{(s,a)\in\SA\,|\,\rho^{p,\pi}_h(s,a)>0\}$, respectively.
Moreover, we define $\S^{p,\pi}\coloneqq \{(s,h) : h \in \dsb{H},\, s\in \S_h^{p,\pi}\}$
and $\Z^{p,\pi}\coloneqq \{(s,a,h) : h \in \dsb{H},\, (s,a)\in \Z_h^{p,\pi}\}$, with cardinality $S^{p,\pi}\le SH$ and $Z^{p,\pi}\le SAH$, respectively. 
We refer to these sets as the ``support'' of $\rho^{p,\pi}$.
We denote the cardinality of the largest set $\S^{p,\pi}_h$ varying $h \in \dsb{H}$, as
$S^{p,\pi}_{\text{max}}\coloneqq\max_{h \in \dsb{H}} |\S^{p,\pi}_{{h}}| \le S$.
Finally, we denote the minimum of the state-action distribution on set $\mathcal{Y}\subseteq\SAH$ as
$\rho^{\pi,\mathcal{Y}}_{\text{min}}\coloneqq\min_{(s,a,h)\in \mathcal{Y}}\rho_h^{p,\pi}(s,a)$.

\textbf{Value Functions and Optimality}~~The \emph{Q-function} of policy $\pi$ with transition model $p$ and reward function $r$ is defined as $Q^\pi_h(s,a;p,r)\coloneqq\E_{p,\pi}[\sum_{t=h}^H r_t(s_t,a_t)|s_h=s, a_h=a]$
and the optimal Q-function as
$Q^*_h(s,a;p,r)\coloneqq\max_{\pi\in\Pi}Q^\pi_h(s,a;p,r)$.
The \emph{utility} (i.e., expected return) of policy $\pi$ under the initial-state distribution $\mu_0$ is given by
$J(\pi;\mu_0,p,r)\coloneqq\E_{s\sim\mu_0, a \sim \pi(\cdot|s)}[Q^\pi_1(s,a;p,r)]$
and the optimal utility by $J^*(\mu_0,p,r)\coloneqq\max_{\pi\in\Pi}J(\pi;\mu_0,p,r)$.
An \emph{optimal policy} $\pi^*$ is a policy that maximizes the  utility
$\pi^*\in \argmax_{\pi\in\Pi}J(\pi;\mu_0,p,r)$.
The existence of a deterministic optimal policy is guaranteed~\citep{puterman1994markov}.

\textbf{Equivalence Relations}~~
We introduce two \textit{equivalence} relations:
$\equiv_{\overline{\S}}$ (over policies) and $\equiv_{\overline{\Z}}$ (over transition models), defined for arbitrary $\overline{\S}\subseteq\SH$ and $\overline{\Z}\subseteq\SAH$.
Specifically, let $\pi,\pi' \in \Pi$ be two policies, we have:
\begin{align}
    \pi\equiv_{\overline{\S}}\pi' \quad \text{iff} \quad \forall (s,h)\in \overline{\S}:\; \pi_h(\cdot|s)=\pi'_h(\cdot|s).
\end{align}
Similarly, let $p,p'\in\mathcal{P}$, be two transition models, we have:
\begin{align}
    p\equiv_{\overline{\Z}}p' \quad \text{iff} \quad\forall (s,a,h)\in \overline{\Z}:\; p_h(\cdot|s,a)=p_h(\cdot|s,a).
\end{align}
We will often use $\overline{\S}=\S^{p,\pi}$ and $\overline{\Z}=\Z^{p,\pi}$
for some $p\in\P$ and $\pi\in\Pi$.
Intuitively, the {equivalence} relation $\equiv_{\S^{p,\pi}}$ (resp. $\equiv_{\Z^{p,\pi}}$) group policies (resp. transition models) indistinguishable given the support $\S^{p,\pi}$ (resp. $\Z^{p,\pi}$) of $\rho^{p,\pi}$.

\textbf{Offline Setting}~~We assume the availability of two datasets 
$\D^b=\{\langle s^{b,i}_1,a_1^{b,i},\dots,s^{b,i}_{H-1},a_{H-1}^{b,i},s^{b,i}_H\rangle\}_{i\in\dsb{\tau^b}}$ and $\D^E=\{\langle s^{E,i}_1,a_1^{E,i},\dots,s^{E,i}_{H-1},a_{H-1}^{E,i},s^{E,i}_H\rangle\}_{i\in\dsb{\tau^E}}$
of $\tau^b$ and $\tau^E$ independent trajectories collected by playing a \emph{behavioral policy} $\pi^b$ and the \emph{expert's policy} $\pi^E$, respectively. Furthermore, we enforce the following assumption.

\begin{ass}[Expert's covering]\label{assumption: coverage of behavioral policy}
    The behavioral policy $\pi^b$
    plays with non-zero probability the actions prescribed by the expert's policy $\pi^E$
    in its support $\suppspie$:
    {
    \begin{align*}
       \forall (s,h)\in\suppspie: \qquad \pi^b_h(\pi^E_h(s)|s)>0.
    \end{align*}}%
\end{ass}

Assumption~\ref{assumption: coverage of behavioral policy} holds when $\pi^b=\pi^E$ and generalizes that setting when the behavioral policy $\pi^b$ is ``more explorative'', possibly playing actions other than expert's ones.\footnote{We elaborate on the limits of learning with just a dataset collected with the expert's policy $\pi^E$ in Section~\ref{section: bitter lesson}. Moreover, we discuss how we can use a single dataset collected with $\pi^b$,
at the price of a slightly larger sample complexity in Appendix~\ref{remark:1Dataset}.} It should be remarked that Assumption \ref{assumption: coverage of behavioral policy} is useful but not strictly necessary. As we will explain later on, it is possible to avoid it by using all samples $\D^E\cup\D^b$ to compute the various estimates that will be needed. Even though this seems reasonable, from a practical viewpoint, it complicates the theoretical analysis of the algorithms. Thus, we will enforce Assumption \ref{assumption: coverage of behavioral policy} in the following for simplicity.

\section{Solution Concepts for Offline IRL}\label{section: framework}
In this section, we introduce a novel definition of \textit{feasible reward set}, discuss its learnability properties, and propose suitable solution concepts to be targeted for the \textit{offline} IRL.


\textbf{A New Definition of Feasible Set}~~
Let us start by recalling the original definition of \textit{feasible set}
presented in the literature and discussing its limitations for offline IRL.
\begin{defi}[``Old'' Feasible Set $\oldfs$, \citealt{metelli2021provably}]\label{def: old fs}
    Let $\M$ be an MDP without reward and let
    $\pi^E$ be the deterministic
    expert's policy. The \emph{``old'' feasible set} $\oldfs$ of rewards compatible with $\pi^E$
    in $\M$ is defined as:\footnote{
    Actually, \citet{metelli2021provably} consider rewards bounded in $[0,1]$, while we consider
    all real-valued rewards in $\mathfrak{R}$.
}
{
\begin{align}
    \oldfs\coloneqq\{&r\in\mathfrak{R}\,|\,\forall(s,h)\in\SH,\, \forall{a\in\A}:\notag\\
    &Q^{\pi^E}_h(s,\pi^E_h(s);p,r) \ge  Q^{\pi^E}_h(s,a;p,r)\}.\label{eq:eqmag}
\end{align}}%
\end{defi}
In words, $\oldfs$ contains all the reward functions that make the expert's
policy optimal \emph{in every} state-stage pair $(s,h)\in\SH$. However, forcing the optimality of $\pi^E$ in states that are never reached from the initial-state distribution $\mu_0$ is unnecessary (and even impossible) if our ultimate goal is to use the learned reward function $r$ to train a policy $\pi^*$ that achieves the maximum utility, i.e., $\pi^* \in \argmax_{\pi \in \Pi} J(\pi;\mu_0,p,r)$. This suggests an alternative definition of {feasible set}.
\begin{defi}[Feasible Set $\fs$]\label{def: new fs}
    Let $\M$ be an MDP without reward and let
    $\pi^E$ be the deterministic
    expert's policy. The \emph{feasible set} $\fs$ of rewards compatible with $\pi^E$
    in $\M$ is defined as:
    {
    \begin{equation*}
        \fs\coloneqq\{
            r\in\mathfrak{R}\,|\,J(\pi^E;\mu_0,p,r)=J^*(\mu_0,p,r)
        \}.
    \end{equation*}}%
\end{defi}
In words, $\fs$ contains all the reward functions that make the expert's
policy $\pi^E$ a utility maximizer. Clearly, since Definition~\ref{def: old fs} enforces optimality \emph{uniformly} over $\S \times \dsb{H}$, we have the inclusion $\oldfs \subseteq \fs$, where the equality holds when $\suppspie=\SH$, i.e., $\eqclasspie{\pi^E}=\{\pi^E\}$.
The following result formalizes the intuition that for $\fs$, differently from $\oldfs$, the expert's policy $\pi^E$ has to be optimal (as in Equation~\ref{eq:eqmag}) in a subset of $\SH$ only.

\begin{restatable}{thr}{relationfs}\label{theorem: alternative representation new fs}
    In the setting of Definition \ref{def: new fs}, the feasible reward set $\fs$ satisfies:
    {
    \begin{align}
     \fs=\{&r\in\mathfrak{R}\,|\,\forall \overline{\pi}\in\eqclasspie{\pi^E},\forall(s,h)\in \suppspie,\,\forall a\in\A:\notag\\
        &Q^{\overline{\pi}}_h(s,\pi^E_h(s);p,r)\ge Q^{\overline{\pi}}_h(s,a;p,r)
        \}.\label{eq:12345}
    \end{align}}%
\end{restatable}
Theorem \ref{theorem: alternative representation new fs}
shows that the optimal action induced by a reward $r\in \fs$
outside $\suppspie$, i.e., outside the support of $\rho^{p,\pi^E}$ induced by the expert's policy $\pi^E$, is not relevant. The optimality condition of Equation~\eqref{eq:12345} is requested for all the policies $\overline{\pi}$ that play the expert's action within its support. Intuitively, those policies cover the same portion of state space as $\pi^E$, i.e.,  $\S^{p,\overline{\pi}}=\S^{p,\pi^E}$ and,
since they all prescribe the same action in there,\footnote{It is worth noting that, since $(s,h)\in \suppspie$, the following identity hold: $Q^{\overline{\pi}}_h(s,\pi^E_h(s);p,r) = Q^{\pi^E}_h(s,\pi^E_h(s);p,r)$.} they all achieve the same utility, i.e., $J(\overline{\pi};\mu_0,p,r)=J(\pi^E;\mu_0,p,r)=J^*(\mu_0,p,r)$. 
Thus, if we train an RL agent with a reward function $\widehat{r} \in \fs\setminus\oldfs$, among the optimal policies we obtain a policy $\widehat{\pi} \in\eqclasspie{\pi^E}$, i.e., a policy that plays optimal (expert) actions inside $\S^{p,\pi^E}$. Clearly, $\widehat{\pi} $ will prescribe different actions than $\pi^E$ outside $\S^{p,\pi^E}$, but this is irrelevant since those states will never be reached by $\widehat{\pi}$.
This has important consequences from the offline IRL perspective. Indeed, we can recover this new notion $\fs$ (Definition~\ref{def: new fs}) without the knowledge of $\pi^E$ in the states outside $\suppspie$.
Instead, to learn the old notion $\oldfs$ (Definition~\ref{def: old fs}), we would need to enforce that the policy used to collect samples (either $\pi^E$ or $\pi^b$) covers the full space $\S\times \dsb{H}$.\footnote{A formal definition of \emph{learnability} and the proofs that $\oldfs$ and $\fs$ are not learnable under partial cover (i.e., $\mathcal{S}^{p,\pi^E} \neq \S \times \dsb{H}$ and $\Z^{p,\pi^b} \neq \SAH$) are reported in Appendix~\ref{section: learnability fs}\label{refnote}.}


\textbf{Solution Concepts and Learnability}~~To compute the feasible set $\fs$, we need to learn the expert's policy $\pi^E_h(s)$ in every $(s,h)\in\suppspie$
and
the transition model $p_h(\cdot|s,a)$ in every $(s,a,h)\in\SAH$, so that we are able to compare the $Q$-functions.
In the \textit{online} setting~\citep[e.g.,][]{metelli2021provably}, this is a reasonable requirement because the learner can explore the environment and,
thus, collect samples over the whole $\SAH$ space.\footnote{
This is true for the \textit{generative} model case. In a \textit{forward} model, in which we are allowed to interact through trajectories, we just need to learn the transition model in all state-action pairs $(s,a,h)$ reachable from $\mu_0$ with \emph{any policy}, i.e., $(s,a,h) \in \bigcup_{\pi \in \Pi} \Z^{p,\pi}$.
}
However, in our \textit{offline} setting, even in the limit of infinite samples,
triples $(s,a,h)\not\in\suppsapib$, i.e., outside the support of $\rho^{p,\pi^b}$ are never sampled.
Thus, we can identify the transition model $p$ up to its equivalence class $\eqclassp{p}$ only. Intuitively, this means that, unless
$\Z^{p,\pi^b}=\SAH$, i.e., $\pi^b$ covers the entire space, since $\fs$ depends on the value of the transition model in the whole $\SAH$, the problem of estimating the feasible set $\fs$ 
{offline} is not \textit{learnable}.\footref{refnote}
Thus, instead of learning $\fs$ directly, we propose to target as solution concepts ($i$) the \emph{largest learnable} set of rewards \textit{contained} into $\fs$, and ($ii$) the \emph{smallest learnable} set of rewards that \textit{contains} $\fs$, 
defined as follows.
\begin{defi}[Sub- and Super-Feasible Sets]\label{def: subset superset fs}
    Let $\M$ be an MDP without reward and let
    $\pi^E$ be the deterministic
    expert's policy.
    We define the \emph{sub-feasible set} $\sub$ and
    the \emph{super-feasible set} $\super$ as:
    {\thinmuskip=1mu
\medmuskip=1mu
\thickmuskip=1mu
    \begin{align*}\resizebox{.98\linewidth}{!}{$\displaystyle
        \sub\coloneqq
        \bigcap\limits_{p'\in\eqclassp{p}}
        \R_{p',\pi^E}, \quad \super\coloneqq
        \bigcup\limits_{p'\in\eqclassp{p}}
        \R_{p',\pi^E}.$}
    \end{align*}}
\end{defi}
Since $p\in \eqclassp{p}$, we ``squeeze'' the {feasible set} $\fs$ between these two learnable solution, i.e., $\sub\subseteq\fs\subseteq
\super$. A more explicit representation is given as follows:
\begin{equation*}\resizebox{\linewidth}{!}{$\displaystyle
    \begin{aligned}
     & \sub = \{r \in \mathfrak{R}|\textcolor{vibrantTeal}{\forall  p'\in\eqclassp{p} },\forall \overline{\pi}  \in \eqclasspie{\pi^E},\\ 
    &\quad \forall (s,h) \in \mathcal{S}^{p,\pi^E}\!\!, \forall a \in \mathcal{A}\,:\,
    Q^{\overline{\pi}}_h(s,\pi^E_h(s);p',r) \ge Q^{\overline{\pi}}_h(s,a;p',r)
    \},\\
\end{aligned}$}
\end{equation*}
\begin{equation*}\resizebox{\linewidth}{!}{$\displaystyle
    \begin{aligned}
    & \super  = \{r \in \mathfrak{R}|\textcolor{vibrantTeal}{\exists  p'\in\eqclassp{p}}, \forall \overline{\pi}  \in \eqclasspie{\pi^E},\\ 
    &\quad \forall (s,h) \in \mathcal{S}^{p,\pi^E}\!\!, \forall a \in \mathcal{A}\,:\,
    Q^{\overline{\pi}}_h(s,\pi^E_h(s);p',r) \ge Q^{\overline{\pi}}_h(s,a;p',r)
    \}.
\end{aligned}$}
\end{equation*}
Intuitively, to be robust against the missing knowledge of the transition model outside $\suppsapib$, we have to account for all the possible $p' \in \eqclassp{p}$ and retain the rewards compatible with \textit{all} of them (for the sub-feasible set $\sub$) and with \textit{at least one} of them (for super-feasible set $\super$), as apparent from the quantifiers.
Moreover, when $\suppsapib=\SAH$, i.e., $\eqclassp{p}=\{p\}$, we have the equality:
$\sub=\fs=
\super$.
We now show that the $\fs^\cap$ and $\fs^\cup$ are indeed the \emph{tightest learnable} subset and superset of $\fs$ (formal statement and proof in Appendix \ref{section: old fs}).
\begin{thr}\label{infthr: tightest learnable}
    \textnormal{(\textbf{Informal})}~Let $\M$ be an MDP without reward, let $\pi^E$ and $\pi^b$ be the deterministic expert's policy and the behavioral policy, respectively. Then, $\sub$ and $\super$
    are the \emph{tightest subset and superset of
    $\fs$ learnable} from data collected in  $\mathcal{M}$ by executing $\pi^b$ and $\pi^E$.
\end{thr}

\section{PAC Framework}\label{section: pac framework}
We now propose a PAC
framework for learning $\sub$ and $\super$ from datasets
 $\D^E$ and $\D^b$, collected with $\pi^E$ and $\pi^b$. We first present the functions to evaluate the dissimilarity between feasible sets and then define the PAC requirement.

\textbf{Dissimilarity Functions}~~Being $\sub$ and $\super$ sets of rewards, we need ($i$) a function to assess the dissimilarity between items (i.e., reward functions), and ($ii$) a way of converting it into a dissimilarity function between sets (i.e., the sub- and super-feasible sets)~\citep{metelli2021provably}. For ($i$), we propose the following two semimetrics.
\begin{defi}[Semimetrics $d$ and $d_\infty$ between rewards]\label{def: metrics d dinf}
    Let $\M$ be an MDP without reward and
    let $\pi^E$ be the expert's policy.
    Let $\pi^b$ be the behavioral policy and let $\{\suppsahpib\}_h$
    be its support.
    Given two reward
    functions $r,\widehat{r}\in\mathfrak{R}$, we define
    $d:\mathfrak{R}\times\mathfrak{R}\rightarrow\mathbb{R}$
    and $d_\infty:\mathfrak{R}\times\mathfrak{R}\rightarrow\mathbb{R}$ as:
{
\begin{align*}
    &d(r,\widehat{r})\coloneqq
    \frac{1}{M(r,\widehat{r})}\sum\limits_{h\in\dsb{H}}\Big(
                \E\limits_{(s,a)\sim\rho^{p,\pi^b}_h}
                    \bigl|
                        r_h(s,a)-\widehat{r}_h(s,a)
                    \bigr|\\
                &\qquad\qquad\qquad+\max\limits_{(s,a)\notin \suppsahpib}\bigl|
                    r_h(s,a)-\widehat{r}_h(s,a)
                \bigr|
            \Big),\\
    &d_\infty(r,\widehat{r})\coloneqq
    \frac{1}{M(r,\widehat{r})}\sum\limits_{h\in\dsb{H}}\|r_h-\widehat{r}_h\|_\infty,
\end{align*}}%
where $M(r,\widehat{r})\coloneqq
    \max\bigl\{
        \|r\|_\infty,\|\widehat{r}\|_\infty
    \bigr\}$.
    Moreover, we conventionally set both $d$ and $d_\infty$ to $0$ when $M(r,\widehat{r})=0$.
\end{defi}
First, $d_\infty$ corresponds to the $\ell_\infty$-norm between reward functions, while $d$
combines the $\ell_1$-norm between rewards in $\suppsapib$ weighted by the visitation distribution of the behavioral policy $\rho^{p,\pi^b}$ and the $\ell_\infty$-norm outside $\suppsapib$. The intuition is that, inside $\suppsapib$, we weigh the error based on the number of samples, which are collected by $\pi^b$. Instead, outside $\suppsapib$, we can afford the $\ell_\infty$-norm because we adopt as solution concepts  $\sub$ and $\super$ that intrinsically manage the lack of samples so that we can confidently achieve zero error in that region. Second, it is easy to verify that both $d$ and $d_\infty$ are \emph{semimetrics}.\footnote{A \emph{semimetric} fulfills all the properties of a \emph{metric} except for the triangular inequality. We show in Appendix~\ref{section: semimetric} that our semimetrics fulfill a ``relaxed'' form of triangular inequality.}
Third, the two semimetrics are related by the following double inequality, where $\rho^{\pi^b,\suppsapib}_{\text{min}} > 0$ by definition:
\begin{restatable}{prop}{relationmetrics}\label{prop: relation metrics}
    For any $r,r'\in\mathfrak{R}$,  it holds that:
    \begin{align*}
        d(r,r') \le 2d_\infty(r,r') \le \frac{2}{\rho^{\pi^b,\suppsapib}_{\text{min}}} d(r,r').
    \end{align*}
\end{restatable}
Moreover, the normalization term $1/M(r,\widehat{r})$ enforces that $d(r,r')$ and $d_\infty(r,r')$ lie in 
$[0,2H]$ for every $r,r'\in\mathfrak{R}$. Differently from previous works~\citep[e.g.,][]{metelli2021provably,
lindner2022active}, this term allows to deal with \emph{(unbounded) real-valued rewards} more naturally and effectively, at the price of accepting a relaxed triangular inequality.
We stress that we have chosen distances $d$ and $d_\infty$ since they enforce non-zero weight to the absolute difference between rewards at all $(s,a,h)\in\mathcal{S}\times\mathcal{A}\times \dsb{H}$. This property allows us to control the distance between the optimal value function and the value function of the policy $\hat{\pi}^*$, i.e., the optimal policy under the recovered reward $\widehat{r}$. This can be obtained with an analogous reasoning as that contained in Section 4.3 of \citet{metelli2023towards}. More specifically, let
\begin{align*}
    & d^G_{V^*}(r,\hat{r}) \\ & \quad \coloneqq \frac{1}{M(r,\hat{r})} \sup_{\hat{\pi}^*\in\Pi^*(\hat{r})}\max_{(s,h)\in\mathcal{S}\times\dsb{H}}|V^*_h(s;r)-V^{\hat{\pi}^*}_h(s;r)|,
\end{align*}
be the adaptation of the dissimilarity index defined in \citet{metelli2023towards}, which measures the distance between the optimal value function $V^*(\cdot;r)$ (under a ground-truth reward $r$) and the value function $V^{\hat{\pi}^*}(\cdot;r)$ (under the same ground-truth reward $r$) of the policy $\hat{\pi}^*$ that is learned using the recovered reward $\widehat{r}$. Then, it can be shown that:
\begin{restatable}{prop}{relationdwithdg}
For any $r,r'\in\mathfrak{R}$,  it holds that:
    \begin{align*}
        d^G_{V^*}(r,\widehat{r})\le 2 d_{\infty}(r,\widehat{r})\le \frac{2d(r,\widehat{r})}{\rho_{\min}^{\pi^b,\mathcal{Z}^{p,\pi^b}}}.
    \end{align*}
\end{restatable}
Clearly, a small value of $d$ entails a small value of $d^G_{V^*}$. Thus, controlling distances $d$ and $d_\infty$, by enforcing non-zero weight to the absolute difference between rewards at all $(s,a,h)\in\mathcal{S}\times\mathcal{A}\times \dsb{H}$, we can control the distance between value functions.
Finally, as mentioned above, notice that we can get rid of Assumption \ref{assumption: coverage of behavioral policy} by replacing the expectation w.r.t. $\rho^{p,\pi^b}$ in the definition of $d$ with some mixture between $\pi^E$ and $\pi^b$. However, for the sake of simplicity, we continue with the current definition.

Next, to obtain a dissimilarity function between reward sets ($ii$), we make use of the Hausdorff distance.
\begin{defi}[Hausdorff distance, \citealt{rockafellar1998variational}]\label{def: Hausdorff distance}
    Let $\R,\widehat{\R} \subseteq \mathfrak{R}$ be two sets of reward functions,
    and let $c \in \{d, d_\infty\}$.
    The \emph{Hausdorff distance} between $\R$ and $\widehat{\R}$ with inner
    distance $c : \mathfrak{R}\times \mathfrak{R} \rightarrow \mathbb{R}$ is defined as:
    {\thinmuskip=2mu
\medmuskip=2mu
\thickmuskip=2mu
    \begin{align}\label{eq:haus}
        \mathcal{H}_c(\R, \widehat{\R})\coloneqq\max\Big\{\sup\limits_{r\in \R}\inf\limits_{\widehat{r}\in \widehat{\R}}c(r,\widehat{r}),
        \sup\limits_{\widehat{r}\in \widehat{\R}}\inf\limits_{r\in \R}c(r,\widehat{r})\Big\}.
    \end{align}}%
    Moreover, we abbreviate $\mathcal{H}_{d_\infty}$ with $\mathcal{H}_\infty$.
\end{defi}
Since the feasible sets are closed (see Appendix \ref{section: semimetric}), using $d$ or $d_\infty$, the Hausdorff distance is a semimetric and satisfies a relaxed triangle inequality as well. Thus, $\mathcal{H}_c(\R, \widehat{\R})=0$ if and only if the two sets coincide, i.e., $\R=\widehat{\R}$.  





\textbf{$(\epsilon,\delta)$-PAC Requirement}~~We now formally define the \emph{sample efficiency} requirement.
To distinguish between the two semimetrics $d$ and $d_\infty$,
we denote by $c$-IRL the problem of estimating $\sub$
and $\super$
under $\mathcal{H}_c$, where $c\in\{d,d_\infty\}$.
\begin{defi}[$(\epsilon,\delta)$-PAC Algorithm]\label{defi:pac}
    Let $\epsilon \in [0,2H]$ and $\delta \in (0,1)$. An algorithm $\mathfrak{A}$ outputting the estimated sub- and super-feasible sets $\widehat{\R}^\cap$ and $\widehat{\R}^\cup$ is $(\epsilon,\delta)$\emph{-PAC}
    for $c$-IRL if:
    \begin{align*}
        \mathop{\mathbb{P}}_{(p,\pi^E,\pi^b)}\bigl(&
            \bigl\{\mathcal{H}_c(\sub,\widehat{\R}^\cap)\le \epsilon\bigr\}\cap\\
            &\bigl\{\mathcal{H}_c(\super,\widehat{\R}^\cup)\le \epsilon\bigr\}
        \bigr)\ge 1-\delta,
    \end{align*}
    where $\mathbb{P}_{(p,\pi^E,\pi^b)}$ denotes the probability
    measure induced by $\pi^E$ and $\pi^b$ in
    $\M$. The \emph{sample complexity} is  the number of trajectories $\tau^E$ and $\tau^b$ in $\D^E$ and $\D^b$, respectively.
\end{defi}

\section{\irloExt (\irlo)}\label{section: irlo}
Our goal is to devise an algorithm
that is ($i$) statistically efficient, ($ii$) computationally efficient, and that provides ($iii$) guarantees about the inclusion monotonicity property. As a warm-up, in this section, we present \irlo (\irloExt), fulfilling ($i$) and ($ii$), but not ($iii$). 


\textbf{Algorithm}~~The pseudo-code of \irlo is reported in Algorithm \ref{alg:irlo} (\textcolor{vibrantOrange}{\irlo} box).
It receives two datasets $\D^E$ and $\D^b$ of trajectories collected
by policies $\pi^E$ and $\pi^b$, respectively, and it outputs the \emph{estimated sub- and super-feasible sets} $\widehat{\R}^\cap$ and $\widehat{\R}^\cup$
as estimates of $\sub$ and $\super$, respectively. \irlo leverages $\D^E$ to compute the {empirical} estimates of the expert's support
$\suppspie$ and policy $\pi^E$, denoted by $\estsuppspie$ and $\widehat{\pi}^E$ (lines~\ref{line:1}-\ref{line:2}), and it uses $\D^b$ to compute the {empirical} estimates of the behavioral policy support $\suppsapib$, and of the transition model $p$, denoted by $\estsuppsapib$ and $\widehat{p}$ (lines~\ref{line:3}-\ref{line:4}).
Finally, it returns the sub- and super-feasible sets computed with the estimated supports, expert's policy, and transition model: $\widehat{\R}^\cap=\R_{\widehat{p},\widehat{\pi}^E}^\cap$
and $\widehat{\R}^\cup=\R_{\widehat{p},\widehat{\pi}^E}^\cup$ (line~\ref{line:5}).

\RestyleAlgo{ruled}
\SetInd{0.5em}{0.5em}
\LinesNumbered
 \begin{algorithm}[t]
    \caption{\irlo and \pirlo.}\label{alg:irlo}
    \small
    \SetKwInOut{Input}{Input}
    \SetKwInOut{Output}{Output}
     \Input{Datasets $\D^E=\{\langle s_h^{E,i},a_h^{E,i} \rangle_{h}\}_{i}$, $\D^b=\{\langle s_h^{b,i},a_h^{b,i} \rangle_{h}\}_{i}$}
    
    \Output{Estimated sub- and super-feasible sets $\widehat{\R}^\cap,\widehat{\R}^\cup$}

     Estimate the expert's support:\\
     $\estsuppspie\gets \{(s,h)\in\SH\,|\, \exists i\in\dsb{\tau^E}:\,s_h^{E,i}=s\}$\label{line:1}
     
    Estimate the expert's policy:\\\For{$(s,h)\in \estsuppspie$}{
     $\widehat{\pi}^E_h(s)\gets a^{E,i}_h$ for some $i\in\dsb{\tau^E}$ s.t. $s_h^i=s$ \label{line:2}
    }

{\thinmuskip=1mu
\medmuskip=1mu
\thickmuskip=1mu    Estimate the state-action behavioral policy support:\\
     $\estsuppsapib\gets \{(s,a,h)\in\SAH\,|\, \exists i\in\dsb{\tau^b}:\,(s_h^{b,i},a_h^{b,i})=(s,a)\}$\label{line:3}}

Compute the counts for every $(s,a,h) \in \estsuppsapib$ and $s'\in\S$:\\ $N_h^b(s,a,s') \leftarrow \sum_{i\in\dsb{\tau^b}}\mathbbm{1}\{(s^{b,i}_h,a^{b,i}_h,s^{b,i}_{h+1})=(s,a,s')\}$ $   N_h^b(s,a)\leftarrow \sum_{s'\in\S}N_h^b(s,a,s')$
     
    Estimate the transition model:\\\For{$(s,a,h)\in \estsuppsapib$}{
    \For{$s'\in\S$}{
     $\widehat{p}_h(s'|s,a)\gets\frac{N_h^b(s,a,s')}{\max\{1,N_h^b(s,a)\}}$\label{line:4}
    }
    }\label{line:endFor}

\begin{tcolorbox}[enhanced, attach boxed title to top right={yshift=-4.5mm,yshifttext=-1mm},colframe=vibrantOrange,colbacktitle=vibrantOrange,colback=white,
  title=\texttt{IRLO},fonttitle=\bfseries,
  boxed title style={size=small, sharp corners}, sharp corners ,boxsep=-1.5mm]\small
     Compute $\R_{\widehat{p},\widehat{\pi}^E}^\cap$ and $\R_{\widehat{p},\widehat{\pi}^E}^\cup$ with Definition~\ref{def: subset superset fs} \\\phantom{A}using
    $\widehat{p}$, $\widehat{\pi}^E$, $\estsuppsapib$, and $\estsuppspie$
    
     \textbf{return} $(\R_{\widehat{p},\widehat{\pi}^E}^\cap,\R_{\widehat{p},\widehat{\pi}^E}^\cup)$
\end{tcolorbox}\label{line:5}
\vspace{-.35cm}
\begin{tcolorbox}[enhanced, attach boxed title to top right={yshift=-4.5mm,yshifttext=-1mm},colframe=vibrantTeal,colbacktitle=vibrantTeal,colback=white,
  title=\texttt{PIRLO},fonttitle=\bfseries,
  boxed title style={size=small, sharp corners}, sharp corners ,boxsep=-1.5mm, ]\small
    Compute the confidence set $\mathcal{C}(\widehat{p},b)$ via Eq.~\eqref{eq:C2}
    
    Compute $\subrelax_{\widehat{p},\widehat{\pi}^E}$ 
    and $\superrelax_{\widehat{p},\widehat{\pi}^E}$ with Eq. \eqref{def: relaxation superset} \\ \phantom{A}using
    $\widehat{p}$, $\widehat{\pi}^E$, $\estsuppsapib$, and $\estsuppspie$
    
     \textbf{return} $(\subrelax_{\widehat{p},\widehat{\pi}^E},\superrelax_{\widehat{p},\widehat{\pi}^E})$
\end{tcolorbox}\label{line:6}
\vspace{-.2cm}
 \end{algorithm}


\textbf{Computationally Efficient Implementation}~~
In Algorithm~\ref{alg:irlo}, \irlo outputs the estimated feasible sets $\widehat{\mathcal{R}}^\cup$ and $\widehat{\mathcal{R}}^\cap$ obtained by computing the intersection and the union of a continuous set of transition models (Definition~\ref{def: subset superset fs}). 
To show the computational efficiency of \irlo, we provide in Appendix~\ref{section: implementation appendix} (Algorithm~\ref{alg: CheckMembirlo}, \irlo box) a \emph{polynomial-time membership checker} that tests whether a candidate reward function  $r\in\mathfrak{R}$ belongs to $\widehat{\mathcal{R}}^\cup$ and/or $\widehat{\mathcal{R}}^\cap$. 
We apply \emph{extended value iteration}~ \citep[EVI,][]{auer2008nearoptimal} to compute an upper bound $Q^{+}$ and a lower bound $Q^{-}$ of the Q-function induced by the candidate reward $r$ and varying the transition model in a set $\mathcal{C}$. For the \irlo algorithm, $\mathcal{C}$ corresponds to the equivalence class of the empirical estimate $\widehat{p}$ induced by the empirical support $\estsuppsapib$, i.e., $[\widehat{p}]_{\equiv_{\estsuppsapib}}$:
\begin{equation}\label{eq:C1}\resizebox{.91\linewidth}{!}{$\displaystyle
    \mathcal{C} \coloneqq  \left\{p' \in \mathcal{P}\,|\, \forall (s,a,h) \in \estsuppsapib: \, p'_h(\cdot|s,a) = \widehat{p}_h(\cdot|s,a) \right\}.$}
\end{equation}
The algorithm has a time complexity of order $\mathcal{O}(HS^2A)$.
 
\textbf{Sample Complexity Analysis}~~We now show that the \irlo algorithm is statistically efficient. The following theorem
provides a \emph{polynomial} upper bound to its sample complexity.

\begin{restatable}{thr}{upperboundirlo}\label{theorem: upper bound d irlo}
    Let $\M$ be an MDP without reward and let
    $\pi^E$ be the expert's policy. Let
    $\D^E$ and $\D^b$ be two datasets of
    $\tau^E$ and $\tau^b$ trajectories collected
    with policies $\pi^E$ and $\pi^b$ in $\M$, respectively.
    Under Assumption~\ref{assumption: coverage of behavioral policy}, \irlo is $(\epsilon,\delta)$-PAC for $d$-IRL
    with a sample complexity at most:
    {
    \begin{align*}
        & \tau^b\le \widetilde{\mathcal{O}}\Bigg(
            \frac{H^3Z^{p,\pi^b}\ln\frac{1}
        {\delta}}{\epsilon^2}\biggl(
            \ln\frac{1}
        {\delta}+S^{p,\pi^b}_{\max}
        \biggr)+\frac{\ln\frac{1}{\delta}}{\ln\frac{1}{1-\rhominpib}}
        \Bigg), \\
        & \tau^E \le \widetilde{\mathcal{O}} \Bigg( \frac{\ln\frac{1}{\delta}}{\ln\frac{1}{1-\rhominpie}} \Bigg).
    \end{align*}}%
\end{restatable}
Some comments are in order. First, we observe that the sample complexity for the expert's dataset $\tau^E$ is constant and depends on the minimum non-zero value of the visitation 
 distribution  $\rhominpie>0$, but it does not depend on the desired accuracy $\epsilon$. This accounts for the minimum number of samples to have $\estsuppspie=\suppspie$, with high probability. Second, the sample complexity for the behavioral policy dataset $\tau^b$ displays a tight dependence on the desired accuracy $\epsilon$ and a dependence of order $H^4$ on the horizon since in the worst case, $Z^{p,\pi^b}\le SAH$.  Moreover, we notice the two-regime behavior represented by $\ln (1/\delta) + S^{p,\pi^b}_{\max}$ (i.e., small and large $\delta$)
as in previous works \cite{kaufmann2021adaptive, metelli2023towards}. 
This term is multiplied by an additional $\ln (1/\delta)$ term, which always appears in {offline} (forward) RL \citep{xie2021bridging} and it is needed to control the minimum number of samples collected from every reachable state-action pair. Finally, we observe a dependence analogous to that of $\tau^E$ on the  minimum non-zero value of the visitation 
 distribution  $\rhominpib>0$, to ensure that $\estsuppsapib=\suppsapib$.
Note that when $\pi^b=\pi^E$, Assumption~\ref{assumption: coverage of behavioral policy} is fulfilled, and the sample complexity reduces to:
{
\begin{align*}\resizebox{\linewidth}{!}{$\displaystyle
    \tau^E\le \widetilde{\mathcal{O}}\Bigg(
            \frac{H^3S^{p,\pi^E}\ln\frac{1}
        {\delta}}{\epsilon^2}\biggl(
            \ln\frac{1}
        {\delta}+S^{p,\pi^E}_{\max}
        \biggr)+\frac{\ln\frac{1}{\delta}}{\ln\frac{1}{1-\rhominpie}}
        \Bigg).$}
\end{align*}}%
Since $S^{p,\pi^E} \le SH$, the dependence on the number of actions is no longer present. An analogous result holds for
$d_\infty$.
\begin{restatable}{thr}{upperbounddinftyirlo}\label{theorem: upper bound d infty irlo}
    Under the conditions of Theorem \ref{theorem: upper bound d irlo},
    \irlo is $(\epsilon,\delta)$-PAC for $d_\infty$-IRL
    with a sample complexity at most:
    {
    \begin{align*}
        \tau^b\le \widetilde{\mathcal{O}}\Bigg(
            \frac{H^4\ln\frac{1}
        {\delta}}{\rhominpib\epsilon^2}\biggl(
            \ln\frac{1}
        {\delta}+S^{p,\pi^b}_{\max}
        \biggr)+\frac{\ln\frac{1}{\delta}}{\ln\frac{1}{1-\rhominpib}}
        \Bigg),
    \end{align*}}%
    and $\tau^E$ is bounded as in Theorem~\ref{theorem: upper bound d irlo}.
\end{restatable}
We note that, since $1/\rhominpib\ge \suppsapib$, Theorem \ref{theorem: upper bound d infty irlo} delivers a larger sample complexity w.r.t. Theorem \ref{theorem: upper bound d irlo}. This is expected because of the relation $d(r,r') \le 2d_\infty(r,r')$ between the two semimetrics (see Proposition~\ref{prop: relation metrics}).

\section{\pirloExt (\pirlo)}\label{section: pirlo}
In this section, we present our main algorithm, \pirlo (\pirloExt). Beyond statistical and computational efficiency,
\pirlo provides guarantees on the \textit{inclusion monotonicity} of the proposed feasible sets by embedding a form of
\textit{pessimism}.\footnote{
We remark on the substantial difference between our use of pessimism and that of \citet{zhao2023inverse}. Indeed, we apply pessimism to \emph{feasible sets} to ensure that the estimated set fulfills the \emph{inclusion monotonicity} property, while \citet{zhao2023inverse} apply pessimism to ensure the \emph{entry-wise monotonicity} of the reward function, i.e., $\widehat{r}(s,a) \preceq r(s,a)$, for all $\widehat{r} \in \widehat{\mathcal{R}}$ and $r \in \mathcal{R}$.}

Before presenting the algorithm, we formally introduce the notion of inclusion monotonicity and intuitively justify it. Thanks to the PAC property (Theorem~\ref{theorem: upper bound d irlo}), in the limit of \emph{infinite samples} $\tau^b,\tau^E \rightarrow + \infty$, \irlo recovers exactly the sub-
$\widehat{\R}^\cap \rightarrow \sub$
and the super- $\widehat{\R}^\cup \rightarrow \super$ feasible sets, and, consequently, the property
$\widehat{\R}^\cap\subseteq\fs
\subseteq\widehat{\R}^\cup$ holds. Because of the meaning of these sets, i.e., the \textit{tightest learnable} subset $\sub$ and superset $\super$ of the \textit{feasible set} $\fs$, it is desirable to ensure the property 
$\widehat{\R}^\cap\subseteq\fs
\subseteq\widehat{\R}^\cup$ (in high probability) in the \emph{finite samples} regime $\tau^b,\tau^E \le +\infty$  too. 
The following definition formalizes the property.

\begin{defi}[Inclusion Monotonic Algorithm]\label{defi:inclusionMon}
    Let $\delta \in (0,1)$. An algorithm $\mathfrak{A}$ outputting the estimated sub- and super-feasible sets $\widehat{\R}^\cap$ and $\widehat{\R}^\cup$ is $\delta$-\emph{inclusion monotonic} if:
    \begin{align*}
\mathop{\mathbb{P}}_{(p,\pi^E,\pi^b)}\left(\widehat{\R}^\cap\subseteq\fs
\subseteq\widehat{\R}^\cup
        \right)\ge 1-\delta.
    \end{align*}
\end{defi}

Clearly, one can always choose $\widehat{\R}^\cap=\{\}$ and $\widehat{\R}^\cup= \mathfrak{R}$ to satisfy Definition~\ref{defi:inclusionMon}. Thus, the inclusion monotonicity property will always be employed in combination with the PAC requirement (Definition~\ref{defi:pac}).
The importance of {monotonicity} will arise from a practical viewpoint in Section~\ref{section: application}.

\textbf{Algorithm}~~
The pseudocode of \pirlo is shown in Algorithm \ref{alg:irlo} (\pirlo box). The first part (lines~\ref{line:1}-\ref{line:endFor}) is analogous to \irlo and the main difference lies in the presence of the confidence set $\mathcal{C}(\widehat{p},b)\subseteq\P$ (line~\ref{line:6}),
containing the transition models in $\P$ close in $\ell_1$-norm to the {empirical estimate} $\widehat{p}$, except the ones that are not compatible with expert's actions. Formally, $\mathcal{C}(\widehat{p},b)$ is defined as:\footnote{
Actually, this definition does not take into account a corner case. See Appendix \ref{subsec: annoying corner case} for details and a more precise definition.
}
\begin{equation}\label{eq:C2}\resizebox{.95\linewidth}{!}{$\displaystyle
    \begin{aligned}
    \mathcal{C}&(\widehat{p},b)\coloneqq\Bigl\{
    p'\in\P|\\
    & \forall (s,h)\in\estsuppspie,\, s' \not\in \estsuppspie_{h+1}:\, p'_h(s'|s,\widehat{\pi}_h^E(s)) = 0 \\
    & \forall (s,a,h) \in \estsuppsapib:\, \|p'_h(\cdot|s,a) - \widehat{p}_h(\cdot|s,a) \|_1 \le \widehat{b}_h(s,a)
    \Bigr\}, \end{aligned}$}\hspace{-.5cm}
\end{equation}
where $\widehat{b}_h(s,a)$ is defined in Equation~\eqref{eq:bhsa}.
The intuition is that, with high probability, the true transition model $p$, and its equivalence class $\eqclassp{p}$, will belong to $\mathcal{C}(\widehat{p},b)$.

Drawing inspiration from \textit{pessimism} in RL, 
\pirlo ``penalizes'' the estimates of the feasible set
by removing from $\widehat{\R}^\cap$ the rewards for which we are \textit{not
confident enough} of their membership to $\sub$,
and by adding to $\widehat{\R}^\cup$ the rewards for which we are \textit{not
confident enough} of their non-membership to $\super$, based on the confidence set $\mathcal{C}(\widehat{p},b)$ on the transition model. This translates into the following expressions:
{
\begin{align}\label{def: sub and super no relaxation pessimism}
    \widehat{\R}^\cap&=\bigcap\limits_{p'\in \mathcal{C}(\widehat{p},b)}
    \R_{p',\widehat{\pi}^E}^\cap, \qquad 
    \widehat{\R}^\cup=\bigcup\limits_{p'\in \mathcal{C}(\widehat{p},b)}\R_{p',\widehat{\pi}^E}^\cup.
\end{align}}%
This way, if $p \in \mathcal{C}(\widehat{p},b)$ and $\widehat{\pi}^E = \pi^E$ with high probability, we have that, simultaneously, $\widehat{\mathcal{R}}^\cap \subseteq \mathcal{R}^{\cap}_{p,\pi^E}$ and $ \mathcal{R}^{\cup}_{p,\pi^E} \subseteq \widehat{\mathcal{R}}^\cup $. This entails the inclusion monotonicity property (Definition~\ref{defi:inclusionMon}) thanks to Definition~\ref{def: subset superset fs}.

\textbf{Computationally Efficient Implementation}~~Differently from \irlo, computing the set operations of Equation~\eqref{def: sub and super no relaxation pessimism} cannot be directly carried out by EVI.\footnote{Membership testing can be here implemented with a \emph{bilinear program}, which is, in general, a difficult problem (Appendix \ref{section: implementation appendix}).} For this reason, we propose a \emph{relaxation} which achieves the double objective of: ($i$) enabling a computationally efficient implementation of \pirlo (Algorithm~\ref{alg: CheckMembirlo}, \pirlo box); and ($ii$) allowing for a simpler statistical analysis, preserving both the PAC and the inclusion monotonicity properties (details in Appendix~\ref{section: implementation appendix}):
\begin{equation}\resizebox{\linewidth}{!}{$\displaystyle
    {\thinmuskip=1mu
\medmuskip=1mu
\thickmuskip=1mu
\small
\begin{aligned}
        \subrelax &\coloneqq\{
    r\in\mathfrak{R}\,|\,
    \forall \overline{\pi} \in [\widehat{\pi}^E]_{\equiv_{\widehat{\S}^{p,\pi^E}}}, \,\forall (s,h)\in \widehat{\S}^{p,\pi^E},\forall a\in\A: \\
    &\textcolor{vibrantTeal}{\min\limits_{p'\in \mathcal{C}(\widehat{p},b)}}Q^{\widehat{\pi}^E}_h(s,\widehat{\pi}_h^E(s);p',r)\ge \textcolor{vibrantTeal}{\max\limits_{p''\in \mathcal{C}(\widehat{p},b)}}Q^{\overline{\pi}}_h(s,a;p'',r)
    \},\notag\\
    \superrelax &\coloneqq\{
    r\in\mathfrak{R}\,|\, \forall \overline{\pi} \in [\widehat{\pi}^E]_{\equiv_{\widehat{\S}^{p,\pi^E}}},\,
    \forall (s,h)\in \widehat{\S}^{p,\pi^E} ,\forall a\in\A: \label{def: relaxation superset}\\
    &\textcolor{vibrantTeal}{\max\limits_{p'\in \mathcal{C}(\widehat{p},b)}}Q^{\widehat{\pi}^E}_h(s,\widehat{\pi}_h^E(s);p',r)\ge \textcolor{vibrantTeal}{\min\limits_{p''\in \mathcal{C}(\widehat{p},b)}}Q^{\overline{\pi}}_h(s,a;p'',r)
    \},\notag
\end{aligned}}$}%
\end{equation}
 where the universal/existential quantification over the transition model of Definition~\ref{def: subset superset fs} has been relaxed by the two $\max-\min$. In other words, we allow a choice of different transition models for the two Q-functions appearing in the two members of the inequality.
Thus, $\widetilde{\mathcal{R}}^\cap \subseteq \widehat{\mathcal{R}}^\cap   $ and $ \widehat{\mathcal{R}}^\cup \subseteq  \widetilde{\mathcal{R}}^\cup$, preserving the inclusion monotonicity. For the membership checking of a candidate reward $r \in \mathfrak{R}$, similarly to the \irlo case, we compute upper and lower bounds $Q^+$ and $Q^-$ to the Q-function by using EVI varying the transition model in the confidence set $\mathcal{C}(\widehat{p},b)$ defined in Equation~\eqref{eq:C2}. Now, the confidence set is made of \emph{$\ell_1$ constraints} and the corresponding $\max$ and $\min$ programs can be solved by using the approach of~\citep[][Figure 2]{auer2008nearoptimal}.
The overall time complexity is of order $\mathcal{O}(HS^2A\log S)$.

\textbf{Sample Efficiency and Inclusion Monotonicity}~~
We now show that \pirlo is statistically efficient, with the additional guarantee (w.r.t. \irlo) of the inclusion {monotonicity}.
\begin{restatable}{thr}{upperboundpirlo}\label{theorem: upper bound d pirlo}
    Let $\M$ be an MDP without reward and let
    $\pi^E$ be the expert's policy. Let
    $\D^E$ and $\D^b$ be two datasets of
    $\tau^E$ and $\tau^b$ trajectories collected
    by executing policies $\pi^E$ and $\pi^b$ in $\M$.
    Under Assumption~\ref{assumption: coverage of behavioral policy}, \pirlo is $(\epsilon,\delta)$-PAC for $d$-IRL
    with a sample complexity at most:
    {
    \begin{align*}
        \tau^b&\le \widetilde{\mathcal{O}}\Biggl(
            \frac{H^3Z^{p,\pi^b}\ln\frac{1}
        {\delta}}{\epsilon^2}\biggl(
            \ln\frac{1}
        {\delta}+S_{\max}^{p,\pi^b}
        \biggr)\\
        &+\frac{H^6\ln\frac{1}
        {\delta}}{\rho_{\min}^{\pi^b,\suppsapie} \epsilon^2}\biggl(
            \ln\frac{1}
        {\delta}+S_{\max}^{p,\pi^b}
        \biggr)+\frac{\ln\frac{1}{\delta}}{\ln\frac{1}{1-\rhominpib}}
        \Biggr),
    \end{align*}}%
    and $\tau^E$ is bounded as in Theorem~\ref{theorem: upper bound d irlo}. Furthermore, \pirlo is inclusion monotonic.
\end{restatable}
The price for the inclusion {monotonicity} is the additional term in the sample complexity which grows with $H^6$ and with $1/\rho_{\min}^{\pi^b,\suppsapie}$. The latter represents the minimum non-zero visitation probability with which policy $\pi^b$ covers $\suppsapie$, i.e., the support of $\rho^{p,\pi^E}$.
Intuitively, since the expert's policy is optimal, this additional term is due to a mismatch between optimal Q-functions under the different transition models of $\mathcal{C}(\widehat{p},b)$.
Notice that, under Assumption~\ref{assumption: coverage of behavioral policy}, $\suppsapie \subseteq \suppsapib$, consequently, $\rho_{\min}^{\pi^b,\suppsapie} \ge \rho_{\min}^{\pi^b,\suppsapib}$.
We can provide an analogous result for $d_\infty$.
\begin{restatable}{thr}{upperbounddinftypirlo}\label{theorem: upper bound d infty pirlo}
    Under the conditions of Theorem \ref{theorem: upper bound d pirlo},
    \pirlo is $(\epsilon,\delta)$-PAC for $d_\infty$-IRL
    with a sample complexity at most:
    {
    \begin{equation*}
    \begin{aligned}
        \tau^b&\le  \widetilde{\mathcal{O}}\Bigg(\frac{H^4\ln\frac{1}
        {\delta}}{\rhominpib \epsilon^2}\biggl(
            \ln\frac{1}
        {\delta}+S^{p,\pi^b}_{\max}
        \biggr) \\
        & 
           + \frac{H^6\ln\frac{1}
        {\delta}}{\rho_{\min}^{\pi^b,\suppsapie }\epsilon^2}\biggl(
            \ln\frac{1}
        {\delta}+S^{p,\pi^b}_{\max}
        \biggr) + \frac{\ln\frac{1}{\delta}}{\ln\frac{1}{1-\rhominpib}}\Bigg),
    \end{aligned}
    \end{equation*}}%
   and $\tau^E$ is bounded as in Theorem~\ref{theorem: upper bound d irlo}.
   Furthermore, \pirlo is inclusion monotonic.
\end{restatable}
Notice that both bounds in Theorem \ref{theorem: upper bound d pirlo} and Theorem \ref{theorem: upper bound d infty pirlo} also hold for the objectives defined in Equation \eqref{def: sub and super no relaxation pessimism}.\footnote{In Appendix \ref{subsec: superset no relaxation bound}, we provide a tighter bound for the superset $\widehat{\R}^\cup$ without using the relaxation. Moreover, in Appendix \ref{subsec: pirlo r = r hat}, we prove a larger sample complexity upper bound, when including an additional useful requirement.}

\section{{Reward Sanity Check with} \pirlo}\label{section: application}
In the literature, IRL algorithms \cite{ratliff2006maximum,ziebart2008maximum} provide \emph{criteria} to select a specific reward function from the {feasible set}.
Our algorithm, \pirlo, thanks to the {inclusion monotonicity} property, provides a partition of the space of rewards $\mathfrak{R}$ in three sets: ($i$) rewards contained in the sub-feasible set $\widehat{\R}^\cap$ (i.e., feasible w.h.p.), ($ii$)
rewards \textit{not} contained in the super-feasible set $\mathfrak{R}\setminus\widehat{\R}^\cup$ (i.e., not feasible w.h.p.), and ($iii$) rewards that we cannot discriminate with the given confidence ($\widehat{\R}^\cup\setminus\widehat{\R}^\cap$).
The situation is illustrated in Figure~\ref{fig: explanation sanity checker v2 online setting}. Thus, \pirlo can be used both as a \emph{sanity checker} on the rewards outputted by a specific IRL algorithm and for defining the set of rewards from which selecting one.
To exemplify this application, we have run \pirlo using highway driving data from \citet{LikmetaMRTGR21} and some human-interpretable reward.
We provide the experimental details and the results in Appendix \ref{section: appendix experiments}.


\section{A Bitter Lesson}\label{section: bitter lesson}
Up to now, we assumed to have two datasets
$\D^E$ and $\D^b$ of trajectories collected by policies $\pi^E$ and $\pi^b$, respectively.
As already noted, this setting generalizes the most common IRL scenario where the only dataset $\D^E$ is collected by the deterministic expert's policy $\pi^E$ and there is no possibility of collecting further data.
A natural question arises: 
\emph{Why not directly considering
the setting with $\D^E$ only?}
The reason lies in the following \emph{negative result} showing that the reward functions that
can be learned from a single expert's dataset $\D^E$ are not completely satisfactory. 
\begin{restatable}{prop}{greedyrewards}\label{theorem: greedy rewards}
    Let $\M$ be the usual MDP without reward with $A\ge 2$ and
    let $\pi^E$ be the deterministic
    expert's policy. Let $\D^E$ be a dataset of trajectories
    collected by following $\pi^E$ in $\M$.
    Then, for any reward in $r \in \sub$ it holds that:
    {
    \begin{align}
        \forall (s,h) \in \suppspie, \, \forall a \in \mathcal{A}: \quad r_h(s,\pi^E_h(s))\ge r_h(s,a).
    \end{align}}%
\end{restatable}
Thus, if we have no information about the transition model in non-expert's actions (as when we have $\mathcal{D}^E$ only), there exists no reward function $r$ that simultaneously: ($i$) surely belongs to the sub-feasible set ($r \in \sub$) and ($ii$) assigns to a non-expert's action a reward value greater than that assigned to
the expert's action in the same $(s,h)$ pair. This is clearly a property that is undesirable as it significantly limits the expressive power of the reward function, making IRL closer to behavioral cloning and, consequently, inheriting its limitations. As mentioned above, this issue can be overcome with a
behavioral policy $\pi^b$ that explores enough.
\begin{restatable}{prop}{nongreedyrewardsbehavioral}\label{theorem: non greedy rewards behavioral}
    Under the conditions of Proposition \ref{theorem: greedy rewards},
    assume that $p_h(\cdot|s,a)$ is known,
    where $a \in \A$ is a non-expert's action in $(s,h)\in \suppspie$.
    Then, if $p_h(\cdot|s,a)\neq p_h(\cdot|s,\pi^E_h(s))$,
    there exists a reward $r\in\sub$ such that:
    {
    \begin{align*}
        r_h(s,\pi_h^E(s))< r_h(s,a).
    \end{align*}}%
\end{restatable}

\section{Conclusion}
In this paper, we have introduced a novel notion of \textit{feasible set} and
an innovative \textit{learning framework} for managing the intrinsic difficulties of the \textit{offline} IRL setting. Furthermore, we have motivated the importance of \textit{inclusion monotonicity}, and we have devised
an original form of \textit{pessimism} to achieve it.
Then, we have presented two provably efficient algorithms, \irlo and \pirlo. We have shown that the latter provides guarantees
of inclusion monotonicity and that it can be employed as a \textit{reward sanity checker}. Finally, we have highlighted an \textit{intrinsic limitation}
of the offline IRL setting when samples from the experts are the only available.


\textbf{Limitations and Future Works}~~
To understand whether our algorithms are \emph{minimax optimal}, future works should focus on the derivation of sample complexity lower bounds for offline IRL.
Moreover, it would be appealing to extend our framework to more challenging (non-tabular) environments.


\section*{Impact Statement}
This paper presents work whose goal is to advance the field of Machine Learning. There are many potential societal consequences of our work, none which we feel must be specifically highlighted here.

\section*{Acknowledgements}
Funded by the European Union – Next Generation EU within the project NRPP M4C2, Investment 1.3 DD. 341 -  15 march 2022 – FAIR – Future Artificial Intelligence Research – Spoke 4 - PE00000013 - D53C22002380006.

%% file: appendix.tex
\appendix

\begin{landscape}
   \begin{table}[t]
    \centering
    \scalebox{0.8}{
    \begin{tabular}{c|c|c|c|c|c}
          & \cite{metelli2021provably} & \cite{metelli2023towards} & \cite{lindner2022active} & \cite{zhao2023inverse} & \textbf{Ours} \\
          Setting & online generative model & online generative model & online forward model & offline with behavioral policy & offline with behavioral policy\\
          Solution concepts & $\popgreen{\oldfs}$ &  $\popgreen{\oldfs}$ &  $\popgreen{\oldfs}$ & $\popred{\mathscr{R}\approx \oldfs}$ & $\popgreen{\fs}$\\
          Monotonicity & \popred{No} & \popred{No} & \popred{No} & $\popyellow{\widehat{r}\preccurlyeq r}$ & $\popgreen{\sub\subseteq\fs\subseteq\super}$\\
         Reward distance $r,\widehat{r}$ & $\|Q^*(p,r)-Q^*(p,\widehat{r})\|_\infty$ & $\|r-\widehat{r}\|_\infty$ & $\sup_{\widehat{\pi}^*}\|Q^*(p,r)-Q^{\widehat{\pi}^*}(p,r)\|_\infty$ & $\sup_{h\in\dsb{H}}\mathbb{E}_{p,\overline{\pi}}|V^{\overline{\pi}}(p,r)-V^{\overline{\pi}}(p,\widehat{r})|$ & $\sum_{h\in\dsb{H}}\mathbb{E}_{p,\pi^b}|r-\widehat{r}|$\\
         Feasible set distance & $\mathcal{H}$ &
         $\mathcal{H}$ & $\mathcal{H}$ & 
         $\sup_{V,A}$ & $\mathcal{H}$\\
         Lower bound? & \popred{No} & \popgreen{Yes} & \popred{No} & \popyellow{Yes} & \popred{No}\\
         Matching upper bound? & \popyellow{?} & \popgreen{Yes} & \popyellow{?} & \popyellow{?} & \popyellow{?}\\
         Trick & $\popred{\|Q\|_\infty \le H}$ & $\popyellow{r=r'/(1+\epsilon)}$ & $\popred{\|Q\|_\infty \le H}$ & $\popred{r=r(V,A)}$ & $\popgreen{d(r,\widehat{r})=d'(r,\widehat{r})/\max\{\|r\|_\infty,\|\widehat{r}\|_\infty\}}$\\
         Recover exact objective? & \popgreen{Yes} & \popgreen{Yes} & \popgreen{Yes} & \popred{No} & \popgreen{Yes}\\
    \end{tabular}
    }
    \caption{Comparison of our paper with the main related works. \emph{Setting} refers to whether the work analyses the online generative model, the online forward model, or the offline setting. To be precise, \citet{lindner2022active} provides some insights also for the online generative model, and \citet{zhao2023inverse} analyses also the online forward model in their setting. However, to keep the table clear and concise, we avoid inserting them. \emph{Solution concepts} indicates the learning objective. By $\mathscr{R}$ we denote the concept of reward mapping, i.e., parametrization of the reward function using $V,A$, introduced by \citet{zhao2023inverse}; in practice, it is equivalent to the ``old'' notion of feasible set $\oldfs$, which is more suitable to the online setting. \emph{Monotonicity} specifies whether the work adopts a notion of monotonicity according to some partial order. By $\preccurlyeq$ we mean the entrywise order among vectors. We stress that we are the first to devise the concept of \emph{inclusion monotonicity}, which permits applications such as the sanity checker. \emph{Reward distance} refers to the distance adopted in the space of rewards $\mathfrak{R}$. We denote $\|\cdot\|_\infty\coloneqq\max_{s,a,h}|\cdot|$ and we neglect the dependence on $s,a,h$ for simplicity. While the works about the online setting can afford to control the error with an $\ell_\infty$ distance, for the offline setting an $\ell_1$ objective is more suitable. It should be remarked that, while \citet{zhao2023inverse} compare the induced V-functions for a single fixed policy $\overline{\pi}$, and so they are not able to recover the exact objective, we compare directly the distance between rewards, and we do not suffer from this issue. \emph{Feasible set distance} indicates the distance among sets. All the works adopt the Hausdorff distance $\mathcal{H}$ except for \citet{zhao2023inverse}, which considers the worst parameters $V,A$; from a technical perspective, this is equivalent to $\mathcal{H}$. \citet{metelli2023towards} provide a matching \emph{lower bound}, while \citet{zhao2023inverse} computes a lower bound that does not depend on the confidence $\delta$, and so, it is not tight. By \emph{trick}, we mean the solution adopted to cope with the intrinsic unboundedness of the space of rewards.
    Observe that the constraint on the Q-function $\|Q\|_\infty \le H$ imposes strange non-box constraints on the rewards; the normalization by $1+\epsilon$ proposed by \citet{metelli2023towards} (see their Lemma B.1) is ``hard-coded'' for their $\ell_\infty$ objective and does not generalize to other notions of distance. Parametrizing the rewards by the bounded pair V-function, advantage function $V,A$, thus defining a reward mapping $\mathscr{R}$ as in \cite{zhao2023inverse}, is, in practice, equivalent to the constraint on the Q-function $\|Q\|_\infty \le H$ and, thus, it does not solve the problem. Instead, notice that our normalization trick is an effective way to cope with this issue, and it can be applied to the settings of all other works, simplifying the sample complexity analysis.
    Finally, by \emph{recover exact objective} we mean whether the setting analyzed is compatible with the solution concept adopted, i.e., if the solution concept can be retrieved exactly in that setting. Notice that this is not true for \citet{zhao2023inverse}, which cannot improve beyond an $\epsilon$-estimate due to the partial coverage of the data and the monotonicity notion that they adopt. Instead, our novel definition of feasible set $\fs$ and our notion of inclusion monotonicity overcome this issue.
    }
    \label{tab: comparison related works}
\end{table} 
\end{landscape}
 \clearpage

\section{Related Works}\label{subsection: related works}
Related works can be
distinguished in theoretical IRL works and works about Reinforcement Learning (RL)
in the \textit{offline} setting. Since the former group of papers
is more closely connected to the subjects of this work, we focus on it here,
and we refer to Appendix \ref{section: additional related works} for a presentation
of the remaining literature.

The notion of \textit{feasible set} has been introduced implicitly by \citet{ng2000algorithms}.
More recently, \citet{metelli2021provably} build upon previous works to
define the \textit{feasible set} explicitly. They
propose two algorithms for the estimation of \textit{feasible set}
in the \textit{online} setting with generative model.
Moreover, the authors analyze the sample complexity of the algorithms and
prove the first upper bound to the number of samples required for the estimation
of the \textit{feasible set} in the discounted infinite-horizon setting.
Such bound is in the order of
$\widetilde{\mathcal{O}}\bigl(\frac{SA}{(1-\gamma)^4\epsilon^2}\bigr)$,
where $S$ and $A$ denote, respectively, the cardinality of the state and action spaces,
$\gamma$ is the discount factor and $\epsilon$ is the accuracy, measured as distance
in max norm between the induced Q-functions.
Later, \citet{lindner2022active} proposes an algorithm, named AceIRL,
to estimate the \textit{feasible set} in the \textit{online} setting with
forward model. The result is an upper bound of
$\widetilde{\mathcal{O}}\bigl(\frac{H^5SA}{\epsilon^2}\bigr)$ to the number of trajectories
required in the episodic finite-horizon setting.
The first lower bound to the sample complexity of IRL has been devised by \citet{Komanduru2021ALB}.
However, their setting concerns state-only rewards in tabular Markov Decision Processes (MDPs)
with only two actions, resulting in a lower bound in the order of $\Omega(S\ln S)$.
\citet{metelli2023towards} analyzes the \textit{online} setting with generative
model by measuring the accuracy using the max norm directly between rewards.
By adopting two different constructions for the hard instances, it proves
a lower bound, along with a matching upper bound, for the sample complexity
of estimating the \textit{feasible set}. The number of samples
is in the order of $\widetilde{\Theta}\bigl(\frac{H^3SA}{\epsilon^2}(\ln\frac{1}{\delta}+S)\bigr)$,
where $\delta$ is the confidence. 
It is worth mentioning the work of \citet{zhao2023inverse},
which analyze the sample complexity of estimating the \textit{reward mapping},
a concept analogous to that of \emph{feasible set}, in the context of offline IRL.
They propose algorithms for solving the problem in
both the \textit{offline} and \textit{online} settings, and analyze the sample
complexity, obtaining an upper bound of
$\widetilde{\mathcal{O}}\bigl(\frac{H^4S^2 C^*}{\epsilon^2}\bigr)$ for the
\textit{offline} setting, where $C^*$ is a
concentrability coefficient.
However, \citet{zhao2023inverse} adopts a solution concept which is intrinsically connected to the
coverage of the state-action-stage space, which is a strong requirement in the \textit{offline} context. As a consequence, the entrywise reward-based pessimistic approach proposed by the authors is not able to recover the solution concept exactly. We also mention \cite{yue2023clare} and \cite{siliang2024when} as two additional IRL works that adopt pessimism but with different settings than ours.

For a clear comparison of these works, see Table \ref{tab: comparison related works}.


\paragraph{Additional Related Works}\label{section: additional related works}
It is worth mentioning also the works that focus on (forward) RL in the
\textit{offline} setting, because they share with our topic some important
concepts and technical tools. We provide a brief overview in this section.

The principle of \textit{optimism} in the face of uncertainty is a
well-established tool for favoring exploration in the context
of \textit{online} bandits \cite{lattimore2020bandit} and \textit{online} (forward)
RL \cite{Kearns2002NearOptimalRL,brafman2003rmax,azar2017minimax,dann2017unifying}.
However, in the \textit{offline} setting, the learning agent is given
a batch dataset and is not allowed to interact with the environment, thus
the adoption of \textit{optimism} does not improve the performances.
Moreover, one of the biggest challenges of the \textit{offline} RL setting
is that the given dataset might suffer from an insufficient coverage
of the space \cite{levine2020offline}.
To improve the performances of algorithms for solving the
\textit{offline} policy optimization problem, the commonly
adopted mechanism is the \textit{pessimism} principle:
``\textit{Behave as though the world was plausibly worse than you observed it to be}''
\cite{buckman2020importance}. As opposed to the principle of \textit{optimism}
in the face of uncertainty which favors exploration, the \textit{pessimism} principle
favors exploitation.
This tool has been adopted in a variety of works for devising
algorithms to solve the \textit{offline} policy optimization problem
\cite{yu2020mopo,Kumar2020ConservativeQF,liu2020provably}.
From a theoretical perspective, \citet{buckman2020importance}
proposes a unified framework for the study of this kind of algorithms,
revealing the reasons why the \textit{pessimism} principle
can demonstrate good performance even when the dataset is
not informative of every policy. Moreover, \citet{Jin2021IsPP} proposes a
\textit{pessimistic} variant of the value iteration algorithm,
named PEVI, and it shows that the \textit{pessimism} principle is not only provably efficient,
but also minimax optimal. 
Another line of research more closely related to the \textit{offline} IRL setting is that
introduced by \citet{Rashidinejad2021bridging}. They analyse the \textit{offline} policy optimization problem
in the novel setting in which the composition of the batch dataset can be located at any point in the range
between \textit{expert data} and \textit{uniform coverage data}. \textit{Expert data} means that the dataset
has been collected by an expert, and the problem reduces to the imitation learning problem
\cite{osa2018algorithmic}, while
\textit{uniform coverage data} refers to a dataset that guarantees a uniform coverage
of the space, which is a common setting in which it is usually required the existence of a
uniformly bounded concentrability coefficient \cite{munos2007performance}. Specifically,
\citet{Rashidinejad2021bridging} proposes a new framework that smoothly interpolates between the two extremes of
data composition, and analyses the information-theoretic limits of LCB, the algorithm they propose,
in three different settings.
%
Finally, \citet{xie2021bridging} builds upon \citet{Rashidinejad2021bridging} and devises
\textit{policy finetuning}, a framework that interpolates between
online and offline RL. Remarkably, \citet{xie2021bridging} designs a novel algorithm,
PEVI-Adv, which achieves a sample complexity of at most
$\widetilde{\mathcal{O}}\bigl(\frac{H^3SC^*}{\epsilon^2}\bigr)$
episodes in the finite-horizon setting, where $C^*$ is the single-policy concentrability
coefficient. Also notice that the authors prove a matching lower bound.

\section{A Framework for the ``old'' Feasible Set}\label{section: old fs}
In this section, we apply the framework we presented in Section \ref{section: framework}
to the ``old'' notion of \textit{feasible set} (Definition \ref{def: old fs}).
In addition, we present a rather negative result on the kind of reward functions contained
in the subset of $\oldfs$.

Let us begin by adapting Definition \ref{def: subset superset fs} to $\oldfs$.
Recall that we use $\D^E$ to estimate $\pi^E$, and $\D^b$ to estimate $p$.
\begin{defi}[Subset and Superset of $\oldfs$]
    Let $\M=\langle\S,\A,p,\mu_0,H\rangle$ be an MDP without reward and let
    $\pi^E$ be the deterministic
    expert's policy and $\pi^b$ the behavioral policy.
    Then, we define the \emph{subset} $\oldfs^\cap$ and
    the \emph{superset} $\oldfs^\cup$ of the \textit{feasible set} $\oldfs$ as:
    \begin{align*}
        \oldfs^\cap&\coloneqq
        \bigcap\limits_{p'\in \eqclassp{p}}
        \bigcap\limits_{\pi'\in \eqclasspie{\pi^E}}
        \overline{\R}_{p',\pi'},\\
        \oldfs^\cup&\coloneqq
        \bigcup\limits_{p'\in \eqclassp{p}}
        \bigcup\limits_{\pi'\in \eqclasspie{\pi^E}}
        \overline{\R}_{p',\pi'}.
    \end{align*}
\end{defi}
Clearly, since $p\in \eqclassp{p}$ and $\pi^E\in \eqclasspie{\pi^E}$, it holds that
$\oldfs^\cap\subseteq\oldfs\subseteq
\oldfs^\cup$. Also notice that when $\suppsapib=\SAH$ (and so, because of Assumption \ref{assumption: coverage of behavioral policy}, $\suppspie=\SH$), then
$\oldfs^\cap=\oldfs=\oldfs^\cup$.
The intuition underlying the definition is analogous to that for $\sub$ and $\super$.

The following theorem shows that $\oldfs^\cap$ and $\oldfs^\cup$ are
``well-defined''.
\begin{infthr}
    Let $\M=\langle\S,\A,p,\mu_0,H\rangle$ be an MDP without reward and
    let $\D^E$ and $\D^b$ be two datasets of trajectories collected by policies $\pi^E$ and $\pi^b$. Then, subset $\oldfs^\cap$ and superset $\oldfs^\cup$
    are the \emph{tightest learnable} subset and superset of
    $\oldfs$ from $\D^E$ and $\D^b$.
\end{infthr}
In Appendix \ref{section: learnability fs}, we enunciate this result formally and we provide a proof.

The following theorem shows a negative result on the kind of reward functions
contained in $\oldfs^\cap$ under reasonable conditions. 
The intuition is that, since $\oldfs$ requires the knowledge of the optimal (expert's) action
in the entire $\SH$, then if there are pairs $(s,h)\in\SH$ in which we cannot have this information, then we are forced to make all actions optimal there (this is exactly what the intersection over policies makes in $\oldfs^\cap$). This imposes many constraints on the structure of the rewards, resulting in the following theorem.
It should be remarked that, if we had only dataset $\D^b$ and
if Assumption \ref{assumption: coverage of behavioral policy} was violated, then
we might not know the expert's action in some $(s,h)\in\suppspie$, and a similar result would also hold for $\fs$;
however, that would be a non-realistic setting for IRL.
\begin{restatable}{thr}{issueoffline}\label{theorem: issue fs offline setting}
    Let $\M$ be an MDP without reward and let $\pi^E$ be the expert's policy and $\pi^b$ be the behavioral policy.
    If for any stage $h\in\dsb{H}$ there is at least a state, say $s_h$, for which
    $(s_h,h)\notin \suppshpib$,
    then $\oldfs^\cap$ is made of ``almost-constant''
    rewards. Formally:
    $r\in\oldfs^\cap$ if and only if
    there exists a sequence $\{k_h\}_h$ of $H$ real numbers
    and a set $\{\bar{r}_s\}_s$ with as many real numbers as the cardinality of
    the support of $\mu_0$,
    such that, for any $(s,a,h)\in\SAH$:
    \[
        \begin{cases}
            r_h(s,a)=x(h,s)\quad\text{if }(s,h)\in \suppspib_h\wedge a=\pi^E_h(s)\\
            r_h(s,a)\le x(h,s)\quad\text{if }(s,h)\in \suppspib_h\wedge a\neq\pi^E_h(s)\\
            r_h(s,a)=k_h\quad\text{if }(s,h)\notin \suppspib_h\\
        \end{cases},
    \]
    where $x(h,s)\coloneqq\bar{r}_s$ if $h=1$, and $x(h,s)\coloneqq k_h$ otherwise.
\end{restatable}
This theorem states that, under reasonable conditions, i.e., if at every stage $h$
there is at least one state not reached by the behavioral policy $\pi^b$,
then all the rewards of the subset $\oldfs^\cap$
have a trivial form, which, i.a.,
does not depend on the specific
value of the transition model, but only on its ``support'' $\suppspib$.
In practice, this means that $\oldfs^\cap$ does not represent
an interesting target for learning.
\begin{proof}[Proof of Theorem \ref{theorem: issue fs offline setting}]
    The theorem expresses a necessary and sufficient condition, thus two proofs are required.
    In the following, recall that $\suppspib\supseteq\suppspie$
    because of Assumption \ref{assumption: coverage of behavioral policy}, and so
    the existence of $(s_h,h)\notin \suppshpib$ for all $h\in\dsb{H}$ entails
    the existence of $(s_h,h)\notin \suppshpie$ for all $h\in\dsb{H}$.
    Moreover, we consider the representation of $\oldfs$ as provided in Eq. \ref{eq: oldfs with Q star}.
    We begin with the proof of the sufficiency. 

    The proof of the sufficiency proceeds in four steps. First, we show that, outside $\suppspie$,
    the definition of $\oldfs^\cap$ enforces all actions to be optimal.
    Then, we use this condition to show that, for any $h\ge 2$, irrespective of
    the transition model, all states take on the
    same value function value $\{\bar{V}_h\}_h$, and that all the actions in $(s,h)$ outside $\suppspie$ have
    the same reward value $\{k_h\}_h$. The subsequent steps build upon these findings
    to show that expert's actions take on constant reward value $\{k_h\}_h$ for $h\ge 2$,
    and that non-expert's actions are smaller than the corresponding expert's action.

    Let us begin by showing that, outside $\suppspie$, all actions shall be optimal.
    By definition of $\oldfs^\cap$, we have:
    \begin{align*}
        \oldfs^\cap&\coloneqq\bigcap\limits_{p'\in \eqclassp{p}}
        \bigcap\limits_{\pi'\in \eqclasspie{\pi^E}}
        \overline{\R}_{p',\pi'}\\
        &=\{r\in\mathfrak{R}\,|\,
        \forall p'\in\eqclassp{p}\,,
        \forall \pi'\in \eqclasspie{\pi^E}\,,\forall(s,h)\in\SH:\;
        \E\limits_{a\sim \pi_h'(\cdot|s)}Q^*_h(s,a;p',r)=\max\limits_{a'\in\A}Q^*_h(s,a';p',r)
        \}\\
        &=\{r\in\mathfrak{R}\,|\,
        \forall p'\in\eqclassp{p}\,,
        \forall(s,h)\in\SH\,,\forall \pi'\in \eqclasspie{\pi^E}:\;
        \E\limits_{a\sim \pi_h'(\cdot|s)}Q^*_h(s,a;p',r)=\max\limits_{a'\in\A}Q^*_h(s,a';p',r)
        \}\\
        &\markref{(1)}{=}\{r\in\mathfrak{R}\,|\,
        \forall p'\in\eqclassp{p}:\bigl(
        \forall(s,h)\in \suppspie: Q^*_h(s,\pi_h^E(s);p',r)=\max\limits_{a\in\A}Q^*_h(s,a;p',r)\wedge\\
        &\quad\quad\quad\quad\quad\quad\quad\quad
        \forall(s,h)\notin \suppspie:\,\forall \pi_h'(\cdot|s)\in\Delta^\A:\,
        \E\limits_{a\sim \pi_h'(\cdot|s)}Q^*_h(s,a;p',r)=\max\limits_{a'\in\A}Q^*_h(s,a';p',r)\bigr)\}\\
        &=\{r\in\mathfrak{R}\,|\,
        \forall p'\in\eqclassp{p}:\bigl(
        \forall(s,h)\in \suppspie: Q^*_h(s,\pi_h^E(s);p',r)=\max\limits_{a\in\A}Q^*_h(s,a;p',r)\wedge\\
        &\quad\quad\quad\quad\quad\quad\quad\quad
        \forall(s,h)\notin \suppspie:\,\forall a\in\A:\,
        Q^*_h(s,a;p',r)=\max\limits_{a'\in\A}Q^*_h(s,a';p',r)\bigr)\},\\
    \end{align*}
    where, at (1), we have partitioned $\SH$ using $\suppspie$ and we have applied the definition of
    $\eqpie$. This shows that, outside $\suppspie$, all actions must be optimal.

    Now, we show that, for any reward $r\in\oldfs^\cap$,
    there exists a sequence of value functions\footnote{The value function of policy $\pi$
    at $(s,h)$ under transition model $p$ and reward $r$ is defined as:
    $V^{\pi}_h(s;p,r)\coloneqq \sum_{a\sim\pi_h(\cdot|s)}Q^{\pi}_h(s,a;p,r)$. The optimal value function is defined as: $V^{*}_h(s;p,r)\coloneqq \max_{a\in\A}Q^{*}_h(s,a;p,r)$.
    } $\{\bar{V}_h\}_{h\in\dsb{2,H}}$
    induced by $r$ which does not depend neither on the transition model nor
    on the state considered: $V^*_h(s;p',r)=\bar{V}_h$
    for any $s\in\S,h\in\dsb{2,H}$, and for any $p'$ of $\eqclassp{p}$.
    So, let $r$ be a reward of $\oldfs^\cap$ and
    let $(s,h)$ be a pair not in $\suppspib\supseteq\suppspie$, for $h\in\dsb{H-1}$.
    Notice that the existence of pair $(s,h)$ is guaranteed by hypothesis.
    For what we have seen at the previous step, for any pair
    of actions $a_1,a_2$, for any $p'\in\eqclassp{p}$, it holds that
    $Q_{h}^*(s,a_1;p',r)=Q_h^*(s,a_2;p',r)$.
    Through the Bellman's equation, we can write:
    \begin{equation}\label{eq: for proof fs}
        r_h(s,a_1)+\E\limits_{s'\sim p'_h(\cdot|s,a_1)}[V_{h+1}^*(s';p',r)]=
        r_h(s,a_2)+\E\limits_{s'\sim p'_h(\cdot|s,a_2)}[V_{h+1}^*(s';p',r)].
    \end{equation}
    Let $s_1,s_2\in\S$. By definition of $\eqp$,
    this condition must hold for any $p'_{h}(\cdot|s,a_1)\in\Delta^\S$
    and $p'_{h}(\cdot|s,a_2)\in\Delta^\S$.
    In particular, let us take two transition models $p^1,p^2\in \eqclassp{p}$
    such that
    :
    \begin{align*}
        p^{1}&=\begin{cases}
            p^{1}_{h}(s_1|s,a_1)=1\\
            p^{1}_{h}(s_1|s,a_2)=1\\
        \end{cases},\\
        p^{2}&=\begin{cases}
            p^{2}_{h}(s_1|s,a_1)=1\\
            p^{2}_{h}(s_1|s,a_2)=0\\
        \end{cases}.
    \end{align*}
    Inserting $p^1$ into Eq. \ref{eq: for proof fs}, we get:
    \[
        r_{h}(s,a_1)+V_{h+1}^*(s_1;p^1,r)=r_{h}(s,a_2)+V_{h+1}^*(s_1,p^1,r)\implies
        r_{h}(s,a_1)=r_{h}(s,a_2).
    \]
    Because $a_1$ and $a_2$ are arbitrary, this holds for any action $a\in\A$ in $(s,h)\notin \suppspib$.
    Therefore, we have shown that (the condition at $H$ is trivial since $V_{H+1}^*(s;p',r)=0$):
    \[
        r\in\oldfs^\cap\implies \exists \{k_h\}_{h\in\dsb{H}}:
        \;\forall (s,h)\notin \suppspib,\,\forall a\in\A:\;r_h(s,a)=k_h.
    \]
    Inserting $p^2$ into Eq. \ref{eq: for proof fs}, we obtain:
    \[
        r_{h}(s,a_1)+V_{h+1}^*(s_1;p^2,r)=r_{h}(s,a_2)+V_{h+1}^*(s_2;p^2,r)\implies
        V_{h+1}^*(s_1;p^2,r)=V_{h+1}^*(s_2;p^2,r),
    \]
    since $r_{h}(s,a_1)=r_{h}(s,a_2)$. Because $p^2$ (and so the next state $s_1$) can be chosen arbitrarily
    in $h+1\ge 2$, then we have proved that:
    \[
        r\in\oldfs^\cap\implies
        \exists \{\bar{V}_h\}_{h\in\dsb{2,H}}:\;
        \forall p'\in \eqclassp{p},\,
        \forall (s,h)\in\SH:\;
        V^*_h(s;p',r)=\bar{V}_h.
    \]
    
    In a similar manner, we can prove the same result also for $(s,h)\in\suppspib\setminus\suppspie$.

    The next step of the proof consists in showing that, for any $h\in\dsb{2,H}$,
    for any $s\in \suppshpie$, the reward value assigned to expert's action coincides with
    $k_h$.
    Let $(s,h)\in \suppspie$ with $h\ge 2$. For any $p'\in\eqclassp{p}$, it holds:
    \begin{align}\label{eq: vbar1}
        \begin{split}
            \bar{V}_h=V^*_h(s;p',r)=Q_h^*(s,a^E;p',r)&=r_h(s,a^E)+\E\limits_{s'\sim p'_h(\cdot|s,a^E)}[V^*_{h+1}(s';p',r)]\\
            &=r_h(s,a^E)+\E\limits_{s'\sim p'_h(\cdot|s,a^E)}[\bar{V}_{h+1}]\\
            &=r_h(s,a^E)+\bar{V}_{h+1}.\\
        \end{split}
    \end{align}
    By hypothesis, there exists $s'\notin \suppspib_h$ such that, for any $a\in\A$:
    \begin{align}\label{eq: vbar2}
        \begin{split}
            \bar{V}_h=V^*_h(s';p',r)=Q_h^*(s',a;p',r)&=r_h(s',a)+\E\limits_{s''\sim p'_h(\cdot|s',a)}[V^*_{h+1}(s'';p',r)]\\
            &=k_h+\E\limits_{s'\sim p'_h(\cdot|s',a)}[\bar{V}_{h+1}]\\
            &=k_h+\bar{V}_{h+1}.\\
        \end{split}
    \end{align}
    Comparing Eq. \ref{eq: vbar1} and Eq. \ref{eq: vbar2}, we
    infer that $r_h(s,a)=k_h$.

    With a similar reasoning, we can prove that the reward value of non-expert's action
    shall be at most $k_h$.
    For simplicity, set $\bar{V}_{H+1}=0$. Let $(s,h)$ be any pair in $\suppspie$
    and let $a$ be any non-expert's action. Then, for any $p'\in\eqclassp{p}$, we have:
    \begin{align*}
        Q^*_h(s,a;p',r)&=r_h(s,a)+\E\limits_{s'\sim p'_h(\cdot|s,a)}[V^*_{h+1}(s';p',r)]\\
        &=r_h(s,a)+\bar{V}_{h+1}\\
        &\le V^*_h(s;p',r)\\
        &=r_h(s,\pi^E_h(s))+\bar{V}_{h+1}\\
        &=k_h+\bar{V}_{h+1}.
    \end{align*}
    From which it follows that $r_h(s,a)\le k_h$.
    This concludes the proof of the sufficiency.

    With regards to the necessity, we have to show that any reward $r$ that can be expressed
    as in the statement of the theorem belongs to $\oldfs^\cap$.
    It is easy to notice that, irrespective of the transition model $p'$,
    the optimal value function of any state $s\in\S$ at stage $h\ge 2$ is
    $V^*_h(s;p',r)=\sum_{i=h}^H k_i$, and that it is achieved by playing
    the expert's policy. At step $h=1$, since all next states take on the same optimal
    value function, the result is immediate.
\end{proof}

\section{On the Learnability of the Feasible Set}\label{section: learnability fs}
In this section, we formalize the notion of (PAC-)\textit{learnability} and we analyze the learnability
properties of the various objects that we introduced in Section \ref{section: framework} (and in Appendix \ref{section: old fs}).
Specifically, we show that both the definitions of \textit{feasible set}
$\oldfs$ and $\fs$ are not \textit{learnable} in the setting of Section \ref{section: framework}
unless the behavioral policy covers the entire $\SAH$ space. Next, we demonstrate that
the framework we have introduced cannot be improved; simply put, we show that
$\oldfs^\cap$ and $\oldfs^\cup$ are the \textit{tightest learnable}
bounds of $\oldfs$ according to partial order $\subseteq$, and that
$\fs^\cap$ and $\fs^\cup$ are the \textit{tightest learnable}
bounds of $\fs$.

We give a definition of \textit{learnability} in the context of the Probably Approximately
Correct (PAC) framework since our main focus is on PAC bounds to the sample complexity in this work.
Let $\phi\in\Phi$ be our target of learning, i.e., the quantity that we aim to estimate,
and let $\mu$ be a certain distribution that provides us with $N$ independent samples
$X_1,X_2,\dotsc,X_N\sim\mu$. Intuitively,
$\phi$ can be \textit{learned} from $\mu$ if there exists a procedure able
to use the samples $X_1,X_2,\dotsc,X_N$ of $\mu$ to create a ``good-enough''
estimate of $\phi$, where the ``goodness'' is measured by a meaningful notion
of distance. We formalize the intuition in the following definition.
\begin{defi}[PAC-learnability]\label{def: pac-learnability}
    A quantity $\phi\in\Phi$ is \textit{PAC-learnable} from a distribution $\mu$
    if there exists a semimetric $d$ in $\Phi$, that satisfies a $\rho$-relaxed triangle inequality with finite $\rho$, and an algorithm $\mathfrak{A}$
    such that, for any $\epsilon,\delta\in(0,1)$,
    there exists a finite $N\in\mathbb{N}$ for which:
    \begin{align*}
        \mathbb{P}_\mu \bigl(
            d(\phi,\widehat{\phi}\le\epsilon)\bigr)\ge 1-\delta,
    \end{align*}
    where $\widehat{\phi}\in\Phi$ is the estimate of $\phi$
    computed by $\mathfrak{A}$ using at least $N$ samples, and $\mathbb{P}_\mu$ is the
    probability measure induced by $\mu$.
\end{defi}
Simply put, $\phi$ is \textit{PAC-learnable} if the samples from $\mu$ leak ``enough'' information
about $\phi$.
Notice that any metric satisfies a $\rho$-relaxed triangle inequality with $\rho=1$. See Appendix \ref{section: semimetric} for
an in-depth analysis of the semimetrics used in this work.

In the context of Section \ref{section: framework}, we identify as quantity of
interest $\phi$ the \textit{feasible set} $\oldfs$ ($\fs$), and as distribution
generating samples\footnote{
Observe that, to avoid mentioning the creation of a mixture of distributions $\mathbb{P}_{p,\pi^b}$ and $\mathbb{P}_{p,\pi^E}$
(because of the two datasets), we assume that $\pi^E$ and $\suppspie$ are given and need not to be learned.
While this simplifies the learning problem, notice that even the estimation of the transition model alone
can end up in non-learnability issues.
} $\mu = \mathbb{P}_{p,\pi^b}$.
We have the following result of \textit{non-learnability} of the \textit{feasible set}
in the \textit{offline} setting.
\begin{thr}\label{theorem: fs not learnable}
    Let $\M=\langle\S,\A,p,\mu_0,H\rangle$ be an MDP without reward and let $\pi^E$ be the deterministic expert policy.
    Assume to know $\pi^E$ in all $(s,h)\in\SH$ and also to know $\suppspie$, i.e., there is no need to learn them.
    Let $\suppsapib$ denote the portion of space covered by a behavioral distribution $\pi^b$ in $\M$. 
    If $\suppsapib \neq \SAH$, then $\oldfs$ is not \textit{PAC-learnable} from $\mathbb{P}_{p,\pi^b}$.
    Moreover, if $\suppsapib \neq \SAH$ and for at least one $(s,h)\notin \suppspib$ there
    exists a policy $\pi\in\Pi$ such that $\mathbb{P}_{p,\pi}(s_h=s)>0$,
    then not even $\fs$ is \textit{PAC-learnable} from $\mathbb{P}_{p,\pi^b}$.
\end{thr}
\begin{proof}
    Let us start with $\oldfs$.
    The idea is to construct two problem instances whose feasible sets
    lie at a fixed non-zero distance and such that
    samples do not allow to discriminate between them.

    We start with the construction of the two instances.
    Let $\pi^E$ be the expert's policy and let $\M_1=\langle\S,\A,p^1,\mu_0,H\rangle$ be an MDP without reward in which policy $\pi^b$
    induces the distribution over trajectories $\mathbb{P}_{p^1,\pi^b}$.
    By hypothesis, there exists triple $(s,a,h)\in\SAH$ not in $\suppsapib$, i.e., such that
    $\mathbb{P}_{p^1,\pi^b}(s_h=s,a_h=a)=0$.
    Let us construct another problem instance $\M_2=\langle\S,\A,p^2,\mu_0,H\rangle$
    such that $p^2 \eqp p^1$. This is possible by simply setting, at triple $(s,a,h)$,
    the condition $p^1_h(\cdot|s,a)\neq p^2_h(\cdot|s,a)$, and equality elsewhere. Observe that, because of this choice,
    $\mathbb{P}_{p^1,\pi^b}=\mathbb{P}_{p^2,\pi^b}$. Let $\overline{\R}_{p^1,\pi^E}$ and $\overline{\R}_{p^2,\pi^E}$ denote, respectively,
    the feasible sets of instances $\M_1$ and $\M_2$ with expert's policy $\pi^E$.

    Now we show that $\overline{\R}_{p^1,\pi^E}\neq\overline{\R}_{p^2,\pi^E}$. To do so, we claim the existence a reward $r\in \overline{\R}_{p^1,\pi^E}$
    such that $r\notin \overline{\R}_{p^2,\pi^E}$. W.l.o.g., assume $a$ be a non-expert's action\footnote{
        Otherwise, we can make the same construction with any non-expert's action $a'\in\A$,
        and we can show that the constraints of the feasible sets of $\M_1$ and $\M_2$
        have the same lower bounds $Q^{\pi^E}_h(s,a';\popblue{p^1},r)= Q^{\pi^E}_h(s,a';\popblue{p^2},r)$,
        but different upper bounds $Q^{\pi^E}_h(s,\pi^E_h(s);\popblue{p^1},r)\neq Q^{\pi^E}_h(s,\pi^E_h(s);\popblue{p^2},r)$.
    }. By definition of $\overline{\R}_{p^1,\pi^E}$,
    at triple $(s,a,h)$, it holds that $Q^{\pi^E}_h(s,a;p^1,r)\le Q^{\pi^E}_h(s,\pi^E_h(s);p^1,r)$. Since $p^2$ coincides with
    $p^1$ everywhere except for triple $(s,a,h)$, the constraint is equivalent to $Q^{\pi^E}_h(s,a;p^1,r)\le Q^{\pi^E}_h(s,\pi^E_h(s);\popblue{p^2},r)$.
    Similarly, we have that $r\in \overline{\R}_{p^2,\pi^E}$ if $Q^{\pi^E}_h(s,a;p^2,r)\le Q^{\pi^E}_h(s,\pi^E_h(s);p^2,r)$.
    Clearly, both the constraints have the same upper bound $Q^{\pi^E}_h(s,\pi^E_h(s);p^2,r)$, and since
    $Q^{\pi^E}_{h+1}(s',\pi^E_{h+1}(s');\popblue{p^1},r)=Q^{\pi^E}_{h+1}(s',\pi^E_{h+1}(s');\popblue{p^2},r)$ for any $s'\in\S$, then $p^1_h(\cdot|s,a)\neq p^2_h(\cdot|s,a)$
    entails $Q^{\pi^E}_h(s,a;\popblue{p^1},r)\neq Q^{\pi^E}_h(s,a;\popblue{p^2},r)$. Therefore, we can find $r\in \overline{\R}_{p^1,\pi^E}$ such that $r\notin \overline{\R}_{p^2,\pi^E}$,
    thus $\overline{\R}_{p^1,\pi^E}\neq\overline{\R}_{p^2,\pi^E}$.
    
    We proceed by contradiction. Let us assume that the feasible set $\oldfs$ is PAC-learnable
    in both $\M_1$ and $\M_2$. By definition of learnability, there exists a semi-metric $d$ and an algorithm $\mathfrak{A}$
    with certain properties. By definition of semi-metric, since $\overline{\R}_{p^1,\pi^E}\neq\overline{\R}_{p^2,\pi^E}$, then
    there exists a certain $c>0$ such that $d(\overline{\R}_{p^1,\pi^E},\overline{\R}_{p^2,\pi^E})=c$. Moreover, by
    $\rho$-relaxed triangle inequality,
    we know that a set of rewards $\widetilde{\R}$ such that
    $d(\widetilde{\R},\overline{\R}_{p^1,\pi^E})<c/(2\rho)$ and $d(\widetilde{\R},\overline{\R}_{p^2,\pi^E})<c/(2\rho)$ at the same time does not exist,
    thus the two events $\bigl\{d(\widetilde{\R},\overline{\R}_{p^1,\pi^E})<c/(2\rho)\bigr\}$ and $\bigl\{d(\widetilde{\R},\overline{\R}_{p^2,\pi^E})<c/(2\rho)\bigr\}$
    are disjoint.
    By the choices $\epsilon< c/(2\rho)$ and $\delta <1/2$, algorithm $\mathfrak{A}$ must satisfy 
    \begin{align*}
        \mathop{\mathbb{P}}\limits_{p^1,\pi^b}\Bigl(d(\widehat{\R}^\mathfrak{A},\overline{\R}_{p^1,\pi^E}) < \frac{c}{2\rho}\Bigr)>\frac{1}{2}
        \quad\quad\wedge\quad\quad
        \mathop{\mathbb{P}}\limits_{p^2,\pi^b}\Bigl(d(\widehat{\R}^\mathfrak{A},\overline{\R}_{p^2,\pi^E}) < \frac{c}{2\rho}\Bigr)>\frac{1}{2}.
    \end{align*}
    By construction, we have $\mathbb{P}_{p^1,\pi^b}=\mathbb{P}_{p^2,\pi^b}$. In other words, samples do not allow
    to discriminate between instances $\M_1$ and $\M_2$, and so between $\overline{\R}_{p^1,\pi^E}$
    and $\overline{\R}_{p^2,\pi^E}$. Therefore, when faced with $\M_1$, independently on the number $N$ of samples,
    algorithm $\mathfrak{A}$ outputs
    $\widehat{\R}^\mathfrak{A}$ such that:
    \begin{align*}
        \mathop{\mathbb{P}}\limits_{p^1,\pi^b}\biggl(&\bigl\{d(\widehat{\R}^\mathfrak{A},\overline{\R}_{p^1,\pi^E}) < \frac{c}{2\rho}\bigr\}\cup
        \bigl\{d(\widehat{\R}^\mathfrak{A},\overline{\R}_{p^2,\pi^E}) < \frac{c}{2\rho}\bigr\}\biggr)\\
        &=\underbrace{\mathop{\mathbb{P}}\limits_{p^1,\pi^b}\biggl(d(\widehat{\R}^\mathfrak{A},\overline{\R}_{p^1,\pi^E}) < \frac{c}{2\rho}\biggr)}_{>1/2}
        + \underbrace{\mathop{\mathbb{P}}\limits_{p^1,\pi^b}\biggl(d(\widehat{\R}^\mathfrak{A},\overline{\R}_{p^2,\pi^E}) < \frac{c}{2\rho}\biggr)}_{>1/2}
        >1,
    \end{align*}
    where we have used that the two events are disjoint.
    This is clearly a contradiction, thus the statement of the theorem holds for the notion of feasible set
    in Definition \ref{def: old fs}.

    With regards to the novel notion of feasible set $\fs$, the proof is analogous. The only difference is in how to show that
    $\R_{p^1,\pi^E}\neq\R_{p^2,\pi^E}$ when the triple $(s,a,h)\notin \suppsapib$ is such that
    $(s,h)\notin \suppspib$. Indeed, by Definition \ref{def: new fs}, in such $(s,h)$ there is no constraint
    on which action shall be optimal. However, by the hypothesis contained in the statement of the theorem,
    there exists a policy that brings to $(s,h)$, so since $\pi^b$ does not reach $(s,h)$, then there exists
    another triple $(s',a',h')$, with $h'<h$, such that $(s',a',h')\notin \suppsapib$ and $(s',h')\in \suppspib$.
    Therefore, the same passages adopted to show that $\R_{p^1,\pi^E}\neq\R_{p^2,\pi^E}$
    can be used to show also that $\R_{p^1,\pi^E}\neq\R_{p^2,\pi^E}$.
    It should be remarked that the hypothesis $\suppsapib \neq \SAH$ alone
    is not sufficient\footnote{
        Consider for instance the MDP without reward $\M$ in which $\SH\setminus \suppspib=\{(\bar{s},1)\}$ and
        $\SAH\setminus \suppsapib=\{(\bar{s},a_1,1),\dotsc,(\bar{s},a_A,1)\}$, i.e.,
        that $\pi^b$ covers the entire space except for a state $\bar{s}$ at stage $h=1$ ($\mu_0(\bar{s})=0$).
        Clearly, such state does not appear in the constraints defining the feasible set, and the feasible set $\fs$ is learnable by $\mathbb{P}_{p,\pi^b}$!
    } for the non-learnability of $\fs$.    
    This concludes the proof.
\end{proof}
The following theorem demonstrates that the solution concepts (subset and superset) that we propose in our
framework are the \emph{tightest learnable}. Observe that Theorem \ref{theorem: tightness of bounds new fs} entails Theorem \ref{theorem: fs not learnable}. However,
since in the proof of Theorem \ref{theorem: tightness of bounds new fs} we make use of the construction
introduced in the proof of Theorem \ref{theorem: fs not learnable}, we prefer to keep the two theorems separated.
\begin{thr}\label{theorem: tightness of bounds new fs}
    Let $\M=\langle\S,\A,p,\mu_0,H\rangle$ be an MDP without reward and let $\pi^E$ be the deterministic expert policy.
    Assume to know $\pi^E$ in all $(s,h)\in\SH$ and also to know $\suppspie$, i.e., there is no need to learn them.
    Let $\suppspib,\suppsapib$ denote the portion of space covered by a behavioral distribution $\pi^b$ in $\M$. 
    Then, $\oldfs^\cap$ and $\oldfs^\cup$ are, respectively, the \emph{tightest}
    subset and superset of $\oldfs$ that can be \emph{learned} from $\mathbb{P}_{p,\pi^b}$. Moreover,
    $\fs^\cap$ and $\fs^\cup$ are, respectively, the \emph{tightest}
    subset and superset of $\fs$ that can be \emph{learned} from $\mathbb{P}_{p,\pi^b}$.
\end{thr}
\begin{proof}
    The theorem states that the considered quantities are the tightest learnable. Thus, we split the proof in two parts.
    First, we prove that such quantities are PAC-learnable, then we show that
    there is no other object that is at the same time learnable and tighter.

    Let us begin with $\fs^\cap$ and $\fs^\cup$. By Definition \ref{def: pac-learnability},
    these quantities are PAC-learnable if we can find a semi-metric $d$ between sets of rewards and an algorithm $\mathfrak{A}$
    such that, for any arbitrarily small choice of the accuracy $\epsilon$ and confidence $\delta$, we can always find a
    \emph{finite} number of samples that algorithm $\mathfrak{A}$ can use to compute a set of rewards
    $\epsilon$-close, according to semi-metric $d$, to $\fs^\cap$ (or to $\fs^\cup$) w.h.p..
    This is exactly what, for instance, Theorem \ref{theorem: upper bound d irlo} states: Algorithm \ref{alg:irlo}
    requires a finite number of samples to compute an $\epsilon$-correct estimate of $\fs^\cap$
    (or of $\fs^\cup$) w.h.p. according to any of the semi-metrics presented in Definition \ref{def: metrics d dinf} (for which we prove in Appendix \ref{section: semimetric} that a $\rho$-relaxed triangle inequality holds).
    We are not going to show that an analogous of Theorem \ref{theorem: upper bound d irlo} holds also for
    $\oldfs^\cap$ and $\oldfs^\cup$.

    Now we show that these quantities are the tightest learnable.
    Let us start with $\oldfs^\cap$, and then we will move to $\oldfs^\cup$, $\sub$, and $\super$.

    The idea is to construct by contradiction another concept $\overline{R}_{p,\pi^E}^\cap$ (non calligraphic) of subset of $\oldfs$
    which is tighter than $\oldfs^\cap$, and then show that we can construct a problem instance in which
    the newly defined concept $\overline{R}_{p,\pi^E}^\cap$ fails at being a subset of $\oldfs$. Thus, by contradiction, let us assume that there exists
    a problem instance $\M=\langle\S,\A,p,\mu_0,H\rangle$ with expert's policy $\pi^E$
    and distribution generating samples $\mathbb{P}_{p,\pi^b}$, in which there exists
    a PAC-learnable set $\overline{R}_{p,\pi^E}^\cap$ from $\mathbb{P}_{p,\pi^b}$
    such that $\oldfs^\cap\subset \overline{R}_{p,\pi^E}^\cap\subseteq\oldfs$.
    If $\suppsapib=\SAH$, then $\oldfs^\cap=\oldfs$, so we consider the case in which $\suppsapib\subset \SAH$.
    Let $\bar{r}$ be a reward of $\overline{R}_{p,\pi^E}^\cap$ which is not present in $\oldfs^\cap$.
    By definition of $\oldfs^\cap$, we have that
    $\bar{r}\notin \oldfs^\cap$ if and only if there exists $p'\in\eqclassp{p}$
    and $\pi'\in\eqclasspi{\pi^E}$ such that $\bar{r}\notin \R_{p',\pi'}$.
    Therefore, similarly to the proof of Theorem \ref{theorem: fs not learnable}, we can construct a new
    problem instance $\M'=\langle\S,\A,p',\mu_0,H\rangle\cup\{\pi'\}$ such that, since
    $\suppsapib\subset \SAH$, $\oldfs\neq \R_{p',\pi'}$.
    By definition of $p'$, we know that $\mathbb{P}_{p',\pi^b}=\mathbb{P}_{p,\pi^b}$, thus
    any algorithm $\mathfrak{A}$ estimating the new concept of subset $\overline{R}_{p,\pi^E}^\cap$
    fails to distinguish instances $\M$ and $\M'$.
    This means that we can choose $\epsilon,\delta$ so that if $\mathfrak{A}$
    returns in $\M$ a set containing $\bar{r}$, then with high probability it will return it
    also in $\M'$. However, $\bar{r}\notin \R_{p',\pi'}$, so we get a contradiction.

    The proof for $\oldfs^\cup$ is analogous.
    By contradiction, we claim the existence of a set $\overline{R}_{p,\pi^E}^\cup$ such that
    $\oldfs\subseteq \overline{R}_{p,\pi^E}^\cup\subset \oldfs^\cup$. The contradiction will be
    shown by considering a reward $\bar{r}$ of $\fs^\cup$ which is not in $\overline{R}_{p,\pi^E}^\cup$,
    and then constructing the problem
    instance in which the feasible set contains exactly that reward, but set $\overline{R}_{p,\pi^E}^\cup$ does not (w.h.p.).

    As far as $\fs^\cap$ and $\fs^\cup$ are concerned, the proofs are analogous
    to those presented above. However, there is a detail that has to be explained.
    Specifically, we have seen in the proof of Theorem \ref{theorem: fs not learnable} that the condition
    $\suppsapib\subset \SAH$ is not a sufficient condition for $\R_{p^1,\pi^E}\neq\R_{p^2,\pi^E}$.
    Therefore, in principle, the proof for $\oldfs^\cap$ ($\oldfs^\cup$)
    cannot be adapted directly to $\fs^\cap$ ($\fs^\cup$).
    However, we observe that $\fs^\cap\subset\fs$
    (respectively, $\fs\subset\fs^\cup$) holds strictly
    if $\suppsapib \neq \SAH$ and for at least one $(s,h)\notin\suppspib$ there
    exists a policy $\pi\in\Pi$ such that $\mathbb{P}_{p,\pi}(s_h=s)>0$.
    Otherwise, it holds that $\fs^\cap=\fs=\fs^\cup$, as explained in Section \ref{section: A remark about non reachable states}.
    This is exactly the condition required in Theorem \ref{theorem: fs not learnable}
    for proving $\R_{p^1,\pi^E}\neq\R_{p^2,\pi^E}$. Therefore, by using this
    observation, we can prove the statement of the theorem also for $\fs^\cap$ and $\fs^\cup$.
\end{proof}

\section{Further considerations}\label{section: further considerations appendix}
In this appendix, we collect a variety of considerations and remarks about the learning framework introduced, about the need of two datasets, alternative representations of the feasible set, and some others.

\subsection{About the need of two datasets}\label{remark:1Dataset}
     We presented \irlo (and, subsequently, \pirlo) in the case two datasets $\mathcal{D}^b$ and $\mathcal{D}^E$ collected with $\pi^b$ and $\pi^E$, respectively, are available. This scenario is common in previous IRL works~\citep{boularias2011relative} but, although convenient for our analysis, it is not strictly necessary to achieve a meaningful sample complexity. Indeed, we remark that the expert's dataset is employed for estimating the expert's support $\suppspie$ and policy $\pi^E$. This task can be anyway achieved under Assumption~\ref{assumption: coverage of behavioral policy} using just one dataset $\D=\{\langle s^{b,i}_1,a_1^{b,i}, a_1^{E,i},\dots,s^{b,i}_{H-1},a_{H-1}^{b,i},a_{H-1}^{E,i},s^{b,i}_H\rangle\}_{i\in\dsb{\tau}}$ playing the behavioral policy $\pi^b$ and keeping track of the expert's actions too. In such a case, we must require that every transition $(s,\pi_h^E(s),s')$ is exercised at least once in $\mathcal{D}$. This leads to a sample complexity bound which is larger in the last constant term in which $\rhominpie$ is replaced with:
    \begin{equation*}
        \min \Bigg\{ \rhominpie,   \min_{\substack{(s,h) \in \mathcal{S}^{p,\pi^E}, \\ s' \in \mathcal{S} : p_h(s'|s,\pi^E_h(s))>0}}  \rho_h^{p,\pi^b}(s,a) p_h(s'|s,\pi^E_h(s))
 \Bigg\}.
    \end{equation*}

\subsection{About the dependence on $\rho_{\min}$}
    The majority of the results presented for the $d_\infty$ semimetric in this paper are characterized by a dependence on the minimum non-zero visitation probability $\rhominpib$ of the behavioral policy $\pi^b$. This is expected since we are targeting as solution concept the \emph{tightest learnable} subset and supersets of the feasible set. Clearly, one can further relax this requirement, accepting to target \emph{non-tightest learnable} sets with a benefit in the sample complexity. Consider a minimum-visitation threshold $\overline{\rho}$, we define $\mathcal{Z}^{p,\pi^b}_{\overline{\rho}} = \{(s,a,h) :\, \rho_h^{p,\pi^b}(s,a) > \overline{\rho} \}$ as the set of triples $(s,a,h)$ that are visited by at least $\overline{\rho}$ probability (notice that $\mathcal{Z}^{p,\pi^b}=\mathcal{Z}^{p,\pi^b}_{0}$).
    We can use this set to employ suitable equivalence $\equiv_{\mathcal{Z}^{p,\pi^b}_{\overline{\rho}}}$ relations over transition models to group together those that differ in triples $(s,a,h)$ visited with probability smaller than $\overline{\rho}$. This allows to redefine the sub- and super-feasible sets as follows:
    \begin{align*}
        \mathcal{R}^{\cap}_{p,\pi^E,\overline{\rho}}\coloneqq
        \bigcap\limits_{p'\in[p]_{\equiv_{\mathcal{Z}^{p,\pi^b}_{\overline{\rho}}}}}
        \R_{p',\pi^E}, \quad \mathcal{R}^{\cup}_{p,\pi^E,\overline{\rho}}\coloneqq
        \bigcup\limits_{p'\in[p]_{\equiv_{\mathcal{Z}^{p,\pi^b}_{\overline{\rho}}}}}
        \R_{p',\pi^E}.
    \end{align*}
    Obviously, by the definition of the equivalence relation, we have that $\mathcal{R}^{\cap}_{p,\pi^E,\overline{\rho}} \subseteq\mathcal{R}^{\cap}_{p,\pi^E}$ and $\mathcal{R}^{\cup}_{p,\pi^E} \subseteq \mathcal{R}^{\cup}_{p,\pi^E,\overline{\rho}}$. Under the assumption that $\overline{\rho} \le \rho_h^{p,\pi^b}(s,\pi_h^E(s))$ for every $(s,h)$, these sets are clearly \emph{learnable}, but lose the property of being the \emph{tightest} ones. The advantage of targeting these feasible sets is that we can reproduce the same proofs done for the original $\mathcal{R}^{\cap}_{p,\pi^E}$ and $\mathcal{R}^{\cup}_{p,\pi^E}$ obtaining a smaller sample complexity that scales with $\overline{\rho}$ instead of $\rho_{\min}^{p,\pi^b}$.

\subsection{Equivalent definitions of the feasible sets}
In both Definition \ref{def: old fs} and Theorem \ref{theorem: alternative representation new fs},
we represent the set of constraints defining the (old) feasible set using the Q-function of policy
$\pi^E$ or of some policy $\overline{\pi}\in\eqclasspie{\pi^E}$. However, it is possible to provide an alternative equivalent representation based on the optimal Q-function $Q^*$. It is easy to notice that the old feasible set $\oldfs$ can be rewritten as:
\begin{align}\label{eq: oldfs with Q star}
    \oldfs=\{r\in\mathfrak{R}\,|\,\forall(s,h)\in\popblue{\SH},\forall a\in\A:\;
    Q^*_h(s,\pi^E_h(s);p,r)\ge Q^*_h(s,a;p,r)\}.
\end{align}
Moreover, thanks to Lemma \ref{lemma: alternative representation new fs Q star}, the new feasible set $\fs$ can be rewritten as:
\begin{align*}
        \fs=\{r\in\mathfrak{R}\,|\,\forall(s,h)\in \popblue{\suppspie},\forall a\in\A:\;
        Q^*_h(s,\pi^E_h(s);p,r)\ge Q^*_h(s,a;p,r)\}.
\end{align*}
We prefer to work with the representations presented in the main paper because
the relaxations (see Section \ref{section: pirlo}) of those representations are ``better'' (See Appendix \ref{section: implementation appendix})
than the relaxations of the representations just introduced.

As a direct consequence of Theorem \ref{theorem: alternative representation new fs}, we have the following corollary (see Appendix \ref{section: appendix proofs section framework} for the proof).
\begin{restatable}{coroll}{relationfscoroll}\label{corollary: relation feasible sets}
    In the setting of Definition \ref{def: new fs} the feasible reward set $\fs$ satisfies:
    \begin{align*}
        \fs=\bigcup\limits_{\pi'\in \eqclasspie{\pi^E}}\overline{\R}_{p,\pi'}.
    \end{align*}
\end{restatable}
This corollary provides the explicit relationship between the old $\oldfs$ and new $\fs$ definitions
of feasible set. Clearly $\oldfs\subseteq\fs$.
By using Corollary \ref{corollary: relation feasible sets},
we can rewrite $\sub$ as $\sub=\bigcap_{p'\in\eqclassp{p}}\bigcup_{\pi'\in\eqclasspie{\pi^E}}\overline{\R}_{p',\pi'}$.
Observe that in general $\sub\neq
\bigcup_{\pi'\in\eqclasspie{\pi^E}}
\bigcap_{p'\in\eqclassp{p}}\overline{\R}_{p',\pi'}$,
because the union of the intersection is different
from the intersection of the union.
Furthermore, because of the different definitions of $\oldfs$
and $\fs$, we have that, in general,
$\sub \nsubseteq\oldfs$
i.e.,  the subset for the new notion of \textit{feasible set}
is not a subset of $\oldfs$.
Differently, with regard to the superset, it holds that
$\super\supseteq\oldfs$.

It should be remarked that Corollary \ref{corollary: relation feasible sets} and Theorem \ref{theorem: alternative representation new fs} are not in contradiction. Indeed, looking at the union over policies in Corollary \ref{corollary: relation feasible sets}, one might expect an existential quantifier inside Theorem \ref{theorem: alternative representation new fs}, but we find a universal quantifier.
By using Corollary \ref{corollary: relation feasible sets} and Eq. \ref{eq: oldfs with Q star}, we can ``transform'' the union into an existential quantifier to obtain:
\begin{align*}
        \fs=\{r\in\mathfrak{R}\,|\,\exists \overline{\pi}\in\eqclasspie{\pi^E}:\forall(s,h)\in \SH,\forall a\in\A:\;
        Q^*_h(s,\overline{\pi}_h(s);p,r)\ge Q^*_h(s,a;p,r)\},
\end{align*}
i.e., we are representing the feasible set $\fs$ as the set of all the rewards that induce an optimal policy in $\eqclasspie{\pi^E}$. From Theorem \ref{theorem: alternative representation new fs}, we have:
\begin{align*}
     \fs=\{r\in\mathfrak{R}\,|\,\forall \overline{\pi}\in\eqclasspie{\pi^E},\forall(s,h)\in \suppspie,\,\forall a\in\A:\;Q^{\overline{\pi}}_h(s,\pi^E_h(s);p,r)\ge Q^{\overline{\pi}}_h(s,a;p,r)
        \},
\end{align*}
i.e., we are representing the feasible set $\fs$ as the set of all the rewards for which playing the expert's action in $\suppspie$ is the optimal strategy irrespective of the optimal action outside $\suppspie$. To put it simple, Corollary \ref{corollary: relation feasible sets} uses the existential quantifier because it says that a certain policy in $\eqclasspie{\pi^E}$ is optimal,
while Theorem \ref{theorem: alternative representation new fs} uses the universal quantifier because it does not care about which policy in $\eqclasspie{\pi^E}$ is optimal, but only that the expert's action is played in $\suppspie$. The $\exists$ gives the optimal policy, while the $\forall$ says that one of the policies in $\eqclasspie{\pi^E}$ is optimal, without telling which.
Clearly, there is no contradiction.

\subsection{A remark about non reachable states}\label{section: A remark about non reachable states}
The strict condition $\suppsapib\subset\SAH$ alone is not a sufficient condition
to have $\sub\neq\fs\neq\super$. Indeed,
if the portion of $\SAH$ not contained into $\suppsapib$ is made only of 
$(s,h)\in\SH$ for which there is no $\pi\in\Pi$ such that
$(s,h)\in \S^{p,\pi}$ for the given $p\in\P$, then
neither the policy nor the transition model in $(s,h)$ appears in the constraints of $\sub,\fs,$ or $\super$
(when viewed using Theorem \ref{theorem: alternative representation new fs}).
In practice, the values of the rewards $r$ of
$\sub$ (and $\fs$, and $\super$) in such $(s,a,h)\notin\suppsapib$
can be chosen arbitrarily, irrespective of the reward in any other $(s',a',h')\in\SAH$,
and therefore we have $\sub=\fs=\super$.


\subsection{An annoying corner case}\label{subsec: annoying corner case}
To cope with the bitter lesson of Section \ref{section: bitter lesson}, we work with two datasets. As aforementioned, we use $\D^E$ to estimate $\suppspie$ and $\pi^E$, and we use $\D^b$ to estimate $p$. However, it might happen the following situation. Let $(s,h)\in\estsuppspie$ (where $\estsuppspie$ is the estimate of $\suppspie$ computed from $\D^E$), and let $\estsuppspie_{h+1}=\{\overline{s}\}$, i.e., dataset $\D^E$ tells us that the support of the expert's policy at $h+1$ is made of state $\overline{s}$ only. Let $a^E\coloneqq\widehat{\pi}^E_h(s)$. By using dataset $\D^b$, we might come up with the estimate of the transition model at $(s,a^E,h)$:
\begin{align*}
    \begin{cases}
        \widehat{p}_h(\overline{s}|s,a^E)>0\\
        \widehat{p}_h(s'|s,a^E)>0
    \end{cases},
\end{align*}
where $s'\in\S$ is some other state not in $\estsuppspie_{h+1}$. Clearly, this means that $s'\in\suppspie_{h+1}$; however, due to finite data, dataset $\D^E$ does not provide us with this information. This fact provides a contradiction between $\widehat{p}$ and $\estsuppspie$. To avoid issues in the implementation of \pirlo, we define the confidence set $\mathcal{C}(\widehat{p},b)$ (see Eq. \ref{eq:C2}) by allowing the support of the transition model of expert's actions to be compatible with the estimate provided by $\D^b$, i.e., we set:
\begin{align*}
    \mathcal{C}(\widehat{p},b)&\coloneqq\Bigl\{
    p'\in\P\,|\,\forall (s,a,h) \in \estsuppsapib:\, \|p'_h(\cdot|s,a) - \widehat{p}_h(\cdot|s,a) \|_1 \le b_h(s,a)\,\wedge\\
    &\qquad\forall (s,h)\in\estsuppspie,\forall s' \not\in \bigl(\estsuppspie_{h+1}\,\popblue{\cup\, \text{supp }\widehat{p}_h(\cdot|s,\widehat{\pi}^E_h(s)}\bigr):\, p'_h(s'|s,\widehat{\pi}^E_h(s)) = 0
    \Bigr\}.
\end{align*}
Observe that the union over the support of $\widehat{p}_h(\cdot|s,a^E)$ solves the potential issue created by the corner case described in this section. It should be remarked that, under good event $\mathcal{E}$ (see Appendix \ref{section: sample complexity}), it holds that $\estsuppspie=\suppspie$, and therefore $\estsuppspie_{h+1}\,\cup\, \text{supp }\widehat{p}_h(\cdot|s,\widehat{\pi}^E_h(s))=\estsuppspie_{h+1}=\suppspie_{h+1}$.

\subsection{Distances $d$ and $d_\infty$ control the distance between value functions}
We provide the proof of the proposition reported in Section \ref{section: pac framework}.
\relationdwithdg*
\begin{proof}
    The proof is similar to that of Theorem 4.1 of \cite{metelli2023towards}. For any $s,h$ and policy $\widehat{\pi}^*$ optimal in some $\widehat{r}$, we can write:
    \begin{align*}
        V_h^*(s;r)-V^{\widehat{\pi}^*}_h(s;r)&=
        V_h^*(s;r)-V^{\widehat{\pi}^*}_h(s;r)\pm V^{\widehat{\pi}^*}_h(s;\widehat{r})\\
        &= \Big(V_h^*(s;r)-V^{\widehat{\pi}^*}_h(s;\widehat{r})\Big)+\Big(V^{\widehat{\pi}^*}_h(s;\widehat{r})-V^{\widehat{\pi}^*}_h(s;r) \Big)\\
        &\markref{(1)} {\le}\Big(V_h^*(s;r)-V^{\popblue{\pi^*}}_h(s;\widehat{r})\Big)+\Big(V^{\widehat{\pi}^*}_h(s;\widehat{r})-V^{\widehat{\pi}^*}_h(s;r) \Big)\\
        &=\sum\limits_{l=h}^H\sum\limits_{(s',a')\in\SA}\mathbb{P}_{p,\pi^*}(s_l=s',a_l=a'|s_h=s)(r_l(s',a')-\widehat{r}_l(s',a'))\\
&\qquad+\sum\limits_{l=h}^H\sum\limits_{(s',a')\in\SA}\mathbb{P}_{p,\widehat{\pi}^*}(s_l=s',a_l=a'|s_h=s)(r_l(s',a')-\widehat{r}_l(s',a'))\\
&\le 2\sum\limits_{l=h}^H\|r_l-\widehat{r}_l\|_\infty,
    \end{align*}
    where at (1) we have used that $\widehat{\pi}^*$ is optimal under $\widehat{r}$.

    Multiplying both sides by $1/M(r,\widehat{r})$ concludes the proof, and noticing that $\rho^{p,\pi^b}$ permits to bound $d_\infty$ by $d$, we get the result.
\end{proof}

\section{Proofs of Section \ref{section: framework} and Section \ref{section: pac framework}}\label{section: appendix proofs section framework}
In this section, we provide the missing proofs of Section \ref{section: framework} and Section \ref{section: pac framework}.

To prove Theorem \ref{theorem: alternative representation new fs}, it is useful to introduce the following lemma.
\begin{lemma}\label{lemma: alternative representation new fs Q star}
In the setting of Definition \ref{def: new fs}, the feasible reward set $\fs$ satisfies:
    \begin{align*}
        \fs=\{&r\in\mathfrak{R}\,|\,\forall(s,h)\in \suppspie,\forall a\in\A:\;Q^*_h(s,\pi^E_h(s);p,r)\ge Q^*_h(s,a;p,r)\}.
    \end{align*}
\end{lemma}
\begin{proof}
    The statement of the theorem is equivalent to the necessary and sufficient condition:
    \begin{align*}
        J(\pi^E;\mu_0,p,r)=\max\limits_{\pi\in\Pi}J(\pi;\mu_0,p,r)
        \iff \forall (s,h)\in \suppspie:\, Q^*_h(s,\pi^E_h(s);p,r)=
        \max\limits_{a\in\A} Q^*_h(s,a;p,r).
    \end{align*}
    We split the proof in two parts. First we show the sufficiency, then the necessity.

    Let us start with the sufficiency.
    Let $r$ be any reward in $\mathfrak{R}$
    and $p$ any transition model in $\P$.
    By contradiction, suppose that there exists a policy $\pi'\in \argmax_{\pi\in\Pi}J(\pi;\mu_0,p,r)$
    for which there exists a $(s',h')$ in the union of the supports of the $h\in\dsb{H}$ distributions $\rho_h^{p,\pi'}(\cdot)$
    in which $Q^*_{h'}(s',\pi_{h'}'(s');p,r)<\max_{a'\in\A} Q^*_{h'}(s',a';p,r)$ (the notation refers to a deterministic $\pi'$
    but it can be taken stochastic by computing the expected value).
    Let $\pi^*\in \argmax_{\pi\in\Pi}V^{\pi}_h(s;p,r)\;\forall (s,h)\in \SH$ be an auxiliary optimal policy whose existence is a widely-known result in RL \citep[see][]{puterman1994markov}.
    By hypothesis, it holds that:
    \begin{align*}
        J(\pi';\mu_0,p,r)=\max\limits_{\pi\in\Pi}J(\pi;\mu_0,p,r)=
        J(\pi^*;\mu_0,p,r).
    \end{align*}
    From the performance difference lemma \cite{kakade2002approximately}, by denoting the advantage function by $A_h^\pi(s,a;p,r)\coloneqq Q_h^\pi (s,a;p,r)-V_h^\pi (s;p,r)$, we can write:
    \begin{align*}
        J(\pi';\mu_0,p,r)-J(\pi^*;\mu_0,p,r)&=\sum\limits_{h\in\dsb{H}}\E\limits_{(s,a)\sim \rho_h^{p,\pi'}(\cdot,\cdot)}\bigl[
        A_h^{\pi^*}(s,a;p,r)
        \bigr]\\
        &\markref{(1)}{=}
        \rho_{h'}^{p,\pi'}(s')
        A_{h'}^{\pi^*}(s',\pi_{h'}'(s');p,r)\\
        &\markref{(2)}{<}0,
    \end{align*}
    where at (1) we have used that in all $(s,h)\in \SH\setminus\{(s',h')\}$
    the policy $\pi'$ prescribes the action greedy w.r.t. $Q^*$, and thus the advantage is 0,
    and (2) holds by (contradiction) hypothesis.
    However, by hypothesis, we know that $J(\pi^*;\mu_0,p,r)-J(\pi';\mu_0,p,r)=0$,
    thus we have obtained a contradiction, so the sufficiency holds.

    As far as the necessity is concerned,
    let us consider again an auxiliary optimal policy $\pi^*$
    and a policy $\pi'$ such that $Q^*_{h'}(s',\pi_{h'}'(s');p,r)=\max_{a'\in\A} Q^*_{h'}(s',a';p,r)$
    in the support of the $h\in\dsb{H}$ distributions $\rho_h^{p,\pi'}(\cdot)$. By applying the performance difference lemma, we can write:
    \begin{align*}
        J(\pi';\mu_0,p,r)-J(\pi^*;\mu_0,p,r)&=\sum\limits_{h\in\dsb{H}}\E\limits_{s\sim \rho_h^{p,\pi'}(\cdot)}\bigl[
        A_h^{\pi^*}(s,a;p,r)
        \bigr]\\
        &=\sum\limits_{h\in\dsb{H}}\sum\limits_{(s,a)\in\SA:\rho_h^{p,\pi'}(s,a)>0} \rho_h^{p,\pi'}(s,a)
        \underbrace{A_h^{\pi^*}(s,a;p,r)}_{=0}\\
        &=0,
    \end{align*}
    where we have simply used the hypothesis.

    By setting $\pi^E\equiv \pi'$, we get the result.
\end{proof}
Now, we are ready to prove Theorem \ref{theorem: alternative representation new fs}:
\relationfs*
\begin{proof}
    Thanks to Lemma \ref{lemma: alternative representation new fs Q star}, to prove the statement
    of the theorem we have to show the equivalence of the constraints:
    \begin{align}
        &\forall(s,h)\in \suppspie,\forall a\in\A:\quad
        Q^*_h(s,\pi^E_h(s);p,r)\ge Q^*_h(s,a;p,r)\label{eq: constraint Q star}\\
        &\qquad\qquad\qquad\qquad\qquad\qquad\qquad\iff\notag\\
        &\forall(s,h)\in \suppspie,\forall a\in\A,
        \forall \overline{\pi}\in\eqclasspie{\pi^E}:\quad
        Q^{\overline{\pi}}_h(s,\pi^E_h(s);p,r)\ge Q^{\overline{\pi}}_h(s,a;p,r),\label{eq: constraint Q piE}
    \end{align}
    where we have exchanged the order of the quantifiers (because they all are of the same type).
    Observe that Eq. \ref{eq: constraint Q star} can be rewritten as:
    \begin{align*}
        \forall(s,h)\in \suppspie,\forall a\in\A:\quad
        Q^{\pi^*}_h(s,\pi^E_h(s);p,r)\ge Q^{\pi^*}_h(s,a;p,r),
    \end{align*}
    because of the existence of some optimal policy $\pi^*$ \citep[see][]{puterman1994markov}.
    Now, by induction over $h\in\dsb{H}$, it is easy to show that Eq. \ref{eq: constraint Q star} entails the existence of an optimal policy $\pi^*\in \eqclasspie{\pi^E}$.
    Therefore, we can rewrite the constraint as:
    \begin{align*}
        \forall(s,h)\in \suppspie,\forall a\in\A:\quad
        Q^{\popblue{\pi^E}}_h(s,\pi^E_h(s);p,r)\ge Q^{\pi^*}_h(s,a;p,r),
    \end{align*}
    since playing $\pi^*$ from $\suppspie$ brings again into $\suppspie$. By definition of $\pi^*$,
    we have:
    \begin{align*}
        Q^{\pi^*}_h(s,a;p,r)&=Q^*_h(s,a;p,r)\coloneqq\max\limits_{\pi\in\Pi}Q^{\pi}_h(s,a;p,r)\ge Q^{\overline{\pi}}_h(s,a;p,r),
    \end{align*}
    for all $\overline{\pi}\in\eqclasspie{\pi^E}$. Since $\pi^*\in \eqclasspie{\pi^E}$, then
    we have shown that the two conditions in Eq. \ref{eq: constraint Q star} and Eq. \ref{eq: constraint Q piE} are equivalent.
\end{proof}
As a direct consequence of Theorem \ref{theorem: alternative representation new fs}, we have the following
corollary.
\relationfscoroll*
\begin{proof}
    Let $(s,a,h)\in\SAH$ and let $p$ be a transition model in $\P$.
    Define:
    \begin{align*}
        \mathfrak{R}_a\coloneqq\{
            r\in\mathfrak{R}\,|\,Q_h^*(s,a;p,r)= \max\limits_{a'\in\A}Q^*_h(s,a';p,r)
        \},
    \end{align*}
    i.e., the set of rewards satisfying the constraint on the optimality of action $a$
    in a single $(s,h)$ pair. It is well known \citep[see][]{puterman1994markov} that,
    given a reward function and a transition model, there always exists an optimal policy
    whose Q-function coincides, by definition, with the optimal Q-function. In other words,
    for any $p\in\P$ and any $r\in\mathfrak{R}$, the optimal Q-function is ``well-defined''.
    Therefore, it holds that:
    \begin{align*}
        \bigcup\limits_{a\in\A}\mathfrak{R}_a=\mathfrak{R},
    \end{align*}
    because \textit{we are making the union of the rewards that induce action $a$ to be optimal
    in $(s,h)$ for any $a\in\A$.}
    To put it simple, if we add the constraint that at pair $(s,h)\in\SH$
    there \emph{exists} an optimal action, we are not actually adding a constraint.
    Notice that we can do the same with policies $\pi\in\Pi$ instead of
    actions $a\in\A$.
    Thanks to Lemma \ref{lemma: alternative representation new fs Q star} and the property
    just highlighted, we can write:
    \begin{align*}
        \fs&=\{
            r\in\mathfrak{R}\,|\,\forall (s,h)\in \suppspie:
            \,Q_h^*(s,\pi_h^E(s);p,r)= \max\limits_{a\in\A}Q^*_h(s,a;p,r)
        \}\\
        &=\{
            r\in\mathfrak{R}\,|\,\exists \pi'\in \eqclasspie{\pi^E}:\,
            \forall (s,h)\in \SH:
            \,Q_h^*(s,\pi_h'(s);p,r)= \max\limits_{a\in\A}Q^*_h(s,a;p,r)
        \}\\
        &=\bigcup\limits_{\pi'\in\eqclasspie{\pi^E}}\overline{\R}_{p,\pi'}.
    \end{align*}
    In this way, the constraints are defined only for $(s,h)\in \suppspie$.
\end{proof}
With regards to Section \ref{section: pac framework}, we provide the following proposition.
\relationmetrics*
\begin{proof}
The first inequality is straightforward. For the second, observe that:
    \begin{align*}
        d_\infty(r,\widehat{r})&\coloneqq\frac{1}{M}\sum\limits_{h\in[H]}\max\limits_{(s,a)\in\SA}|r_h(s,a)-\widehat{r}_h(s,a)|\\
        &=\frac{1}{M}\sum\limits_{h\in[H]}\max\bigl\{
            \max\limits_{(s,a)\in \suppsahpib}|r_h(s,a)-\widehat{r}_h(s,a)|,
            \max\limits_{(s,a)\notin \suppsahpib}|r_h(s,a)-\widehat{r}_h(s,a)|
        \bigr\}\\
        &=
        \frac{1}{\rhominpib}\frac{1}{M}\sum\limits_{h\in[H]}\max\bigl\{
            \max\limits_{(s,a)\in \suppsahpib}\rhominpib|r_h(s,a)-\widehat{r}_h(s,a)|,
            \rhominpib\max\limits_{(s,a)\notin \suppsahpib}|r_h(s,a)-\widehat{r}_h(s,a)|
        \bigr\}\\
        &\markref{(1)}{\le}
        \frac{1}{\rhominpib}\frac{1}{M}\sum\limits_{h\in[H]}\max\bigl\{
            \max\limits_{(s,a)\in \suppsahpib}\rho_h^{p,\pi^b}(s,a)|r_h(s,a)-\widehat{r}_h(s,a)|,
            \max\limits_{(s,a)\notin \suppsahpib}|r_h(s,a)-\widehat{r}_h(s,a)|
        \bigr\}\\
        &\markref{(2)}{\le}
        \frac{1}{\rhominpib}\frac{1}{M}\sum\limits_{h\in[H]}\max\bigl\{
            \sum\limits_{(s,a)\in \suppsahpib}\rho_h^{p,\pi^b}(s,a)|r_h(s,a)-\widehat{r}_h(s,a)|,
            \max\limits_{(s,a)\notin \suppsahpib}|r_h(s,a)-\widehat{r}_h(s,a)|
        \bigr\}\\
        &=\frac{1}{\rhominpib}\frac{1}{M}\sum\limits_{h\in[H]}\max\bigl\{
            \E\limits_{(s,a)\sim \rho_h^{p,\pi^b}(\cdot,\cdot)}|r_h(s,a)-\widehat{r}_h(s,a)|,
            \max\limits_{(s,a)\notin \suppsahpib}|r_h(s,a)-\widehat{r}_h(s,a)|
        \bigr\}\\
        &\markref{(3)}{\le}
        \frac{1}{\rhominpib}\frac{1}{M}\sum\limits_{h\in[H]}\biggl(
            \E\limits_{(s,a)\sim \rho_h^{p,\pi^b}(\cdot,\cdot)}|r_h(s,a)-\widehat{r}_h(s,a)|+
            \max\limits_{(s,a)\notin \suppsahpib}|r_h(s,a)-\widehat{r}_h(s,a)|
        \biggr)\\
        &=: \frac{1}{\rhominpib} d(r,\widehat{r}),
    \end{align*}
    where at (1) we have upper bounded 
    $\rhominpib\le \rho_h^{p,\pi^b}(s,a)$ for $(s,a)\in \suppsahpib$,
    and $\rhominpib\le 1$, and
    at (2) and (3) we have used that $\max\{a,b\}\le a+b$ for $a,b\ge 0$. 
\end{proof}

\section{Sample Complexity Analysis}\label{section: sample complexity}
In this section, we present our results on the sample complexity of \irlo (Algorithm \ref{alg:irlo} - \irlo box)
and \pirlo (Algorithm \ref{alg:irlo} - \pirlo box) with both distances $d$ and $d_\infty$. 

The section is organized into various subsections. We begin with Section \ref{subsection: concentration lemmas}, in which we present the concentration results that will be used in all the sample complexity proofs. Section \ref{subsec: irlo d} contains the proofs of sample complexity of \irlo w.r.t.
distances $d$ and $d_\infty$. Analogously, Section \ref{subsec: pirlo d} contains the proofs of sample complexity of \pirlo w.r.t. distances $d$ and $d_\infty$.
In Section \ref{subsec: pirlo r = r hat}, we present additional sample complexity results for \pirlo w.r.t. $d,d_\infty$ under additional requirements.
We conclude with Section \ref{subsec: superset no relaxation bound}, in which we present a result of sample complexity on the estimation problem of the superset only, as defined in Equation \eqref{def: sub and super no relaxation pessimism}.

\subsection{Concentration Lemmas}\label{subsection: concentration lemmas}
We define \textit{good event} $\mathcal{E}$ as the intersection of four events $\mathcal{E}_1,\mathcal{E}_2,\mathcal{E}_3,\mathcal{E}_4$. Events $\mathcal{E}_1$ and $\mathcal{E}_2$
allow to obtain exact estimates of $\suppsapib$ and $\suppspie$ w.h.p., while events $\mathcal{E}_3$ and $\mathcal{E}_4$ allow to concentrate the estimates of the transition models around their means.
\begin{lemma}[Coverage Events]\label{lemma: concentration support datasets}
    Let $\M$ be an MDP without reward and let
    $\pi^E$ be the expert's policy. Let
    $\D^b=\{\langle s_h^{b,i},a_h^{b,i} \rangle_{h\in\dsb{H}}\}_{i\in\dsb{\tau^b}}$
    and $\D^E=\{\langle s_h^{E,j},a_h^{E,j} \rangle_{h\in\dsb{H}}\}_{j\in\dsb{\tau}^E}$
    be datasets of
    $\tau^b$ and $\tau^E$ trajectories collected
    by executing policies $\pi^b$ and $\pi^E$ in $\M$.
    Denote with $N_h^b(s,a)$
    the visitation count of triple $(s,a,h)\in \suppsapib$
    computed using $\D^b$, and by $N_h^E(s,a)$ the analogous for $\D^E$.
    For any $\delta\in (0,1)$, define events $\mathcal{E}_1, \mathcal{E}_2$ as:
    \begin{align*}
        \mathcal{E}_1&\coloneqq\Biggl\{N_h^b(s,a)\ge 1, \;\forall (s,a,h)\in \suppsapib
                \quad\text{when }\tau^b\ge c_1\frac{\ln\frac{|\suppsapib|}{\delta}}
                {\ln\frac{1}{1-\rhominpib}}\Biggr\},\\
        \mathcal{E}_2&\coloneqq\Biggl\{N^E_h(s,a)\ge 1, \;\forall (s,a,h)\in \suppsapie
        \quad\text{when }\tau^E\ge c_2\frac{\ln\frac{|\suppspie|}{\delta}}
        {\ln\frac{1}{1-\rhominpie}}\Biggr\},
    \end{align*}
    where $c_1$ and $c_2$ are universal constants. Then, event $\mathcal{E}_1\cap\mathcal{E}_2$ holds with
    probability at least $1-\delta/2$.
\end{lemma}
\begin{proof}
    Let us begin with event $\mathcal{E}_1$. We observe that
$N_h^b(s,a)\sim \text{Bin}(\tau^b, \rho_h^{p,\pi^b}(s,a))$.
In an analogous manner as Lemma E.5 of \citet{metelli2023towards},
we can write:
\begin{align*}
    \mathop{\mathbb{P}}\limits_{p,\pi^b}(\mathcal{E}_1^\complement)&=
    \mathop{\mathbb{P}}\limits_{p,\pi^b}(\exists (s,a,h)\in \suppsapib:\; N_h^b(s,a)=0)\\
    &\markref{(1)}{\le}\sum\limits_{(s,a,h)\in \suppsapib}\mathop{\mathbb{P}}\limits_{p,\pi^b}(N_h^b(s,a)=0)\\
    &=\sum\limits_{(s,a,h)\in \suppsapib}(1-\rho_h^{p,\pi^b}(s,a))^{\tau^b}\\
    &\markref{(2)}{\le}\sum\limits_{(s,a,h)\in \suppsapib}(1-\rhominpib)^{\tau^b}\\
    &=|\suppsapib|(1-\rhominpib)^{\tau^b}\le \frac{\delta}{4},
\end{align*}
where at (1) we used a union bound, and at (2) the definition of
$\rhominpib\coloneqq\min_{(s,a,h)\in \suppsapib}\rho_h^{p,\pi^b}(s,a)$.
Solving w.r.t. $\tau^b$ we get:
\begin{align*}
    \tau^b\ge \frac{\ln\frac{4|\suppsapib|}{\delta}}{\ln\frac{1}{1-\rhominpib}},
\end{align*}
from which the bound for event $\mathcal{E}_1$ holds for some uninteresting constant $c_1$.

Observe that, by following a similar reasoning, we can prove the bound for event $\mathcal{E}_2$,
by recalling that, by hypothesis, the expert's policy is deterministic, so $|\suppsapie|=|\suppspie|$.
The statement of the theorem follows by an application of the union bound.
\end{proof}
Before presenting the next lemma, we introduce the symbols:
\begin{align}
    {b}_h(s,a)\coloneqq \sqrt{\frac{2{\beta}(N_h^b(s,a),\delta)}{\max\{N^b_h(s,a), 1\}}}, 
\end{align}
and 
\begin{align}\label{eq: definition beta}
    {\beta}(n,\delta)\coloneqq\ln(4{Z}^{p,\pi^b}/{\delta})+
    ({S}^{p,\pi^b}_{\max}-1)\ln(
        e(1+n/{({S}^{p,\pi^b}_{\max}-1)})
    ),
\end{align}
with {${Z}^{p,\pi^b} \coloneqq |{\mathcal{Z}}^{p,\pi^b}|$} and {${S}^{p,\pi^b}_{\max} \coloneqq \max_{h \in \dsb{H}}|{\mathcal{S}}_h^{p,\pi^b}|$}. The corresponding counterparts with the estimated quantities are given as follows:
\begin{align}\label{eq:bhsa}
    \widehat{b}_h(s,a)\coloneqq \sqrt{\frac{2\widehat{\beta}(N_h^b(s,a),\delta)}{\max\{N^b_h(s,a), 1\}}}, 
\end{align}
and 
\begin{align}
    \widehat{\beta}(n,\delta)\coloneqq\ln(4\widehat{Z}^{p,\pi^b}/{\delta})+
    (\widehat{S}^{p,\pi^b}_{\max}-1)\ln(
        e(1+n/{(\widehat{S}^{p,\pi^b}_{\max}-1)})
    ),
\end{align}
with {$\widehat{Z}^{p,\pi^b} \coloneqq |\widehat{\mathcal{Z}}^{p,\pi^b}|$} and {$\widehat{S}^{p,\pi^b}_{\max} \coloneqq \max_{h \in \dsb{H}}|\widehat{\mathcal{S}}_h^{p,\pi^b}|$}. Clearly, under the good event $\mathcal{E}_1\cap\mathcal{E}_2$, the two versions coincide.
\begin{lemma}[Concentration]\label{lemma: concentration}
    Let $\M$ be an MDP without reward and let
    $\pi^E$ be the expert's policy. Let
    $\D^b=\{\langle s_h^{b,i},a_h^{b,i} \rangle_{h\in\dsb{H}}\}_{i\in\dsb{\tau^b}}$
    and $\D^E=\{\langle s_h^{E,j},a_h^{E,j} \rangle_{h\in\dsb{H}}\}_{j\in\dsb{\tau}^E}$
    be datasets of
    $\tau^b$ and $\tau^E$ trajectories collected
    by executing policies $\pi^b$ and $\pi^E$ in $\M$.
    Denote with $\widehat{p}_h(\cdot|s,a)$ the empirical transition model of triple $(s,a,h)\in \suppsapib$ computed using $\D^b$.
    For any $\delta\in (0,1)$, define events $\mathcal{E}_3, \mathcal{E}_4$ as:
    \begin{align*}
        \mathcal{E}_3&\coloneqq\Biggl\{
            N_h^b(s,a) KL(\widehat{p}_h(\cdot|s,a)\|p_h(\cdot|s,a))
                \le \beta(N_h^b(s,a),\delta),
                \quad\forall \tau^b\in\mathbb{N},\,\forall (s,a,h)\in \suppsapib\Biggr\},\\
        \mathcal{E}_4&\coloneqq\Biggl\{
            \frac{1}{N_h^b(s,a)\vee 1}\le
                c_4 \frac{\ln\frac{|\suppsapib|}{\delta}}{\tau^b \rho_h^{p,\pi^b}(s,a)},
                \quad\forall (s,a,h)\in \suppsapib\Biggr\},
    \end{align*}
    where $c_4$ is  a universal constant.
    Then, event $\mathcal{E}_3\cap \mathcal{E}_4$ holds with probability at least $1-\delta/2$.
\end{lemma}
\begin{proof}
We show that both events $\mathcal{E}_3,\mathcal{E}_4$ hold with
probability at least $1-\frac{\delta}{4}$. The thesis follows through the application
of a union bound.

Let us begin with
event $\mathcal{E}_3$.
Similarly to the proof of Lemma 10 in \citet{kaufmann2021adaptive}, we apply Lemma
    \ref{lemma: jonsson} and a union bound to get:
    \begin{align*}
        \mathop{\mathbb{P}}\limits_{p,\pi^b}(\mathcal{E}_3^\complement)&=
        \mathop{\mathbb{P}}\limits_{p,\pi^b}\biggl(
            \exists (s,a,h)\in \suppsapib,\,\exists\tau^b\in\mathbb{N}:\;
            N_h^b(s,a) KL(\widehat{p}_h(\cdot|s,a)\|p_h(\cdot|s,a))
                > \beta(N_h^b(s,a),\delta)
        \biggr)\\
        &\le\sum\limits_{(s,a,h)\in \suppsapib}\mathop{\mathbb{P}}\limits_{p,\pi^b}\biggl(
            \exists\tau^b\in\mathbb{N}:\;
            N_h^b(s,a) KL(\widehat{p}_h(\cdot|s,a)\|p_h(\cdot|s,a))
                > \beta(N_h^b(s,a),\delta)
        \biggr)\\
        &\le\sum\limits_{(s,a,h)\in \suppsapib}\frac{\delta}{4|\suppsapib|}=\frac{\delta}{4}.
    \end{align*}
    It should be remarked that, in the definition of $\beta$ (Equation \ref{eq: definition beta}),
    we have used $|\suppspibmax|$
    instead of $\S$ because it better represents the support of the transition model
    in triples $(s,a,h)\in \suppsapib$.

    As far as event $\mathcal{E}_4$ is concerned, consider an arbitrary triple $(s,a,h)\in \suppsapib$.
    Observe that the visitation count $N_h^b(s,a)$
    is binomially distributed, i.e.,
    $N_h^b(s,a)\sim \text{Bin}(\tau, \rho_h^{p,\pi^b}(s,a))$.
    Therefore, similarly to Lemma B.1 of \cite{xie2021bridging},
    by applying Lemma \ref{lemma: binomial concentration} with confidence
    $\delta/(4|\suppsapib|)$, we can concentrate the binomial as:
    \[
        \frac{\rho_h^{p,\pi^b}(s,a)}{N_h^b(s,a)\vee 1}\le
        \frac{8\ln\frac{4|\suppsapib|}{\delta}}{\tau},
    \]
    from which we get:
    \[
    \mathop{\mathbb{P}}\limits_{p,\pi^b}\Bigl(
        \frac{1}{N_h^b(s,a)\vee 1}\le
        \frac{8\ln\frac{4|\suppsapib|}{\delta}}{\tau\rho_h^{p,\pi^b}(s,a)}
    \Bigr)\ge 1-\frac{\delta}{4|\suppsapib|}.
    \]
    We can perform a union bound over $(s,a,h)\in \suppsapib$ to get:
    \[
        \mathop{\mathbb{P}}\limits_{p,\pi^b}\Bigl(
            \exists (s,a,h)\in \suppsapib:\;
            \frac{1}{N_h^b(s,a)\vee 1}>
            \frac{8\ln\frac{4|\suppsapib|}{\delta}}{\tau\rho_h^{p,\pi^b}(s,a)}
        \Bigr)\le \frac{\delta}{4}.
        \]
    By choosing $c_4$ appropriately, we get the result.
\end{proof}
Since $\mathcal{E}\coloneqq\mathcal{E}_1\cap \mathcal{E}_2\cap\mathcal{E}_3\cap \mathcal{E}_4$, then,
by combining Lemma \ref{lemma: concentration support datasets} with Lemma \ref{lemma: concentration} through a union bound, we get that $\mathcal{E}$ holds w.p. $1-\delta$.

\subsection{Proof of Theorem \ref{theorem: upper bound d irlo} and Theorem \ref{theorem: upper bound d infty irlo}}\label{subsec: irlo d}
The next lemmas show that, for any reward in $\sub$ ($\super$),
it is possible to find a ``similar'' reward in the estimate $\estsub$ ($\estsuper$).
Notice that, under events $\mathcal{E}_1$ and $\mathcal{E}_2$ we have that, respectively,
$\estsuppsapib=\suppsapib$ and $\estsuppspie=\suppspie$.
For the sake of simplicity, we provide here the (recursive) definitions of $p^m,p^M\in\eqclassp{p}$ for any $r\in\mathfrak{R}$:
\begin{align}\label{eq: def pM and pm}
\begin{split}
    p^M&\coloneqq\begin{cases}
    p^M_h(\cdot|s,a)=p_h(\cdot|s,a),\qquad \forall (s,a,h)\in \suppsapib\\
    p^M_h(\cdot|s,a)=\mathbbm{1}\{\cdot=\argmax\limits_{s'\in\S}V^{\pi^M}_{h+1}(s';p^M,r\}, \qquad\text{otherwise}
\end{cases},\\
p^m&\coloneqq\begin{cases}
    p^m_h(\cdot|s,a)=p_h(\cdot|s,a),\qquad \forall (s,a,h)\in \suppsapib\\
    p^m_h(\cdot|s,a)=\mathbbm{1}\{\cdot=\argmin\limits_{s'\in\S}V^{\pi^m}_{h+1}(s';p^m,r\}, \qquad\text{otherwise}
\end{cases},
\end{split}  
    \end{align}
where we have used the following (recursive) policy definitions $\pi^M,\widetilde{\pi}^m\in\eqclasspie{\pi^E}$:
\begin{align}\label{eq: def piM and pim}
\begin{split}
    \pi^M&\coloneqq\begin{cases}
    \pi^M_h(s)=\pi^E_h(s)
    ,\;\text{if } (s,h)\in \suppspie\\
    \pi^M_h(\cdot|s)=\mathbbm{1}\{\cdot=\argmax\limits_{a\in\A}Q^{\pi^M}_{h}(s,a;p^M,r)
    \},\;\text{if } (s,h)\notin \suppspie\\
\end{cases},\\
\pi^m&\coloneqq\begin{cases}
    \pi^m_h(s)=\pi^E_h(s)
    ,\;\text{if } (s,h)\in \suppspie\\
    \pi^m_h(\cdot|s)=\mathbbm{1}\{\cdot=\popblue{\argmax\limits_{a\in\A}}Q^{\pi^m}_{h}(s,a;p^m,r)
    \},\;\text{if } (s,h)\notin \suppspie\\
\end{cases}.
\end{split}
\end{align}
Thanks to these definitions, we can rewrite $\sub$ and $\super$ as:
\begin{align}\label{eq: representation sub super with pm pM}
\begin{split}
        \sub &=\{
    r\in\mathfrak{R}\,|\,
    \forall (s,h)\in \suppspie,\forall a\in\A\setminus\{a^E\}:
    Q^{\pi^E}_h(s,a^E;p,r)\ge Q^{\pi^M}_h(s,a;p^M,r)
    \},\\
    \super &=\{
    r\in\mathfrak{R}\,|\,
    \forall (s,h)\in \suppspie,\forall a\in\A\setminus\{a^E\}:
    Q^{\pi^E}_h(s,a^E;p,r)\ge Q^{\pi^m}_h(s,a;p^m,r)
    \}.
\end{split}
\end{align}
We will denote by $\widehat{p}^M,\widehat{p}^m,\widehat{\pi}^M,\widehat{\pi}^m$
the transition models and policies defined as in Eq. \ref{eq: def pM and pm} and Eq. \ref{eq: def piM and pim} but for transition model $\widehat{p}$.
\begin{lemma}[Reward Choice Subset]\label{lemma: error propagation irlo subset}
    Let $\sub$ be the subset of the feasible set $\fs$ estimated through
    $\estsub$ outputted by Algorithm \ref{alg:irlo}.
    Under event $\mathcal{E}$, for any $r\in\sub$,
    the reward $\widehat{r}$ constructed as:
    \begin{align*}
        \begin{cases}
            \widehat{r}_h(s,a)=r_h(s,a)+\sum\limits_{s'\in\S}
            p_h(s'|s,a)V^{\pi^M}_{h+1}(s';p^M,r)
            -\sum\limits_{s'\in\S}\widehat{p}_h(s'|s,a)
            V^{\widehat{\pi}^M}_{h+1}(s';\widehat{p}^M,\widehat{r}),
            \quad\forall (s,a,h)\in \suppsapib\\
            \widehat{r}_h(s,a)=r_h(s,a),
            \quad\forall (s,a,h)\notin \suppsapib\\
        \end{cases},
    \end{align*}
    belongs to $\estsub$. Moreover, for any reward $\widehat{r}\in\estsub$, we can construct in the
    same manner a reward $r$ that belongs to $\sub$.
\end{lemma}
\begin{proof}
    The idea of the proof is to show that $Q^{\pi^M}_h(s,a;p^M,r)=Q^{\widehat{\pi}^M}_h(s,a;\widehat{p}^M,\widehat{r})$ for all $(s,a,h)\in\SAH$. Indeed, in this way, since $r\in\sub$, then
    it holds that:
    \begin{align*}
        \forall (s,h)\in \suppspie,\forall a\in\A\setminus\{a^E\}:
        Q^{\pi^E}_h(s,a^E;\widehat{p},\widehat{r})- Q^{\widehat{\pi}^M}_h(s,a;\widehat{p}^M,\widehat{r})=Q^{\pi^E}_h(s,a^E;p,r)-Q^{\pi^M}_h(s,a;p^M,r)\ge 0.
    \end{align*}

    Let us begin with any $(s,a,h)\in\suppsapib$. By definition of $\widehat{r}$ and by rearranging the terms, we obtain:
    \begin{align}\label{eq: equality Q subset irlo d}
        &\widehat{r}_h(s,a)+\sum\limits_{s'\in\S}\widehat{p}_h(s'|s,a)
            V^{\widehat{\pi}^M}_{h+1}(s';\widehat{p}^M,\widehat{r})=r_h(s,a)+\sum\limits_{s'\in\S}
            p_h(s'|s,a)V^{\pi^M}_{h+1}(s';p^M,r)\notag\\
        &\iff
        Q^{\widehat{\pi}^M}_h(s,a;\widehat{p}^M,\widehat{r})=Q^{\pi^M}_h(s,a;p^M,r).
    \end{align}
    In particular, observe that, by Assumption \ref{assumption: coverage of behavioral policy}, it holds $\suppsapie\subseteq\suppsapib$; moreover, by definition of $\widehat{p}^M$ and $p^M$, playing an expert's action from $\suppspie$ brings again into $\suppspie$, therefore, for $(s,a^E,h)\in\suppsapie$, this means:
    \begin{align*}
    Q^{\pi^E}_h(s,a^E;\widehat{p},\widehat{r})=Q^{\pi^E}_h(s,a^E;p,r).
    \end{align*}
    
    Now, we show by induction that, for any $(s,a,h)\notin\suppsapib$,
    it holds that:
    \begin{align*}
    Q^{\widehat{\pi}^M}_h(s,a;\widehat{p}^M,\widehat{r})=Q^{\pi^M}_h(s,a;p^M,r).
    \end{align*}
    As case base, consider stage $H$. Clearly, for any $(s,a)\notin\suppsapib_H$, we have:
    \begin{align*}
        Q^{\widehat{\pi}^M}_H(s,a;\widehat{p}^M,\widehat{r})&=\widehat{r}_H(s,a)\\
        &=r_H(s,a)\\
        &=Q^{\pi^M}_h(s,a;p^M,r),
    \end{align*}
    where we have used the definition of $\widehat{r}$.
    Make the inductive hypothesis that, at stage $h+1$, for any
    $(s,a)\notin\suppsapib_{h+1}$,
    it holds that
    $Q^{\widehat{\pi}^M}_{h+1}(s,a;\widehat{p}^M,\widehat{r})=Q^{\pi^M}_{h+1}(s,a;p^M,r)$, and consider stage $h$:
    \begin{align*}
        Q^{\widehat{\pi}^M}_h(s,a;\widehat{p}^M,\widehat{r})&=
        \widehat{r}_h(s,a)+\sum\limits_{s'\in\S}\widehat{p}^M_h(s'|s,a)
        V_{h+1}^{\widehat{\pi}^M}(s';\widehat{p}^M,\widehat{r})\\
        &\markref{(1)}{=}
        \popblue{r_h(s,a)}+\sum\limits_{s'\in\S}\widehat{p}^M_h(s'|s,a)
        V_{h+1}^{\widehat{\pi}^M}(s';\widehat{p}^M,\widehat{r})\\
        &\markref{(2)}{=}
        r_h(s,a)+\max\limits_{s'\in\S}
        V_{h+1}^{\widehat{\pi}^M}(s';\widehat{p}^M,\widehat{r})\\
        &=
        r_h(s,a)+\max\bigl\{\max\limits_{s'\in\suppspib_{h+1}}
        V_{h+1}^{\widehat{\pi}^M}(s';\widehat{p}^M,\widehat{r}),
        \max\limits_{s'\notin\suppspib_{h+1}}
        V_{h+1}^{\widehat{\pi}^M}(s';\widehat{p}^M,\widehat{r})
        \bigr\}\\
        &\markref{(3)}{=}
        r_h(s,a)+\max\bigl\{\max\limits_{s'\in\suppspib_{h+1}}
        V_{h+1}^{\popblue{\pi^M}}(s';\popblue{p^M,r}),
        \max\limits_{s'\notin\suppspib_{h+1}}
        V_{h+1}^{\widehat{\pi}^M}(s';\widehat{p}^M,\widehat{r})
        \bigr\}\\
        &\markref{(4)}{=}
        r_h(s,a)+\max\bigl\{\max\limits_{s'\in\suppspib_{h+1}}
        V_{h+1}^{\pi^M}(s';p^M,r),
        \max\limits_{s'\notin\suppspib_{h+1}}
        V_{h+1}^{\popblue{\pi^M}}(s';\popblue{p^M,r})
        \bigr\}\\
        &=
        r_h(s,a)+\max\limits_{s'\in\S}
        V_{h+1}^{\pi^M}(s';p^M,r)\\
        &=Q_h^{\pi^M}(s,a;p^M,r),
    \end{align*}
    where at (1) we use the definition of $\hat{r}$ outside $\suppsapib$,
    at (2) we use the definition of $\widehat{p}^M$,    
    at (3) we use Eq. \ref{eq: equality Q subset irlo d},
    and at (4) we use the inductive hypothesis.

    Notice that the same passages can be carried out if we exchanged $p$ and $\widehat{p}$.
    This concludes the proof.
\end{proof}
Notice that the reward function chosen in Lemma \ref{lemma: error propagation irlo subset}
can be interpreted, in an analogous manner as in the proof of Theorem 3.1 of \citet{metelli2021provably},
as the reward that provides, in transition model $\widehat{p}$, the same Q-function provided by the given reward in $p$.

We can prove an analogous result for the superset $\R_{\widehat{p},\widehat{\pi}^E}^\cup$.
\begin{lemma}[Reward Choice Superset]\label{lemma: error propagation irlo superset}
    Let $\super$ be the subset of the feasible set $\fs$ estimated through
    $\estsuper$ outputted by Algorithm \ref{alg:irlo}.
    Under event $\mathcal{E}$, for any $r\in\super$,
    the reward $\widehat{r}$ constructed as:
    \begin{align*}
        \begin{cases}
            \widehat{r}_h(s,a)=r_h(s,a)+\sum\limits_{s'\in\S}
            p_h(s'|s,a)V^{\pi^m}_{h+1}(s';p^m,r)
            -\sum\limits_{s'\in\S}\widehat{p}_h(s'|s,a)
            V^{\widehat{\pi}^m}_{h+1}(s';\widehat{p}^m,\widehat{r}),
            \quad\forall (s,a,h)\in \suppsapib\\
            \widehat{r}_h(s,a)=r_h(s,a),
            \quad\forall (s,a,h)\notin \suppsapib\\
        \end{cases},
    \end{align*}
    belongs to $\estsuper$. Moreover, for any reward $\widehat{r}\in\estsuper$, we can construct in the
    same manner a reward $r$ that belongs to $\super$.
\end{lemma}
\begin{proof}
    The idea of the proof is analogous to that of Lemma \ref{lemma: error propagation irlo subset}, and is reported here for completeness. We aim to show that $Q^{\pi^m}_h(s,a;p^m,r)=Q^{\widehat{\pi}^m}_h(s,a;\widehat{p}^m,\widehat{r})$ for all $(s,a,h)\in\SAH$. Indeed, in this way, since $r\in\super$, then
    it holds that:
    \begin{align*}
        \forall (s,h)\in \suppspie,\forall a\in\A\setminus\{a^E\}:
        Q^{\pi^E}_h(s,a^E;\widehat{p},\widehat{r})- Q^{\widehat{\pi}^m}_h(s,a;\widehat{p}^m,\widehat{r})=Q^{\pi^E}_h(s,a^E;p,r)-Q^{\pi^m}_h(s,a;p^m,r)\ge 0.
    \end{align*}

    Let us begin with any $(s,a,h)\in\suppsapib$. By definition of $\widehat{r}$ and by rearranging the terms, we obtain:
    \begin{align}\label{eq: equality Q superset irlo d}
        &\widehat{r}_h(s,a)+\sum\limits_{s'\in\S}\widehat{p}_h(s'|s,a)
            V^{\widehat{\pi}^m}_{h+1}(s';\widehat{p}^m,\widehat{r})=r_h(s,a)+\sum\limits_{s'\in\S}
            p_h(s'|s,a)V^{\pi^m}_{h+1}(s';p^m,r)\notag\\
        &\iff
        Q^{\widehat{\pi}^m}_h(s,a;\widehat{p}^m,\widehat{r})=Q^{\pi^m}_h(s,a;p^m,r).
    \end{align}
    In particular, observe that, by Assumption \ref{assumption: coverage of behavioral policy}, it holds $\suppsapie\subseteq\suppsapib$; moreover, by definition of $\widehat{p}^m$ and $p^m$, playing an expert's action from $\suppspie$ brings again into $\suppspie$, therefore, for $(s,a^E,h)\in\suppsapie$, this means:
    \begin{align*}
    Q^{\pi^E}_h(s,a^E;\widehat{p},\widehat{r})=Q^{\pi^E}_h(s,a^E;p,r).
    \end{align*}
    
    Now, we show by induction that, for any $(s,a,h)\notin\suppsapib$,
    it holds that:
    \begin{align*}
    Q^{\widehat{\pi}^m}_h(s,a;\widehat{p}^m,\widehat{r})=Q^{\pi^m}_h(s,a;p^m,r).
    \end{align*}
    As case base, consider stage $H$. Clearly, for any $(s,a)\notin\suppsapib_H$, we have:
    \begin{align*}
        Q^{\widehat{\pi}^m}_H(s,a;\widehat{p}^m,\widehat{r})&=\widehat{r}_H(s,a)\\
        &=r_H(s,a)\\
        &=Q^{\pi^m}_h(s,a;p^m,r),
    \end{align*}
    where we have used the definition of $\widehat{r}$.
    Make the inductive hypothesis that, at stage $h+1$, for any
    $(s,a)\notin\suppsapib_{h+1}$,
    it holds that
    $Q^{\widehat{\pi}^m}_{h+1}(s,a;\widehat{p}^m,\widehat{r})=Q^{\pi^m}_{h+1}(s,a;p^m,r)$, and consider stage $h$:
    \begin{align*}
        Q^{\widehat{\pi}^m}_h(s,a;\widehat{p}^m,\widehat{r})&=
        \widehat{r}_h(s,a)+\sum\limits_{s'\in\S}\widehat{p}^m_h(s'|s,a)
        V_{h+1}^{\widehat{\pi}^m}(s';\widehat{p}^m,\widehat{r})\\
        &\markref{(1)}{=}
        \popblue{r_h(s,a)}+\sum\limits_{s'\in\S}\widehat{p}^m_h(s'|s,a)
        V_{h+1}^{\widehat{\pi}^m}(s';\widehat{p}^m,\widehat{r})\\
        &\markref{(2)}{=}
        r_h(s,a)+\min\limits_{s'\in\S}
        V_{h+1}^{\widehat{\pi}^m}(s';\widehat{p}^m,\widehat{r})\\
        &=
        r_h(s,a)+\min\bigl\{\min\limits_{s'\in\suppspib_{h+1}}
        V_{h+1}^{\widehat{\pi}^m}(s';\widehat{p}^m,\widehat{r}),
        \min\limits_{s'\notin\suppspib_{h+1}}
        V_{h+1}^{\widehat{\pi}^m}(s';\widehat{p}^m,\widehat{r})
        \bigr\}\\
        &\markref{(3)}{=}
        r_h(s,a)+\min\bigl\{\min\limits_{s'\in\suppspib_{h+1}}
        V_{h+1}^{\popblue{\pi^m}}(s';\popblue{p^m,r}),
        \min\limits_{s'\notin\suppspib_{h+1}}
        V_{h+1}^{\widehat{\pi}^m}(s';\widehat{p}^m,\widehat{r})
        \bigr\}\\
        &\markref{(4)}{=}
        r_h(s,a)+\min\bigl\{\min\limits_{s'\in\suppspib_{h+1}}
        V_{h+1}^{\pi^m}(s';p^m,r),
        \min\limits_{s'\notin\suppspib_{h+1}}
        V_{h+1}^{\popblue{\pi^m}}(s';\popblue{p^m,r})
        \bigr\}\\
        &=
        r_h(s,a)+\max\limits_{s'\in\S}
        V_{h+1}^{\pi^m}(s';p^m,r)\\
        &=Q_h^{\pi^m}(s,a;p^m,r),
    \end{align*}
    where at (1) we use the definition of $\hat{r}$ outside $\suppsapib$,
    at (2) we use the definition of $\widehat{p}^m$,    
    at (3) we use Eq. \ref{eq: equality Q superset irlo d},
    and at (4) we use the inductive hypothesis.

    Notice that the same passages can be carried out if we exchanged $p$ and $\widehat{p}$.
    This concludes the proof.
\end{proof}
From the proofs of Lemma \ref{lemma: error propagation irlo subset} and Lemma
\ref{lemma: error propagation irlo superset}, we notice that proving the result for the superset is ``easier'', because we simply have to consider a single transition model; instead, for the subset, we have to consider all the transition models in the equivalence
class. Luckily, we can single out a ``worst'' transition model from this class and
provide the proof only for it. We will see in the proofs of the results with pessimism
how to cope with the trickier problem in which there exist many ``worst'' transition models, and thus the recursion cannot be applied directly.

Thanks to Lemma \ref{lemma: error propagation irlo subset} and Lemma \ref{lemma: error propagation irlo superset},
we can upper bound the distance between sets of rewards by a term that depends on the distance between the transition models. We do not have error for the policy because, under good event $\mathcal{E}$, we have that $\widehat{\pi}^E=\pi^E$ in $\estsuppspie=\suppspie$.
\begin{lemma}[Performance Decomposition Subset and Superset]\label{lemma: performance decomposition irlo}
    Let $\widehat{\R}^\cap\coloneqq\estsub$
    and $\widehat{\R}^\cup\coloneqq\estsuper$
    be the output of IRLO (Algorithm \ref{alg:irlo}).
    Under the good event $\mathcal{E}$, it holds that:
    \begin{align*}
        \mathcal{H}_d(\sub, \widehat{\R}^\cap)
        \le
        H\sum\limits_{h\in\dsb{H}}\E\limits_{(s,a)\sim\rho^{p,\pi^b}_h(\cdot,\cdot)}
        b_h(s,a),
    \end{align*}
    and:
    \begin{align*}
        \mathcal{H}_d(\super, \widehat{\R}^\cup)
        \le
        H\sum\limits_{h\in\dsb{H}}\E\limits_{(s,a)\sim\rho^{p,\pi^b}_h(\cdot,\cdot)}
        b_h(s,a).
    \end{align*}
\end{lemma}
\begin{proof}
    By Definition \ref{def: Hausdorff distance}, we can write:
    \begin{align*}
        \mathcal{H}(\sub,\widehat{\R}^\cap)&\coloneqq
        \max\{\sup\limits_{r\in \sub}
            \inf\limits_{\widehat{r}\in \widehat{\R}^\cap}d(r,\widehat{r}),
            \sup\limits_{\widehat{r}\in \widehat{\R}^\cap}
            \inf\limits_{r\in \sub}d(r,\widehat{r})\}\\
            &\markref{(1)}{\le}
            \max\{\sup\limits_{r\in \sub}
            d(r,\popblue{\widetilde{r}^1}),
            \sup\limits_{\widehat{r}\in \widehat{\R}^\cap}
            d(\popblue{\widetilde{r}^2},\widehat{r})\}\\
            &\markref{(2)}{=}
            \max\Biggl\{\sup\limits_{r\in \sub}
            \frac{1}{M}\sum\limits_{h\in\dsb{H}}\biggl(\E\limits_{(s,a)\sim\rho_h^{p,\pi^b}(\cdot,\cdot)}
            \bigl|r_h(s,a)-\widetilde{r}^1_h(s,a)\bigr|+
            \underbrace{\max\limits_{(s,a)\notin \suppsahpib}\bigl|
                    r_h(s,a)-\widetilde{r}^1_h(s,a)
                \bigr|}_{=0}\biggr),\\
            &\quad\quad\quad
            \sup\limits_{\widehat{r}\in \widehat{\R}^\cap}
            \frac{1}{M}\sum\limits_{h\in\dsb{H}}\biggl(
            \E\limits_{(s,a)\sim\rho_h^{p,\pi^b}(\cdot,\cdot)}
            \bigl|\widehat{r}_h(s,a)-\widetilde{r}^2_h(s,a)\bigr|+
            \underbrace{\max\limits_{(s,a)\notin \suppsahpib}\bigl|
                    \widehat{r}_h(s,a)-\widetilde{r}^2_h(s,a)
                \bigr|}_{=0}\biggr)
            \Biggr\}\\
            &\markref{(3)}{=}
            \max\biggl\{\sup\limits_{r\in \sub}
            \frac{1}{M}\sum\limits_{h\in\dsb{H}}\E\limits_{(s,a)\sim\rho_h^{p,\pi^b}(\cdot,\cdot)}
            \bigl|\sum\limits_{s'\in\S}
            (p_h(s'|s,a)-\widehat{p}_h(s'|s,a))
            V^{\pi^M}_{h+1}(s';p^M,r)\bigr|,\\
            &\quad\quad\quad
            \sup\limits_{\widehat{r}\in \widehat{\R}^\cap}
            \frac{1}{M}\sum\limits_{h\in\dsb{H}}\E\limits_{(s,a)\sim\rho_h^{p,\pi^b}(\cdot,\cdot)}
            \bigl|\sum\limits_{s'\in\S}
            (\widehat{p}_h(s'|s,a)-p_h(s'|s,a))
            V^{\widehat{\pi}^M}_{h+1}(s';\widehat{p}^M,\widehat{r})\bigr|
            \biggr\}\\
            &\markref{(4)}{\le}
            \max\biggl\{\sup\limits_{r\in \sub}
            \frac{1}{M}\sum\limits_{h\in\dsb{H}}\E\limits_{(s,a)\sim\rho_h^{p,\pi^b}(\cdot,\cdot)}
            \sum\limits_{s'\in\S}\popblue{\bigl|}
            (p_h(s'|s,a)-\widehat{p}_h(s'|s,a))
            V^{\pi^M}_{h+1}(s';p^M,r)\popblue{\bigr|},\\
            &\quad\quad\quad
            \sup\limits_{\widehat{r}\in \widehat{\R}^\cap}
            \frac{1}{M}\sum\limits_{h\in\dsb{H}}\E\limits_{(s,a)\sim\rho_h^{p,\pi^b}(\cdot,\cdot)}
            \sum\limits_{s'\in\S}\popblue{\bigl|}
            (\widehat{p}_h(s'|s,a)-p_h(s'|s,a))
            V^{\widehat{\pi}^M}_{h+1}(s';\widehat{p}^M,\widehat{r})\popblue{\bigr|}
            \biggr\}\\
            &\markref{(5)}{\le}
            \max\biggl\{\sup\limits_{r\in \sub}
            \sum\limits_{h\in\dsb{H}}\E\limits_{(s,a)\sim\rho_h^{p,\pi^b}(\cdot,\cdot)}
            \sum\limits_{s'\in\S}\bigl|
            (p_h(s'|s,a)-\widehat{p}_h(s'|s,a))
            \popblue{H}\bigr|,\\
            &\quad\quad\quad
            \sup\limits_{\widehat{r}\in \widehat{\R}^\cap}
            \sum\limits_{h\in\dsb{H}}\E\limits_{(s,a)\sim\rho_h^{p,\pi^b}(\cdot,\cdot)}
            \sum\limits_{s'\in\S}\bigl|
            (\widehat{p}_h(s'|s,a)-p_h(s'|s,a))
            \popblue{H}\bigr|
            \biggr\}\\
            &=
            \sum\limits_{h\in\dsb{H}}\E\limits_{(s,a)\sim\rho_h^{p,\pi^b}(\cdot,\cdot)}
            \sum\limits_{s'\in\S}\bigl|
            (p_h(s'|s,a)-\widehat{p}_h(s'|s,a))
            H\bigr|\\
            &=
            \sum\limits_{h\in\dsb{H}}\E\limits_{(s,a)\sim\rho_h^{p,\pi^b}(\cdot,\cdot)}
            H\bigl\|
            p_h(\cdot|s,a)-\widehat{p}_h(\cdot|s,a)
            \bigr\|_1\\
            &\markref{(6)}{\le}
            \sum\limits_{h\in\dsb{H}}\E\limits_{(s,a)\sim\rho_h^{p,\pi^b}(\cdot,\cdot)}
            H\sqrt{2KL\bigl(p_h(\cdot|s,a)\|\widehat{p}_h(\cdot|s,a)\bigr)}\\
            &\markref{(7)}{\le}
            \sum\limits_{h\in\dsb{H}}\E\limits_{(s,a)\sim\rho_h^{p,\pi^b}(\cdot,\cdot)}
            H\sqrt{2 \frac{\beta(N_h^b(s,a),\delta)}{N_h^b(s,a)}}\\
            &=
            H\sum\limits_{h\in\dsb{H}}\E\limits_{(s,a)\sim\rho_h^{p,\pi^b}(\cdot,\cdot)}
            b_h(s,a),
    \end{align*}
    where at (1) we have applied Lemma \ref{lemma: error propagation irlo subset},
    denoting by $\widetilde{r}^1\in\widehat{\R}^\cap$ and $\widetilde{r}^2\in\sub$ the choices of rewards,
    and used that $\inf_{x\in\X} f(x)\le f(\bar{x})$ for any $\bar{x}\in\X$;
    at (2) we have used the definition of distance $d$, 
    at (3) we have inserted the definitions of $\widetilde{r}^1$ and $\widetilde{r}^2$
    as provided by Lemma \ref{lemma: error propagation irlo subset},
    in particular using that $Q^{\pi^M}_h(s,a;p^M,r)=Q^{\widehat{\pi}^M}_h(s,a;\widehat{p}^M,\widehat{r})$,
    at (4) we have applied triangle inequality, to bring the absolute value inside
    the summation,
    at (5) we upper bound
    $ V^{\pi^M}_{h+1}(s';p^M,r)\le H\|r\|_\infty$ and
    $V^{\widehat{\pi}^M}_{h+1}(s';\widehat{p}^M,\widehat{r})\le H\|\widehat{r}\|_\infty$,
    and since $M\coloneqq1/\max\{\|r\|_\infty,\|\widehat{r}\|_\infty\}$, we obtain $H\|r\|_\infty/M\le H$
    and $H\|\widehat{r}\|_\infty/M\le H$; at (6) we have applied Pinsker's inequality,
    and at (7) we have applied the bound of $\mathcal{E}_3\supseteq\mathcal{E}$
    noticing that $N_h^b(s,a)\ge 1$ because of event $\mathcal{E}_1\supseteq\mathcal{E}$.

    A similar procedure can be carried out also for the supersets, using Lemma \ref{lemma: error propagation irlo superset}
    instead of Lemma \ref{lemma: error propagation irlo subset}.
\end{proof}
Finally, we have all the tools we need to prove Theorem \ref{theorem: upper bound d irlo}.
\upperboundirlo*
\begin{proof}
First, observe that, thanks to Lemma \ref{lemma: concentration support datasets} and Lemma \ref{lemma: concentration}, we have that
    good event $\mathcal{E}$ holds w.p. $1-\delta$
    with a number of trajectories $\tau^E$ and $\tau^b$ upper bounded as in events $\mathcal{E}_1$ and $\mathcal{E}_2$.
    Now, under good event $\mathcal{E}$,
    the idea of the proof is to compute the number of trajectories $\tau^b$ needed to
    have a distance between sets of rewards smaller than $\epsilon$.
    Then, we combine it with the number of trajectories required by event $\mathcal{E}$
    through $\max\{a,b\}\le a+b$ for $a,b\ge 0$.

    Let us begin with the subset. Thanks to Lemma \ref{lemma: performance decomposition irlo},
    we can write:
    \begin{align*}
        \mathcal{H}_d(\sub, \widehat{\R}^\cap)
        &\le
        H\sum\limits_{h\in\dsb{H}}\E\limits_{(s,a)\sim\rho^{p,\pi^b}_h(\cdot,\cdot)}
        b_h(s,a)\\
        &= 
        H\sum\limits_{h\in\dsb{H}}\sum\limits_{(s,a)\in \suppsahpib}
        \rho^{p,\pi^b}_h(s,a)b_h(s,a)\\
        &=
        H\sum\limits_{h\in\dsb{H}}\sum\limits_{(s,a)\in \suppsahpib}
        \rho^{p,\pi^b}_h(s,a)
        \sqrt{2\frac{\beta(N^b_h(s,a),\delta)}{N^b_h(s,a)}}\\
        &\markref{(1)}{\le}
        \sqrt{2}H\sum\limits_{h\in\dsb{H}}\sum\limits_{(s,a)\in \suppsahpib}
        \rho^{p,\pi^b}_h(s,a)
        \sqrt{\frac{\beta(\tau^b,\delta)}{N^b_h(s,a)}}\\
        &=
        \sqrt{2\beta(\tau^b,\delta)}H\sum\limits_{h\in\dsb{H}}\sum\limits_{(s,a)\in \suppsahpib}
        \rho^{p,\pi^b}_h(s,a)
        \sqrt{\frac{1}{N^b_h(s,a)}}\\
        &\markref{(2)}{\le}
        \sqrt{2\beta(\tau^b,\delta)}H
        \sum\limits_{h\in\dsb{H}}\sum\limits_{(s,a)\in \suppsahpib}
        \rho^{p,\pi^b}_h(s,a)\sqrt{c_4 \frac{\ln\frac{|\suppsapib|}{\delta}}
        {\tau^b \rho_h^{p,\pi^b}(s,a)}}\\
        &\markref{(3)}{=}
        c_5\sqrt{\frac{\beta(\tau^b,\delta)\ln\frac{|\suppsapib|}{\delta}}{\tau^b}}H
        \sum\limits_{h\in\dsb{H}}\sum\limits_{(s,a)\in \suppsahpib}
        \sqrt{\rho_h^{p,\pi^b}(s,a)}\\
        &\markref{(4)}{\le}
        c_5\sqrt{\frac{\beta(\tau^b,\delta)\ln\frac{|\suppsapib|}{\delta}}{\tau^b}}H
        \sum\limits_{h\in\dsb{H}}\sqrt{|\suppsahpib|}
        \sqrt{\sum_{(s,a)\in\suppsahpib}\rho_h^{p,\pi^b}(s,a)}\\
        &=
        c_5\sqrt{\frac{\beta(\tau^b,\delta)\ln\frac{|\suppsapib|}{\delta}}{\tau^b}}H
        \sum\limits_{h\in\dsb{H}}\sqrt{|\suppsahpib|}\\
        &\markref{(5)}{\le}
        c_5\sqrt{\frac{\beta(\tau^b,\delta)\ln\frac{|\suppsapib|}{\delta}}{\tau^b}}H
        \sqrt{H |\suppsapib|}\le\epsilon,
    \end{align*}
    where at (1) we have used that $\tau^b\ge N^b_h(s,a)$ for all $(s,a,h)\in \suppsapib$, and that function $\beta(\cdot,\delta)$
    is monotonically increasing; at (2) we have applied the result in Lemma \ref{lemma: concentration} for
    event $\mathcal{E}_4$, at (3) we define constant $c_5\coloneqq\sqrt{2c_4}$, and at (4) and (5)
    we have applied the Cauchy-Schwarz's inequality.

    To compute an upper bound to the number of trajectories required to have
    $\mathcal{H}_d(\sub, \widehat{\R}^\cap)\le \epsilon$,
    we compute the smallest $\tau^b$ that satisfies:
    \begin{align*}
        &c_5\sqrt{\frac{\beta(\tau^b,\delta)\ln\frac{|\suppsapib|}{\delta}}{\tau^b}}H
        \sqrt{H |\suppsapib|}\le\epsilon.
    \end{align*}
    By using the definition of $\beta(\tau^b,\delta)$
    from Eq. \ref{eq: definition beta} and rearranging the terms, this is equivalent
    to finding the smallest $\tau^b$ such that:
    \begin{align*}
        \tau^b \ge c_6\frac{H^3|\suppsapib|\ln\frac{|\suppsapib|}
        {\delta}\bigl(\ln\frac{4|\suppsapib|}{\delta}+
        (|\suppspibmax|-1)\ln\bigl(
            e(1+\tau^b/(|\suppspibmax|-1))
        \bigr)\bigr)}
        {\epsilon^2},
    \end{align*}
    where $c_6\coloneqq c_5^2$.
    If we define:
    \begin{align*}
        a&\coloneqq c_6\frac{H^3|\suppsapib|\ln\frac{|\suppsapib|}
        {\delta}\ln\frac{4|\suppsapib|}{\delta}}
        {\epsilon^2},\\
        b&\coloneqq c_6\frac{H^3|\suppsapib|\ln\frac{|\suppsapib|}
        {\delta}(|\suppspibmax|-1)
        }{\epsilon^2},\\
        c&\coloneqq \frac{e}{|\suppspibmax|-1},\\
        d&\coloneqq e,\\
    \end{align*}
    then we can rewrite the previous expression as:
    \begin{align*}
        \tau^b\ge a+b\ln (c\tau^b+d).
    \end{align*}
    To solve it, we can notice that
    $a,b,c,d> 0$ and $2bc>e$, thus we can apply Lemma \ref{lemma: lambert}\footnote{
        It should be remarked that the adoption of Lemma 15 of \cite{kaufmann2021adaptive}
        provides the same asymptotical dependence on the quantities of interest.
        However, Lemma \ref{lemma: lambert} is more concise.
    } to obtain:
    \begin{align*}
        \tau^b\ge 2a+3b\ln(2bc)+d/c.
    \end{align*}
    Replacing $a,b,c,d$ with their values, we get:
    \begin{align*}
        \tau^b&\ge 2c_6\frac{H^3|\suppsapib|\ln\frac{|\suppsapib|}
        {\delta}\ln\frac{4|\suppsapib|}{\delta}}
        {\epsilon^2}+\bigl(|\suppspibmax|-1\bigr)\\
        &\qquad+3c_6\frac{H^3|\suppsapib|\ln\frac{|\suppsapib|}
        {\delta}(|\suppspibmax|-1)
        }{\epsilon^2}\ln\biggl(
            2c_6e\frac{H^3|\suppsapib|\ln\frac{|\suppsapib|}
        {\delta}
        }{\epsilon^2}
        \biggr)\\
        &\le\widetilde{\mathcal{O}}\Biggl(
            \frac{H^3|\suppsapib|\ln\frac{1}
        {\delta}}{\epsilon^2}\biggl(
            \ln\frac{1}
        {\delta}+|\suppspibmax|
        \biggr)
        \Biggr).
    \end{align*}
    Now, observe that an upper bound to the number of trajectories needed to have
    $\mathcal{H}_d(\super, \widehat{\R}^\cup)\le \epsilon$
    can be obtained with an identical derivation, and, thus, it is of the same order.
    The statement of the theorem follows by the considerations at the beginning of the proof.
\end{proof}
We provide here the proof of Theorem \ref{theorem: upper bound d infty irlo}.
The proof is analogous to that of Theorem \ref{theorem: upper bound d irlo}.
\upperbounddinftyirlo*
\begin{proof}[Sketch of proof]
    The proof is exactly the same as that of Theorem \ref{theorem: upper bound d irlo}.
    First, observe that it is possible to prove a lemma analogous to Lemma \ref{lemma: performance decomposition irlo}, so that:
    \begin{align*}
        \mathcal{H}_\infty(\sub, \widehat{\R}^\cap)
        &\le H\sum\limits_{h\in\dsb{H}}\max\limits_{(s,a)\in\suppsahpib} b_h(s,a).
    \end{align*}
    Next, when applying Lemma \ref{lemma: binomial concentration}, we simply notice that, for all $h\in\dsb{H}$:
    \begin{align*}
        \max\limits_{(s,a)\in\suppsahpib} \sqrt{\frac{1}{\rho_h^{p,\pi^b}(s,a)}}
        \le\sqrt{\frac{1}{\rhominpib}}.
    \end{align*}
    We can do the same also for the superset. Following the derivation in the proof of Theorem \ref{theorem: upper bound d irlo}, the result can be obtained.
\end{proof}

\subsection{Proof of Theorem \ref{theorem: upper bound d pirlo} and Theorem \ref{theorem: upper bound d infty pirlo}}\label{subsec: pirlo d}
We will denote by $a^E\coloneqq \pi^E_h(s)$ for all $(s,h)\in\suppspie$, the expert's action.
Given any reward $r\in\mathfrak{R}$, it is useful to define (recursively) the transition models $\widetilde{p}^M$ and $\widetilde{p}^m$ as:
\begin{align}\label{eq: def ptildeM and ptildem}
\begin{split}
    \widetilde{p}^M&\coloneqq\begin{cases}
    \widetilde{p}^M_h(\cdot|s,a)=\argmax\limits_{\substack{p': \|p'_h(\cdot|s,a)-\widehat{p}_h(\cdot|s,a)\|_1\le b_h(s,a)\\
    \wedge \forall s'\notin \suppspie_{h+1}:\, p'_h(s'|s,a)=0}}
    \E\limits_{s'\sim p_h'(\cdot|s,a)}
    V^{\widetilde{\pi}^M}_{h+1}(s';\widetilde{p}^M,r)
    \}
    ,\;\text{if } (s,a,h)\in \suppsapib\\
    \widetilde{p}^M_h(\cdot|s,a)=\mathbbm{1}\{\cdot=\argmax\limits_{
    s'\in\S}V^{\widetilde{\pi}^M}_{h+1}(s';\widetilde{p}^M,r)
    \},\;\text{if } (s,a,h)\notin \suppsapib\\
\end{cases},\\
\widetilde{p}^m&\coloneqq\begin{cases}
    \widetilde{p}^m_h(\cdot|s,a)=\argmin\limits_{\substack{p': \|p'_h(\cdot|s,a)-\widehat{p}_h(\cdot|s,a)\|_1\le b_h(s,a)\\
    \wedge \forall s'\notin \suppspie_{h+1}:\, p'_h(s'|s,a)=0}}
    \E\limits_{s'\sim p_h'(\cdot|s,a)}
    V^{\widetilde{\pi}^m}_{h+1}(s';\widetilde{p}^m,r)
    \}
    ,\;\text{if } (s,a,h)\in \suppsapib\\
    \widetilde{p}^m_h(\cdot|s,a)=\mathbbm{1}\{\cdot=\argmin\limits_{
    s'\in\S}V^{\widetilde{\pi}^m}_{h+1}(s';\widetilde{p}^m,r)
    \},\;\text{if } (s,a,h)\notin \suppsapib\\
\end{cases},
\end{split}
\end{align}
where we have used the following (recursive) policy definitions $\widetilde{\pi}^M,\widetilde{\pi}^m\in\eqclasspie{\pi^E}$:
\begin{align}\label{eq: def pitildeM and pitildem}
\begin{split}
    \widetilde{\pi}^M&\coloneqq\begin{cases}
    \widetilde{\pi}^M_h(s)=\pi^E_h(s)
    ,\;\text{if } (s,h)\in \suppspie\\
    \widetilde{\pi}^M_h(\cdot|s)=\mathbbm{1}\{\cdot=\argmax\limits_{a\in\A}Q^{\widetilde{\pi}^M}_{h}(s,a;\widetilde{p}^M,r)
    \},\;\text{if } (s,h)\notin \suppspie\\
\end{cases},\\
\widetilde{\pi}^m&\coloneqq\begin{cases}
    \widetilde{\pi}^m_h(s)=\pi^E_h(s)
    ,\;\text{if } (s,h)\in \suppspie\\
    \widetilde{\pi}^m_h(\cdot|s)=\mathbbm{1}\{\cdot=\popblue{\argmax\limits_{a\in\A}}Q^{\widetilde{\pi}^m}_{h}(s,a;\widetilde{p}^m,r)
    \},\;\text{if } (s,h)\notin \suppspie
\end{cases}.
\end{split}
\end{align}
Thanks to these definitions, we can rewrite $\subrelax$ and $\superrelax$ as:
\begin{align}\label{eq: representation relaxations with ptildem ptildeM}
\begin{split}
        \subrelax &=\{
    r\in\mathfrak{R}\,|\,
    \forall (s,h)\in \suppspie,\forall a\in\A\setminus\{a^E\}:
    Q^{\pi^E}_h(s,a^E;\popblue{\widetilde{p}^m},r)\ge Q^{\widetilde{\pi}^M}_h(s,a;\popblue{\widetilde{p}^M},r)
    \},\\
    \superrelax &=\{
    r\in\mathfrak{R}\,|\,
    \forall (s,h)\in \suppspie,\forall a\in\A\setminus\{a^E\}:
    Q^{\pi^E}_h(s,a^E;\popblue{\widetilde{p}^M},r)\ge Q^{\widetilde{\pi}^m}_h(s,a;\popblue{\widetilde{p}^m},r)
    \}.
\end{split}
\end{align}
Both Theorem \ref{theorem: upper bound d pirlo} and
Theorem \ref{theorem: upper bound d infty pirlo} uses $d_\infty$ instead of $d$
use the same reward choice lemmas, but differ for the performance decomposition lemmas.
\begin{lemma}[Reward Choice Subset]\label{lemma: error propagation pirlo subset}
    For any $r\in \sub$,
    the reward $\widehat{r}$ constructed as:
    \begin{align*}
        \begin{cases}
            \widehat{r}_h(s,a^E)=r_h(s,a^E)+\sum\limits_{s'\in\S}
            p_h(s'|s,a^E)V^{\pi^E}_{h+1}(s';p,r)-\sum\limits_{s'\in\S}
            \widetilde{p}^m_h(s'|s,a^E)V^{\pi^E}_{h+1}(s';\widetilde{p}^m,\widehat{r}),
            \quad\forall (s,a^E,h)\in \suppsapie\\
            \widehat{r}_h(s,a)=r_h(s,a)+\sum\limits_{s'\in\S}
            p^M_h(s'|s,a)V^{\pi^M}_{h+1}(s';p^M,r)-\sum\limits_{s'\in\S}
            \widetilde{p}^M_h(s'|s,a)V^{\widetilde{\pi}^M}_{h+1}(s';\widetilde{p}^M,\widehat{r}),
            \quad\text{otherwise},
        \end{cases},
    \end{align*}
    belongs to $\subrelax$.
\end{lemma}
\begin{proof}
    Consider any $(s,a^E,h)\in \suppsapie$.
    By definition of $\widehat{r}$, by rearranging the terms, we have that:
    \begin{align}\label{eq: equality of Q for proof error prop pirlo}
        &\widehat{r}_h(s,a^E)+\sum\limits_{s'\in\S}
            \widetilde{p}^m_h(s'|s,a^E)V^{\pi^E}_{h+1}(s';\widetilde{p}^m,\widehat{r})=r_h(s,a^E)+\sum\limits_{s'\in\S}
            p_h(s'|s,a^E)V^{\pi^E}_{h+1}(s';p,r)\notag\\
        &\iff Q^{\pi^E}_h(s,a^E;\widetilde{p}^m,\widehat{r})=Q^{\pi^E}_h(s,a^E;p,r).
    \end{align}
    Now, consider any other triple $(s,a,h)\notin\suppsapie$. Similarly, by rearranging the terms,
    we obtain:
    \begin{align}
    \label{eq: equality Q for non expert actions}
        Q^{\widetilde{\pi}^M}_h(s,a;\widetilde{p}^M,\widehat{r})=Q^{\pi^M}_h(s,a;p^M,r).
    \end{align}
    By hypothesis, $r\in \sub$, therefore:
    \begin{align*}
        \forall (s,h)\in \suppspie,\forall a\in \A\setminus\{a^E\}:
        Q^{\pi^E}_h(s,a^E;p,r)\ge Q^{\pi^M}_h(s,a;p^M,r),
    \end{align*}
    from which it follows that:
    \begin{align*}
       \forall (s,h)\in \suppspie,\forall a\in \A\setminus\{a^E\}: Q^{\pi^E}_h(s,a^E;\widetilde{p}^m,\widehat{r})\ge Q^{\widetilde{\pi}^M}_h(s,a;\widetilde{p}^M,\widehat{r}).
    \end{align*}
\end{proof}
For the superset, we have an analogous result.
\begin{lemma}[Reward Choice Superset]\label{lemma: error propagation pirlo superset}
    For any $\widehat{r}\in \superrelax$,
    the reward $r$ constructed as:
    \begin{align*}
        \begin{cases}
            r_h(s,a^E)=\widehat{r}_h(s,a^E)+\sum\limits_{s'\in\S}
            \widetilde{p}^M_h(s'|s,a^E)V^{\pi^E}_{h+1}(s';\widetilde{p}^M,\widehat{r})
            -\sum\limits_{s'\in\S}
            p_h(s'|s,a^E)V^{\pi^E}_{h+1}(s';p,r),
            \quad\forall (s,a^E,h)\in \suppsapie\\
            r_h(s,a)=\widehat{r}_h(s,a)+\sum\limits_{s'\in\S}
            \widetilde{p}^m_h(s'|s,a)V^{\widetilde{\pi}^m}_{h+1}(s';\widetilde{p}^m,\widehat{r})-
            \sum\limits_{s'\in\S}
            p^m_h(s'|s,a)V^{\pi^m}_{h+1}(s';p^m,r),
            \quad\text{otherwise},
        \end{cases},
    \end{align*}
    belongs to $\super$.
\end{lemma}
\begin{proof}
    Consider any $(s,a^E,h)\in \suppsapie$.
    By definition of $r$, by rearranging the terms, we have that:
    \begin{align}\label{eq: equality of Q for proof error prop pirlo superset}
        &\widehat{r}_h(s,a^E)+\sum\limits_{s'\in\S}
            \widetilde{p}^M_h(s'|s,a^E)V^{\pi^E}_{h+1}(s';\widetilde{p}^M,\widehat{r})=r_h(s,a^E)+\sum\limits_{s'\in\S}
            p_h(s'|s,a^E)V^{\pi^E}_{h+1}(s';p,r)\notag\\
        &\iff Q^{\pi^E}_h(s,a^E;\widetilde{p}^M,\widehat{r})=Q^{\pi^E}_h(s,a^E;p,r).
    \end{align}
    Now, consider any other triple $(s,a,h)\notin\suppsapie$. Similarly, by rearranging the terms,
    we obtain:
    \begin{align}
    \label{eq: equality Q for non expert actions superset}
        Q^{\widetilde{\pi}^m}_h(s,a;\widetilde{p}^m,\widehat{r})=Q^{\pi^m}_h(s,a;p^m,r).
    \end{align}
    By hypothesis, $\widehat{r}\in \superrelax$, therefore:
    \begin{align*}
       \forall (s,h)\in \suppspie,\forall a\in \A\setminus\{a^E\}: Q^{\pi^E}_h(s,a^E;\widetilde{p}^M,\widehat{r})\ge Q^{\widetilde{\pi}^m}_h(s,a;\widetilde{p}^m,\widehat{r}),
    \end{align*}
    from which it follows that:
    \begin{align*}
    \forall (s,h)\in \suppspie,\forall a\in \A\setminus\{a^E\}:
        Q^{\pi^E}_h(s,a^E;p,r)\ge Q^{\pi^m}_h(s,a;p^m,r).
    \end{align*}
    Since $p^m\in\eqclassp{p}$ and $\pi^m$ is the worst policy in $\eqclassp{\pi^E}$ for $p^m$, then we have $r\in\super$.
\end{proof}
\subsubsection{Lemmas for Theorem \ref{theorem: upper bound d pirlo}}
\begin{lemma}[Performance Decomposition Subset]\label{lemma: performance decomposition pirlo}
    Under good event $\mathcal{E}$, it holds that:
    \begin{align*}
        \mathcal{H}_d(\sub, \subrelax)&\le
        2H\sum\limits_{h\in\dsb{H}}
        \E\limits_{(s,a)\sim \rho^{p,\pi^b}_h(\cdot,\cdot)} b_h(s,a)
        +8H^3\max\limits_{(s,a,h)\in\suppsapie}
        b_{h}(s,a).
    \end{align*}
\end{lemma}
\begin{proof}
Observe that:
\begin{align}\label{eq: bound d hausdorff subset pirlo}
    \mathcal{H}_d(\sub,\subrelax)&\coloneqq
    \max\{\sup\limits_{r\in \sub}
            \inf\limits_{\widetilde{r}\in \subrelax}d(r,\widetilde{r}),
            \sup\limits_{\widetilde{r}\in \subrelax}
            \inf\limits_{r\in \sub}d(r,\widetilde{r})\}\notag\\
    &\markref{(1)}{=}
    \sup\limits_{r\in \sub}
            \inf\limits_{\widetilde{r}\in \subrelax}d(r,\widetilde{r})\notag\\
    &\eqqcolon\sup\limits_{r\in \sub}
            \inf\limits_{\widetilde{r}\in \subrelax}
            \frac{1}{M}
        \sum\limits_{h\in\dsb{H}}\Bigl(\E\limits_{(s,a)\sim \rho^{p,\pi^b}_h(\cdot,\cdot)}\bigl|
            r_h(s,a)-\widetilde{r}_h(s,a)
        \bigr|+\max\limits_{(s,a)\notin \suppsahpib}
            \bigl|
            r_h(s,a)-\widetilde{r}_h(s,a)
        \bigr|\Bigr)\notag\\
    &\markref{(2)}{\le}
    \sup\limits_{r\in \sub}
    \frac{1}{M}
        \sum\limits_{h\in\dsb{H}}\Bigl(\E\limits_{(s,a)\sim \rho^{p,\pi^b}_h(\cdot,\cdot)}\bigl|
            r_h(s,a)-\widehat{r}_h(s,a)
        \bigr|+\max\limits_{(s,a)\notin \suppsahpib}
            \bigl|
            r_h(s,a)-\widehat{r}_h(s,a)
        \bigr|\Bigr),
\end{align}
where at (1) we have used that, under event $\mathcal{E}$, we have $\subrelax\subseteq\sub$,
and at (2) we have applied Lemma \ref{lemma: error propagation pirlo subset}, denoting with $\widehat{r}$ the reward chosen from $\subrelax$.

Now, we consider the various triples $(s,a,h)\in\SAH$ differently according to the definition
of $\widehat{r}$ in Lemma \ref{lemma: error propagation pirlo subset}.
Let us begin with any $(s,a^E,h)\in \suppsapie$.
Thanks to Eq. \ref{eq: equality of Q for proof error prop pirlo},
we know that, for any $(s,h)\in \suppspie$, it holds that
$Q^{\pi^E}_h(s,a^E;p,r)=Q^{\pi^E}_h(s,a^E;\widetilde{p}^m,\widehat{r})$.
Since any expert's action when played from $\suppspie$ brings to
$\suppspie$ (even under $\widetilde{p}^m$, by definition)
, then we can write:
\begin{align}\label{eq: norm 1 bound difference expert rewards}
    \bigl|r_h(s,a^E)-\widehat{r}_h(s,a^E)\bigr|&=
    \bigl|\sum\limits_{s'\in\S}
            p_h(s'|s,a^E)V^{\pi^E}_{h+1}(s';p,r)-\sum\limits_{s'\in\S}
            \widetilde{p}^m_h(s'|s,a^E)V^{\pi^E}_{h+1}(s';\widetilde{p}^m,\widehat{r})\bigr|\notag\\
            &=
            \bigl|\sum\limits_{\popblue{s'\in \suppspie}}
            p_h(s'|s,a^E)V^{\pi^E}_{h+1}(s';p,r)-\sum\limits_{\popblue{s'\in \suppspie}}
            \widetilde{p}^m_h(s'|s,a^E)V^{\pi^E}_{h+1}(s';\widetilde{p}^m,\widehat{r})\bigr|\notag\\
            &\markref{(1)}{=}
            \bigl|\sum\limits_{s'\in \suppspie}
            p_h(s'|s,a^E)V^{\pi^E}_{h+1}(s';p,r)-\sum\limits_{s'\in \suppspie}
            \widetilde{p}^m_h(s'|s,a^E)V^{\pi^E}_{h+1}(s';\popblue{p,r})\bigr|\notag\\
            &=
            \bigl|
            \sum\limits_{s'\in\S}
            (p_h(s'|s,a^E)-\widetilde{p}^m_h(s'|s,a^E))V^{\pi^E}_{h+1}(s';p,r)
            \bigr|\notag\\
            &\markref{(2)}{\le}
            \sum\limits_{s'\in\S}
            \bigl|(p_h(s'|s,a^E)-\widetilde{p}^m_h(s'|s,a^E))HM
            \bigr|\notag\\
            &\le
            HM\bigl\|p_h(\cdot|s,a^E)-\hat{p}_h(\cdot|s,a^E)\bigr\|_1+
            HM\bigl\|\widetilde{p}^m_h(\cdot|s,a^E)-\hat{p}_h(\cdot|s,a^E)\bigr\|_1\notag\\
            &\markref{(3)}{\le}
            2MHb_h(s,a^E),
\end{align}
where at (1) we use Eq. \ref{eq: equality of Q for proof error prop pirlo},
at (2) we use triangle inequality and we upper bound the
value function by $H$ times the maximum reward,
and at (3) we use the event in Lemma
\ref{lemma: concentration} twice.

Now, let us consider any triple $(s,a,h)\in \suppsapib\setminus\suppsapie$.
Thanks to Lemma \ref{lemma: error propagation pirlo subset}, we can write:
\begin{align*}
    \bigl|r_h(s,a)-\widehat{r}_h(s,a)\bigr|&=
            \bigl|
            \sum\limits_{s'\in\S}
            p^M_h(s'|s,a)V^{\pi^M}_{h+1}(s';p^M,r)-\sum\limits_{s'\in\S}
            \widetilde{p}^M_h(s'|s,a)V^{\widetilde{\pi}^M}_{h+1}(s';\widetilde{p}^M,\widehat{r})
            \bigr|\\
            &\markref{(1)}{=}
            \bigl|
            \sum\limits_{s'\in\S}
            \popblue{p_h}(s'|s,a)V^{\pi^M}_{h+1}(s';p^M,r)-\sum\limits_{s'\in\S}
            \widetilde{p}^M_h(s'|s,a)V^{\widetilde{\pi}^M}_{h+1}(s';\widetilde{p}^M,\widehat{r})\\
            &\qquad \pm
            \sum\limits_{s'\in\S}
            p_h(s'|s,a)V^{\widetilde{\pi}^M}_{h+1}(s';\widetilde{p}^M,\widehat{r})
            \bigr|\\
            &\markref{(2)}{\le}
            \bigl|
            \sum\limits_{s'\in\S}
            (p_h(s'|s,a)-\widetilde{p}^M_h(s'|s,a))V^{\widetilde{\pi}^M}_{h+1}(s';\widetilde{p}^M,\widehat{r})
            \bigr|\\
            &\qquad+
            \bigl|
            \sum\limits_{s'\in\S}
            p_h(s'|s,a)\bigl(
            V^{\pi^M}_{h+1}(s';p^M,r)-
            V^{\widetilde{\pi}^M}_{h+1}(s';\widetilde{p}^M,\widehat{r})
            \bigr)
            \bigr|\\
            &\markref{(3)}{\le}
            2MHb_h(s,a)+
            \bigl|
            \sum\limits_{s'\in\S}
            p_h(s'|s,a)\bigl(
            V^{\pi^M}_{h+1}(s';p^M,r)-
            V^{\widetilde{\pi}^M}_{h+1}(s';\widetilde{p}^M,\widehat{r})
            \bigr)
            \bigr|\\
            &\markref{(4)}{\le}
            2MHb_h(s,a)+
            \bigl|
            \sum\limits_{s'\in \suppspie_{h+1}}
            p_h(s'|s,a)\bigl(
            Q^{\popblue{\pi^E}}_{h+1}(s',a^E;\popblue{p},r)-
            Q^{\popblue{\pi^E}}_{h+1}(s',a^E;\widetilde{p}^M,\widehat{r})
            \bigr)\bigr|\\
            &\qquad+\bigl|\sum\limits_{s'\notin \suppspie_{h+1}}
            p_h(s'|s,a)\bigl(
            \max\limits_{a'\in\A}Q^{\pi^M}_{h+1}(s',a';p^M,r)-
            \max\limits_{a''\in\A}Q^{\widetilde{\pi}^M}_{h+1}(s',a'';\widetilde{p}^M,\widehat{r})
            \bigr)
            \bigr|\\
            &\markref{(5)}{\le}
            2MHb_h(s,a)+
            \bigl|
            \sum\limits_{s'\in \suppspie_{h+1}}
            p_h(s'|s,a)\bigl(
            Q^{\pi^E}_{h+1}(s',a^E;p,r)-
            Q^{\pi^E}_{h+1}(s',a^E;\widetilde{p}^M,\widehat{r})
            \bigr)\bigr|\\
            &\qquad+\bigl|\sum\limits_{s'\notin \suppspie_{h+1}}
            p_h(s'|s,a)\bigl(
            \underbrace{\max\limits_{a'\in\A}Q^{\pi^M}_{h+1}(s',a';p^M,r)-
            \max\limits_{a''\in\A}Q^{\popblue{\pi^M}}_{h+1}(s',a'';\popblue{p^M,r})}_{=0}
            \bigr)
            \bigr|\\ 
            &\le
            2MHb_h(s,a)+
            \sum\limits_{s'\in \suppspie_{h+1}}
            p_h(s'|s,a)
            \underbrace{\bigl|
            Q^{\pi^E}_{h+1}(s',a^E;p,r)-
            Q^{\pi^E}_{h+1}(s',a^E;\widetilde{p}^M,\widehat{r})
            \bigr|}_{\eqqcolon X_{h+1}(s')},\\ 
\end{align*}
where at (1) we have used that, since $(s,a,h)\in \suppsapib$, then $p^M_h(\cdot|s,a)=p_h(\cdot|s,a)$,
at (2) we have applied triangle inequality, at (3) we upper bound the
value function by $H$ times the maximum reward and we use the event in Lemma
\ref{lemma: concentration} twice, at (4) we use triangle inequality and the Bellman's equation and that the expert's action $a^E$ is played by both $\pi^M$ and $\widetilde{\pi}^M$ in any $(s,h)\in \suppspie$,
at (5) we apply Eq. \ref{eq: equality Q for non expert actions}.

Now, having defined terms $X_h(s)$ for all $(s,h)\in\suppspie$ as above, we recursively bound term $X_{h+1}(s')$.
\begin{align*}
            X_{h+1}(s')&\coloneqq \bigl|Q^{\pi^E}_{h+1}(s',a^E;p,r)-
            Q^{\pi^E}_{h+1}(s',a^E;\widetilde{p}^M,\widehat{r})\bigr|\\
            &\markref{(6)}{=}
            \bigl|
            r_{h+1}(s',a^E)-\widehat{r}_{h+1}(s',a^E)\bigr|\\
            &\qquad+\bigl|
            \sum\limits_{s''\in \suppspie_{h+2}}
            p_{h+1}(s''|s',a^E)V^{\pi^E}_{h+2}(s'';p,r)-
            \sum\limits_{s''\in \suppspie_{h+2}}
            \widetilde{p}^M_{h+1}(s''|s',a^E)V^{\pi^E}_{h+2}(s'';\widetilde{p}^M,\widehat{r})
            \bigr)
            \bigr|\\
            &\markref{(7)}{\le}
            2MHb_{h+1}(s',a^E)+
            \bigl|
            \sum\limits_{s''\in \suppspie_{h+2}}
            p_{h+1}(s''|s',a^E)V^{\pi^E}_{h+2}(s'';p,r)-
            \sum\limits_{s''\in \suppspie_{h+2}}
            \widetilde{p}^M_{h+1}(s''|s',a^E)V^{\pi^E}_{h+2}(s'';\widetilde{p}^M,\widehat{r})
            \bigr)\\
            &\qquad\pm \sum\limits_{s''\in \suppspie_{h+2}}
            p_{h+1}(s''|s',a^E)V^{\pi^E}_{h+2}(s'';\widetilde{p}^M,\widehat{r})
            \bigr|\\
            &\markref{(8)}{\le}
            2MHb_{h+1}(s',a^E)+
            \sum\limits_{s''\in \suppspie_{h+2}}
            \bigl|\bigl(p_{h+1}(s''|s',a^E)-\widetilde{p}^M_{h+1}(s''|s',a^E)\bigr)V^{\pi^E}_{h+2}(s'';\widetilde{p}^M,\widehat{r})
            \bigr|\\
            &\qquad+
            \sum\limits_{s''\in \suppspie_{h+2}}
            p_{h+1}(s''|s',a^E)
            \bigl|V^{\pi^E}_{h+2}(s'';p,r)-V^{\pi^E}_{h+2}(s'';\widetilde{p}^M,\widehat{r})
            \bigr|\\
            &\markref{(9)}{\le}
            4MHb_{h+1}(s',a^E)+
            \sum\limits_{s''\in \suppspie_{h+2}}
            p_{h+1}(s''|s',a^E)
            \underbrace{\bigl|Q^{\pi^E}_{h+2}(s'',a^E;p,r)-Q^{\pi^E}_{h+2}(s'',a^E;\widetilde{p}^M,\widehat{r})
            \bigr|}_{\eqqcolon X_{h+2}(s'')},\\
\end{align*}
where at (6) we use the definition of Q function, and we apply triangle inequality;
at (7) we use again triangle inequality and Eq. \ref{eq: norm 1 bound difference expert rewards},
at (8) we apply triangle inequality, and at (9) we use again the event in Lemma \ref{lemma: concentration} twice.
The recursion on the $X$ terms tells us that:
\begin{align}\label{eq: recursion X terms}
    X_h(s)\le 4MH b_h(s,a^E) + \E\limits_{s'\sim p_h(\cdot|s,a^E)}X_{h+1}(s').
\end{align}
Therefore,
we can upper bound the difference between rewards in $(s,a,h)\in \suppsapib\setminus\suppsapie$ as:
\begin{align}\label{eq: bound b difference rewards sampled triples}
\begin{split}
    \bigl|r_h(s,a)-\widehat{r}_h(s,a)\bigr|&\le
    2MHb_h(s,a)+ 4MH\sum\limits_{s'\in \suppspie_{h+1}}
    p_h(s'|s,a)\sum\limits_{h'\in\dsb{h+1,H}}\E\limits_{s''\sim
    \rho_{h'}^{p,\pi^E}(\cdot|s_{h+1}=s')}
    b_{h'}(s'',a^E).
    \end{split}
\end{align}
Now, the only missing triples to consider are those $(s,a,h)\notin \suppsapib$.
Similarly to the triples just considered, we can write:
\begin{align}\label{eq: bound b difference rewards non sampled triples}
    \bigl|r_h(s,a)-\widehat{r}_h(s,a)\bigr|&=
            \bigl|
            \sum\limits_{s'\in\S}
            p^M_h(s'|s,a)V^{\pi^M}_{h+1}(s';p^M,r)-\sum\limits_{s'\in\S}
            \widetilde{p}^M_h(s'|s,a)V^{\widetilde{\pi}^M}_{h+1}(s';\widetilde{p}^M,\widehat{r})
            \bigr|\notag\\
            &\markref{(1)}{=}
            \bigl|
            \max\limits_{s'\in\S}V^{\pi^M}_{h+1}(s';p^M,r)-\max\limits_{s''\in\S}
            V^{\widetilde{\pi}^M}_{h+1}(s'';\widetilde{p}^M,\widehat{r})\bigr|\notag\\
            &\markref{(2)}{\le}
            \max\limits_{s'\in\S}\bigl|
            V^{\pi^M}_{h+1}(s';p^M,r)-
            V^{\widetilde{\pi}^M}_{h+1}(s';\widetilde{p}^M,\widehat{r})\bigr|\notag\\
            &\markref{(3)}{=}
            \max\Bigl\{\max\limits_{s'\in \suppspie_{h+1}}\bigl|
            V^{\popblue{\pi^E}}_{h+1}(s';\popblue{p},r)-
            V^{\popblue{\pi^E}}_{h+1}(s';\widetilde{p}^M,\widehat{r})\bigr|,\notag\\
            &\qquad\max\limits_{s'\notin \suppspie_{h+1}}\bigl|
            \underbrace{V^{\pi^M}_{h+1}(s';p^M,r)-
            V^{\widetilde{\pi}^M}_{h+1}(s';\widetilde{p}^M,\widehat{r})}_{=0}\bigr|
            \Bigr\}\notag\\
            &=
            \max\limits_{s'\in \suppspie_{h+1}}\bigl|
            Q^{\pi^E}_{h+1}(s',a^E;p,r)-
            Q^{\pi^E}_{h+1}(s',a^E;\widetilde{p}^M,\widehat{r})\bigr|\notag\\
            &\eqqcolon
            \max\limits_{s'\in \suppspie_{h+1}}X_{h+1}(s')\notag\\
            &\markref{(4)}{\le}
            4MH\max\limits_{s'\in \suppspie_{h+1}}\Bigl(
            b_{h+1}(s',a^E)+\sum\limits_{h'\in\dsb{h+2,H}}
            \E\limits_{s''\sim \rho^{p,\pi^E}_{h'}(\cdot|s_{h+1}=s')}
            b_{h'}(s'',a^E)
            \Bigr),
\end{align}
where at (1) we have used that $(s,a,h)\notin \suppsapib$ and the definitions of $p^M$
and $\widetilde{p}^M$, at (2) we have used that for any pair of real-valued functions $f,g$
it holds that $|\max_{x}f(x)-\max_{x}g(x)|\le \max_{x}|f(x)-g(x)|$,
at (3) we use Eq. \ref{eq: equality Q for non expert actions} to realize that
in $(s,h)$ outside $\suppspie_{h+1}$ we have an equality of Q-functions, and thus the difference is 0,
at (4) we have unfolded the recursion on the $X$ terms by using Eq. \ref{eq: recursion X terms}.

By combining Eq. \ref{eq: bound d hausdorff subset pirlo}
with Eq. \ref{eq: norm 1 bound difference expert rewards},
Eq. \ref{eq: bound b difference rewards sampled triples},
and Eq. \ref{eq: bound b difference rewards non sampled triples},
we get:
\begin{align*}
    \mathcal{H}_d(\sub, \subrelax)&\le
    \sum\limits_{h\in\dsb{H}}
    \biggl(
    \sum\limits_{(s,a)\in\suppsapie_h}\rho^{p,\pi^b}_h(s,a)2Hb_h(s,a)\\
    &\qquad
    +\sum\limits_{(s,a)\in \suppsapib_h\setminus\suppsapie_h}\rho^{p,\pi^b}_h(s,a)
    \Bigl(
    2Hb_h(s,a)\\
    &+4H\sum\limits_{s'\in \suppspie_{h+1}}
    p_h(s'|s,a)\sum\limits_{h'\in\dsb{h+1,H}}\E\limits_{s''\sim
    \rho_{h'}^{p,\pi^E}(\cdot|s_{h+1}=s')}
    b_{h'}(s'',a^E)
    \Bigr)\\
    &\qquad
        +\max\limits_{(s,a)\notin \suppsapib_h} 4H\max\limits_{s'\in \suppspie_{h+1}}\sum\limits_{h'\in\dsb{h+1,H}}
            \E\limits_{s''\sim \rho^{p,\pi^E}_{h'}(\cdot|s_{h+1}=s')}
            b_{h'}(s'',a^E)
        \biggr)\\
&\le
    \sum\limits_{h\in\dsb{H}}
    \biggl(
    \sum\limits_{(s,a)\in\suppsapie_h}\rho^{p,\pi^b}_h(s,a)2Hb_h(s,a)\\
    &\qquad
    +\sum\limits_{(s,a)\in \suppsapib_h\setminus\suppsapie_h}\rho^{p,\pi^b}_h(s,a)
    \Bigl(
    2Hb_h(s,a)+ 4H^2\max\limits_{(s',a^E,h')\in\suppsapie}
    b_{h'}(s',a^E)
    \Bigr)\\
    &\qquad
        +\max\limits_{(s,a)\notin \suppsapib_h} 4H^2\max\limits_{(s',a^E,h')\in\suppsapie}
    b_{h'}(s',a^E)
        \biggr)\\
    &\le
    \sum\limits_{h\in\dsb{H}}
    \biggl(\sum\limits_{(s,a)\in\suppsapie_h}\rho^{p,\pi^b}_h(s,a)2Hb_h(s,a)\\
    &\qquad+\sum\limits_{(s,a)\in \suppsapib_h\setminus\suppsapie_h}\rho^{p,\pi^b}_h(s,a)
    2Hb_h(s,a)
    +8H^2\max\limits_{(s',a^E,h')\in\suppsapie}
    b_{h'}(s',a^E)
    \biggr)\\
    &\le
        2H\sum\limits_{h\in\dsb{H}}
        \E\limits_{(s,a)\sim \rho^{p,\pi^b}_h(\cdot,\cdot)} b_h(s,a)
        +8H^3\max\limits_{(s,a,h)\in\suppsapie}
        b_{h}(s,a).
\end{align*}
\end{proof}

\begin{lemma}[Performance Decomposition Superset]\label{lemma: performance decomposition pirlo superset}
    Under good event $\mathcal{E}$, it holds that:
    \begin{align*}
        \mathcal{H}_d(\super \superrelax)&\le
        2H\sum\limits_{h\in\dsb{H}}
        \E\limits_{(s,a)\sim \rho^{p,\pi^b}_h(\cdot,\cdot)} b_h(s,a)
        +8H^3\max\limits_{(s,a,h)\in\suppsapie}
        b_{h}(s,a).
    \end{align*}
\end{lemma}
\begin{proof}
Observe that:
\begin{align}\label{eq: bound d hausdorff superset pirlo}
    \mathcal{H}_d(\super,\superrelax)&\coloneqq
    \max\{\sup\limits_{r\in \super}
            \inf\limits_{\widetilde{r}\in \superrelax}d(r,\widetilde{r}),
            \sup\limits_{\widetilde{r}\in \superrelax}
            \inf\limits_{r\in \super}d(r,\widetilde{r})\}\notag\\
    &\markref{(1)}{=}
    \sup\limits_{\widetilde{r}\in \superrelax}
            \inf\limits_{r\in \super}d(r,\widetilde{r})\notag\\
    &\eqqcolon\sup\limits_{\widetilde{r}\in \superrelax}
            \inf\limits_{r\in \super}
            \frac{1}{M}
        \sum\limits_{h\in\dsb{H}}\Bigl(\E\limits_{(s,a)\sim \rho^{p,\pi^b}_h(\cdot,\cdot)}\bigl|
            r_h(s,a)-\widetilde{r}_h(s,a)
        \bigr|+\max\limits_{(s,a)\notin \suppsahpib}
            \bigl|
            r_h(s,a)-\widetilde{r}_h(s,a)
        \bigr|\Bigr)\notag\\
    &\markref{(2)}{\le}
    \sup\limits_{\widetilde{r}\in \superrelax}
    \frac{1}{M}
        \sum\limits_{h\in\dsb{H}}\Bigl(\E\limits_{(s,a)\sim \rho^{p,\pi^b}_h(\cdot,\cdot)}\bigl|
            r_h(s,a)-\widetilde{r}_h(s,a)
        \bigr|+\max\limits_{(s,a)\notin \suppsahpib}
            \bigl|
            r_h(s,a)-\widetilde{r}_h(s,a)
        \bigr|\Bigr),
\end{align}
where at (1) we have used that, under event $\mathcal{E}$, we have $\super\subseteq\superrelax$,
and at (2) we have applied Lemma \ref{lemma: error propagation pirlo superset}, denoting with $r$ the reward chosen from $\super$.

Now, we consider the various triples $(s,a,h)\in\SAH$ differently according to the definition
of $r$ in Lemma \ref{lemma: error propagation pirlo superset}.
Let us begin with any $(s,a^E,h)\in \suppsapie$.
We can write:
\begin{align}\label{eq: norm 1 bound difference expert rewards superset}
\begin{split}
    \bigl|r_h(s,a^E)-\widehat{r}_h(s,a^E)\bigr|&=
    \bigl|\sum\limits_{s'\in\S}
            p_h(s'|s,a^E)V^{\pi^E}_{h+1}(s';p,r)-\sum\limits_{s'\in\S}
            \widetilde{p}^M_h(s'|s,a^E)V^{\pi^E}_{h+1}(s';\widetilde{p}^M,\widehat{r})\\
            &=
    \bigl|\sum\limits_{\popblue{s'\in\suppspie_{h+1}}}
            p_h(s'|s,a^E)V^{\pi^E}_{h+1}(s';p,r)-\sum\limits_{\popblue{s'\in\suppspie_{h+1}}}
            \widetilde{p}^M_h(s'|s,a^E)V^{\pi^E}_{h+1}(s';\widetilde{p}^M,\widehat{r})\bigr|\\
            &=
    \bigl|\sum\limits_{s'\in\suppspie_{h+1}}
            p_h(s'|s,a^E)V^{\pi^E}_{h+1}(s';p,r)-\sum\limits_{s'\in\suppspie_{h+1}}
            \widetilde{p}^M_h(s'|s,a^E)V^{\pi^E}_{h+1}(s';\popblue{p,r})\bigr|\\
            &\le
            \bigl\|p_h(\cdot|s,a^E)-\widetilde{p}^M_h(\cdot|s,a^E)\bigr\|_1MH\\
            &\le
            \bigl\|p_h(\cdot|s,a^E)-\hat{p}_h(\cdot|s,a^E)\bigr\|_1MH
            + \bigl\|\hat{p}_h(\cdot|s,a^E)
            -\widetilde{p}^M_h(\cdot|s,a^E)\bigr\|_1MH\\
            &\le
            2MHb_h(s,a^E).
\end{split}
\end{align}
Now, let us consider any triple $(s,a,h)\in \suppsapib\setminus\suppsapie$.
Thanks to Lemma \ref{lemma: error propagation pirlo superset}, we can write:
\begin{align*}
    \bigl|r_h(s,a)-\widehat{r}_h(s,a)\bigr|&=
            \bigl|
            \sum\limits_{s'\in\S}
            p^m_h(s'|s,a)V^{\pi^m}_{h+1}(s';p^m,r)-\sum\limits_{s'\in\S}
            \widetilde{p}^m_h(s'|s,a)V^{\widetilde{\pi}^m}_{h+1}(s';\widetilde{p}^m,\widehat{r})
            \bigr|\\
            &\markref{(1)}{=}
            \bigl|
            \sum\limits_{s'\in\S}
            \popblue{p_h(s'|s,a)}V^{\pi^m}_{h+1}(s';p^m,r)-\sum\limits_{s'\in\S}
            \widetilde{p}^m_h(s'|s,a)V^{\widetilde{\pi}^m}_{h+1}(s';\widetilde{p}^m,\widehat{r})\\
            &\qquad \pm
            \sum\limits_{s'\in\S}
            p_h(s'|s,a)V^{\widetilde{\pi}^m}_{h+1}(s';\widetilde{p}^m,\widehat{r})
            \bigr|\\
            &\markref{(2)}{\le}
            \bigl|
            \sum\limits_{s'\in\S}
            (p_h(s'|s,a)-\widetilde{p}^m_h(s'|s,a))V^{\widetilde{\pi}^m}_{h+1}(s';\widetilde{p}^m,\widehat{r})
            \bigr|\\
            &\qquad+
            \bigl|
            \sum\limits_{s'\in\S}
            p_h(s'|s,a)\bigl(
            V^{\pi^m}_{h+1}(s';p^m,r)-
            V^{\widetilde{\pi}^m}_{h+1}(s';\widetilde{p}^m,\widehat{r})
            \bigr)
            \bigr|\\
            &\markref{(3)}{\le}
            2MHb_h(s,a)+
            \bigl|
            \sum\limits_{s'\in\S}
            p_h(s'|s,a)\bigl(
            V^{\pi^m}_{h+1}(s';p^m,r)-
            V^{\widetilde{\pi}^m}_{h+1}(s';\widetilde{p}^m,\widehat{r})
            \bigr)
            \bigr|\\
            &\markref{(4)}{\le}
            2MHb_h(s,a)+
            \bigl|
            \sum\limits_{s'\in \suppspie_{h+1}}
            p_h(s'|s,a)\bigl(
            Q^{\pi^E}_{h+1}(s',a^E;p,r)-
            Q^{\pi^E}_{h+1}(s',a^E;\widetilde{p}^m,\widehat{r})
            \bigr)\bigr|\\
            &\qquad+\bigl|\sum\limits_{s'\notin \suppspie_{h+1}}
            p_h(s'|s,a)\bigl(
            \max\limits_{a'\in\A}Q^{\pi^m}_{h+1}(s',a';p^m,r)-
            \max\limits_{a''\in\A}Q^{\widetilde{\pi}^m}_{h+1}(s',a'';\widetilde{p}^m,\widehat{r})
            \bigr)
            \bigr|\\
            &\markref{(5)}{\le}
            2MHb_h(s,a)+
            \bigl|
            \sum\limits_{s'\in \suppspie_{h+1}}
            p_h(s'|s,a)\bigl(
            Q^{\pi^E}_{h+1}(s',a^E;p,r)-Q^{\pi^E}_{h+1}(s',a^E;\widetilde{p}^m,\widehat{r})
            \bigr)\bigr|\\
            &\qquad+\bigl|\sum\limits_{s'\notin \suppspie_{h+1}}
            p_h(s'|s,a)\bigl(
            \underbrace{\max\limits_{a'\in\A}Q^{\pi^m}_{h+1}(s',a';p^m,r)-
            \max\limits_{a''\in\A}Q^{\widetilde{\pi}^m}_{h+1}(s',a'';\popblue{p^m,r})}_{=0}
            \bigr)
            \bigr|\\ 
            &\le
            2MHb_h(s,a)+
            \sum\limits_{s'\in \suppspie_{h+1}}
            p_h(s'|s,a)
            \underbrace{\bigl|
            Q^{\pi^E}_{h+1}(s',a^E;p,r)-Q^{\pi^E}_{h+1}(s',a^E;\widetilde{p}^m,\widehat{r})\bigr|}_{\eqqcolon Y_{h+1}(s')},\\ 
\end{align*}
where at (1) we have used that, since $(s,a,h)\in \suppsapib$, then $p^m_h(\cdot|s,a)=p_h(\cdot|s,a)$,
at (2) we have applied triangle inequality, at (3) we upper bound the
value function by $H$ times the maximum reward and we use the event in Lemma
\ref{lemma: concentration} twice, at (4) we use triangle inequality and the Bellman optimality equation and that the expert's action $a^E$ is optimal in any $(s,h)\in \suppspie$,
at (5) we apply Eq. \ref{eq: equality Q for non expert actions superset}.

Now, having defined terms $\{Y_h(s)\}_h$ for all $s\in\suppspie$ as above, we recursively bound term $Y_{h+1}(s')$. It should be remarked that
$(s',h+1)\in \suppspie$.
\begin{align*}
            Y_{h+1}(s')&\coloneqq \bigl|Q^{\pi^E}_{h+1}(s',a^E;p,r)-Q^{\pi^E}_{h+1}(s',a^E;\widetilde{p}^m,\widehat{r})\bigr|\\
            &\markref{(6)}{=}
            \bigl|
            r_{h+1}(s',a^E)-\widehat{r}_{h+1}(s',a^E)\bigr|\\
            &\qquad+\bigl|
            \sum\limits_{s''\in \suppspie_{h+2}}
            p_{h+1}(s''|s',a^E)V^{\pi^E}_{h+2}(s'';p,r)-
            \sum\limits_{s''\in \suppspie_{h+2}}
            \widetilde{p}^m_{h+1}(s''|s',a^E)V^{\pi^E}_{h+2}(s'';\widetilde{p}^m,\widehat{r})
            \bigr)
            \bigr|\\
            &\markref{(7)}{\le}
            2MHb_{h+1}(s',a^E)+
            \bigl|
            \sum\limits_{s''\in \suppspie_{h+2}}
            p_{h+1}(s''|s',a^E)V^{\pi^E}_{h+2}(s'';p,r)-
            \sum\limits_{s''\in \suppspie_{h+2}}
            \widetilde{p}^m_{h+1}(s''|s',a^E)V^{\pi^E}_{h+2}(s'';\widetilde{p}^m,\widehat{r})
            \bigr)\\
            &\qquad\pm \sum\limits_{s''\in \suppspie_{h+2}}
            p_{h+1}(s''|s',a^E)V^{\pi^E}_{h+2}(s'';\widetilde{p}^m,\widehat{r})
            \bigr|\\
            &\markref{(8)}{\le}
            2MHb_{h+1}(s',a^E)+
            \sum\limits_{s''\in \suppspie_{h+2}}
            \bigl|\bigl(p_{h+1}(s''|s',a^E)-\widetilde{p}^m_{h+1}(s''|s',a^E)\bigr)V^{\pi^E}_{h+2}(s'';\widetilde{p}^m,\widehat{r})
            \bigr|\\
            &\qquad+
            \sum\limits_{s''\in \suppspie_{h+2}}
            p_{h+1}(s''|s',a^E)
            \bigl|V^{\pi^E}_{h+2}(s'';p,r)-V^{\pi^E}_{h+2}(s'';\widetilde{p}^m,\widehat{r})
            \bigr|\\
            &\markref{(9)}{\le}
            4MHb_{h+1}(s',a^E)+
            \sum\limits_{s''\in \suppspie_{h+2}}
            p_{h+1}(s''|s',a^E)
            \underbrace{\bigl|Q^{\pi^E}_{h+2}(s'',a^E;p,r)-Q^{\pi^E}_{h+2}(s'',a^E;\widetilde{p}^m,\widehat{r})
            \bigr|}_{\eqqcolon Y_{h+2}(s'')},\\
\end{align*}
where at (6) we use the definition of Q function, and we apply triangle inequality;
at (7) we use again triangle inequality and Eq. \ref{eq: norm 1 bound difference expert rewards superset},
at (8) we apply triangle inequality, and at (9) we use again the event in Lemma \ref{lemma: concentration} twice.
The recursion on the $Y$ terms tells us that:
\begin{align}\label{eq: recursion Y terms}
    Y_h(s)\le 4MH b_h(s,a^E) + \E\limits_{s'\sim p_h(\cdot|s,a^E)}Y_{h+1}(s').
\end{align}
Therefore,
we can upper bound the difference between rewards in $(s,a,h)\in \suppsapib\setminus\suppsapie$ as:
\begin{align}\label{eq: bound b difference rewards sampled triples superset}
\begin{split}
    \bigl|r_h(s,a)-\widehat{r}_h(s,a)\bigr|&\le
    2MHb_h(s,a)+ 4MH\sum\limits_{s'\in \suppspie_{h+1}}
    p_h(s'|s,a)\sum\limits_{h'\in\dsb{h+1,H}}\E\limits_{s''\sim
    \rho_{h'}^{p,\pi^E}(\cdot|s_{h+1}=s')}
    b_{h'}(s'',a^E).
    \end{split}
\end{align}
Now, the only missing triples to consider are those $(s,a,h)\notin \suppsapib$.
Similarly to the triples just considered, we can write:
\begin{align}\label{eq: bound b difference rewards non sampled triples superset}
    \bigl|r_h(s,a)-\widehat{r}_h(s,a)\bigr|&=
            \bigl|
            \sum\limits_{s'\in\S}
            p^m_h(s'|s,a)V^{\pi^m}_{h+1}(s';p^m,r)-\sum\limits_{s'\in\S}
            \widetilde{p}^m_h(s'|s,a)V^{\widetilde{\pi}^m}_{h+1}(s';\widetilde{p}^m,\widehat{r})
            \bigr|\notag\\
            &\markref{(1)}{=}
            \bigl|
            \min\limits_{s'\in\S}V^{\pi^m}_{h+1}(s';p^m,r)-\min\limits_{s''\in\S}
            V^{\widetilde{\pi}^m}_{h+1}(s'';\widetilde{p}^m,\widehat{r})\bigr|\notag\\
            &\markref{(2)}{\le}
            \max\limits_{s'\in\S}\bigl|
            V^{\pi^m}_{h+1}(s';p^m,r)-
            V^{\widetilde{\pi}^m}_{h+1}(s';\widetilde{p}^m,\widehat{r})\bigr|\notag\\
            &\markref{(3)}{=}
            \max\Bigl\{\max\limits_{s'\in \suppspie_{h+1}}\bigl|
            V^{\popblue{\pi^E}}_{h+1}(s';\popblue{p},r)-
            V^{\popblue{\pi^E}}_{h+1}(s';\widetilde{p}^m,\widehat{r})\bigr|,\notag\\
            &\qquad\max\limits_{s'\notin \suppspie_{h+1}}\bigl|
            \underbrace{V^{\pi^m}_{h+1}(s';p^m,r)-
            V^{\widetilde{\pi}^m}_{h+1}(s';\widetilde{p}^m,\widehat{r})}_{=0}\bigr|
            \Bigr\}\notag\\
            &=
            \max\limits_{s'\in \suppspie_{h+1}}\bigl|
            Q^{\pi^E}_{h+1}(s',a^E;p,r)-
            Q^{\pi^E}_{h+1}(s',a^E;\widetilde{p}^m,\widehat{r})\bigr|\notag\\
            &\eqqcolon
            \max\limits_{s'\in \suppspie_{h+1}}Y_{h+1}(s')\notag\\
            &\markref{(4)}{\le}
            4MH\max\limits_{s'\in \suppspie_{h+1}}\Bigl(
            b_{h+1}(s',a^E)+\sum\limits_{h'\in\dsb{h+2,H}}
            \E\limits_{s''\sim \rho^{p,\pi^E}_{h'}(\cdot|s_{h+1}=s')}
            b_{h'}(s'',a^E)
            \Bigr),
\end{align}
where at (1) we have used that $(s,a,h)\notin \suppsapib$ and the definitions of $p^m$
and $\widetilde{p}^m$, at (2) we have used that for any pair of real-valued functions $f,g$
it holds that $|\min_{x}f(x)-\min_{x}g(x)|\le \max_{x}|f(x)-g(x)|$,
at (3) we use Eq. \ref{eq: equality Q for non expert actions superset} to realize that
in $(s,h)$ outside $\suppspie_{h+1}$ we have an equality of Q-functions, and thus the difference is 0,
and that in $\suppspie$ the optimal action is always the expert's action,
at (4) we have unfolded the recursion on the $Y$ terms by using Eq. \ref{eq: recursion Y terms}.

By combining Eq. \ref{eq: bound d hausdorff superset pirlo}
with Eq. \ref{eq: norm 1 bound difference expert rewards superset},
Eq. \ref{eq: bound b difference rewards sampled triples superset},
and Eq. \ref{eq: bound b difference rewards non sampled triples superset},
we get:
\begin{align*}
    \mathcal{H}_d(\super, \superrelax)&\le
    \sum\limits_{h\in\dsb{H}}
    \biggl(
    \sum\limits_{(s,a)\in\suppsapie_h}\rho^{p,\pi^b}_h(s,a)2Hb_h(s,a)\\
    &\qquad
    +\sum\limits_{(s,a)\in \suppsapib_h\setminus\suppsapie_h}\rho^{p,\pi^b}_h(s,a)
    \Bigl(
    2Hb_h(s,a)\\
    &+4H\sum\limits_{s'\in \suppspie_{h+1}}
    p_h(s'|s,a)\sum\limits_{h'\in\dsb{h+1,H}}\E\limits_{s''\sim
    \rho_{h'}^{p,\pi^E}(\cdot|s_{h+1}=s')}
    b_{h'}(s'',a^E)
    \Bigr)\\
    &\qquad
        +\max\limits_{(s,a)\notin \suppsapib_h} 4H\max\limits_{s'\in \suppspie_{h+1}}\sum\limits_{h'\in\dsb{h+1,H}}
            \E\limits_{s''\sim \rho^{p,\pi^E}_{h'}(\cdot|s_{h+1}=s')}
            b_{h'}(s'',a^E)
        \biggr)\\
&\le
    \sum\limits_{h\in\dsb{H}}
    \biggl(
    \sum\limits_{(s,a)\in\suppsapie_h}\rho^{p,\pi^b}_h(s,a)2Hb_h(s,a)\\
    &\qquad
    +\sum\limits_{(s,a)\in \suppsapib_h\setminus\suppsapie_h}\rho^{p,\pi^b}_h(s,a)
    \Bigl(
    2Hb_h(s,a)+ 4H^2\max\limits_{(s',a^E,h')\in\suppsapie}
    b_{h'}(s',a^E)
    \Bigr)\\
    &\qquad
        +\max\limits_{(s,a)\notin \suppsapib_h} 4H^2\max\limits_{(s',a^E,h')\in\suppsapie}
    b_{h'}(s',a^E)
        \biggr)\\
    &\le
    \sum\limits_{h\in\dsb{H}}
    \biggl(\sum\limits_{(s,a)\in\suppsapie_h}\rho^{p,\pi^b}_h(s,a)2Hb_h(s,a)\\
    &\qquad+\sum\limits_{(s,a)\in \suppsapib_h\setminus\suppsapie_h}\rho^{p,\pi^b}_h(s,a)
    2Hb_h(s,a)
    +8H^2\max\limits_{(s',a^E,h')\in\suppsapie}
    b_{h'}(s',a^E)
    \biggr)\\
    &\le
        2H\sum\limits_{h\in\dsb{H}}
        \E\limits_{(s,a)\sim \rho^{p,\pi^b}_h(\cdot,\cdot)} b_h(s,a)
        +8H^3\max\limits_{(s,a,h)\in\suppsapie}
        b_{h}(s,a).
\end{align*}
\end{proof}

\subsubsection{Lemmas for Theorem \ref{theorem: upper bound d infty pirlo}}
\begin{lemma}[Performance Decomposition Subset]\label{lemma: performance decomposition pirlo d infty}
    Under good event $\mathcal{E}$, it holds that:
    \begin{align*}
        \mathcal{H}_\infty(\sub, \subrelax)&\le
        2H^2
        \max\limits_{(s,a,h)\in \suppsapib} b_h(s,a)
        +4H^3\max\limits_{(s,a,h)\in\suppsapie}
        b_{h}(s,a).
    \end{align*}
\end{lemma}
\begin{proof}
Observe that:
\begin{align}\label{eq: bound max norm hausdorff subset pirlo}
    \mathcal{H}_\infty(\sub,\subrelax)&\coloneqq
    \max\{\sup\limits_{r\in \sub}
            \inf\limits_{\widetilde{r}\in \subrelax}d_\infty(r,\widetilde{r}),
            \sup\limits_{\widetilde{r}\in \subrelax}
            \inf\limits_{r\in \sub}d_\infty(r,\widetilde{r})\}\notag\\
    &\markref{(1)}{=}
    \sup\limits_{r\in \sub}
            \inf\limits_{\widetilde{r}\in \subrelax}d_\infty(r,\widetilde{r})\notag\\
    &\eqqcolon\sup\limits_{r\in \sub}
            \inf\limits_{\widetilde{r}\in \subrelax}
            \frac{1}{M}
        \sum\limits_{h\in\dsb{H}}\max\limits_{(s,a)\in \SA}
            \bigl|
            r_h(s,a)-\widetilde{r}_h(s,a)
        \bigr|\notag\\
    &\markref{(2)}{\le}
    \sup\limits_{r\in \sub}
    \frac{1}{M}
        \sum\limits_{h\in\dsb{H}}\max\limits_{(s,a)\in \SA}
            \bigl|
            r_h(s,a)-\widehat{r}_h(s,a)
        \bigr|,
\end{align}
where at (1) we have used that, under event $\mathcal{E}$, we have $\subrelax\subseteq\sub$,
and at (2) we have applied Lemma \ref{lemma: error propagation pirlo subset}, denoting with $\widehat{r}$ the reward chosen from $\subrelax$.

By combining Eq. \ref{eq: bound max norm hausdorff subset pirlo}
with Eq. \ref{eq: norm 1 bound difference expert rewards},
Eq. \ref{eq: bound b difference rewards sampled triples},
and Eq. \ref{eq: bound b difference rewards non sampled triples},
we get:
\begin{align*}
    \mathcal{H}_\infty(\sub, \subrelax)&\le
    \sum\limits_{h\in\dsb{H}}
        \max\biggl\{
        \max\limits_{(s,a)\in \suppsapie_h} 2 Hb_h(s,a),\\
        &\qquad\max\limits_{(s,a)\in \suppsapib_h\setminus\suppsapie_h}
        2Hb_h(s,a)+ 4H\sum\limits_{s'\in \suppspie_{h+1}}
    p_h(s'|s,a)\sum\limits_{h'\in\dsb{h+1,H}}\E\limits_{s''\sim
    \rho_{h'}^{p,\pi^E}(\cdot|s_{h+1}=s')}
    b_{h'}(s'',a^E),\\
    &\qquad
        \max\limits_{(s,a)\notin \suppsapib_h} 4H\max\limits_{s'\in \suppspie_{h+1}}\sum\limits_{h'\in\dsb{h+1,H}}
            \E\limits_{s''\sim \rho^{p,\pi^E}_{h'}(\cdot|s_{h+1}=s')}
            b_{h'}(s'',a^E)
        \biggr\}\\
        &\le
    \sum\limits_{h\in\dsb{H}}
        \max\biggl\{
        \max\limits_{(s,a)\in \suppsapie_h} 2H b_h(s,a),\\
        &\qquad\max\limits_{(s,a)\in \suppsapib_h\setminus\suppsapie_h}
        2Hb_h(s,a)+ 4H^2\max\limits_{(s',a^E,h')\in\suppsapie}
    b_{h'}(s',a^E),\\
    &\qquad
        \max\limits_{(s,a)\notin \suppsapib_h} 4H^2\max\limits_{(s',a^E,h')\in\suppsapie}
    b_{h'}(s',a^E)
        \biggr\}\\
                &\le
    \sum\limits_{h\in\dsb{H}}
        \max\biggl\{
        \max\limits_{(s,a)\in \suppsapie_h} 2H b_h(s,a),\max\limits_{(s,a)\in \suppsapib_h\setminus\suppsapie_h}
        2Hb_h(s,a)+ 4H^2\max\limits_{(s',a^E,h')\in\suppsapie}
    b_{h'}(s',a^E)
        \biggr\}\\
    &\le
    \sum\limits_{h\in\dsb{H}}
        \max\limits_{(s,a)\in \suppsapib_h} 2Hb_h(s,a)
        +\sum\limits_{h\in\dsb{H}}4H^2\max\limits_{(s',a^E,h')\in\suppsapie}
        b_{h'}(s',a^E)
        \biggr\}\\
        &\le
        2H^2
        \max\limits_{(s,a,h)\in \suppsapib} b_h(s,a)
        +4H^3\max\limits_{(s,a,h)\in\suppsapie}
        b_{h}(s,a).
\end{align*}
\end{proof}

\begin{lemma}[Performance Decomposition Superset]\label{lemma: performance decomposition pirlo superset d infty}
    Under good event $\mathcal{E}$, it holds that:
    \begin{align*}
        \mathcal{H}_\infty(\super \superrelax)&\le
        2H^2
        \max\limits_{(s,a,h)\in \suppsapib} b_h(s,a)
        +4H^3\max\limits_{(s,a,h)\in\suppsapie}
        b_{h}(s,a).
    \end{align*}
\end{lemma}
\begin{proof}
Observe that:
\begin{align}\label{eq: bound max norm hausdorff superset pirlo}
    \mathcal{H}_\infty(\super,\superrelax)&\coloneqq
    \max\{\sup\limits_{r\in \super}
            \inf\limits_{\widetilde{r}\in \superrelax}d_\infty(r,\widetilde{r}),
            \sup\limits_{\widetilde{r}\in \superrelax}
            \inf\limits_{r\in \super}d_\infty(r,\widetilde{r})\}\notag\\
    &\markref{(1)}{=}
    \sup\limits_{\widetilde{r}\in \superrelax}
            \inf\limits_{r\in \super}d_\infty(r,\widetilde{r})\notag\\
    &\eqqcolon\sup\limits_{\widetilde{r}\in \superrelax}
            \inf\limits_{r\in \super}
            \frac{1}{M}
        \sum\limits_{h\in\dsb{H}}\max\limits_{(s,a)\in \SA}
            \bigl|
            r_h(s,a)-\widetilde{r}_h(s,a)
        \bigr|\notag\\
    &\markref{(2)}{\le}
    \sup\limits_{\widetilde{r}\in \superrelax}
    \frac{1}{M}
        \sum\limits_{h\in\dsb{H}}\max\limits_{(s,a)\in \SA}
            \bigl|
            r_h(s,a)-\widetilde{r}_h(s,a)
        \bigr|,
\end{align}
where at (1) we have used that, under event $\mathcal{E}$, we have $\super\subseteq\superrelax$,
and at (2) we have applied Lemma \ref{lemma: error propagation pirlo superset}, denoting with $r$ the reward chosen from $\super$.

By combining Eq. \ref{eq: bound max norm hausdorff superset pirlo}
with Eq. \ref{eq: norm 1 bound difference expert rewards superset},
Eq. \ref{eq: bound b difference rewards sampled triples superset},
and Eq. \ref{eq: bound b difference rewards non sampled triples superset},
we get:
\begin{align*}
    \mathcal{H}_\infty(\super, \superrelax)&\le
    \sum\limits_{h\in\dsb{H}}
        \max\biggl\{
        \max\limits_{(s,a)\in \suppsapie_h} 2H b_h(s,a),\\
        &\qquad\max\limits_{(s,a)\in \suppsapib_h\setminus\suppsapie_h}
        2Hb_h(s,a)+ 4H\sum\limits_{s'\in \suppspie_{h+1}}
    p_h(s'|s,a)\sum\limits_{h'\in\dsb{h+1,H}}\E\limits_{s''\sim
    \rho_{h'}^{p,\pi^E}(\cdot|s_{h+1}=s')}
    b_{h'}(s'',a^E),\\
    &\qquad
        \max\limits_{(s,a)\notin \suppsapib_h} 4H\max\limits_{s'\in \suppspie_{h+1}}\sum\limits_{h'\in\dsb{h+1,H}}
            \E\limits_{s''\sim \rho^{p,\pi^E}_{h'}(\cdot|s_{h+1}=s')}
            b_{h'}(s'',a^E)
        \biggr\}\\
        &\le
    \sum\limits_{h\in\dsb{H}}
        \max\biggl\{
        \max\limits_{(s,a)\in \suppsapie_h} 2H b_h(s,a),\\
        &\qquad\max\limits_{(s,a)\in \suppsapib_h\setminus\suppsapie_h}
        2Hb_h(s,a)+ 4H^2\max\limits_{(s',a^E,h')\in\suppsapie}
    b_{h'}(s',a^E),\\
    &\qquad
        \max\limits_{(s,a)\notin \suppsapib_h} 4H^2\max\limits_{(s',a^E,h')\in\suppsapie}
    b_{h'}(s',a^E)
        \biggr\}\\
                &\le
    \sum\limits_{h\in\dsb{H}}
        \max\biggl\{
        \max\limits_{(s,a)\in \suppsapie_h} 2H b_h(s,a),\max\limits_{(s,a)\in \suppsapib_h\setminus\suppsapie_h}
        2Hb_h(s,a)+ 4H^2\max\limits_{(s',a^E,h')\in\suppsapie}
    b_{h'}(s',a^E)
        \biggr\}\\
    &\le
    \sum\limits_{h\in\dsb{H}}
        \max\limits_{(s,a)\in \suppsapib_h} 2Hb_h(s,a)
        +\sum\limits_{h\in\dsb{H}}4H^2\max\limits_{(s',a^E,h')\in\suppsapie}
        b_{h'}(s',a^E)
        \biggr\}\\
        &\le
        2H^2
        \max\limits_{(s,a,h)\in \suppsapib} b_h(s,a)
        +4H^3\max\limits_{(s,a,h)\in\suppsapie}
        b_{h}(s,a).
\end{align*}
\end{proof}

\subsubsection{Proofs of the main theorems}
Thanks to Lemma \ref{lemma: performance decomposition pirlo}
and Lemma \ref{lemma: performance decomposition pirlo superset},
we can conclude the proof of Theorem \ref{theorem: upper bound d pirlo}.
\upperboundpirlo*
\begin{proof}
The proof for the subset and superset is completely analogous.
Under good event $\mathcal{E}$, the performance decomposition
lemma for the subset (Lemma \ref{lemma: performance decomposition pirlo}) tells us that:
\begin{align*}
    \mathcal{H}_d(\sub, \subrelax)&\le
        2H\sum\limits_{h\in\dsb{H}}
        \E\limits_{(s,a)\sim \rho^{p,\pi^b}_h(\cdot,\cdot)} b_h(s,a)
        +8H^3\max\limits_{(s,a,h)\in\suppsapie}
        b_{h}(s,a)\le \epsilon.
\end{align*}
We upper bound both the terms of the sum by $\epsilon/2$.
The bound of the first term is analogous to the bound of the term provided in the proof of Theorem \ref{theorem: upper bound d irlo}, so we will not rewrite it here; with regards to the other term, we have:
    \begin{align*}
    8H^3\max\limits_{(s,a^E,h)\in \suppsapie}b_{h}(s,a^E))&\markref{(1)}{=}
    8H^3\max\limits_{(s,a^E,h)\in \suppsapie}\sqrt{2\frac{\beta(N_h^b(s,a^E),\delta)}{N_h^b(s,a^E)}}\\
    &\markref{(2)}{\le}
    8\sqrt{2}H^3\sqrt{\beta(\tau^b,\delta)}\max\limits_{(s,a^E,h)\in \suppsapie}
    \sqrt{\frac{1}{N_h^b(s,a^E)}}\\
   &\markref{(3)}{\le}
    8\sqrt{2}H^3\sqrt{\beta(\tau^b,\delta)}\max\limits_{(s,a^E,h)\in \suppsapie}
    \sqrt{c_4 \frac{\ln\frac{|\suppsapib|}{\delta}}
        {\tau^b \rho_h^{p,\pi^b}(s,a^E)}}\\
        &\markref{(4)}{=}
    c_5H^3\sqrt{\frac{\beta(\tau^b,\delta)\ln\frac{|\suppsapib|}{\delta}}
        {\tau^b \rho_{\min}^{\pi^b,\suppsapie}}}\le \epsilon/2,
    \end{align*}
    where at (1) we have used the definition of the $b$ terms,
    at (2) we have upper bounded $\beta(N_{\bar{h}}^b(\bar{s},a^E),\delta)\le \beta(\tau^b,\delta)$
    for all $(\bar{s},\bar{h})\in \suppspie$, at (3) we use event $\mathcal{E}_4$,
    at (4) we define $c_5\coloneqq 8\sqrt{2c_4}$ and we use the definition of $\rho_{\min}^{\pi^b,\suppsapie}$.

    Similarly to the proof of Theorem \ref{theorem: upper bound d irlo},
    we apply Lemma \ref{lemma: lambert} to ($c_6\coloneqq 4c_5^2$):
    \begin{align*}
        \tau^b \ge c_6\frac{H^6\ln\frac{|\suppsapib|}
        {\delta}}{\rho_{\min}^{\pi^b,\suppsapie}\epsilon^2}\bigl(\ln\frac{4|\suppsapib|}{\delta}+
        (|\suppspibmax|-1)\ln\bigl(
            e(1+\tau^b/(|\suppspibmax|-1))
        \bigr)\bigr),
    \end{align*}
    to obtain:
    \begin{align*}
        \tau^b &\le\widetilde{\mathcal{O}}\Biggl(
            \frac{H^6\ln\frac{1}
        {\delta}}{\rho_{\min}^{\pi^b,\suppsapie}\epsilon^2}\biggl(
            \ln\frac{1}
        {\delta}+|\suppspibmax|
        \biggr)
        \Biggr).
    \end{align*}
We can do the same for the superset through Lemma \ref{lemma: performance decomposition pirlo superset}. By combining the various bounds, we get the result.
\end{proof}
Thanks to Lemma \ref{lemma: performance decomposition pirlo d infty}
and Lemma \ref{lemma: performance decomposition pirlo superset d infty},
we can conclude the proof of Theorem \ref{theorem: upper bound d infty pirlo}.
\upperbounddinftypirlo*
\begin{proof}[Proof Sketch]
The proof is analogous to that of Theorem \ref{theorem: upper bound d pirlo}.
The only difference is that we use Lemma \ref{lemma: performance decomposition pirlo d infty} and Lemma \ref{lemma: performance decomposition pirlo d infty}, and that
we follow the proof of Theorem \ref{theorem: upper bound d infty irlo} instead
of that of Theorem \ref{theorem: upper bound d irlo} to bound the first term of:
\begin{align*}
        2H^2
        \max\limits_{(s,a,h)\in \suppsapib} b_h(s,a)
        +4H^3\max\limits_{(s,a,h)\in\suppsapie}
        b_{h}(s,a)\le\epsilon.
\end{align*}
\end{proof}

\subsection{Sample complexity for \pirlo with additional requirements}\label{subsec: pirlo r = r hat}
In the proofs of Theorem \ref{theorem: upper bound d pirlo} and Theorem \ref{theorem: upper bound d infty pirlo}, we have used reward choice lemmas that set $\widehat{r}_h(s,a)\neq r_h(s,a)$ in $(s,a,h)\notin\suppsapib$. However, it might be interesting to relate directly the error in the estimation of the transition model with the difference in the reward functions, so that where we do not have samples we have zero error.
We would like to have $\widehat{r}_h(s,a)= r_h(s,a)$ in $(s,a,h)\notin\suppsapib$. Notice that this property is satisfied in the proofs of Theorem \ref{theorem: upper bound d irlo}
and Theorem \ref{theorem: upper bound d infty irlo}. Moreover, for a notion of distance other than $d$ or $d_\infty$, the condition $\widehat{r}_h(s,a)= r_h(s,a)$ in $(s,a,h)\notin\suppsapib$ might even be needed.
Therefore, in this section, we provide reward choice lemmas that satisfy this property, and we show that this selection ends up in a $H^8$ dependence in the sample complexity instead of $H^6$.
\begin{lemma}[Reward Choice Subset]\label{lemma: error propagation pirlo subset v2}
    Under good event $\mathcal{E}$, for any $r\in \sub$,
    the reward $\widehat{r}$ constructed (recursively) as:
    \begin{align*}
        \begin{cases}
            \widehat{r}_h(s,a^E)=r_h(s,a^E)+\sum\limits_{s'\in\S}
            p_h(s'|s,a^E)V^{\pi^E}_{h+1}(s';p,r)-\sum\limits_{s'\in\S}
            \widetilde{p}^m_h(s'|s,a^E)V^{\pi^E}_{h+1}(s';\widetilde{p}^m,\widehat{r})\\
            \qquad\qquad\qquad+\max\limits_{s'\in\S}
            V^{\widetilde{\pi}^M}_{h+1}(s';\widetilde{p}^M,\widehat{r})
            -\max\limits_{s'\in\S}V^{\pi^M}_{h+1}(s';p^M,r),
            \qquad\forall (s,h)\in \suppspie\\
            \widehat{r}_h(s,a)=r_h(s,a),\quad\forall (s,a,h)\notin \suppsapib\\
            \widehat{r}_h(s,a)=r_h(s,a)+\sum\limits_{s'\in\S}
            p^M_h(s'|s,a)V^{\pi^M}_{h+1}(s';p^M,r)-\sum\limits_{s'\in\S}
            \widetilde{p}^M_h(s'|s,a)V^{\widetilde{\pi}^M}_{h+1}(s';\widetilde{p}^M,\widehat{r}),
            \quad\text{otherwise}\\            
        \end{cases},
    \end{align*}
    belongs to $\subrelax$.
\end{lemma}
\begin{proof}
    By definition of $\subrelax$, the reward $\widehat{r}$ belongs to $\subrelax$
    if and only if:
    \begin{align*}
        \forall (s,h)\in \suppspie,\forall a\in\A\setminus\{a^E\}:
    Q^{\pi^E}_h(s,a^E;\widetilde{p}^m,\widehat{r})\ge Q^{\widetilde{\pi}^M}_h(s,a;\widetilde{p}^M,\widehat{r}).
    \end{align*}
    By hypothesis, $r\in \sub$, therefore:
    \begin{align*}
        \forall (s,h)\in \suppspie,\forall a\in \A\setminus\{a^E\}:
        Q^{\pi^E}_h(s,a^E;p,r)\ge Q^{\pi^M}_h(s,a;p^M,r),
    \end{align*}
    thus, if we show that $\forall (s,h)\in \suppspie,\forall a\in\A\setminus\{a^E\}$,
    it holds that:
    \begin{align*}
    Q^{\pi^E}_h(s,a^E;\widetilde{p}^m,\widehat{r})-Q^{\widetilde{\pi}^M}_h(s,a;\widetilde{p}^M,\widehat{r})
        \ge
        Q^{\pi^E}_h(s,a^E;p,r)- Q^{\pi^M}_h(s,a;p^M,r)
        ,
    \end{align*}
    then we are done.

    Let us begin with triples $(s,a,h)\notin \suppsapib$ such that $(s,h)\in\suppspie$. By rearranging the terms in the definition of $\widehat{r}$, we observe that:
    \begin{align*}
        &\widehat{r}_h(s,a^E)+\sum\limits_{s'\in\S}
            \widetilde{p}^m_h(s'|s,a^E)V^{\pi^E}_{h+1}(s';\widetilde{p}^m,\widehat{r})=r_h(s,a^E)+\sum\limits_{s'\in\S}
            p_h(s'|s,a^E)V^{\pi^E}_{h+1}(s';p,r)\\
            &\qquad\qquad+\max\limits_{s'\in\S}
            V^{\widetilde{\pi}^M}_{h+1}(s';\widetilde{p}^M,\widehat{r})
            -\max\limits_{s'\in\S}V^{\pi^M}_{h+1}(s';p^M,r)\\
        &\iff Q^{\pi^E}_h(s,a^E;\widetilde{p}^m,\widehat{r})=Q^{\pi^E}_h(s,a^E;p,r)+\max\limits_{s'\in\S}
            V^{\widetilde{\pi}^M}_{h+1}(s';\widetilde{p}^M,\widehat{r})
            -\max\limits_{s'\in\S}V^{\pi^M}_{h+1}(s';p^M,r)\popblue{\pm \widehat{r}_h(s,a)}\\
            &\markref{(1)}{\iff} Q^{\pi^E}_h(s,a^E;\widetilde{p}^m,\widehat{r})=Q^{\pi^E}_h(s,a^E;p,r)+\widehat{r}_h(s,a)+\max\limits_{s'\in\S}
            V^{\widetilde{\pi}^M}_{h+1}(s';\widetilde{p}^M,\widehat{r})
            -\bigl(\popblue{r_h(s,a)}+\max\limits_{s'\in\S}V^{\pi^M}_{h+1}(s';p^M,r)\bigr)\\
            &\iff Q^{\pi^E}_h(s,a^E;\widetilde{p}^m,\widehat{r})=Q^{\pi^E}_h(s,a^E;p,r)+Q^{\widetilde{\pi}^M}_h(s,a;\widetilde{p}^M,\widehat{r})
            -Q^{\pi^M}_h(s,a;p^M,r)\\
            &\implies Q^{\pi^E}_h(s,a^E;\widetilde{p}^m,\widehat{r})- Q^{\widetilde{\pi}^M}_h(s,a;\widetilde{p}^M,\widehat{r})\ge
            Q^{\pi^E}_h(s,a^E;p,r)- Q^{\pi^M}_h(s,a;p^M,r),
    \end{align*}
    where at (1) we have used that $\widehat{r}_h(s,a)=r_h(s,a)$ by definition.
    
    Now, consider any other triple $(s,a,h)\in \suppsapib\setminus\suppsapie$ such that
    $(s,h)\in \suppspie$. By rearranging the terms,
    we obtain:
    \begin{align}\label{eq: def rhat equality Q non-expert v2}
        Q^{\widetilde{\pi}^M}_h(s,a;\widetilde{p}^M,\widehat{r})=Q^{\pi^M}_h(s,a;p^M,r),
    \end{align}
    therefore, it suffices to show that 
    \begin{align*}
        Q^{\pi^E}_h(s,a^E;\widetilde{p}^m,\widehat{r})\ge Q^{\pi^E}_h(s,a^E;p,r).
    \end{align*}
    By using again the definition of $\widehat{r}$ for $(s,a,h)\in\suppsapie$, we know that:
    \begin{align}\label{eq: def hat r equality Q pirlo subset v2}
        Q^{\pi^E}_h(s,a^E;\widetilde{p}^m,\widehat{r})=Q^{\pi^E}_h(s,a^E;p,r)+\max\limits_{s'\in\S}
            V^{\widetilde{\pi}^M}_{h+1}(s';\widetilde{p}^M,\widehat{r})
            -\max\limits_{s'\in\S}V^{\pi^M}_{h+1}(s';p^M,r),
    \end{align}
    therefore, if we show that
    \begin{align*}
        \max\limits_{s'\in\S}
            V^{\widetilde{\pi}^M}_{h+1}(s';\widetilde{p}^M,\widehat{r})
            \ge\max\limits_{s'\in\S}V^{\pi^M}_{h+1}(s';p^M,r),
    \end{align*}
    then we are done. We do it by induction. At stage $H-1$,
    we have that:
    \begin{align*}
        \max\limits_{s'\in\S}
            V^{\widetilde{\pi}^M}_{H}(s';\widetilde{p}^M,\widehat{r})&=\max\limits_{s'\in\S}\E\limits_{a'\sim \widetilde{\pi}^M_H(\cdot|s')}
            \widehat{r}_H(s',a')\\
            &\markref{(1)}{=}
            \max\limits_{s'\in\S}\E\limits_{a'\sim \pi^M_H(\cdot|s')}r_H(s',a')\\
            &=\max\limits_{s'\in\S}
            V^{\pi^M}_{H}(s';p^M,r),
    \end{align*}
    where at (1) we have used the definition of $\widehat{r}$ at stage $H$, and the definitions of $\widetilde{\pi}^M$ and $\pi^M$.
    We make the inductive hypothesis that, at stage $h+1$, it holds that
    $\max_{s'\in\S}V^{\widetilde{\pi}^M}_{h+2}(s';\widetilde{p}^M,\widehat{r})
    \ge\max_{s'\in\S}V^{\pi^M}_{h+2}(s';p^M,r)$, and we consider stage $h$:
    \begin{align*}
        \max\limits_{s'\in\S}
            V^{\widetilde{\pi}^M}_{h+1}(s';\widetilde{p}^M,\widehat{r})
            &\markref{(1)}{=}\max\Bigl\{
            \max\limits_{s'\in \suppspie_{h+1}}
            Q^{\popblue{\pi^E}}_{h+1}(s',a^E;\widetilde{p}^M,\widehat{r}),
            \max\limits_{s'\notin \suppspie_{h+1}}\max\limits_{a'\in\A}
            Q^{\widetilde{\pi}^M}_{h+1}(s',a';\widetilde{p}^M,\widehat{r})
            \Bigr\}\\
            &\markref{(2)}{\ge}\max\Bigl\{
            \max\limits_{s'\in \suppspie_{h+1}}
            Q^{\pi^E}_{h+1}(s',a^E;\popblue{\widetilde{p}^m},\widehat{r}),
            \max\limits_{s'\notin \suppspie_{h+1}}\max\limits_{a'\in\A}
            Q^{\widetilde{\pi}^M}_{h+1}(s',a';\widetilde{p}^M,\widehat{r})\Bigr\}\\
            &\markref{(3)}{\ge}\max\Bigl\{
            \max\limits_{s'\in \suppspie_{h+1}}
            Q^{\pi^E}_{h+1}(s',a^E;\popblue{p},\popblue{r}),
            \max\limits_{s'\notin \suppspie_{h+1}}\max\limits_{a'\in\A}
            Q^{\widetilde{\pi}^M}_{h+1}(s',a';\widetilde{p}^M,\widehat{r})
            \Bigr\}\\
            &=\max\Bigl\{
            \max\limits_{s'\in \suppspie_{h+1}}
            Q^{\pi^E}_{h+1}(s',a^E;p,r),\\
            &\qquad\max\limits_{s'\notin \suppspie_{h+1}}\max\bigl\{\max\limits_{a'\in\A:
            (s',a',h+1)\in \suppsapib}
            Q^{\widetilde{\pi}^M}_{h+1}(s',a';\widetilde{p}^M,\widehat{r}),
            \max\limits_{a'\in\A:
            (s',a',h+1)\notin \suppsapib}
            Q^{\widetilde{\pi}^M}_{h+1}(s',a';\widetilde{p}^M,\widehat{r})
            \bigr\}
            \Bigr\}\\
            &\markref{(4)}{=}\max\Bigl\{
            \max\limits_{s'\in \suppspie_{h+1}}
            Q^{\pi^E}_{h+1}(s',a^E;p,r),\\
            &\qquad\max\limits_{s'\notin \suppspie_{h+1}}\max\bigl\{\max\limits_{a'\in\A:
            (s',a',h+1)\in \suppsapib}
            Q^{\popblue{\pi^M}}_{h+1}(s',a';\popblue{p^M},\popblue{r}),\\
            &\qquad\max\limits_{a'\in\A:
            (s',a',h+1)\notin \suppsapib}
            \widehat{r}_{h+1}(s',a')+\max\limits_{s''\in\S}V^{\widetilde{\pi}^M}_{h+2}(s'';\widetilde{p}^M,\widehat{r})
            \bigr\}
            \Bigr\}\\
            &\markref{(5)}{\ge}\max\Bigl\{
            \max\limits_{s'\in \suppspie_{h+1}}
            Q^{\pi^E}_{h+1}(s',a^E;p,r),\\
            &\qquad\max\limits_{s'\notin \suppspie_{h+1}}\max\bigl\{\max\limits_{a'\in\A:
            (s',a',h+1)\in \suppsapib}
            Q^{\pi^M}_{h+1}(s',a';p^M,r),\\
            &\qquad\max\limits_{a'\in\A:
            (s',a',h+1)\notin \suppsapib}
            \popblue{r}_{h+1}(s',a')+\max\limits_{s''\in\S}V^{\pi^M}_{h+2}(s'';\popblue{p^M},\popblue{r})
            \bigr\}
            \Bigr\}\\
            &=\max\Bigl\{
            \max\limits_{s'\in \suppspie_{h+1}}
            Q^{\pi^E}_{h+1}(s',a^E;p,r),\\
            &\qquad\max\limits_{s'\notin \suppspie_{h+1}}\max\bigl\{\max\limits_{a'\in\A:
            (s',a',h+1)\in \suppsapib}
            Q^{\pi^M}_{h+1}(s',a';p^M,r),
            \max\limits_{a'\in\A:
            (s',a',h+1)\notin \suppsapib}
            Q^{\pi^M}_{h+1}(s',a';p^M,r)
            \bigr\}
            \Bigr\}\\
            &=\max\Bigl\{
            \max\limits_{s'\in \suppspie_{h+1}}
            Q^{\pi^E}_{h+1}(s',a^E;p,r),\max\limits_{s'\notin \suppspie_{h+1}}
            \max\limits_{a'\in\A}
            Q^{\pi^M}_{h+1}(s',a';p^M,r)
            \Bigr\}\\
            &=\max\limits_{s'\in\S}
            V^{\pi^M}_{h+1}(s';p^M,r),
    \end{align*}
    where at (1) we use the definition of $\widetilde{\pi}^M$, at (2) we use the definition of $\widetilde{p}^m$ and $\widetilde{p}^M$,
    at (3) we use the inductive hypothesis along with Eq. \ref{eq: def hat r equality Q pirlo subset v2},
    at (4) we use Eq. \ref{eq: def rhat equality Q non-expert v2} and the definition of Q-function,
    at (5) we use the definition of $\widehat{r}$ and the inductive hypothesis.

    This concludes the proof.  
\end{proof}
\begin{lemma}[Reward Choice Superset]\label{lemma: error propagation pirlo superset v2}
    Under good event $\mathcal{E}$, for any $\widehat{r}\in \superrelax$,
    the reward $r$ constructed (recursively) as:
    \begin{align*}
        \begin{cases}
            r_h(s,a^E)=\widehat{r}_h(s,a^E)+\sum\limits_{s'\in\S}
            \widetilde{p}^M_h(s'|s,a^E)V^{\pi^E}_{h+1}(s';\widetilde{p}^M,\widehat{r})-\sum\limits_{s'\in\S}
            p_h(s'|s,a^E)V^{\pi^E}_{h+1}(s';p,r)\\
            \qquad\qquad\qquad+\min\limits_{s'\in\S}
            V^{\pi^m}_{h+1}(s';p^m,r)
            -\min\limits_{s'\in\S}V^{\widetilde{\pi}^m}_{h+1}(s';\widetilde{p}^m,\widehat{r}),
            \qquad\forall (s,h)\in \suppspie\\
            r_h(s,a)=\widehat{r}_h(s,a),\quad\forall (s,a,h)\notin \suppsapib\\
            r_h(s,a)=\widehat{r}_h(s,a)+
            \sum\limits_{s'\in\S}
            \widetilde{p}^m_h(s'|s,a)V^{\widetilde{\pi}^m}_{h+1}(s';\widetilde{p}^m,\widehat{r})
            -\sum\limits_{s'\in\S}p^m_h(s'|s,a)V^{\pi^m}_{h+1}(s';p^m,r),
            \quad\text{otherwise}\\            
        \end{cases},
    \end{align*}
    belongs to $\super$.
\end{lemma}
\begin{proof}
    By definition of $\super$, a sufficient condition for having the reward $r$ belong to $\super$ is:
    \begin{align*}
        \forall (s,h)\in \suppspie,\forall a\in\A\setminus\{a^E\}:
        Q^{\pi^E}_h(s,a^E;p,r)\ge Q^{\pi^m}_h(s,a;p^m,r).
    \end{align*}
    By hypothesis, $r\in \superrelax$, therefore:
    \begin{align*}
        \forall (s,h)\in \suppspie,\forall a\in \A\setminus\{a^E\}:
    Q^{\pi^E}_h(s,a^E;\widetilde{p}^M,\widehat{r})\ge Q^{\widetilde{\pi}^m}_h(s,a;\widetilde{p}^m,\widehat{r}),
    \end{align*}
    thus, if we show that $\forall (s,h)\in \suppspie,\forall a\in\A\setminus\{a^E\}$,
    it holds that:
    \begin{align*}
    Q^{\pi^E}_h(s,a^E;p,r)- Q^{\pi^m}_h(s,a;p^m,r)\ge
    Q^{\pi^E}_h(s,a^E;\widetilde{p}^M,\widehat{r})-Q^{\widetilde{\pi}^m}_h(s,a;\widetilde{p}^m,\widehat{r}),
    \end{align*}
    then we are done.

    Let us begin with triples $(s,a,h)\notin \suppsapib$ such that $(s,h)\in\suppspie$. By rearranging the terms in the definition of $r$, we observe that:
    \begin{align*}
        &\widehat{r}_h(s,a^E)+\sum\limits_{s'\in\S}
            \widetilde{p}^M_h(s'|s,a^E)V^{\pi^E}_{h+1}(s';\widetilde{p}^M,\widehat{r})=r_h(s,a^E)+\sum\limits_{s'\in\S}
            p_h(s'|s,a^E)V^{\pi^E}_{h+1}(s';p,r)\\
            &\qquad\qquad+\min\limits_{s'\in\S}
            V^{\widetilde{\pi}^m}_{h+1}(s';\widetilde{p}^m,\widehat{r})
            -\min\limits_{s'\in\S}V^{\pi^m}_{h+1}(s';p^m,r)\\
        &\iff Q^{\pi^E}_h(s,a^E;\widetilde{p}^M,\widehat{r})=Q^{\pi^E}_h(s,a^E;p,r)+\min\limits_{s'\in\S}
            V^{\widetilde{\pi}^m}_{h+1}(s';\widetilde{p}^m,\widehat{r})
            -\min\limits_{s'\in\S}V^{\pi^m}_{h+1}(s';p^m,r)\popblue{\pm r_h(s,a)}\\
            &\markref{(1)}{\iff} Q^{\pi^E}_h(s,a^E;\widetilde{p}^M,\widehat{r})=Q^{\pi^E}_h(s,a^E;p,r)+\popblue{\widehat{r}_h(s,a)}+\min\limits_{s'\in\S}
            V^{\widetilde{\pi}^m}_{h+1}(s';\widetilde{p}^m,\widehat{r})
            -\bigl(r_h(s,a)+\min\limits_{s'\in\S}V^{\pi^m}_{h+1}(s';p^m,r)\bigr)\\
            &\iff Q^{\pi^E}_h(s,a^E;\widetilde{p}^M,\widehat{r})=Q^{\pi^E}_h(s,a^E;p,r)+Q^{\widetilde{\pi}^m}_h(s,a;\widetilde{p}^m,\widehat{r})
            -Q^{\pi^m}_h(s,a;p^m,r)\\
            &\implies Q^{\pi^E}_h(s,a^E;\widetilde{p}^M,\widehat{r})- Q^{\widetilde{\pi}^m}_h(s,a;\widetilde{p}^m,\widehat{r})\ge
            Q^{\pi^E}_h(s,a^E;p,r)- Q^{\pi^m}_h(s,a;p^m,r),
    \end{align*}
    where at (1) we have used that $r_h(s,a)=\widehat{r}_h(s,a)$ by definition.
    
    Now, consider any other triple $(s,a,h)\in \suppsapib\setminus\suppsapie$ such that
    $(s,h)\in \suppspie$. By rearranging the terms,
    we obtain:
    \begin{align}\label{eq: def rhat equality Q non-expert v2 superset}
        Q^{\widetilde{\pi}^m}_h(s,a;\widetilde{p}^m,\widehat{r})=Q^{\pi^m}_h(s,a;p^m,r),
    \end{align}
    therefore, it suffices to show that 
    \begin{align*}
        Q^{\pi^E}_h(s,a^E;p,r)\ge Q^{\pi^E}_h(s,a^E;\widetilde{p}^M,\widehat{r}).
    \end{align*}
    By using again the definition of $\widehat{r}$ for $(s,a,h)\in\suppsapie$, we know that:
    \begin{align}\label{eq: def hat r equality Q pirlo superset v2}
        Q^{\pi^E}_h(s,a^E;\widetilde{p}^M,\widehat{r})=Q^{\pi^E}_h(s,a^E;p,r)+\min\limits_{s'\in\S}
            V^{\widetilde{\pi}^m}_{h+1}(s';\widetilde{p}^m,\widehat{r})
            -\min\limits_{s'\in\S}V^{\pi^m}_{h+1}(s';p^m,r),
    \end{align}
    therefore, if we show that
    \begin{align*}
        \min\limits_{s'\in\S}
            V^{\widetilde{\pi}^m}_{h+1}(s';\widetilde{p}^m,\widehat{r})
            \le\min\limits_{s'\in\S}V^{\pi^m}_{h+1}(s';p^m,r),
    \end{align*}
    then we are done. We do it by induction. At stage $H-1$,
    we have that:
    \begin{align*}
        \min\limits_{s'\in\S}
            V^{\widetilde{\pi}^m}_{H}(s';\widetilde{p}^m,\widehat{r})&=\min\limits_{s'\in\S}\E\limits_{a'\sim \widetilde{\pi}^m_H(\cdot|s')}
            \widehat{r}_H(s',a')\\
            &\markref{(1)}{=}
            \min\limits_{s'\in\S}\E\limits_{a'\sim \pi^m_H(\cdot|s')}r_H(s',a')\\
            &=\min\limits_{s'\in\S}
            V^{\pi^m}_{H}(s';p^m,r),
    \end{align*}
    where at (1) we have used the definition of $\widehat{r}$ at stage $H$, and also the definitions of $\widetilde{\pi}^m$ and $\pi^m$.
    We make the inductive hypothesis that, at stage $h+1$, it holds that
    $\min_{s'\in\S}V^{\widetilde{\pi}^m}_{h+2}(s';\widetilde{p}^m,\widehat{r})
    \le\min_{s'\in\S}V^{\pi^m}_{h+2}(s';p^m,r)$, and we consider stage $h$:
    \begin{align*}
        \min\limits_{s'\in\S}
            V^{\widetilde{\pi}^m}_{h+1}(s';\widetilde{p}^m,\widehat{r})
            &\markref{(1)}{=}\min\Bigl\{
            \min\limits_{s'\in \suppspie_{h+1}}
            Q^{\popblue{\pi^E}}_{h+1}(s',a^E;\widetilde{p}^m,\widehat{r}),
            \min\limits_{s'\notin \suppspie_{h+1}}\max\limits_{a'\in\A}
            Q^{\widetilde{\pi}^m}_{h+1}(s',a';\widetilde{p}^m,\widehat{r})
            \Bigr\}\\
            &\markref{(2)}{\le}\min\Bigl\{
            \min\limits_{s'\in \suppspie_{h+1}}
            Q^{\pi^E}_{h+1}(s',a^E;\popblue{\widetilde{p}^M},\widehat{r}),
            \min\limits_{s'\notin \suppspie_{h+1}}\max\limits_{a'\in\A}
            Q^{\widetilde{\pi}^m}_{h+1}(s',a';\widetilde{p}^m,\widehat{r})\Bigr\}\\
            &\markref{(3)}{\le}\min\Bigl\{
            \min\limits_{s'\in \suppspie_{h+1}}
            Q^{\pi^E}_{h+1}(s',a^E;\popblue{p},\popblue{r}),
            \min\limits_{s'\notin \suppspie_{h+1}}\max\limits_{a'\in\A}
            Q^{\widetilde{\pi}^m}_{h+1}(s',a';\widetilde{p}^m,\widehat{r})
            \Bigr\}\\
            &=\min\Bigl\{
            \min\limits_{s'\in \suppspie_{h+1}}
            Q^{\pi^E}_{h+1}(s',a^E;p,r),\\
            &\qquad\min\limits_{s'\notin \suppspie_{h+1}}\max\bigl\{\max\limits_{a'\in\A:
            (s',a',h+1)\in \suppsapib}
            Q^{\widetilde{\pi}^m}_{h+1}(s',a';\widetilde{p}^m,\widehat{r}),
            \max\limits_{a'\in\A:
            (s',a',h+1)\notin \suppsapib}
            Q^{\widetilde{\pi}^m}_{h+1}(s',a';\widetilde{p}^m,\widehat{r})
            \bigr\}
            \Bigr\}\\
            &\markref{(4)}{=}\min\Bigl\{
            \min\limits_{s'\in \suppspie_{h+1}}
            Q^{\pi^E}_{h+1}(s',a^E;p,r),\\
            &\qquad\min\limits_{s'\notin \suppspie_{h+1}}\max\bigl\{\max\limits_{a'\in\A:
            (s',a',h+1)\in \suppsapib}
            Q^{\popblue{\pi^m}}_{h+1}(s',a';\popblue{p^m},\popblue{r}),\\
            &\qquad\max\limits_{a'\in\A:
            (s',a',h+1)\notin \suppsapib}
            \widehat{r}_{h+1}(s',a')+\min\limits_{s''\in\S}V^{\widetilde{\pi}^m}_{h+2}(s'';\widetilde{p}^m,\widehat{r})
            \bigr\}
            \Bigr\}\\
            &\markref{(5)}{\le}\min\Bigl\{
            \min\limits_{s'\in \suppspie_{h+1}}
            Q^{\pi^E}_{h+1}(s',a^E;p,r),\\
            &\qquad\min\limits_{s'\notin \suppspie_{h+1}}\max\bigl\{\max\limits_{a'\in\A:
            (s',a',h+1)\in \suppsapib}
            Q^{\pi^m}_{h+1}(s',a';p^m,r),\\
            &\qquad\max\limits_{a'\in\A:
            (s',a',h+1)\notin \suppsapib}
            \popblue{r}_{h+1}(s',a')+\max\limits_{s''\in\S}V^{\pi^m}_{h+2}(s'';\popblue{p^m},\popblue{r})
            \bigr\}
            \Bigr\}\\
            &=\min\Bigl\{
            \min\limits_{s'\in \suppspie_{h+1}}
            Q^{\pi^E}_{h+1}(s',a^E;p,r),\\
            &\qquad\min\limits_{s'\notin \suppspie_{h+1}}\max\bigl\{\max\limits_{a'\in\A:
            (s',a',h+1)\in \suppsapib}
            Q^{\pi^m}_{h+1}(s',a';p^m,r),
            \max\limits_{a'\in\A:
            (s',a',h+1)\notin \suppsapib}
            Q^{\pi^m}_{h+1}(s',a';p^m,r)
            \bigr\}
            \Bigr\}\\
            &=\min\Bigl\{
            \min\limits_{s'\in \suppspie_{h+1}}
            Q^{\pi^E}_{h+1}(s',a^E;p,r),\min\limits_{s'\notin \suppspie_{h+1}}
            \max\limits_{a'\in\A}
            Q^{\pi^m}_{h+1}(s',a';p^m,r)
            \Bigr\}\\
            &=\min\limits_{s'\in\S}
            V^{\pi^m}_{h+1}(s';p^m,r),
    \end{align*}
    where at (1) we use the definition of $\widetilde{\pi}^m$, at (2) we use the definition of $\widetilde{p}^m$ and $\widetilde{p}^M$,
    at (3) we use the inductive hypothesis along with Eq. \ref{eq: def hat r equality Q pirlo superset v2},
    at (4) we use Eq. \ref{eq: def rhat equality Q non-expert v2 superset} and the Bellman's equation,
    at (5) we use the definition of $r$ and the inductive hypothesis.

    This concludes the proof.  
\end{proof}

\subsubsection{Lemmas for Theorem \ref{theorem: upper bound d pirlo H 8}}
We can exploit Lemma \ref{lemma: error propagation pirlo subset v2} to bound the error for the subset.
\begin{lemma}[Performance Decomposition Subset]\label{lemma: performance decomposition pirlo v2}
    Under good event $\mathcal{E}$, it holds that:
    \begin{align*}
        \mathcal{H}_d(\sub, \subrelax)&\le
        2H\sum\limits_{h\in\dsb{H}}\E\limits_{(s,a)\sim \rho^{p,\pi^b}_h(\cdot,\cdot)}   
        b_{h}(s,a)+
        4H^4\max\limits_{(s,h)\in \suppspie}b_{h}(s,a^E).
    \end{align*}
\end{lemma}
\begin{proof}
We can write:
\begin{align}
\begin{split}\label{eq: bound hausdorff temporary v2}
    \mathcal{H}_d(\sub, \subrelax)&\coloneqq
\max\{\sup\limits_{r\in \sub}
            \inf\limits_{\widetilde{r}\in \subrelax}d(r,\widetilde{r}),
            \sup\limits_{\widetilde{r}\in \subrelax}
            \inf\limits_{r\in \sub}d(r,\widetilde{r})\}\\
    &\markref{(1)}{=}
    \sup\limits_{r\in \sub}
            \inf\limits_{\widetilde{r}\in \subrelax}d(r,\widetilde{r})\\
    &\eqqcolon\sup\limits_{r\in \sub}
            \inf\limits_{\widetilde{r}\in \subrelax}
            \frac{1}{M}\sum\limits_{h\in\dsb{H}}\biggl(\E\limits_{(s,a)\sim\rho_h^{p,\pi^b}(\cdot,\cdot)}
            \bigl|r_h(s,a)-\widetilde{r}_h(s,a)\bigr|+
            \max\limits_{(s,a)\notin \suppsahpib}\bigl|
                    r_h(s,a)-\widetilde{r}_h(s,a)
                \bigr|\biggr)\\
    &\markref{(2)}{\le}
    \sup\limits_{r\in \sub}
    \frac{1}{M}\sum\limits_{h\in\dsb{H}}\biggl(\E\limits_{(s,a)\sim\rho_h^{p,\pi^b}(\cdot,\cdot)}
            \bigl|r_h(s,a)-\widehat{r}_h(s,a)\bigr|+
            \underbrace{\max\limits_{(s,a)\notin \suppsahpib}\bigl|
                    r_h(s,a)-\widehat{r}_h(s,a)
                \bigr|}_{=0}\biggr)\\
    &=
    \sup\limits_{r\in \sub}
    \frac{1}{M}\sum\limits_{h\in\dsb{H}}\E\limits_{(s,a)\sim\rho_h^{p,\pi^b}(\cdot,\cdot)}
            \bigl|r_h(s,a)-\widehat{r}_h(s,a)\bigr|
\end{split}
\end{align}
where at (1) we have used that, under good event $\mathcal{E}$, $\subrelax\subseteq\sub$,
and at (2) we apply Lemma \ref{lemma: error propagation pirlo subset v2}, by denoting with $\widehat{r}$ the chosen reward from $\subrelax$.

In the following, it is useful to denote, for any $h\in\dsb{H}$:
\begin{align*}
    X_h\coloneqq \max\limits_{s'\in\S}\bigl|
    V^{\widetilde{\pi}^M}_{h+1}(s';\widetilde{p}^M,\widehat{r})-V^{\pi^M}_{h+1}(s';p^M,r)
    \bigr|.
\end{align*}
Let us consider any $(s,h)\in \suppspie$. The difference between the rewards of expert's action can be bounded by:
\begin{align*}
    \bigl|\widehat{r}_h(s,a^E)-r_h(s,a^E)\bigr|&\markref{(1)}{\le}
    \bigl|
    \E\limits_{s'\sim p_h(\cdot|s,a^E)}V^{\pi^E}_{h+1}(s';p,r)-
    \E\limits_{s'\sim \widetilde{p}^m_h(\cdot|s,a^E)}V^{\pi^E}_{h+1}(s';\widetilde{p}^m,\widehat{r})
    \bigr|\\
    &\qquad+\bigl|\max\limits_{s'\in\S}
            V^{\widetilde{\pi}^M}_{h+1}(s';\widetilde{p}^M,\widehat{r})
            -\max\limits_{s'\in\S}V^{\pi^M}_{h+1}(s';p^M,r)\bigr|\\
        &\markref{(2)}{\le}
\bigl|
\E\limits_{s'\sim p_h(\cdot|s,a^E)}V^{\pi^E}_{h+1}(s';p,r)-
\E\limits_{s'\sim \widetilde{p}^m_h(\cdot|s,a^E)}V^{\pi^E}_{h+1}(s';\widetilde{p}^m,\widehat{r})
\pm \E\limits_{s'\sim p_h(\cdot|s,a^E)}V^{\pi^E}_{h+1}(s';\widetilde{p}^m,\widehat{r})\bigr|\\
&\qquad+\popblue{\max\limits_{s'\in\S}}\bigl|
        V^{\widetilde{\pi}^M}_{h+1}(s';\widetilde{p}^M,\widehat{r})
        -V^{\pi^M}_{h+1}(s';p^M,r)\bigr|\\
&\markref{(3)}{\le}
MH\bigl|\bigl|p_h(\cdot|s,a^E)-\widetilde{p}^m_h(\cdot|s,a^E)
\bigr|\bigr|_1
+\E\limits_{s'\sim p_h(\cdot|s,a^E)}\bigl|V^{\pi^E}_{h+1}(s';p,r)-
V^{\pi^E}_{h+1}(s';\widetilde{p}^m,\widehat{r})
\bigr|+X_h\\
&\markref{(4)}{\le}
2MHb_h(s,a^E)
+\E\limits_{s'\sim p_h(\cdot|s,a^E)}\bigl|Q^{\pi^E}_{h+1}(s',a^E;p,r)-
Q^{\pi^E}_{h+1}(s',a^E;\widetilde{p}^m,\widehat{r})
\bigr|+X_h\\
&\markref{(5)}{=}
2MHb_h(s,a^E)
+\E\limits_{s'\sim p_h(\cdot|s,a^E)}\bigl|
\max\limits_{s''\in\S}V^{\pi^M}_{h+2}(s'';p^M,r)-
\max\limits_{s''\in\S}V^{\widetilde{\pi}^M}_{h+2}(s'';\widetilde{p}^M,\widehat{r})
\bigr|+X_h\\
&\le
2MHb_h(s,a^E)
+\E\limits_{s'\sim p_h(\cdot|s,a^E)}\bigl[X_{h+1}\bigr]+X_h\\
&\markref{(6)}{=}
2MHb_h(s,a^E)+X_h+X_{h+1},
\end{align*}
where at (1) we use the definition of $\widehat{r}$ in Lemma \ref{lemma: error propagation pirlo subset v2} and triangle inequality, at (2) we use that, for any pair $f,g$ of real-valued functions, it holds that $|\max_x f(x)-\max_x g(x)|\le \max_x|f(x)-g(x)|$, at (3) we apply triangle inequality twice, we recognize the definition of $X_h$, we upper bound the value function by $MH$, and we recognize the definition of $\ell_1$ norm,
at (4) we first use triangle inequality $\|p_h(\cdot|s,a^E)-\widetilde{p}^m_h(\cdot|s,a^E)\|_1
\le \|p_h(\cdot|s,a^E)-\widehat{p}_h(\cdot|s,a^E)\|_1+
\|\widetilde{p}^m_h(\cdot|s,a^E)-\widehat{p}_h(\cdot|s,a^E)\|_1$, then we use Pinsker's inequality, event $\mathcal{E}_3$ from Lemma \ref{lemma: concentration}, and the definition of $b_h(s,a^E)$;
at (5) we use the definition of $\widehat{r}$ in Lemma \ref{lemma: error propagation pirlo subset v2} in the form of Eq. \ref{eq: def hat r equality Q pirlo subset v2}, by noticing that the support of $p_h(\cdot|s,a^E)$
is contained in $\suppspie$, and at (6) we realize that $X_{h+1}$ depends
only on $h$ and not on $s'$.

In order to upper bound the term $X_h$, we write:
\begin{align*}
    X_h&\coloneqq \max\limits_{s'\in\S}\bigl|
    V^{\widetilde{\pi}^M}_{h+1}(s';\widetilde{p}^M,\widehat{r})-V^{\pi^M}_{h+1}(s';p^M,r)
    \bigr|\\
    &\markref{(1)}{=}\max\biggl\{
    \max\limits_{s'\notin \suppspie_{h+1}}\Bigl|
    \max\limits_{a'\in\A} Q^{\widetilde{\pi}^M}_{h+1}(s',a';\widetilde{p}^M,\widehat{r})-
    \max\limits_{a'\in\A} Q^{\pi^M}_{h+1}(s',a';p^M,r)
    \Bigr|,\\
    &\qquad\qquad\max\limits_{s'\in \suppspie_{h+1}}\Bigl|
    Q^{\popblue{\pi^E}}_{h+1}(s',a^E;\widetilde{p}^M,\widehat{r})-
    Q^{\popblue{\pi^E}}_{h+1}(s',a^E;\popblue{p},r)
    \Bigr|
    \biggr\}\\
    &\markref{(2)}{\le}\max\biggl\{
    \max\limits_{s'\notin \suppspie_{h+1}}
\popblue{\max\limits_{a'\in\A}}\Bigl|
     Q^{\widetilde{\pi}^M}_{h+1}(s',a';\widetilde{p}^M,\widehat{r})-
     Q^{\pi^M}_{h+1}(s',a';p^M,r)
    \Bigr|,\\
    &\qquad\qquad\max\limits_{s'\in \suppspie_{h+1}}\Bigl|
    Q^{\pi^E}_{h+1}(s',a^E;\widetilde{p}^M,\widehat{r})-
    Q^{\pi^E}_{h+1}(s',a^E;p,r)\pm Q^{\pi^E}_{h+1}(s',a^E;\widetilde{p}^m,\widehat{r}) 
    \Bigr|
    \biggr\}\\
    &\markref{(3)}{=}\max\biggl\{
    \max\limits_{s'\notin \suppspie_{h+1}}
    \max\Bigl\{
    \max\limits_{a'\in\A: (s',a',h+1)\in \suppsapib}\Bigl|
     \underbrace{Q^{\widetilde{\pi}^M}_{h+1}(s',a';\widetilde{p}^M,\widehat{r})-
     Q^{\pi^M}_{h+1}(s',a';p^M,r)}_{=0}
    \Bigr|,\\
    &\qquad\qquad\max\limits_{a'\in\A: (s',a',h+1)\notin \suppsapib}\Bigl|
     Q^{\widetilde{\pi}^M}_{h+1}(s',a';\widetilde{p}^M,\widehat{r})-
     Q^{\pi^M}_{h+1}(s',a';p^M,r)
    \Bigr|
    \Bigr\},\\
    &\qquad\qquad\max\limits_{s'\in \suppspie_{h+1}}\Bigl|
    Q^{\pi^E}_{h+1}(s',a^E;\widetilde{p}^M,\widehat{r})-Q^{\pi^E}_{h+1}(s',a^E;\widetilde{p}^m,\widehat{r})
    +\max\limits_{s''\in\S}
            V^{\widetilde{\pi}^M}_{h+2}(s'';\widetilde{p}^M,\widehat{r})
            -\max\limits_{s''\in\S}V^{\pi^M}_{h+2}(s';p^M,r)
    \Bigr|
    \biggr\}\\
    &\markref{(4)}{\le}\max\biggl\{
    \max\limits_{s'\notin \suppspie_{h+1}}
    \max\limits_{a'\in\A: (s',a',h+1)\notin \suppsapib}\Bigl|
     \underbrace{\widehat{r}_{h+1}(s',a')-r_{h+1}(s',a')}_{=0}+\max\limits_{s''\in\S}V^{\widetilde{\pi}^M}_{h+2}(s'';\widetilde{p}^M,\widehat{r})-\max\limits_{s''\in\S}
     V^{\pi^M}_{h+2}(s'';p^M,r)
    \Bigr|,\\
    &\qquad\qquad\max\limits_{s'\in \suppspie_{h+1}}\Bigl|
    Q^{\pi^E}_{h+1}(s',a^E;\widetilde{p}^M,\widehat{r})-Q^{\pi^E}_{h+1}(s',a^E;\widetilde{p}^m,\widehat{r})
    \Bigr|
    + \popblue{X_{h+1}}
    \biggr\}\\
    &\markref{(5)}{\le}\max\biggl\{
    X_{h+1},\max\limits_{s'\in \suppspie_{h+1}}\Bigl(\Bigl|
    \E\limits_{s''\sim \widetilde{p}^M_{h+1}(\cdot|s',a^E)}V^{\pi^E}_{h+2}(s'';\widetilde{p}^M,\widehat{r})-
    \E\limits_{s''\sim \widetilde{p}^m_{h+1}(\cdot|s',a^E)}V^{\pi^E}_{h+2}(s'';\widetilde{p}^m,\widehat{r})
    \Bigr|
    + X_{h+1}\Bigr)
    \biggr\}\\
    &=
    X_{h+1}+\max\limits_{s'\in \suppspie_{h+1}}\Bigl|
    \E\limits_{s''\sim \widetilde{p}^M_{h+1}(\cdot|s',a^E)}
    V^{\pi^E}_{h+2}(s'';\widetilde{p}^M,\widehat{r})-
    \E\limits_{s''\sim \widetilde{p}^m_{h+1}(\cdot|s',a^E)}V^{\pi^E}_{h+2}(s'';\widetilde{p}^m,\widehat{r})
    \pm\E\limits_{s''\sim \widetilde{p}^m_{h+1}(\cdot|s',a^E)}
    V^{\pi^E}_{h+2}(s'';\widetilde{p}^M,\widehat{r}) 
    \Bigr|\\
    &\markref{(6)}{\le}
    X_{h+1}+\max\limits_{s'\in \suppspie_{h+1}}\biggl(MH\Bigl\|\widetilde{p}^M_{h+1}(\cdot|s',a^E)-\widetilde{p}^m_{h+1}(\cdot|s',a^E)
\Bigr\|_1+
    \E\limits_{s''\sim \widetilde{p}^m_{h+1}(\cdot|s',a^E)}
    \Bigl|V^{\pi^E}_{h+2}(s'';\widetilde{p}^M,\widehat{r})-
    V^{\pi^E}_{h+2}(s'';\widetilde{p}^m,\widehat{r})\Bigr|\biggr)\\
    &\markref{(7)}{\le}
    X_{h+1}+\max\limits_{s'\in \suppspie_{h+1}}\biggl(2MHb_{h+1}(s',a^E)\\
    &\qquad\qquad+\E\limits_{s''\sim \widetilde{p}^m_{h+1}(\cdot|s',a^E)}
    \Bigl|
    \E\limits_{s'''\sim \widetilde{p}^M_{h+2}(\cdot|s'',a^E)}
    V^{\pi^E}_{h+3}(s''';\widetilde{p}^M,\widehat{r})-
    \E\limits_{s'''\sim \widetilde{p}^m_{h+2}(\cdot|s'',a^E)}
    V^{\pi^E}_{h+3}(s''';\widetilde{p}^m,\widehat{r})\Bigr|\biggr)\\
    &\markref{(8)}{\le}
    X_{h+1}+2MH\max\limits_{s'\in \suppspie_{h+1}}
    \sum\limits_{h'\in\dsb{h+1,H-1}}
    \E\limits_{s''\sim\rho^{\widetilde{p}^m,\pi^E}_{h'}(\cdot|s_{h+1}=s')}
    b_{h'}(s'',a^E)\\
    &\markref{(9)}{\le}
    2MH\sum\limits_{h'\in\dsb{h,H-1}}
    \max\limits_{s'\in \suppspie_{h'+1}}
    \sum\limits_{h''\in\dsb{h'+1,H-1}}
    \E\limits_{s''\sim\rho^{\widetilde{p}^m,\pi^E}_{h''}(\cdot|s_{h'+1}=s')}
    b_{h''}(s'',a^E)\\
    &\le
    2MH^3
    \max\limits_{(s',h')\in \suppspie}
    b_{h'}(s',a^E),
\end{align*}
where at (1) we apply the Bellman's equation and the definition of $\widetilde{\pi}^M$ and $\pi^M$, at (2) we use that
$|\max_x f(x)-\max_x g(x)|\le \max_x|f(x)-g(x)|$,
at (3) we use the definition of $\widehat{r}$ in the form of Eq.
\ref{eq: def rhat equality Q non-expert v2} and Eq.\ref{eq: def hat r equality Q pirlo subset v2}, at (4) we apply the Bellman's equation and the definition of $\widehat{r}$ to recognize that $\widehat{r}_{h+1}(s',a')-r_{h+1}(s',a')=0$; moreover, we apply triangle inequality along with the usual bound
$|\max_x f(x)-\max_x g(x)|\le \max_x|f(x)-g(x)|$, and we recognize the definition of $X_{h+1}$. At (5) we proceed similarly as (4) and we use the Bellman optimality equation, and we observe that $X_{h+1}$ does not depend on
$s'$; at (6) we upper bound the value function by $HM$ and recognize the $\ell_1$-norm, at (7) we use the concentration bound of event $\mathcal{E}_3$ in Lemma \ref{lemma: concentration} (both $\widetilde{p}^M$ and $\widetilde{p}^m$ lie at a ``distance'' of $b$ from $\widehat{p}$). At (8) we have unfolded the recursion to bound the difference of value functions between transition models
$\widetilde{p}^m$ and $\widetilde{p}^M$, at (9) we have unfolded the recursion on the $X$ terms.

Thanks to this expression, we can upper bound the difference of rewards in expert's action as:
\begin{align*}
    \bigl|\widehat{r}_h(s,a^E)-r_h(s,a^E)\bigr|&\le
    2MHb_h(s,a^E)+4MH^3
    \max\limits_{(s',h')\in \suppspie}
    b_{h'}(s',a^E).
\end{align*}

With regards to visited $(s,a,h)\in\suppsapib\setminus\suppsapie$, we can write:
\begin{align*}
    \bigl|\widehat{r}_h(s,a)-r_h(s,a)\bigr|
    &=
    \bigl|
    \E\limits_{s'\sim \widetilde{p}^M_h(\cdot|s,a)}V^{\widetilde{\pi}^M}_{h+1}(s';\widetilde{p}^M,\widehat{r})
    -\E\limits_{s'\sim p_h(\cdot|s,a)}V^{\pi^M}_{h+1}(s';p^M,r)
    \bigr|\\
    &\le 2MH b_h(s,a)+
    \E\limits_{s'\sim \widetilde{p}^M_h(\cdot|s,a)}\Bigl|V^{\widetilde{\pi}^M}_{h+1}(s';\widetilde{p}^M,\widehat{r})
    -V^{\pi^M}_{h+1}(s';p^M,r)
    \Bigr|\\
    &\le 2MH b_h(s,a)+
    \max\limits_{s'\in\S}\Bigl|V^{\widetilde{\pi}^M}_{h+1}(s';\widetilde{p}^M,\widehat{r})
    -V^{\pi^M}_{h+1}(s';p^M,r)
    \Bigr|\\
    &=2MHb_h(s,a)+X_h\\
    &\le 2MHb_h(s,a)+2MH^3
    \max\limits_{(s',h')\in \suppspie}
    b_{h'}(s',a^E)\\
    &\le 2MHb_h(s,a)+\popblue{4}MH^3
    \max\limits_{(s',h')\in \suppspie}
    b_{h'}(s',a^E).
\end{align*}
Obviously, for $(s,a,h)\notin \suppsapib$, we have:
\begin{align*}
    \bigl|\widehat{r}_h(s,a)-r_h(s,a)\bigr|&=0.
\end{align*}

Therefore, by Eq. \ref{eq: bound hausdorff temporary v2}, we can write:
\begin{align*}
\mathcal{H}_d(\sub, \subrelax)&\le
    \sup\limits_{r\in \sub}
    \frac{1}{M}\sum\limits_{h\in\dsb{H}}\E\limits_{(s,a)\sim\rho_h^{p,\pi^b}(\cdot,\cdot)}
            \bigl|r_h(s,a)-\widehat{r}_h(s,a)\bigr|\\
    &\le \sum\limits_{h\in\dsb{H}}\E\limits_{(s,a)\sim\rho_h^{p,\pi^b}(\cdot,\cdot)}
    \biggl(
    2Hb_h(s,a)+4H^3
    \max\limits_{(s',h')\in \suppspie}
    b_{h'}(s',a^E)
    \biggr)\\
    &=
    2H\sum\limits_{h\in\dsb{H}}\E\limits_{(s,a)\sim\rho_h^{p,\pi^b}(\cdot,\cdot)}
    b_h(s,a)+4H^3
    \sum\limits_{h\in\dsb{H}}\E\limits_{(s,a)\sim\rho_h^{p,\pi^b}(\cdot,\cdot)}
    \max\limits_{(s',h')\in \suppspie}
    b_{h'}(s',a^E)\\
    &=
    2H\sum\limits_{h\in\dsb{H}}\E\limits_{(s,a)\sim\rho_h^{p,\pi^b}(\cdot,\cdot)}
    b_h(s,a)+4H^3
    \sum\limits_{h\in\dsb{H}}
    \max\limits_{(s',h')\in \suppspie}
    b_{h'}(s',a^E)\\
    &=
    2H\sum\limits_{h\in\dsb{H}}\E\limits_{(s,a)\sim\rho_h^{p,\pi^b}(\cdot,\cdot)}
    b_h(s,a)+4H^4
    \max\limits_{(s',h')\in \suppspie}
    b_{h'}(s',a^E).
\end{align*}

This concludes the proof w.r.t. the subset.
\end{proof}

Now we can exploit Lemma \ref{lemma: error propagation pirlo superset v2} to bound the error for the superset.
\begin{lemma}[Performance Decomposition Superset]\label{lemma: performance decomposition pirlo v2 superset}
    Under good event $\mathcal{E}$, it holds that:
    \begin{align*}
        \mathcal{H}_d(\super, \superrelax)&\le
        2H\sum\limits_{h\in\dsb{H}}\E\limits_{(s,a)\sim \rho^{p,\pi^b}_h(\cdot,\cdot)}   
        b_{h}(s,a)+
        4H^4\max\limits_{(s,h)\in \suppspie}b_{h}(s,a^E).
    \end{align*}
\end{lemma}
\begin{proof}
We can write:
\begin{align}
\begin{split}\label{eq: bound hausdorff temporary v2 super}
    \mathcal{H}_d(\super, \superrelax)&\coloneqq
\max\{\sup\limits_{r\in \super}
            \inf\limits_{\widetilde{r}\in \superrelax}d(r,\widetilde{r}),
            \sup\limits_{\widetilde{r}\in \superrelax}
            \inf\limits_{r\in \super}d(r,\widetilde{r})\}\\
    &\markref{(1)}{=}
    \sup\limits_{\widetilde{r}\in \superrelax}
            \inf\limits_{r\in \super}
            d(r,\widetilde{r})\\
    &\eqqcolon
    \sup\limits_{\widetilde{r}\in \superrelax}
            \inf\limits_{r\in \super}
            \frac{1}{M}\sum\limits_{h\in\dsb{H}}\biggl(\E\limits_{(s,a)\sim\rho_h^{p,\pi^b}(\cdot,\cdot)}
            \bigl|r_h(s,a)-\widetilde{r}_h(s,a)\bigr|+
            \max\limits_{(s,a)\notin \suppsahpib}\bigl|
                    r_h(s,a)-\widetilde{r}_h(s,a)
                \bigr|\biggr)\\
    &\markref{(2)}{\le}
    \sup\limits_{\widetilde{r}\in \superrelax}
    \frac{1}{M}\sum\limits_{h\in\dsb{H}}\biggl(\E\limits_{(s,a)\sim\rho_h^{p,\pi^b}(\cdot,\cdot)}
            \bigl|r_h(s,a)-\widetilde{r}_h(s,a)\bigr|+
            \underbrace{\max\limits_{(s,a)\notin \suppsahpib}\bigl|
                    r_h(s,a)-\widetilde{r}_h(s,a)
                \bigr|}_{=0}\biggr)\\
    &=
    \sup\limits_{\widetilde{r}\in \superrelax}
    \frac{1}{M}\sum\limits_{h\in\dsb{H}}\E\limits_{(s,a)\sim\rho_h^{p,\pi^b}(\cdot,\cdot)}
            \bigl|r_h(s,a)-\widetilde{r}_h(s,a)\bigr|
\end{split}
\end{align}
where at (1) we have used that, under good event $\mathcal{E}$, $\super\subseteq\superrelax$,
and at (2) we apply Lemma \ref{lemma: error propagation pirlo superset v2}, by denoting with $r$ the chosen reward from $\super$.

In the following, it is useful to denote, for any $h\in\dsb{H}$:
\begin{align*}
    Y_h\coloneqq \max\limits_{s'\in\S}\bigl|
    V^{\widetilde{\pi}^m}_{h+1}(s';\widetilde{p}^m,\widehat{r})-V^{\pi^m}_{h+1}(s';p^m,r)
    \bigr|.
\end{align*}
Let us consider any $(s,h)\in \suppspie$. The difference between the rewards of expert's action can be bounded by:
\begin{align*}
    \bigl|\widehat{r}_h(s,a^E)-r_h(s,a^E)\bigr|&\markref{(1)}{\le}
    \bigl|
    \E\limits_{s'\sim p_h(\cdot|s,a^E)}V^{\pi^E}_{h+1}(s';p,r)-
    \E\limits_{s'\sim \widetilde{p}^M_h(\cdot|s,a^E)}V^{\pi^E}_{h+1}(s';\widetilde{p}^M,\widehat{r})
    \bigr|\\
    &\qquad+\bigl|\min\limits_{s'\in\S}
            V^{\widetilde{\pi}^m}_{h+1}(s';\widetilde{p}^m,\widehat{r})
            -\min\limits_{s'\in\S}V^{\pi^m}_{h+1}(s';p^m,r)\bigr|\\
        &\markref{(2)}{\le}
\bigl|
\E\limits_{s'\sim p_h(\cdot|s,a^E)}V^{\pi^E}_{h+1}(s';p,r)-
\E\limits_{s'\sim \widetilde{p}^M_h(\cdot|s,a^E)}V^{\pi^E}_{h+1}(s';\widetilde{p}^M,\widehat{r})
\pm \E\limits_{s'\sim p_h(\cdot|s,a^E)}V^{\pi^E}_{h+1}(s';\widetilde{p}^M,\widehat{r})\bigr|\\
&\qquad+\popblue{\max\limits_{s'\in\S}}\bigl|
        V^{\widetilde{\pi}^m}_{h+1}(s';\widetilde{p}^m,\widehat{r})
        -V^{\pi^m}_{h+1}(s';p^m,r)\bigr|\\
&\markref{(3)}{\le}
MH\bigl|\bigl|p_h(\cdot|s,a^E)-\widetilde{p}^M_h(\cdot|s,a^E)
\bigr|\bigr|_1
+\E\limits_{s'\sim p_h(\cdot|s,a^E)}\bigl|V^{\pi^E}_{h+1}(s';p,r)-
V^{\pi^E}_{h+1}(s';\widetilde{p}^M,\widehat{r})
\bigr|+Y_h\\
&\markref{(4)}{\le}
2MHb_h(s,a^E)
+\E\limits_{s'\sim p_h(\cdot|s,a^E)}\bigl|Q^{\pi^E}_{h+1}(s',a^E;p,r)-
Q^{\pi^E}_{h+1}(s',a^E;\widetilde{p}^M,\widehat{r})
\bigr|+Y_h\\
&\markref{(5)}{=}
2MHb_h(s,a^E)
+\E\limits_{s'\sim p_h(\cdot|s,a^E)}\bigl|
\min\limits_{s''\in\S}V^{\pi^m}_{h+2}(s'';p^m,r)-
\min\limits_{s''\in\S}V^{\widetilde{\pi}^m}_{h+2}(s'';\widetilde{p}^m,\widehat{r})
\bigr|+Y_h\\
&\le
2MHb_h(s,a^E)
+\E\limits_{s'\sim p_h(\cdot|s,a^E)}\bigl[Y_{h+1}\bigr]+Y_h\\
&\markref{(6)}{=}
2MHb_h(s,a^E)+Y_h+Y_{h+1},
\end{align*}
where at (1) we use the definition of $r$ in Lemma \ref{lemma: error propagation pirlo superset v2} and triangle inequality, at (2) we use that, for any pair $f,g$ of real-valued functions, it holds that $|\min_x f(x)-\min_x g(x)|\le \max_x|f(x)-g(x)|$, at (3) we apply triangle inequality twice, we recognize the definition of $X_h$, we upper bound the value function by $MH$, and we recognize the definition of $\ell_1$ norm,
at (4) we first use triangle inequality $\|p_h(\cdot|s,a^E)-\widetilde{p}^M_h(\cdot|s,a^E)\|_1
\le \|p_h(\cdot|s,a^E)-\widehat{p}_h(\cdot|s,a^E)\|_1+
\|\widetilde{p}^M_h(\cdot|s,a^E)-\widehat{p}_h(\cdot|s,a^E)\|_1$, then we use Pinsker's inequality, event $\mathcal{E}_3$ from Lemma \ref{lemma: concentration}, and the definition of $b_h(s,a^E)$;
at (5) we use the definition of $\widehat{r}$ in Lemma \ref{lemma: error propagation pirlo superset v2} in the form of Eq. \ref{eq: def hat r equality Q pirlo superset v2}, by noticing that the support of $p_h(\cdot|s,a^E)$
is contained in $\suppspie$, and at (6) we realize that $Y_{h+1}$ depends
only on $h$ and not on $s'$.

In order to upper bound the term $Y_h$, we write:
\begin{align*}
    Y_h&\coloneqq \max\limits_{s'\in\S}\bigl|
    V^{\widetilde{\pi}^m}_{h+1}(s';\widetilde{p}^m,\widehat{r})-V^{\pi^m}_{h+1}(s';p^m,r)
    \bigr|\\
    &\markref{(1)}{=}\max\biggl\{
    \max\limits_{s'\notin \suppspie_{h+1}}\Bigl|
    \max\limits_{a'\in\A} Q^{\widetilde{\pi}^m}_{h+1}(s',a';\widetilde{p}^m,\widehat{r})-
    \max\limits_{a'\in\A} Q^{\pi^m}_{h+1}(s',a';p^m,r)
    \Bigr|,\\
    &\qquad\qquad\max\limits_{s'\in \suppspie_{h+1}}\Bigl|
    Q^{\popblue{\pi^E}}_{h+1}(s',a^E;\widetilde{p}^m,\widehat{r})-
    Q^{\popblue{\pi^E}}_{h+1}(s',a^E;\popblue{p},r)
    \Bigr|
    \biggr\}\\
    &\markref{(2)}{\le}\max\biggl\{
    \max\limits_{s'\notin \suppspie_{h+1}}
\popblue{\max\limits_{a'\in\A}}\Bigl|
     Q^{\widetilde{\pi}^m}_{h+1}(s',a';\widetilde{p}^m,\widehat{r})-
     Q^{\pi^m}_{h+1}(s',a';p^m,r)
    \Bigr|,\\
    &\qquad\qquad\max\limits_{s'\in \suppspie_{h+1}}\Bigl|
    Q^{\pi^E}_{h+1}(s',a^E;\widetilde{p}^m,\widehat{r})-
    Q^{\pi^E}_{h+1}(s',a^E;p,r)\pm Q^{\pi^E}_{h+1}(s',a^E;\widetilde{p}^M,\widehat{r}) 
    \Bigr|
    \biggr\}\\
    &\markref{(3)}{=}\max\biggl\{
    \max\limits_{s'\notin \suppspie_{h+1}}
    \max\Bigl\{
    \max\limits_{a'\in\A: (s',a',h+1)\in \suppsapib}\Bigl|
     \underbrace{Q^{\widetilde{\pi}^m}_{h+1}(s',a';\widetilde{p}^m,\widehat{r})-
     Q^{\pi^m}_{h+1}(s',a';p^m,r)}_{=0}
    \Bigr|,\\
    &\qquad\qquad\max\limits_{a'\in\A: (s',a',h+1)\notin \suppsapib}\Bigl|
     Q^{\widetilde{\pi}^m}_{h+1}(s',a';\widetilde{p}^m,\widehat{r})-
     Q^{\pi^m}_{h+1}(s',a';p^m,r)
    \Bigr|
    \Bigr\},\\
    &\qquad\qquad\max\limits_{s'\in \suppspie_{h+1}}\Bigl|
    Q^{\pi^E}_{h+1}(s',a^E;\widetilde{p}^m,\widehat{r})-Q^{\pi^E}_{h+1}(s',a^E;\widetilde{p}^M,\widehat{r})
    +\min\limits_{s''\in\S}
            V^{\widetilde{\pi}^m}_{h+2}(s'';\widetilde{p}^m,\widehat{r})
            -\min\limits_{s''\in\S}V^{\pi^m}_{h+2}(s';p^m,r)
    \Bigr|
    \biggr\}\\
    &\markref{(4)}{\le}\max\biggl\{
    \max\limits_{s'\notin \suppspie_{h+1}}
    \max\limits_{a'\in\A: (s',a',h+1)\notin \suppsapib}\Bigl|
     \underbrace{\widehat{r}_{h+1}(s',a')-r_{h+1}(s',a')}_{=0}+\min\limits_{s''\in\S}V^{\widetilde{\pi}^m}_{h+2}(s'';\widetilde{p}^m,\widehat{r})-\min\limits_{s''\in\S}
     V^{\pi^m}_{h+2}(s'';p^m,r)
    \Bigr|,\\
    &\qquad\qquad\max\limits_{s'\in \suppspie_{h+1}}\Bigl|
    Q^{\pi^E}_{h+1}(s',a^E;\widetilde{p}^M,\widehat{r})-Q^{\pi^E}_{h+1}(s',a^E;\widetilde{p}^m,\widehat{r})
    \Bigr|
    + \popblue{Y_{h+1}}
    \biggr\}\\
    &\markref{(5)}{\le}\max\biggl\{
    Y_{h+1},\max\limits_{s'\in \suppspie_{h+1}}\Bigl(\Bigl|
    \E\limits_{s''\sim \widetilde{p}^M_{h+1}(\cdot|s',a^E)}V^{\pi^E}_{h+2}(s'';\widetilde{p}^M,\widehat{r})-
    \E\limits_{s''\sim \widetilde{p}^m_{h+1}(\cdot|s',a^E)}V^{\pi^E}_{h+2}(s'';\widetilde{p}^m,\widehat{r})
    \Bigr|
    + Y_{h+1}\Bigr)
    \biggr\}\\
    &=
    Y_{h+1}+\max\limits_{s'\in \suppspie_{h+1}}\Bigl|
    \E\limits_{s''\sim \widetilde{p}^M_{h+1}(\cdot|s',a^E)}
    V^{\pi^E}_{h+2}(s'';\widetilde{p}^M,\widehat{r})-
    \E\limits_{s''\sim \widetilde{p}^m_{h+1}(\cdot|s',a^E)}V^{\pi^E}_{h+2}(s'';\widetilde{p}^m,\widehat{r})
    \pm\E\limits_{s''\sim \widetilde{p}^m_{h+1}(\cdot|s',a^E)}
    V^{\pi^E}_{h+2}(s'';\widetilde{p}^M,\widehat{r}) 
    \Bigr|\\
    &\markref{(6)}{\le}
    Y_{h+1}+\max\limits_{s'\in \suppspie_{h+1}}\biggl(MH\Bigl\|\widetilde{p}^M_{h+1}(\cdot|s',a^E)-\widetilde{p}^m_{h+1}(\cdot|s',a^E)
\Bigr\|_1+
    \E\limits_{s''\sim \widetilde{p}^m_{h+1}(\cdot|s',a^E)}
    \Bigl|V^{\pi^E}_{h+2}(s'';\widetilde{p}^M,\widehat{r})-
    V^{\pi^E}_{h+2}(s'';\widetilde{p}^m,\widehat{r})\Bigr|\biggr)\\
    &\markref{(7)}{\le}
    Y_{h+1}+\max\limits_{s'\in \suppspie_{h+1}}\biggl(2MHb_{h+1}(s',a^E)\\
    &\qquad\qquad+\E\limits_{s''\sim \widetilde{p}^m_{h+1}(\cdot|s',a^E)}
    \Bigl|
    \E\limits_{s'''\sim \widetilde{p}^M_{h+2}(\cdot|s'',a^E)}
    V^{\pi^E}_{h+3}(s''';\widetilde{p}^M,\widehat{r})-
    \E\limits_{s'''\sim \widetilde{p}^m_{h+2}(\cdot|s'',a^E)}
    V^{\pi^E}_{h+3}(s''';\widetilde{p}^m,\widehat{r})\Bigr|\biggr)\\
    &\markref{(8)}{\le}
    Y_{h+1}+2MH\max\limits_{s'\in \suppspie_{h+1}}
    \sum\limits_{h'\in\dsb{h+1,H-1}}
    \E\limits_{s''\sim\rho^{\widetilde{p}^m,\pi^E}_{h'}(\cdot|s_{h+1}=s')}
    b_{h'}(s'',a^E)\\
    &\markref{(9)}{\le}
    2MH\sum\limits_{h'\in\dsb{h,H-1}}
    \max\limits_{s'\in \suppspie_{h'+1}}
    \sum\limits_{h''\in\dsb{h'+1,H-1}}
    \E\limits_{s''\sim\rho^{\widetilde{p}^m,\pi^E}_{h''}(\cdot|s_{h'+1}=s')}
    b_{h''}(s'',a^E)\\
    &\le
    2MH^3
    \max\limits_{(s',h')\in \suppspie}
    b_{h'}(s',a^E),
\end{align*}
where at (1) we apply the Bellman's equation and the definition of $\widetilde{\pi}^m$ and $\pi^m$, at (2) we use that
$|\min_x f(x)-\min_x g(x)|\le \max_x|f(x)-g(x)|$,
at (3) we use the definition of $r$ in the form of Eq.
\ref{eq: def rhat equality Q non-expert v2 superset} and Eq.\ref{eq: def hat r equality Q pirlo superset v2}, at (4) we apply the Bellman's equation and the definition of $r$ to recognize that $r_{h+1}(s',a')-\widehat{r}_{h+1}(s',a')=0$; moreover, we apply triangle inequality along with the usual bound
$|\min_x f(x)-\min_x g(x)|\le \max_x|f(x)-g(x)|$, and we recognize the definition of $X_{h+1}$. At (5) we proceed similarly as (4) and we use the Bellman's equation, and we observe that $Y_{h+1}$ does not depend on
$s'$; at (6) we upper bound the value function by $HM$ and recognize the $\ell_1$-norm, at (7) we use the concentration bound of event $\mathcal{E}_3$ in Lemma \ref{lemma: concentration} (both $\widetilde{p}^M$ and $\widetilde{p}^m$ lie at a ``distance'' of $b$ from $\widehat{p}$). At (8) we have unfolded the recursion to bound the difference of value functions between transition models
$\widetilde{p}^m$ and $\widetilde{p}^M$, at (9) we have unfolded the recursion on the $Y$ terms.

Thanks to this expression, we can upper bound the difference of rewards in expert's action as:
\begin{align*}
    \bigl|\widehat{r}_h(s,a^E)-r_h(s,a^E)\bigr|&\le
    2MHb_h(s,a^E)+4MH^3
    \max\limits_{(s',h')\in \suppspie}
    b_{h'}(s',a^E).
\end{align*}

With regards to visited $(s,a,h)\in\suppsapib\setminus\suppsapie$, we can write:
\begin{align*}
    \bigl|\widehat{r}_h(s,a)-r_h(s,a)\bigr|
    &=
    \bigl|
    \E\limits_{s'\sim \widetilde{p}^m_h(\cdot|s,a)}V^{\widetilde{\pi}^m}_{h+1}(s';\widetilde{p}^m,\widehat{r})
    -\E\limits_{s'\sim p_h(\cdot|s,a)}V^{\pi^m}_{h+1}(s';p^m,r)
    \bigr|\\
    &\le 2MH b_h(s,a)+
    \E\limits_{s'\sim \widetilde{p}^m_h(\cdot|s,a)}\Bigl|V^{\widetilde{\pi}^m}_{h+1}(s';\widetilde{p}^m,\widehat{r})
    -V^{\pi^m}_{h+1}(s';p^m,r)
    \Bigr|\\
    &\le 2MH b_h(s,a)+
    \max\limits_{s'\in\S}\Bigl|V^{\widetilde{\pi}^m}_{h+1}(s';\widetilde{p}^m,\widehat{r})
    -V^{\pi^m}_{h+1}(s';p^m,r)
    \Bigr|\\
    &=2MHb_h(s,a)+Y_h\\
    &\le 2MHb_h(s,a)+2MH^3
    \max\limits_{(s',h')\in \suppspie}
    b_{h'}(s',a^E)\\
    &\le 2MHb_h(s,a)+\popblue{4}MH^3
    \max\limits_{(s',h')\in \suppspie}
    b_{h'}(s',a^E).
\end{align*}
Obviously, for $(s,a,h)\notin \suppsapib$, we have:
\begin{align*}
    \bigl|\widehat{r}_h(s,a)-r_h(s,a)\bigr|&=0.
\end{align*}

Therefore, by Eq. \ref{eq: bound hausdorff temporary v2 super}, we can write:
\begin{align*}
\mathcal{H}_d(\super, \superrelax)&\le
    \sup\limits_{\widetilde{r}\in \superrelax}
    \frac{1}{M}\sum\limits_{h\in\dsb{H}}\E\limits_{(s,a)\sim\rho_h^{p,\pi^b}(\cdot,\cdot)}
            \bigl|r_h(s,a)-\widehat{r}_h(s,a)\bigr|\\
    &\le \sum\limits_{h\in\dsb{H}}\E\limits_{(s,a)\sim\rho_h^{p,\pi^b}(\cdot,\cdot)}
    \biggl(
    2Hb_h(s,a)+4H^3
    \max\limits_{(s',h')\in \suppspie}
    b_{h'}(s',a^E)
    \biggr)\\
    &=
    2H\sum\limits_{h\in\dsb{H}}\E\limits_{(s,a)\sim\rho_h^{p,\pi^b}(\cdot,\cdot)}
    b_h(s,a)+4H^3
    \sum\limits_{h\in\dsb{H}}\E\limits_{(s,a)\sim\rho_h^{p,\pi^b}(\cdot,\cdot)}
    \max\limits_{(s',h')\in \suppspie}
    b_{h'}(s',a^E)\\
    &=
    2H\sum\limits_{h\in\dsb{H}}\E\limits_{(s,a)\sim\rho_h^{p,\pi^b}(\cdot,\cdot)}
    b_h(s,a)+4H^3
    \sum\limits_{h\in\dsb{H}}
    \max\limits_{(s',h')\in \suppspie}
    b_{h'}(s',a^E)\\
    &=
    2H\sum\limits_{h\in\dsb{H}}\E\limits_{(s,a)\sim\rho_h^{p,\pi^b}(\cdot,\cdot)}
    b_h(s,a)+4H^4
    \max\limits_{(s',h')\in \suppspie}
    b_{h'}(s',a^E).
\end{align*}
This concludes the proof w.r.t. the superset.
\end{proof}

\subsubsection{Lemmas for Theorem \ref{theorem: upper bound d infty pirlo H 8}}
\begin{lemma}[Performance Decomposition Subset]\label{lemma: performance decomposition pirlo v2 subset d infty}
    Under good event $\mathcal{E}$, it holds that:
    \begin{align*}
        \mathcal{H}_\infty(\sub, \subrelax)&\le
        2H^2\max\limits_{(s,a,h)\in\suppsapib}   
        b_{h}(s,a)+
        4H^4\max\limits_{(s,h)\in \suppspie}b_{h}(s,a^E).
    \end{align*}
\end{lemma}
\begin{proof}[Proof Sketch]
    The proof is analogous to that of Lemma \ref{lemma: performance decomposition pirlo v2}.
    We can reuse the bounds for the difference between rewards proved in there and insert them into $\mathcal{H}_\infty$ to get the result.
\end{proof}

\begin{lemma}[Performance Decomposition Superset]\label{lemma: performance decomposition pirlo v2 superset d infty}
    Under good event $\mathcal{E}$, it holds that:
    \begin{align*}
        \mathcal{H}_\infty(\sub, \subrelax)&\le
        2H^2\max\limits_{(s,a,h)\in\suppsapib}   
        b_{h}(s,a)+
        4H^4\max\limits_{(s,h)\in \suppspie}b_{h}(s,a^E).
    \end{align*}
\end{lemma}
\begin{proof}[Proof Sketch]
    The proof is analogous to that of Lemma \ref{lemma: performance decomposition pirlo v2 superset}.
    We can reuse the bounds for the difference between rewards proved in there and insert them into $\mathcal{H}_\infty$ to get the result.
\end{proof}

\subsubsection{Proofs of the main theorems}
Thanks to Lemma \ref{lemma: performance decomposition pirlo v2} and Lemma \ref{lemma: performance decomposition pirlo v2 superset},
we can conclude the proof of the main theorem for $d$.
\begin{thr}\label{theorem: upper bound d pirlo H 8}
    Under the conditions of Theorem \ref{theorem: upper bound d pirlo}, \pirlo is $(\epsilon,\delta)$-PAC for $d$-IRL
    with a sample complexity at most:
    \begin{align*}
        \tau^b&\le \widetilde{\mathcal{O}}\Biggl(
            \frac{H^3Z^{p,\pi^b}\ln\frac{1}
        {\delta}}{\epsilon^2}\biggl(
            \ln\frac{1}
        {\delta}+S_{\max}^{p,\pi^b}
        \biggr)\\
        &+\frac{H^8\ln\frac{1}
        {\delta}}{\rho_{\min}^{\pi^b,\suppsapie} \epsilon^2}\biggl(
            \ln\frac{1}
        {\delta}+S_{\max}^{p,\pi^b}
        \biggr)+\frac{\ln\frac{1}{\delta}}{\ln\frac{1}{1-\rhominpib}}
        \Biggr),
    \end{align*}
    and $\tau^E$ is bounded as in Theorem~\ref{theorem: upper bound d irlo}. Furthermore, \pirlo is inclusion monotonic.
\end{thr}
\begin{proof}[Proof Sketch]
The proof for the subset and superset is completely analogous.
Thanks to Lemma \ref{lemma: performance decomposition pirlo v2} and Lemma \ref{lemma: performance decomposition pirlo v2 superset}, we realize that we have to bound the sum of two terms, which are completely analogous to those in the proof of Theorem \ref{theorem: upper bound d pirlo}, with the only difference of $H^4$ instead of $H^3$. By proceeding similarly, we get the result.
\end{proof}
Thanks to Lemma \ref{lemma: performance decomposition pirlo v2 subset d infty} and Lemma \ref{lemma: performance decomposition pirlo v2 superset d infty},
we can conclude the proof of the main theorem for $d_\infty$.
\begin{thr}\label{theorem: upper bound d infty pirlo H 8}
    Under the conditions of Theorem \ref{theorem: upper bound d pirlo},
    \pirlo is $(\epsilon,\delta)$-PAC for $d_\infty$-IRL
    with a sample complexity at most:
    \begin{equation*}
    \begin{aligned}
        \tau^b&\le  \widetilde{\mathcal{O}}\Bigg(\frac{H^4\ln\frac{1}
        {\delta}}{\rhominpib \epsilon^2}\biggl(
            \ln\frac{1}
        {\delta}+S^{p,\pi^b}_{\max}
        \biggr) \\
        & 
           + \frac{H^8\ln\frac{1}
        {\delta}}{\rho_{\min}^{\pi^b,\suppsapie }\epsilon^2}\biggl(
            \ln\frac{1}
        {\delta}+S^{p,\pi^b}_{\max}
        \biggr) + \frac{\ln\frac{1}{\delta}}{\ln\frac{1}{1-\rhominpib}}\Bigg),
    \end{aligned}
    \end{equation*}
   and $\tau^E$ is bounded as in Theorem~\ref{theorem: upper bound d irlo}.
   Furthermore, \pirlo is inclusion monotonic.
\end{thr}
\begin{proof}[Proof Sketch]
    The proof for the subset and superset is completely analogous.
Thanks to Lemma \ref{lemma: performance decomposition pirlo v2 subset d infty} and Lemma \ref{lemma: performance decomposition pirlo v2 superset d infty}, we realize that we have to bound the sum of two terms, which are completely analogous to those in the proof of Theorem \ref{theorem: upper bound d infty pirlo}, with the only difference of $H^4$ instead of $H^3$. By proceeding similarly, we get the result.
\end{proof}

\subsection{A note on the superset without relaxation}\label{subsec: superset no relaxation bound}
In this section, we show that, if we use the superset definition $\widehat{\R}^\cup$ of Eq. \ref{def: sub and super no relaxation pessimism}, i.e., the definition without relaxation, then we are able to obtain the same performance decomposition result (see Lemma \ref{lemma: performance decomposition irlo}) that we had for the case without pessimism, and, thus, we end up with the same sample complexity result, which is much smaller than those computed for the relaxations. Observe that we are not able to have an analogous result for the subset definition $\widehat{\R}^\cap$ of Eq. \ref{def: sub and super no relaxation pessimism}.
Indeed, differently from the relaxations defined in Eq. \ref{def: relaxation superset}, the subset and the superset definitions of Eq. \ref{def: sub and super no relaxation pessimism} are not exactly simmetric.
While $r\in\widehat{\R}^\cup$ entails, by definition, the existence of (at least) one transition model in $\mathcal{C}(\widehat{p},b)$ in which $r$ induces an optimal policy $\pi^*\in\eqclasspie{\pi^E}$, this is not true for $r'\in\widehat{\R}^\cap$.
Indeed, potentially, there might exist a (worst)\footnote{Worst because there is a universal quantifier $\forall$ over transition models in the definition of $\widehat{\R}^\cap$.} transition model for every $(s,a,h)\in\SAH$. Therefore, intuitively, in the reward choice lemma, we cannot make a choice of a single (worst) transition model, but we have to choose many of them. This fact ``breaks'' the recursion and it does not allow us to perform $\ell_1$-norm bounds.
Instead, since the superset is of different nature, we can.
It should be remarked that a membership checker algorithm for superset $\widehat{\R}^\cup$ is inefficient to implement in practice because it requires to solve a bilinear optimization problem (see Appendix \ref{section: implementation appendix}).

We will denote by $\widehat{\R}^\cup$ the superset definition of Eq. \ref{def: sub and super no relaxation pessimism}.
\begin{lemma}[Reward Choice]\label{lemma: error propagation pirlo superset no relaxation}
    Under good event $\mathcal{E}$,
    for any $\widehat{r}\in \widehat{\R}^\cup$,
    the reward $r$ constructed as:
    \begin{align*}
        \begin{cases}
            r_h(s,a)=\widehat{r}_h(s,a)+\sum\limits_{s'\in\S}
            (\widecheck{p}_h(s'|s,a)-p_h(s'|s,a))V^*_{h+1}(s';\widecheck{p},\widehat{r}),
            \quad\forall (s,a,h)\in \suppsapib\\
            r_h(s,a)=\widehat{r}_h(s,a),
            \quad\forall (s,a,h)\notin \suppsapib\\
        \end{cases},
    \end{align*}
    where $\widecheck{p}$ is some transition model in $\mathcal{C}(\widehat{p},b)$,
    belongs to $\super$.
\end{lemma}
\begin{proof}
    By definition of $\widehat{\R}^\cup$,
    we have that
    $\widehat{r}\in \widehat{\R}^\cup$ if and only if there exists a transition model $\bar{p}\in \mathcal{C}(\widehat{p},b)$
    such that $\widehat{r}\in \R^\cup_{\bar{p},\pi^E}$. Let us construct
    $r$ by choosing $\widecheck{p}=\bar{p}$. To show that $r\in \super$,
    we are going to show that the transition model $\widetilde{p}$
    defined as:
    \begin{align*}
        \begin{cases}
            \widetilde{p}_h(\cdot|s,a)=p_h(\cdot|s,a),\quad \forall (s,a,h)\in \suppsapib\\
            \widetilde{p}_h(\cdot|s,a)=\bar{p}_h(\cdot|s,a),\quad \forall (s,a,h)\notin \suppsapib\\
        \end{cases},
    \end{align*}
    belongs to $\eqclassp{p}$ and is such that, for all $(s,h)\in \suppspie$, for all $a\in\A\setminus\{a^E\}$:
    \begin{align*}
        Q^*_h(s,a;\widetilde{p},r)\le Q^*_h(s,a^E;\widetilde{p},r).
    \end{align*}
    Then, by Lemma \ref{lemma: alternative representation new fs Q star}, we can conclude that $r\in\super$.
    
    Trivially, notice that $\widetilde{p}\eqp p$.
    We proceed by induction to show that, for all $(s,a,h)\in\SAH$, the following identity holds:
    \begin{align*}
        Q^*_h(s,a;\bar{p},\widehat{r})= Q^*_h(s,a;\widetilde{p},r).
    \end{align*}
    Then, since $\widehat{r}\in \widehat{\R}^\cup$, for all $(s,h)\in \suppspie$ and for all $a\in\A\setminus\{a^E\}$,
    the inequality
    $Q^*_h(s,a;\bar{p},\widehat{r})\le Q^*_h(s,a^E;\bar{p},\widehat{r})$ (Lemma \ref{lemma: alternative representation new fs Q star}) entails
    $Q^*_h(s,a;\widetilde{p},r)\le Q^*_h(s,a^E;\widetilde{p},r)$, and the thesis follows.

    As case base, consider stage $H$. For any $(s,a)\in\SA$, thanks to the definition of $r$,
    we can write:
    \begin{align*}
        Q^*_H(s,a;\widetilde{p},r)&=r_H(s,a)\\
        &=\widehat{r}_H(s,a)\\
        &=Q^*_H(s,a;\bar{p},\widehat{r}).
    \end{align*}
    Now, make the inductive hypothesis that for all $(s',a')\in\SA$, it holds that
    $Q^*_{h+1}(s',a';\widetilde{p},r)=Q^*_{h+1}(s',a';\bar{p},\widehat{r})$, and consider stage $h$.
    For any $(s,a)\in \suppsapib_h$, we can write:
    \begin{align*}
        Q^*_h(s,a;\widetilde{p},r)&\markref{(1)}{=}
        r_h(s,a)+\sum\limits_{s'\in\S}\widetilde{p}_h(s'|s,a)\max\limits_{a'\in\A}Q^*_{h+1}(s',a';\widetilde{p},r)\\
        &\markref{(2)}{=}
        r_h(s,a)+\sum\limits_{s'\in\S}\widetilde{p}_h(s'|s,a)\max\limits_{a'\in\A}Q^*_{h+1}(s',a';\popblue{\bar{p},\widehat{r}})\\
        &\markref{(3)}{=}
        \widehat{r}_h(s,a)+\sum\limits_{s'\in\S}
        (\widecheck{p}_h(s'|s,a)-p_h(s'|s,a))
        \max\limits_{a'\in\A}Q^*_{h+1}(s',a';\widecheck{p},\widehat{r})
        +\sum\limits_{s'\in\S}\widetilde{p}_h(s'|s,a)\max\limits_{a'\in\A}Q^*_{h+1}(s',a';\bar{p},\widehat{r})\\
        &\markref{(4)}{=}
        \widehat{r}_h(s,a)+\sum\limits_{s'\in\S}
        \popblue{\bar{p}}_h(s'|s,a)\max\limits_{a'\in\A}Q^*_{h+1}(s',a';\bar{p},\widehat{r})\\
        &\markref{(5)}{=}
        Q^*_h(s,a;\bar{p},\widehat{r}),
    \end{align*}
    where at (1) we have applied the Bellman's optimality equation, at (2) we have used the inductive hypothesis,
    at (3) we have inserted the definition of $r_h(s,a)$ along with the fact that $(s,a,h)\in \suppsapib$,
    at (4) we have noticed that $\widecheck{p}=\bar{p}$ (by choice) and that $\widetilde{p}_h(\cdot|s,a)=p_h(\cdot|s,a)$
    (by definition); finally, at (5), we have applied again the Bellman's optimality equation.

    On the other side, for any $(s,a)\notin \suppsapib_h$, we can write:
    \begin{align*}
        Q^*_h(s,a;\widetilde{p},r)&=
        r_h(s,a)+\sum\limits_{s'\in\S}\widetilde{p}_h(s'|s,a)\max\limits_{a'\in\A}Q^*_{h+1}(s',a';\widetilde{p},r)\\
        &=
        r_h(s,a)+\sum\limits_{s'\in\S}\widetilde{p}_h(s'|s,a)\max\limits_{a'\in\A}Q^*_{h+1}(s',a';\popblue{\bar{p},\widehat{r}})\\
        &\markref{(1)}{=}
        \popblue{\widehat{r}}_h(s,a)+\sum\limits_{s'\in\S}\popblue{\bar{p}}_h(s'|s,a)\max\limits_{a'\in\A}Q^*_{h+1}(s',a';\bar{p},\widehat{r})\\
        &=
        Q^*_h(s,a;\bar{p},\widehat{r}),
    \end{align*}
    where at (1) we have used both the definition of $r_h(s,a)$
    and that $\widetilde{p}_h(\cdot|s,a)=\bar{p}_h(\cdot|s,a)$, for $(s,a,h)\notin \suppsapib$.

    This concludes the proof.
\end{proof}
Thanks to the reward choice lemma just presented, we obtain the following sample complexity result.
\begin{thr}
    Let $\M$ be an MDP without reward and let
    $\pi^E$ be the expert's policy. Let
    $\D^E$ and $\D^b$ be two datasets of
    $\tau^E$ and $\tau^b$ trajectories collected
    with policies $\pi^E$ and $\pi^b$ in $\M$, respectively.
    Under Assumption~\ref{assumption: coverage of behavioral policy}, any algorithm $\mathfrak{A}$ that outputs $\widehat{\R}^\cup$ (defined as in Eq. \ref{def: sub and super no relaxation pessimism}) is such that, for any $\epsilon,\delta\in (0,1)$:
    \begin{align*}
        \mathop{\mathbb{P}}_{(p,\pi^E,\pi^b)}\bigl(\bigl\{\mathcal{H}_c(\super,\widehat{\R}^\cup)\le \epsilon\bigr\}
        \wedge \{\super\subseteq\widehat{\R}^\cup\}
        \bigr)\ge 1-\delta,
    \end{align*}
    with a sample complexity at most:
    \begin{align*}
        & \tau^b\le \widetilde{\mathcal{O}}\Bigg(
            \frac{H^3Z^{p,\pi^b}\ln\frac{1}
        {\delta}}{\epsilon^2}\biggl(
            \ln\frac{1}
        {\delta}+S^{p,\pi^b}_{\max}
        \biggr)+\frac{\ln\frac{1}{\delta}}{\ln\frac{1}{1-\rhominpib}}
        \Bigg), \\
        & \tau^E \le \widetilde{\mathcal{O}} \Bigg( \frac{\ln\frac{1}{\delta}}{\ln\frac{1}{1-\rhominpie}} \Bigg).
    \end{align*}
    if $c=d$, and a sample complexity at most:
    \begin{align*}
        & \tau^b\le \widetilde{\mathcal{O}}\Bigg(
            \frac{H^4\ln\frac{1}
        {\delta}}{\rhominpib\epsilon^2}\biggl(
            \ln\frac{1}
        {\delta}+S^{p,\pi^b}_{\max}
        \biggr)+\frac{\ln\frac{1}{\delta}}{\ln\frac{1}{1-\rhominpib}}
        \Bigg), \\
        & \tau^E \le \widetilde{\mathcal{O}} \Bigg( \frac{\ln\frac{1}{\delta}}{\ln\frac{1}{1-\rhominpie}} \Bigg).
    \end{align*}
    if $c=d_\infty$.
\end{thr}
\begin{proof}[Proof Sketch]
    Observe that, thanks to Lemma \ref{lemma: error propagation pirlo superset no relaxation}, we are able to obtain a performance decomposition lemma analogous to Lemma \ref{lemma: performance decomposition irlo} for distance $d$ (or analogous for distance $d_\infty$). Next, following the steps in the proof of Theorem \ref{theorem: upper bound d irlo} (Theorem \ref{theorem: upper bound d infty irlo}), we obtain the result.
\end{proof}

\section{Implementation}\label{section: implementation appendix}
In this appendix, we provide some comments on the implementation of membership checker algorithms for \irlo, \pirlo, and some comments about the subset and superset defined in Eq. \ref{def: sub and super no relaxation pessimism}. In Section \ref{subsec: implementation appendix algorithm}, we present the pseudocode of the membership checker algorithms. In Section
\ref{subsec: considerations irlo pirlo}, we provide more details on the definitions of the relaxations $\subrelax$ and $\superrelax$.
In Section \ref{subsec: issues no relaxation implementation} we show that the implementation of a membership checker algorithm for the subset and superset defined in Eq. \ref{def: sub and super no relaxation pessimism} is inefficient. Finally, in Section \ref{subsec: relaxations with Q start}, we give the intuition that a straightforward relaxation of the representation of the sets provided by Lemma \ref{lemma: alternative representation new fs Q star} is worse than that obtained by relaxing the representation in Theorem \ref{theorem: alternative representation new fs}.

\subsection{Algorithm}\label{subsec: implementation appendix algorithm}
The pseudocode of the membership checker algorithms for \irlo and \pirlo is provided in Algorithm \ref{alg: CheckMembirlo}.
\clearpage
  \RestyleAlgo{ruled}
\SetInd{0.5em}{0.5em}
\LinesNumberedHidden{
 \begin{algorithm}[t]
    \caption{Membership checker for \irlo and \pirlo.} \label{alg: CheckMembirlo}
    \small
    \SetKwInOut{Input}{Input}
    \SetKwInOut{Output}{Output}
     \Input{Datasets $\D^E=\{\langle s_h^{E,i},a_h^{E,i} \rangle_{h}\}_{i}$, $\D^b=\{\langle s_h^{b,i},a_h^{b,i} \rangle_{h}\}_{i}$, candidate reward function $r \in \mathfrak{R}$}
    
    \Output{$\mathsf{True}$ if $r \in (\widehat{\R}^\cup,\widehat{\R}^\cap)$}

Run lines~\ref{line:1}-\ref{line:endFor} of Algorithm~\ref{alg:irlo}

Define $\mathcal{C}$ as in Eq.~\eqref{eq:C1} for \irlo and as in Eq.~\eqref{eq:C2} for \pirlo

    Extended value iteration:\\
    {\thinmuskip=1mu
\medmuskip=1mu
\thickmuskip=1mu$\mathcal{A}_h(s) \leftarrow \text{\textbf{ if }} (s,h) \in \widehat{\mathcal{S}}^{p,\pi^E} \text{\textbf{ then }}
        \{\widehat{\pi}^E_h(s)\}    \text{ \textbf{else} }
        \mathcal{A},\; \forall (s,h) \in \S \times \dsb{H}$}
    
     $Q_H^{+}(s,a), Q_H^{-}(s,a) \leftarrow r_H(s,a)$  $\forall (s,a,h) \in \S \times \A \times \dsb{H}$

     \For{$h=H-1$ {\textnormal{\textbf{to}}} $1$}{\label{line:45}
    \For{$(s,a)\in\SA$}{
\thinmuskip=1mu
\medmuskip=1mu
\thickmuskip=1mu
     $Q_h^+(s,a)\leftarrow r_h(s,a)+\max\limits_{p' \in \mathcal{C}}\sum\limits_{s'\in\S}{p}'_h(s'|s,a)\max\limits_{a'\in\A_{h+1}(s)}Q_{h+1}^+(s',a')$\label{line:max}

     $Q_h^-(s,a) \leftarrow r_h(s,a)+\min\limits_{p' \in \mathcal{C}} \sum\limits_{s'\in\S}{p}'_h(s'|s,a)\max\limits_{a'\in\A_{h+1}(s)}Q_{h+1}^-(s',a')$\label{line:min}

    }
    }\label{line:46}

Membership test:\\
    in$^{\cup}\leftarrow \mathsf{True}$, in$^{\cap} \leftarrow \mathsf{True}$

     \For{$(s,h)\in \estsuppspie$}{
    \For{$a\in\A\setminus\{\widehat{\pi}^E(s)\}$}{
    \begin{tcolorbox}[enhanced, attach boxed title to top right={yshift=-4.5mm,yshifttext=-1mm},colframe=vibrantOrange,colbacktitle=vibrantOrange,colback=white,
  title=\texttt{IRLO},fonttitle=\bfseries,
  boxed title style={size=small, sharp corners}, sharp corners ,boxsep=-1.5mm, width=7.5cm]\small
        \If{$Q_h^+(s,\widehat{\pi}_h^E(s))<Q_h^+(s,a)$}{
         in$^{\cup}\leftarrow \mathsf{False}$
        }
        \If{$Q_h^-(s,\widehat{\pi}_h^E(s))<Q_h^-(s,a)$}{
         in$^{\cap}\leftarrow \mathsf{False}$
        }
        \end{tcolorbox}\label{line:55}
        \vspace{-.35cm}
        \begin{tcolorbox}[enhanced, attach boxed title to top right={yshift=-4.5mm,yshifttext=-1mm},colframe=vibrantTeal,colbacktitle=vibrantTeal,colback=white,
  title=\texttt{PIRLO},fonttitle=\bfseries,
  boxed title style={size=small, sharp corners}, sharp corners ,boxsep=-1.5mm, width=7.5cm]\small
  \uIf{$Q_h^+(s,\widehat{\pi}_h^E(s))<Q_h^-(s,a)$}{
         in$^{\cup}\leftarrow \mathsf{False}$
        }
        \ElseIf{$Q_h^-(s,\widehat{\pi}_h^E(s))<Q_h^+(s,a)$}{
         in$^{\cap}\leftarrow \mathsf{False}$
        }
        \end{tcolorbox}
        \vspace{-.25cm}
    }} \textbf{return} (in$^{\cup}$, in$^{\cap}$)\label{line:55}
 \end{algorithm}
}
 
The idea is to find the worst (resp. best) transition model for the subset (resp. superset) among all those feasible. In practice, what we do is to exploit the representation provided in Eq. \ref{eq: representation sub super with pm pM} for \irlo and the representation provided in Eq. \ref{eq: representation relaxations with ptildem ptildeM} for \pirlo.

Observe that, inside the support $(s,h)\in\estsuppspie$, we use the estimated expert's action $\widehat{\pi}^E_h(s)$, while outside the support we always play the action that maximizes the Q-function (both $Q^+$ and $Q^-$). Concerning the transition model, notice that, for $Q^+$, we consider the $p' \in \mathcal{C}$ that maximizes the expected $Q^+$, while for $Q^-$ we consider the $p' \in \mathcal{C}$ that minimizes the expected $Q^-$. Observe also that, for \irlo, because of the definition of $\mathcal{C}$, inside the support $\estsuppsapib$ we use $p'=\widehat{p}$, and outside the support we consider $p'=\argmax_{s' \in \mathcal{S}}$ for $Q^+$ and $p'=\argmin_{s' \in \mathcal{S}}$ for $Q^-$.


Finally, we check the Bellman optimality conditions to assess the membership of the candidate reward $r$ in the estimated sets $\widehat{\mathcal{R}}^\cup$ (boolean variable in$^{\cup}$) and $\widehat{\mathcal{R}}^\cap$ (boolean variable in$^{\cap}$) (line~\ref{line:55}).

\subsection{A better understanding of the relaxations}\label{subsec: considerations irlo pirlo}
To get a better understanding of why $\subrelax\subseteq\sub$ and $\super\subseteq\superrelax$, under the hypothesis (good event) that the true transition model $p\in\mathcal{C}(\widehat{p},b)$ and that $\widehat{\pi}^E=\pi^E$ in all $\suppspie$, observe that, for the subset:
\begin{align*}
    \sub&\supseteq\bigcap\limits_{p'\in \mathcal{C}(\widehat{p},b)}
    \R_{p',\pi^E}^\cap\\
    &= \{r\in\mathfrak{R}\,|\,\forall p'\in \mathcal{C}(\widehat{p},b),\forall\Bar{\pi}\in\eqclasspie{\pi^E},\forall(s,h)\in\suppspie,\forall a\in\A:\;
    Q^{\pi^E}_h(s,\pi_h^E(s);p',r)\ge Q^{\Bar{\pi}}_h(s,a;p',r)
    \}\\
    &\markref{(1)}{\supseteq}
    \{r\in\mathfrak{R}\,|\,\forall \Bar{\pi}\in\eqclasspie{\pi^E},\forall(s,h)\in\suppspie,\forall a\in\A:\;
    \min_{p'\in\mathcal{C}(\widehat{p},b)}Q^{\pi^E}_h(s,\pi_h^E(s);p',r)\ge \max_{p'\in\mathcal{C}(\widehat{p},b)}Q^{\Bar{\pi}}_h(s,a;p',r)
    \}\\
    &=
    \{r\in\mathfrak{R}\,|\,\forall(s,h)\in\suppspie,\forall a\in\A:\;
    \min_{\Bar{\pi}\in\eqclasspie{\pi^E}}\Bigl(\min_{p'\in\mathcal{C}(\widehat{p},b)}Q^{\pi^E}_h(s,\pi_h^E(s);p',r)-\max_{p'\in\mathcal{C}(\widehat{p},b)}Q^{\Bar{\pi}}_h(s,a;p',r)\Bigr)\ge 0
    \}\\
    &\markref{(2)}{=}
    \{r\in\mathfrak{R}\,|\,\forall(s,h)\in\suppspie,\forall a\in\A:\;
    \min_{p'\in\mathcal{C}(\widehat{p},b)}Q^{\pi^E}_h(s,\pi_h^E(s);p',r)\ge \max_{\Bar{\pi}\in\eqclasspie{\pi^E}}\max_{p'\in\mathcal{C}(\widehat{p},b)}Q^{\Bar{\pi}}_h(s,a;p',r)
    \}\\
    &\eqqcolon \subrelax,
\end{align*}
where at (1) we have exchanged the order of the quantifiers, and we can do so because they all are of the same type, and then we have observed that $\min_x(f(x)-g(x))\ge \min_x f(x)-\max_x g(x)$, and at (2) we recognize that the first term does not depend on $\Bar{\pi}$.
W.r.t. the superset, in an analogous manner, observe that:
\begin{align*}
    \super&\subseteq\bigcup\limits_{p'\in \mathcal{C}(\widehat{p},b)}
    \R_{p',\pi^E}^\cup\\
        &= \{r\in\mathfrak{R}\,|\,\exists p'\in \mathcal{C}(\widehat{p},b),\forall\Bar{\pi}\in\eqclasspie{\pi^E},\forall(s,h)\in\suppspie,\forall a\in\A:\;
    Q^{\pi^E}_h(s,\pi_h^E(s);p',r)\ge Q^{\Bar{\pi}}_h(s,a;p',r)
    \}\\
    &\markref{(1)}{\subseteq}
    \{r\in\mathfrak{R}\,|\,\forall \Bar{\pi}\in\eqclasspie{\pi^E},\forall(s,h)\in\suppspie,\forall a\in\A:\;
    \max_{p'\in\mathcal{C}(\widehat{p},b)}Q^{\pi^E}_h(s,\pi_h^E(s);p',r)\ge \min_{p'\in\mathcal{C}(\widehat{p},b)}Q^{\Bar{\pi}}_h(s,a;p',r)
    \}\\
    &=
    \{r\in\mathfrak{R}\,|\,\forall(s,h)\in\suppspie,\forall a\in\A:\;
    \min_{\Bar{\pi}\in\eqclasspie{\pi^E}}\Bigl(\max_{p'\in\mathcal{C}(\widehat{p},b)}Q^{\pi^E}_h(s,\pi_h^E(s);p',r)-\min_{p'\in\mathcal{C}(\widehat{p},b)}Q^{\Bar{\pi}}_h(s,a;p',r)\Bigr)\ge 0
    \}\\
    &\markref{(2)}{=}
    \{r\in\mathfrak{R}\,|\,\forall(s,h)\in\suppspie,\forall a\in\A:\;
    \max_{p'\in\mathcal{C}(\widehat{p},b)}Q^{\pi^E}_h(s,\pi_h^E(s);p',r)\ge \max_{\Bar{\pi}\in\eqclasspie{\pi^E}}\min_{p'\in\mathcal{C}(\widehat{p},b)}Q^{\Bar{\pi}}_h(s,a;p',r)
    \}\\
    &\eqqcolon \superrelax,
\end{align*}
where at (1) we have relaxed and at (2) we recognize that the first term does not depend on $\Bar{\pi}$.

\subsection{Testing the membership without relaxations}\label{subsec: issues no relaxation implementation}
We show that the problem of testing the membership to the
subset and superset defined as in Eq. \ref{def: sub and super no relaxation pessimism}
is equivalent to solving a bilinear optimization problem, which is in general hard.
We will denote the expert's action by $a^E$ and we use the representation provided by Lemma \ref{lemma: alternative representation new fs Q star}.

Let us begin with the subset $\widehat{\R}^\cap$.
A given reward $r$ belongs to $\widehat{\R}^\cap$ if and only if:
\begin{align*}
    \forall p'\in \mathcal{C}(\widehat{p},b),\forall (s,h)\in \estsuppspie,
    \forall a\in \A\setminus\{a^E\}:\;
    Q^*_h(s,a;p',r)\le Q^*_h(s,a^E;p',r),
\end{align*}
where $\mathcal{C}(\widehat{p},b)$ is defined in Eq. \ref{eq:C2}.
By applying the Bellman optimality equation, changing the order of the quantifiers,
and considering the worst possible transition model, we obtain:
\begin{align*}
    \forall (s,h)\in \estsuppspie,&
    \forall a\in \A\setminus\{a^E\}:\\
    &r_h(s,a)\le r_h(s,a^E)+\min\limits_{p'\in \mathcal{C}(\widehat{p},b)}
    \biggl(\sum\limits_{s'\in\S}p_h'(s'|s,a^E)V^*_{h+1}(s';p',r)
    -\sum\limits_{s'\in\S}p_h'(s'|s,a)V^*_{h+1}(s';p',r)
    \biggr).
\end{align*}
It should be remarked that, differently from \irlo and \pirlo,
to check whether a given reward $r$ belongs to $\widehat{\R}^\cap$,
we cannot optimize the value function, but we have to optimize the advantage function.
Therefore, for all $(\bar{s},\bar{h})\in \estsuppspie$, and for all $\bar{a}\in \A\setminus\{a^E\}$,
we have to solve the optimization problem:
\begin{align*}
    &\min\limits_{p'}
    \sum\limits_{s'\in \S}\Bigl(p_{\bar{h}}'(s'|\bar{s},a^E)-p_{\bar{h}}'(s'|\bar{s},\bar{a})\Bigr)
    V^*_{\bar{h}+1}(s';p',r)\\
    &\text{s.t. }
    \|p_{h}'(\cdot|s,a)-\widehat{p}_{h}(\cdot|s,a)\|_1
    \le b_{h}(s,a)
    \quad\quad
    \forall (s,a,h)\in \estsuppsapib
    \\
    &\quad\quad\; p_h'(\cdot|s,a)\in\Delta^\S
    \quad\quad\forall (s,a,h)\in \SAH
    \\
    &\quad\quad\; p_h'(s'|s,a)=0
    \quad\quad\forall (s,a,h)\in\widehat{\Z}^{p,\pi^E} \wedge s'\notin \big(\widehat{\S}^{p,\pi^E}_{h+1}\cup \text{supp }\widehat{p}_h(\cdot|s,a^E)\big).
\end{align*}
Observe that, while the set of constraints define a convex set, the objective involves the product of optimization variables. We can introduce $H$ variables $\{V_s\}_{s\in\dsb{S}}$ to replace $V^*_{\bar{h}+1}(s';p',r)$ by adding suitable constraints that keep into account also the presence of the maximum operator inside $V^*$ (because of the Bellman's optimality equation).
In this way, due to the product between variables $p'$ and $V$, we can conclude that this is a bilinear optimization problem, which is in general difficult to solve.

W.r.t. the superset, we have that a reward $r$ belongs to $\widehat{\R}^\cup$ if and only if:
\begin{align*}
    \exists p'\in \mathcal{C}(\widehat{p},b),\forall (s,h)\in \estsuppspie,
    \forall a\in \A\setminus\{a^E\}:\;
    Q^*_h(s,a;p',r)\le Q^*_h(s,a^E;p',r).
\end{align*}
This time, we cannot bring the transition model inside because we have different quantifiers.
We can formulate the problem as a feasibility problem by adding constraints because of the presence of $\forall (s,h)\in \estsuppspie,\forall a\in \A\setminus\{a^E\}$. In practice, the presence of the product between ``optimization'' variables is now in the constraints, so the problem is again a bilinear problem.

\subsection{Relaxing the representation provided by Lemma \ref{lemma: alternative representation new fs Q star}}\label{subsec: relaxations with Q start}
We have seen that we can represent the feasible set by using Theorem \ref{theorem: alternative representation new fs} or Lemma \ref{lemma: alternative representation new fs Q star}. While the two representations are equivalent, observe that a straightforward relaxation of the constraints present in Lemma \ref{lemma: alternative representation new fs Q star} provides a different relaxation of the subset and superset w.r.t. $\subrelax$ and $\superrelax$ (which are obtained by relaxing the representation in Theorem \ref{theorem: alternative representation new fs}).
Indeed, by relaxing the representation with the $Q^*$ (Lemma \ref{lemma: alternative representation new fs Q star}), we would obtain constraints of the form (ex. subset):
\begin{align}\label{eq: ciao}
    &\min_{p'\in\mathcal{C}(\widehat{p},b)}Q^*_h(s,\pi_h^E(s);p',r)\ge \max_{p'\in\mathcal{C}(\widehat{p},b)}Q^*_h(s,a;p',r)\notag\\
    &\iff\quad \min_{p'\in\mathcal{C}(\widehat{p},b)}\max\limits_{\pi\in\Pi}Q^\pi_h(s,\pi_h^E(s);p',r)\ge \max_{p'\in\mathcal{C}(\widehat{p},b)}\max\limits_{\pi\in\Pi}Q^\pi_h(s,a;p',r).
\end{align}
Clearly, this is different from $\subrelax$, whose constraints can be written as:
\[
\min_{p'\in\mathcal{C}(\widehat{p},b)}Q^{\popblue{\pi^E}}_h(s,\pi_h^E(s);p',r)\ge \max_{p'\in\mathcal{C}(\widehat{p},b)}\max\limits_{\pi\in\popblue{\eqclasspie{\pi^E}}}Q^\pi_h(s,a;p',r).
\]
Indeed, $\subrelax$ puts the additional constraint that the $Q^*$ is achieved by a policy in $\eqclasspie{\pi^E}$, which is not present in Eq. \ref{eq: ciao}.

An analogous reasoning can be carried out also for the superset.

\section{Proofs of Section \ref{section: bitter lesson}}\label{section: appendix bitter lesson}
In this section, we provide the missing proofs of Section \ref{section: bitter lesson}.

\greedyrewards*
\begin{proof}
    Let $r$ be an arbitrary reward function of $\fs^\cap$.
    Consider a certain $(s,h)\in \suppspie$, with expert's action $\pi^E_h(s)=a^E$, and
    let $a\in\A$ be a non-expert's action.
    By Lemma \ref{lemma: alternative representation new fs Q star},
    we know that, for any $p'\in [p]_{\equiv_{\suppsapie}}$, it must hold:
    \begin{align*}
        r_h(s,a)&\le r_h(s,a^E)+\E\limits_{s'\sim p_h'(\cdot|s,a^E)}V_{h+1}^*(s';p',r)-
        \E\limits_{s'\sim p_h'(\cdot|s,a)}V_{h+1}^*(s';p',r)\\
        &=r_h(s,a^E)+\E\limits_{s'\sim p_h(\cdot|s,a^E)}V_{h+1}^*(s';p,r)-
        \E\limits_{s'\sim p_h'(\cdot|s,a)}V_{h+1}^*(s';p',r),
    \end{align*}
    where we have used the definition of $[p]_{\equiv_{\suppsapie}}$.
    Since $(s,a,h)\notin \suppsapie$, then the constraint must 
    hold $\forall p_h'(\cdot|s,a)\in\Delta^\S$. In particular,
    it must hold for the transition model such that:
    \begin{align*}
        r_h(s,a)\le r_h(s,a^E)+\underbrace{\E\limits_{s'\sim p_h(\cdot|s,a^E)}V_{h+1}^*(s';p,r)-
        \max\limits_{s'\in\S}V_{h+1}^*(s';p',r)}_{\le 0},
    \end{align*}
    from which the thesis follows.
\end{proof}

\nongreedyrewardsbehavioral*
\begin{proof}
    Let $a^E\coloneqq\pi^E_h(s)$.
    Similarly to the proof of Proposition \ref{theorem: greedy rewards},
    we can write:
    $p_h(\cdot|s,a)$:
    \begin{align*}
        r_h(s,a)&\le r_h(s,a^E)+\E\limits_{s'\sim p_h'(\cdot|s,a^E)}V_{h+1}^*(s';p',r)-
        \E\limits_{s'\sim p_h'(\cdot|s,a)}V_{h+1}^*(s';p',r)\\
        &=r_h(s,a^E)+\E\limits_{s'\sim p_h(\cdot|s,a^E)}V_{h+1}^*(s';p,r)-
        \E\limits_{s'\sim p_h(\cdot|s,a)}V_{h+1}^*(s';p',r),
    \end{align*}
    where we have used that we have access to samples about $p_h(\cdot|s,a)$.
    By hypothesis, $p_h(\cdot|s,a)\neq p_h(\cdot|s,a^E)$, therefore,
    by taking $r$ such that
    $\E_{s'\sim p_h(\cdot|s,a^E)}V_{h+1}^*(s';p,r)>
    \E_{s'\sim p_h(\cdot|s,a)}V_{h+1}^*(s';p',r)$, we can obtain a reward $r$ in $\fs^\cap$
    such that
    $r_h(s,\pi^E_h(s))< r_h(s,a)$.

\end{proof}

\section{A relaxed triangle inequality}\label{section: semimetric}
In this section, we show that both our notions of distance $d,d_\infty$, defined in Section \ref{section: pac framework}, are semimetrics, and that they satisfy a $\rho$-relaxed triangle inequality \citep[see][]{fagin1998relaxing} with finite $\rho >1$ for any pair of rewards $r,r'\in\mathfrak{R}$.
Furthermore, we show that the Hausdorff distance $\mathcal{H}$, when applied to the sets of rewards considered in this work, inherits the relaxed triangle inequality property. It should be remarked that we need the $\rho$-relaxed triangle inequality property with finite $\rho$ just for the \emph{learnability} proofs of Appendix \ref{section: learnability fs}. Moreover, notice that we do not care about a tight value of $\rho$, but only that it is finite. Instead, if we wanted to compute a minimax lower bound, then we would need a tight value of $\rho$ in order to obtain a tight lower bound.

\subsection{$d$ and $d_\infty$ satisfy a relaxed triangle inequality}
In the following, for the sake of simplicity, we denote reward functions by vectors $x,y,z,\dotsc \in \mathbb{R}^k$. Moreover,
for any pair $x,y\in\mathbb{R}^k$, we will consider distance $d$ for some distribution $q\in\Delta^{\dsb{k}}$ as:
\begin{align*}
    d(x,y)= \frac{\sum\limits_{i\in\dsb{k}}q_i|x_i-y_i|}{\max\{\|x\|_\infty,\|y\|_\infty\}},
\end{align*}
and distance $d_\infty$ as:
\begin{align*}
        d_\infty(x,y)= \frac{\|x-y\|_\infty}{\max\{\|x\|_\infty,\|y\|_\infty\}}.
\end{align*}
    
First of all, let us see that neither $d$ nor $d_\infty$ are metrics:
\begin{prop}
    Both the functions $d$ and $d_\infty$ do not satisfy the triangle inequality.
\end{prop}
\begin{proof}
To show that the triangle inequality property is not satisfied, we simply provide some counterexamples.
For the sake of simplicity, let $k=2$.

W.r.t. distance $d$, let the vectors $x,y,z\in\mathbb{R}^2$ be defined as:
\begin{align*}
    \begin{cases}
        x=[1,0]^\intercal\\
        y=[-1,-1]^\intercal\\
        z=[-2,-1]^\intercal
    \end{cases},
\end{align*}
and observe that, for any $q\in\Delta^{\dsb{2}}$ such that $q_2>0$:
\begin{align*}
    &d(x,y)=2q_1+q_2\markref{?}{\le}d(x,z)+d(y,z)=3/2q_1+q_2/2+q_1/2=2q_1+q_2/2\\
    &\iff\, q_2\le 0.
\end{align*}
If $q_2=0$, we can take the second component of $z$ to be arbitrary large so that the inequality is not verified.

As far as $d_\infty$ is concerned,
let $x,y,z\in \mathbb{R}^2$ be the vectors defined as:
\begin{align*}
    \begin{cases}
        x=[0,2]^\intercal\\
        y=[2,2]^\intercal\\
        z=[1,3]^\intercal
    \end{cases}.
\end{align*}
We have:
\begin{align*}
    &d_\infty(x,y)=\frac{2}{2}=1\markref{?}{\le} d_\infty(x,z)+d_\infty(y,z)=1/3+1/3=2/3\\
    &\,\iff 1\le 2/3,
\end{align*}
which is clearly false.

Notice that it is possible to generate counterexamples for higher dimensions $k>2$
by using a simple script of code.
\end{proof}
Having verified that distances $d,d_\infty$ do not satisfy the triangle inequality, we can conclude that they are \textit{semimetrics}, as it is easy to check the other three properties of positivity, simmetry, and that the distance between two points is zero if and only if the two points coincide. We are interested in verifying whether
they satisfy a relaxed form of triangle inequality \cite{fagin1998relaxing}.
Specifically, for any finite $\rho\in\mathbb{R}$ with $\rho>1$, we say that a function $d$ satisfies
the $\rho$-relaxed triangle inequality if, for any $x,y,z\in\mathbb{R}^k$:
\begin{align*}
    d(x,y)\le \rho\bigl(
    d(x,z)+d(y,z)\bigr).
\end{align*}
We aim to show that both $d$ and $d_\infty$ satisfy the $\rho$-relaxed triangle inequality for some $\rho$.
Let us begin with a useful lemma.
\begin{lemma}\label{lemma: d2 triangle inequality}
    Let $d_2:\mathbb{R}^k\times\mathbb{R}^k\rightarrow \mathbb{R}$ be the function that, for any pair $x,y\in\mathbb{R}^k$, it returns:
    \begin{align*}
        d_2(x,y)&\coloneqq \frac{\|x-y\|_2}{\max\{\|x\|_2,\|y\|_2\}}.
    \end{align*}
    Then, $d_2$ is a metric.
\end{lemma}
\begin{proof}
    It is easy to observe that $d_2(x,y)=0$ if and only if $x=y$. Moreover, notice that $d_2(x,y)\ge0$ for all $x,y\in\mathbb{R}^k$, and also that $d_2(x,y)=d_2(y,x)$.

    It remains to prove that $d_2$ satisfies the triangle inequality property, i.e., for any $x,y,z\in\mathbb{R}^k$, it satisfies:
    \begin{align*}
        d_2(x,y)\le d_2(x,z)+d_2(y,z).
    \end{align*}
    We distinguish two cases, one in which $\max\{\|x\|_2,\|y\|_2,\|z\|_2\}\neq \|z\|_2$ and the other in which $\max\{\|x\|_2,\|y\|_2,\|z\|_2\}= \|z\|_2$.
    
    Let us begin with the former case. W.l.o.g., assume that $\argmax\{\|x\|_2,\|y\|_2,\|z\|_2\}=y$. Then, we can write:
    \begin{align*}
        d_2(x,y)&\coloneqq \frac{\|x-y\|_2}{\max\{\|x\|_2,\|y\|_2\}}\\
        &=\frac{\|x-y\|_2}{\max\{\|x\|_2,\|y\|_2,\popblue{\|z\|_2\}}}\\
        &\markref{(1)}{\le}
        \frac{\|x-z\|_2}{\max\{\|x\|_2,\|y\|_2,\|z\|_2\}}+\frac{\|y-z\|_2}{\max\{\|x\|_2,\|y\|_2,\|z\|_2\}}\\
        &\markref{(2)}{=}
        \frac{\|x-z\|_2}{\max\{\|x\|_2,\|y\|_2,\|z\|_2\}}+\frac{\|y-z\|_2}{\popblue{\max\{\|y\|_2,\|z\|_2\}}}\\
        &\markref{(3)}{\le}
        \frac{\|x-z\|_2}{\popblue{\max\{\|x\|_2,\|z\|_2\}}}+\frac{\|y-z\|_2}{\max\{\|y\|_2,\|z\|_2\}}\\
        &\eqqcolon d_2(x,z)+d_2(y,z),
    \end{align*}
    where at (1) we apply triangle inequality of the $\|\cdot\|_2$ norm,
    at (2) we use that $\max\{\|x\|_2,\|y\|_2,\|z\|_2\}=\max\{\|y\|_2,\|z\|_2\}=\|y\|_2$, at (3) we use that, since
    $\max\{\|x\|_2,\|y\|_2,\|z\|_2\}=\|y\|_2$, then $\max\{\|x\|_2,\|y\|_2,\|z\|_2\}\ge \max\{\|x\|_2,\|z\|_2\}$.

    Now, w.l.o.g., consider the case in which $\|x\|_2\le\|y\|_2\le\|z\|_2$. Since the normed vector space $\mathbb{R}^k$ with $\|\cdot\|_2$ is an inner product space, then the Ptolemy's inequality \cite{steele2004cauchy} holds:
    \begin{align*}
        \|x-y\|_2\|z\|_2&\le \|x-z\|_2\|y\|_2+\|y-z\|_2\|x\|_2\\
        &\le \|x-z\|_2\|y\|_2+\|y-z\|_2\popblue{\|y\|_2}\\
        &=\|y\|_2\bigl( \|x-z\|_2+\|y-z\|_2\bigr).
    \end{align*}
    By dividing both sides of the inequality by $\|z\|_2$ and $\|y\|_2$, we can write:
    \begin{align*}
        &\frac{\|x-y\|_2}{\|y\|_2}\le\frac{\|x-z\|_2}{\|z\|_2}+\frac{\|y-z\|_2}{\|z\|_2}\\
        &\iff \frac{\|x-y\|_2}{\max\{\|x\|_2,\|y\|_2\}}\le\frac{\|x-z\|_2}{\max\{\|x\|_2,\|z\|_2\}}+\frac{\|y-z\|_2}{\max\{\|y\|_2,\|z\|_2\}}.
    \end{align*}    
    This concludes the proof.
\end{proof}
It should be remarked that the Ptolemy's inequality holds in inner product spaces only, and that the unique $p$-normed vector space to be an inner product space is that with $p=2$. This is why our proof of Lemma \ref{lemma: d2 triangle inequality} works for function $d_p$ defined as:
\begin{align*}
    d_p(x,y)\coloneqq \frac{\|x-y\|_p}{\max\{\|x\|_p,\|y\|_p\}},
\end{align*}
if and only if $p=2$.
Thanks to Lemma \ref{lemma: d2 triangle inequality}, we are able to prove the main theorem of this section.
\begin{thr}\label{theorem: d dinfty are rhorelaxed}
    Let $q\in\Delta^{\dsb{k}}$ such that $q_i>0$ for all $i\in\dsb{k}$, and denote    
    $q_{\min}\coloneqq\min_{i\in\dsb{k}}q_i$.
    Then, both the semimetrics $d$ and $d_\infty$ satisfy
    the $\rho$-relaxed triangle inequality with $\rho$ upper bounded, respectively, by $k/q_{\min}^2$ and $k$.
\end{thr}
\begin{proof}
    First, we prove the statement of the theorem for $d_\infty$, and then we use it to prove the statement for $d$.

    Observe that, for any $x\in\mathbb{R}^k$:
    \begin{align}\label{eq: relation norm 2 norm infty}
        &\|x\|_\infty\le \|x\|_2 \le \sqrt{k}\|x\|_\infty.
    \end{align}
    Let us consider any three vectors $x,y,z\in\mathbb{R}^k$.
    If $\argmax\{\|x\|_\infty,\|y\|_\infty,\|z\|_\infty\}\neq z$, then we can proceed as in the first part of the proof of Lemma \ref{lemma: d2 triangle inequality} to show that
    $d_\infty(x,y)\le d_\infty(x,z)+d_\infty(y,z)$.
    Therefore, w.l.o.g., we consider the case in which $\argmax\{\|x\|_\infty,\|y\|_\infty,\|z\|_\infty\}= z$.
    We can write:    
    \begin{align*}
        d_\infty(x,y)&=\frac{\|x-y\|_\infty}{\max\{\|x\|_\infty,\|y\|_\infty\}}\\
        &\markref{(1)}{\le} \frac{\|x-y\|_{\popblue{2}}}{\max\{\|x\|_\infty,\|y\|_\infty\}}\\
        &\markref{(2)}{\le} \frac{\sqrt{k}\|x-y\|_2}{\max\{\|x\|_{\popblue{2}},\|y\|_{\popblue{2}}\}}\\
        &= \sqrt{k}d_2(x,y)\\
        &\markref{(3)}{\le} \sqrt{k}d_2(x,z)+\sqrt{k}d_2(y,z)\\
        &= \sqrt{k}
        \frac{\|x-z\|_2}{\max\{\|x\|_2,\|z\|_2\}}+\sqrt{k}\frac{\|y-z\|_2}{\max\{\|y\|_2,\|z\|_2\}}\\
        &\markref{(4)}{\le}\sqrt{k}
        \frac{\|x-z\|_2}{\max\{\|x\|_{\popblue{\infty}},\|z\|_{\popblue{\infty}}\}}+\sqrt{k}\frac{\|y-z\|_2}{\max\{\|y\|_{\popblue{\infty}},\|z\|_{\popblue{\infty}}\}}\\
        &\markref{(5)}{\le} \popblue{k}
        \frac{\|x-z\|_{\popblue{\infty}}}{\max\{\|x\|_\infty,\|z\|_\infty\}}+ \popblue{k}\frac{\|y-z\|_{\popblue{\infty}}}{\max\{\|y\|_\infty,\|z\|_\infty\}}\\
        &=k\bigl(d_\infty(x,z)+d_\infty(y,z)\bigr),
    \end{align*}
    where at (1) and at (2) we use Eq. \ref{eq: relation norm 2 norm infty},
    at (3) we use the result in Lemma \ref{lemma: d2 triangle inequality},
    and at (4) and at (5) we use again Eq. \ref{eq: relation norm 2 norm infty}.

    Now, we move to prove the statement concerning $d$.
    In a similar way as in the proof of Proposition \ref{prop: relation metrics}, we have that, for any $x,y\in\mathbb{R}^k$:
    \begin{align*}
        d(x,y)\le d_\infty(x,y)&\coloneqq \frac{\|x-y\|_\infty}{\max\{\|x\|_\infty,\|y\|_\infty\}}\\
        &=\frac{\max\limits_{i\in\dsb{k}}\frac{q_i}{q_i}|x_i-y_i|}{\max\{\|x\|_\infty,\|y\|_\infty\}}\\
        &\le\frac{\max\limits_{i\in\dsb{k}}\frac{q_i}{q_{\min}}|x_i-y_i|}{\max\{\|x\|_\infty,\|y\|_\infty\}}\\
        &\le\frac{\sum\limits_{i\in\dsb{k}}q_i|x_i-y_i|}{q_{\min}\max\{\|x\|_\infty,\|y\|_\infty\}}\\
        &=\frac{d(x,y)}{q_{\min}}.
    \end{align*}
    By using this relation in place of that in Eq. \ref{eq: relation norm 2 norm infty}, we can carry out the
    same derivation made for $d_\infty$ using $d_2$ for the semimetric $d$ using $d_\infty$.
\end{proof}
It should be remarked that we are not claiming here that the values of $\rho$ provided in Theorem \ref{theorem: d dinfty are rhorelaxed} are tight\footnote{
Indeed, we do not believe so. By using a script to generate a large number of vectors, and using the intuition
that the diagonal of the unit square ($\|\cdot\|_\infty$) is $\sqrt{2}$ the radius of the unit circle ($\|\cdot\|_2$),
we conjecture that a tighter value of $\rho$ for $d_\infty$ is $\rho=2$, irrespective of the dimension.
}.

\subsection{The Hausdorff distance inherits the relaxed triangle inequality property}
First, we show that, thanks to the definitions of $d$ and $d_\infty$, if we apply the Hausdorff distance to \emph{closed} sets, then the (relaxed) triangle inequality property is satisfied. Next, we show that the sets of rewards we work with are \emph{closed}.

Let us begin with the following proposition.
\begin{prop}
    Let $\mathcal{H}_d$ and $\mathcal{H}_\infty$ be defined as in Section \ref{section: pac framework}. The \emph{closedness} of the sets to which these distances are applied is a sufficient condition for the (relaxed) triangle inequality property to hold.
\end{prop}
\begin{proof}[Proof Sketch]
    We will not provide an exhaustive proof, since it is completely analogous to the proof that shows that \emph{compactness} is a sufficient condition for the Hausdorff distance with inner metric to satisfy triangle inequality. Instead, we simply give an idea of why for $d$ and $d_\infty$ \emph{closedness} (instead of \emph{compactness}) suffices.

    In practice, the compactness requirement is just needed to guarantee that the \emph{infimum} is actually a \emph{minimum} over the sets in input to the Hausdorff distance. For a generic notion of inner distance, closedness is not sufficient because the infimum might be at $\infty$ and, thus, the minimum would not exist. However, observe that both $d$ and $d_\infty$ contain the normalization term $1/M$ (see Section \ref{section: pac framework}), therefore, for any finite vector $x\in\mathbb{R}^k$, getting to infinity
    $\lim_{y\to\infty}\|x-y\|_\infty/M=1$ worsens the distance to $x$ w.r.t. any other finite $z$ in the set containing $y$. This shows that boundedness is not required anymore, but closedness suffices. This concludes the proof.
\end{proof}
In this work we consider unbounded sets of rewards, so clearly compactness does not hold. The following proposition shows the closedness of some sets of rewards.
\begin{prop}
    The following sets are closed:
    \begin{align*}
        \oldfs,\fs,\sub,\super,\subrelax,\superrelax.
    \end{align*}
\end{prop}
\begin{proof}
    From Theorem 3 of \citet{ng2000algorithms}, we observe that the old feasible set $\oldfs$ is closed because it is defined by linear less than or \emph{equal to} $\le$ inequalities.
    
    The new feasible set $\fs$ can be expressed, from Corollary \ref{corollary: relation feasible sets}, as an arbitrary union of closed sets $\fs=\bigcup_{\pi'\in\eqclasspie{\pi^E}}\overline{\R}_{p,\pi'}$. However, observe that the feasible sets $\overline{\R}_{p,\pi'}$ with stochastic $\pi'$ are contained in the feasible sets of some deterministic policies. Since there is a finite number of deterministic policies, then $\fs$ can be expressed as a finite union of closed sets, so it is closed.

    The subset $\sub$ is an arbitrary intersection of $\fs$, i.e., closed sets, thus it is closed.

    The superset $\super$ is an arbitrary union of $\fs$, so, potentially, it might not be non-closed. However, thanks to the definitions of $p^m$ and $\pi^m$ in Eq. \ref{eq: def pM and pm} and Eq. \ref{eq: def piM and pim}, we know that the arbitrary union representing $\super$ coincides with the feasible set $\R_{p^m,\pi^m}$, which is closed, thus $\super$ is closed.

    In an analogous manner, by using Eq. \ref{eq: def ptildeM and ptildem} and Eq. \ref{eq: def pitildeM and pitildem}, we observe that the relaxations $\subrelax$ and $\superrelax$ can be expressed by a finite number of linear less than or \emph{equal to} $\le$ constraints, thus they are closed.
\end{proof}

\section{Technical Lemmas}\label{apx:tech}
In this section, we report some technical lemmas that are useful in the analysis of the sample complexity
of \irlo and \pirlo (see Appendix \ref{section: sample complexity}).
Lemma \ref{lemma: binomial concentration} and Lemma \ref{lemma: jonsson} are taken from
other works, while Lemma \ref{lemma: lambert} takes inspiration from Lemma B.9 of \citet{metelli2021provably}.
\begin{lemma}[Lemma A.1 of \cite{xie2021bridging}]\label{lemma: binomial concentration}
    Suppose that $N\sim\text{Bin}(n,p)$ is a binomially distributed random variable,
    with $n\ge 1$ and $p\in[0,1]$. Then, with probability
    at least $1-\delta$, we have that:
    \[
        \frac{p}{N\vee 1}\le\frac{8\ln\frac{1}{\delta}}{n}.
    \]
\end{lemma}
\begin{lemma}[Lemma 8 of \cite{kaufmann2021adaptive}]\label{lemma: jonsson}
    Let $X_1,X_2,\dotsc,X_n,\dotsc$ be i.i.d. samples from a distribution supported over
    $\dsb{m}$, of probabilities given by $p\in\Delta^{\dsb{m}}$. We denote by $\widehat{p}_n$
    the empirical vector of probabilities, i.e., for all $k\in\dsb{m}$:
    \[
        \widehat{p}_{n,k}=\frac{1}{n}\sum\limits_{l=1}^n \mathbbm{1}\{X_l=k\}. 
    \]
    For all $p\in\Delta^{\dsb{m}}$, for all $\delta\in[0,1]$:
    \[
        \mathbb{P}\biggl(
            \exists n \in \mathbb{N}_{\ge 0},\,
            n KL(\widehat{p}_n\|p)>\ln(1/\delta)+(m-1)\ln\bigl(
                e(1+n/(m-1))
            \bigr)
        \biggr)\le\delta.
    \]
\end{lemma}
\begin{lemma}\label{lemma: lambert}
    Let $a,b,c,d> 0$ such that $2bc > e$. Then, the inequality $x\ge a+b\ln(cx+d)$ is satisfied
    by all $x\ge 2a+3b\ln(2bc)+d/c$.
\end{lemma}
\begin{proof}
    Observe that, since function $x$ grows faster than function $a+b\ln(cx+d)$,
    then there exists $\bar{x}$ such that, for all $x\ge \bar{x}$,
    the inequality is satisfied. Our goal here is to show that
    such $\bar{x}$ can be upper bounded by $2a+3b\ln(2bc)+d/c$.

    Let us consider any $x\ge 2a+d/c$. We can write:
    \begin{align*}
        x\ge a+b\ln(cx+d)\quad&\Longleftrightarrow\quad \frac{x-a}{b}\ge \ln(c(x\pm a)+d)\\
        &\Longleftrightarrow\quad e^\frac{x-a}{b}\ge c(x-a)+ca+d\\
        &\markref{(1)}{\Longleftarrow}\quad e^\frac{x-a}{b}\ge 2c(x-a)\\
        &\Longleftrightarrow\quad \frac{a-x}{b}e^\frac{a-x}{b}\ge -\frac{1}{2bc},\tag{I}\label{eq: ineq for lambert}
    \end{align*}
    where at (1) we have used that, since $x\ge 2a+d/c$,
    then $c(x-a)\ge ca+d$, and thus we have replaced the constraint with a stronger one.

    Hypothesis $2bc > e$ entails that $-\frac{1}{2bc}\ge -\frac{1}{e}$, thus we can apply the
    Lambert function, which provides as solution to inequality \ref{eq: ineq for lambert} all the $x$ such that:
    \[
        \frac{a-x}{b}\le W_{-1}\biggl(-\frac{1}{2bc}\biggr)\quad\text{ or }\quad
        \frac{a-x}{b}\ge W_{0}\biggl(-\frac{1}{2bc}\biggr),
    \]
    where $W_0$ is the principal component of the Lambert W function.
    Consider the first inequality. We can write:
    \begin{align*}
        x&\ge a-bW_{-1}\biggl(-\frac{1}{2bc}\biggr)\\
        &\markref{(1)}{\le}
        a+b+b\sqrt{2\ln(2bc)-2}+b\ln(2bc)-b\\
        &\le
        a+3b\ln(2bc),
    \end{align*}
    where at (1) we have applied the inequality $W_{-1}(-e^{-u-1})\ge -1-\sqrt{2u}-u$
    from \cite{chatzigeorgiou2013boundlambert}.

    To obtain the result, we use that $\max\{a,b\}\le a+b$ for any $a,b\ge 0$ to upper bound:
    \begin{align*}
        \max\bigl\{
            2a+d/c,a+3b\ln(2bc)
        \bigr\}=a+\max\bigl\{
            a+d/c,3b\ln(2bc)
        \bigr\}\le 2a+3b\ln(2bc)+d/c.
    \end{align*}
\end{proof}

\section{Illustrative Experiment}\label{section: appendix experiments}
We have applied \pirlo to the highway driving application domain.
To this aim, we have used the data\footnote{
The data is publicly available at \url{https://github.com/amarildolikmeta/irl_real_life/tree/main/datasets/highway}.
} gathered by \citet{LikmetaMRTGR21}.

\paragraph{Data Description}
The dataset consists of trajectories of $H=400$ stages collected by $10$ different human experts driving in a simulator. The highway has 3 lanes. The goal of each expert is to change lane in order to drive safely and to minimize the trip time. The action space $\A$ is made of 3 actions: Turn left, turn right, continue forward. The state space $\S$ is continue, and it is represented by $25$ features, keeping into account the speed and position of the car, and the speed and position of the surrounding cars.

\paragraph{Data Preprocessing}
We have to transform the data to obtain a tabular MDP. To this aim, we construct $5$ discrete features from the $25$ present in the original data: We use three binary features, free left, free right, free forward, that say whether there is a vehicle on the left, on the right, or in front of our car; next, we use a binary feature that says whether the car is changing lane, and a discrete feature with $5$ possible values for the speed of the vehicle. In this way, we obtain a tabular MDP with $S=80$.

\paragraph{Experiments Design}
As mentioned by \citet{LikmetaMRTGR21}, this lane-change scenario represents a multi-objective task, because humans consider several objectives while driving. We manually design some reward functions coherent with the most common driving objectives and we use \pirlo to verify whether they are compatible w.h.p. with the demonstrations of behavior provided by the 10 experts in the dataset.
First, we construct a single behavioral dataset $\D^b$ by joining the trajectories of all the 10 experts, and then we consider one expert at a time to construct $\D^E$. 
Next, we design the reward functions and we give them in input to the membership checker implementation of \pirlo.

\paragraph{Experiments Results}
We design 3 kinds of reward functions:
\begin{itemize}
    \item reward $r_{\text{BC}}$, i.e., the ``behavioral cloning'' reward, which is the reward that assigns positive values to actions played by the expert's policy;
    \item reward $r$, which is coherent with the observations provided in Section 5.3 of \citet{LikmetaMRTGR21}. In words, it assigns negative reward when ($i$) the right lane is not free, ($ii$) there is a car in front of us (and so it decreases our speed), ($iii$) we change lane;
    \item reward $\overline{r}$, which is $-r$, i.e., it assigns positive reward to all the bad actions;
\end{itemize}
We provide the output of \pirlo in Table \ref{tab: exp results}. 
   \begin{table}[t]
    \centering
    \scalebox{0.95}{
    \begin{tabular}{c||c|c|c|c|c|c|c|c|c|c}
          & Alice & Bob & Carol & Chuck & Craig & Dan & Erin & Eve & Grace & Judy\\
           $r_{\text{BC}}$ & \popgreen{Y,Y} & \popgreen{Y,Y} & \popgreen{Y,Y} & \popgreen{Y,Y} & \popgreen{Y,Y} & \popgreen{Y,Y} & \popgreen{Y,Y} & \popgreen{Y,Y} & \popgreen{Y,Y} & \popgreen{Y,Y}\\
           $r$ & \popred{N,N} & \popyellow{Y,N} & \popyellow{Y,N} & \popyellow{Y,N} & \popred{N,N} & \popyellow{Y,N} & \popyellow{Y,N} & \popred{N,N} & \popred{N,N} & \popyellow{Y,N}\\
           $\overline{r}$ & \popred{N,N} & \popred{N,N} & \popred{N,N} & \popred{N,N} & \popred{N,N} & \popred{N,N} & \popred{N,N} & \popred{N,N} & \popred{N,N} & \popred{N,N}\\
    \end{tabular}
    }
    \caption{The output of \pirlo when fed with the rewards designed for the highway driving task. The first letter refers to the superset, while the second letter refers to the subset. ``N'' means that the reward does not belong to the set, while ``Y'' means that it belongs to the set.}
    \label{tab: exp results}
\end{table} 
Some comments are in order. First, our reduction to a smaller state space has caused the policies of the agents to be (more) stochastic. Moreover, this reduction has increased the number of times that the corner case described in Appendix \ref{subsec: annoying corner case} takes place. Since this corner case is outside the good event, we have removed such data from $\D^b$; in this way, we improve the performances of \pirlo.

Observe that the behavioral cloning reward $r_{\text{BC}}$ belongs to the subset and superset for all the experts. This is reasonable since it assigns positive reward only to expert's actions in the support of the expert's policy. However, it should be remarked that if we had not removed the ``corner-case'' samples, then $r_{\text{BC}}$ would not belong to the subsets.

The reward $r$ compatible with the analysis provided in \citet{LikmetaMRTGR21} belongs to the superset of some experts only. Specifically, for the experts Alice, Eve, Grace, and Craig, that belong to the clusters 1 and 3 of Table 1 of \citet{LikmetaMRTGR21}, the reward $r$ is not in the superset. However, it should be remarked that reward $r$ is not exactly the same as the reward described by \citet{LikmetaMRTGR21}, and also that we are working with a more aggregated state space.

Notice that, as expected, reward $\overline{r}=-r$, which rewards ``bad'' actions, does not belong neither to the subset nor to the superset of any expert.